\documentclass{article}
\usepackage[a4paper,margin=2cm]{geometry}
\usepackage{authblk}
\usepackage[utf8]{inputenc}
\usepackage{graphicx}
\usepackage{amsmath}
\usepackage{amsfonts}
\usepackage{amssymb}
\usepackage{amsthm}
\usepackage{subfig}
\usepackage{natbib}
\usepackage{url}
\usepackage[boxed,linesnumbered]{algorithm2e}
\usepackage{setspace} 
\usepackage{array}
\usepackage{comment}
\usepackage{soul}
\usepackage{yhmath}
\usepackage{enumerate}
\usepackage{morefloats}
\usepackage{mathtools}
\usepackage{hyperref}

\usepackage{dsfont}
\newcommand{\Real}{\mathds{R}}
\newcommand{\Zeta}{\mathds{Z}}
\newcommand{\N}{\mathds{N}}
\newcommand{\LR}{\Leftrightarrow}

\newcommand*{\defeq}{\stackrel{\text{def}}{=}}
\DeclareMathOperator*{\argmin}{argmin}

\newtheorem{proposition}{Proposition}
\newtheorem{definition}{Definiton}
\newtheorem{lemma}{Lemma}

\newcommand{\shortVersion}{0}

\newif\ifexpandexplanation
\expandexplanationtrue 

\newif\ifithasappendixforlemmas

\begin{document}

  \title{On the throughput of the common target area for robotic swarm strategies
   \if\shortVersion 0
   \ifexpandexplanation
     -- extended version
   \fi 
   \fi
  }
  
  \author[*,1,2]{Yuri Tavares dos Passos}
  
  \author[2]{Xavier Duquesne}
  
  \author[2]{Leandro Soriano Marcolino}
  
  \affil[*]{Corresonding author.}
  
  \affil[1]{Centro de Ciências Exatas e Tecnológicas. Universidade Federal do Reconcâvo da Bahia. Rua Rui Barbosa, 710. Centro. Cruz das Almas. 44380-000. Bahia. Brazil.}

  \affil[2]{School of Computing and Communications. Lancaster University. Bailrigg. Lancaster. LA1 4WA. Lancashire. United Kingdom.}

  \affil[ ]{yuri.passos@ufrb.edu.br, duquesne.xavier.13@gmail.com, l.marcolino@lancaster.ac.uk}

  \date{}

  \maketitle

  \begin{abstract}
     A robotic swarm may encounter traffic congestion when many robots simultaneously attempt to reach the same area. For solving that efficiently, robots must execute decentralised traffic control algorithms. In this work, we propose a measure for evaluating the access efficiency of a common target area as the number of robots in the swarm rises: the common target area throughput. We also employ here the target area asymptotic throughput -- that is, the throughput of a target region with a limited area as the time tends to infinity -- because it is always finite as the number of robots grows, opposed to the relation arrival time at the target per number of robots that tends to infinity. Using this measure, we can analytically compare the effectiveness of different algorithms. In particular, we propose and formally evaluate three different theoretical strategies for getting to a circular target area: (i) forming parallel queues towards the target area, (ii) forming a hexagonal packing through a corridor going to the target, and (iii) making multiple curved trajectories towards the boundary of the target area. We calculate the throughput for a fixed time and the asymptotic throughput for these strategies. Additionally, we corroborate these results by simulations, showing that when a strategy has higher throughput, its arrival time per number of robots is lower. Thus, we conclude that using throughput is well suited for comparing congestion algorithms for a common target area in robotic swarms even if we do not have their closed asymptotic equation.
  \end{abstract}
  

\section{Introduction}

Swarms of robots are systems composed of a large number of robots that can only interact with direct neighbours and follow simple algorithms. Interestingly, complex behaviours may emerge from such straightforward rules \citep{navarro2013introduction, Garnier2007}. An advantage of such systems is the usage of low-priced robots instead of a few expensive ones to solve problems. Robotic swarms accurately projected for simple robots may solve complex tasks with greater efficiency and fault-tolerance, while being cheaper than a small group of complex robots oriented for a specific problem domain. They can also be seen as a multiagent system with spatial computers, which is a group of devices displaced in the space such that its objective is defined in terms of spatial structure and its interaction depends on the distance between them  \citep{giavitto:hal-00821901}. Swarms have been recently receiving attention in the multi-agents systems literature in problems such as logistics \citep{10.5555/3463952.3464142}, flocking formation \citep{10.5555/3463952.3463999}, pattern formation \citep{10.5555/3463952.3463998} and coordination of unmanned aerial vehicles swarms \citep{10.5555/3463952.3464114}. In such problems relating to spatial distribution, conflicts may be created by the trajectories of its robots, which may slow down the system, in special when a group is intended to go to a common region of the space. Some examples where this happens are waypoint navigation \citep{marcolinoNoRobotLeft2008} and foraging \citep{ducatelleCommunicationAssistedNavigation2011}.

Related works on multi-agents systems \citep{carlinoAuctionbasedAutonomousIntersection2013, Sharon2017, CuiStoneScalable} deals with a similar problem, but they consider autonomous cars navigating over lanes and roads, and coordination is needed at the junctions. In  \citep{ChoudhurySKP21} and \citep{jair112397}, they also deal with multi-agents and pathfinding, but not in a situation in which the target of every agent is the same area. \cite{jair112397} present theoretical analysis for their proposed matter, alongside experimentation by simulations corroborating the results as we do. Furthermore, we are considering agents with only local information and distributed solutions, while \citep{ChoudhurySKP21} and \citep{jair112397} propose centralised solutions. \cite{jmse9121324} investigate the topology of the neighbourhood relations between multiple unmanned surface vehicles in a swarm. In our work, we analyse the impact of the throughput of the target area when using formation packing in squares and hexagons. They deal with maintaining formation in swarms, but they have to keep virtual leaders, and their goal is not minimising congestion.

Moreover, there has not been much research on the problem of reducing congestion when a swarm of robots are aimed at the same target. Surveys about robotic swarms \citep{sahin04swarm,SahinGBT08,Barca2013swarm,Brambilla2013Swarm,Bayindir2016,8424838} do not provide information regarding these situations. Even a recent survey on collision avoidance \citep{hoyAlgorithmsCollisionfreeNavigation2015} do not address this issue though it provides insights into multi-vehicle navigation. Congestion in robotic swarms is mostly managed by collision avoidance in a decentralised fashion, which allows for improved scalability of the algorithms. 

However, solely avoiding collisions does not necessarily lead to a good performance in this problem with a common target. For example, we showed  \citep{Marcolino2016} that the ORCA algorithm  \citep{Berg2011} reaches an equilibrium where robots could not arrive at the target despite avoiding collisions. In that work, we also presented three algorithms using artificial potential fields for the common target congestion problem, but no formal analysis of the cluttered environment was done. Hence, congestion is still not well understood, and more theoretical work is needed to measure the optimality of the algorithms. A better understanding of this topic should lead to a variety of new algorithms adapted to specific environments.

Furthermore, any elaborated analysis on that subject must investigate the effect of the increase of the number of individuals on the swarm congestion, as we desire for the system to perform well as it grows in size. If we have a finite measure that abstracts the optimality of any algorithm as the number of robots goes to infinity, we can use this as a metric to compare different approaches to the same problem. Thus, in this work, we present as metric the common target area throughput. That is, we are proposing a measure of the rate of arrival in this area as the time tends to infinity as an alternative approach to analyse the congestion in swarms with a common target area. In network and parallel computing studies \citep{asymptotic1,asymptotic2}, asymptotic throughput is used to measure the throughput when the message size is assumed to have infinite length. We use the same idea here, but instead of message size, we work with infinite time, as if the algorithms run forever. As we will present in the next section, this implies dealing with an infinite number of robots. Thus, here we use time instead of message size or bytes as in computer network studies.

Therefore, the contributions in this paper are the following.
\begin{enumerate}[(i)]
  \item We propose a method for evaluating algorithms for the common target problem in a robotic swarm by using the throughput in theoretical or experimental scenarios.
  \item We present an extensive theoretical study of the common target problem, allowing one to understand better how to measure the access to a common target using a metric not yet used in other works on the same problem. 
  \item Assuming a circular target area and robots with a constant linear velocity and a fixed distance from each other, we develop theoretical strategies for entering the area and calculate their theoretical throughput for a fixed time and their asymptotic throughput when time goes to infinite. Additionally, we verify the correctness of these calculations by simulations.
\end{enumerate}

The presented theoretical strategies are based on forming a corridor towards the target area or making multiple curved trajectories towards the boundary of the target area. For the corridor strategy, we also discuss the throughput when the robots are going to the target in square and hexagonal packing formations. We evaluate our theoretical strategies by realistic Stage \citep{PlayerStage} simulations with holonomic and non-holonomic robots. Our experiments corroborate that whenever an algorithm makes a swarm take less time to reach the target region than another algorithm, the throughput of the former is higher than the latter. These strategies are the inspiration to new distributed algorithms for robotic swarms in our concurrent work \citep{arxivAlgorithms}.

This paper is organised as follows. In the next section, we briefly explain the mathematical notation we are using.  In Section \ref{sec:theoreticalresults}, we formally define the common target area throughput and prove statements about this measure for theoretical strategies that allow robots to enter the common target area. Section \ref{sec:experimentresults} describes the experiments and present its results to verify the correctness of the theoretical strategies results. Finally, we summarise our results and make final remarks in Section \ref{sec:conclusion}.

\section{Notation}

Geometric notation is used as follows. $\overleftrightarrow{AB}, \overrightarrow{AB}$ and $\overline{AB}$ represents a line passing through points A and B, a ray starting at A and passing through B and a segment from A to B, respectively. $\vert \overline{AB}\vert $ is the size of $\overline{AB}$. $\overleftrightarrow{AB} \parallel \overleftrightarrow{CD}$ means $\overleftrightarrow{AB}$ is parallel to $\overleftrightarrow{CD}$. If a two-dimensional point is represented by a vector $P_{1}$, its x- and y-coordinates are denoted by $P_{1,x}$ and $P_{1,y}$, respectively. 

$\bigtriangleup ABC$ express the triangle formed by the points A, B and C. $\bigtriangleup ABC \cong \bigtriangleup DEF$ and $\bigtriangleup ABC \sim \bigtriangleup DEF$ mean the triangles ABC and DEF are congruent (same angles and same size) and similar (same angles), respectively. Depending on the context, the notation is omitted for brevity. 

$\widehat{AOB}$ means an angle with vertex O, one ray passing through point A and another through B. Depending on the context, if we are dealing only with one $\bigtriangleup EFG$, we will name its angles only by  $\widehat{E}$, $\widehat{F}$ and $\widehat{G}$. All angles are measured in radians in this paper. 

\section{Theoretical Analysis}
\label{sec:theoreticalresults}

We consider in this paper the scenario where a large number of robots must reach a common target. After reaching the target, each robot moves towards another destination which may or may not be common among the robots. We assume the target is defined by a circular area of radius $s$. A robot reaches the target if its centre of mass is at a distance below or equal to the radius $s$ from the centre of the target. We assume that there is no minimum amount of time to stay at the target. Additionally, the angle and the speed of arrival have no impact on whether the robot reached the target or not. In this section, theoretical strategies are constructed to solve that task and show limits for the efficiency of real-life implementations, which we developed in a concurrent work \citep{arxivAlgorithms}. To measure performance, we start with the following definition.

\begin{definition}
The \emph{throughput} is the inverse of the average time between arrivals at the target. 
\label{def:throughput}
\end{definition} 

Informally speaking, the throughput is measured by someone located on the common target (i.e., on its perspective). We consider that an optimal algorithm minimises the average time between two arrivals or, equivalently, maximises throughput. The unit for throughput can be in $s^{-1}$.   
It will be noted $f$ (as in frequency).
In the rest of the paper, we focus on maximising throughput.

Assume we have run an experiment with $N \ge 2$ robots for $T$ units of time, such that the time between the arrival of the $i$-th robot and the $i+1$-th robot is $t_{i}$, for $i$ from $1$ to $N-1$. Then, by Definition \ref{def:throughput}, we have
\if\shortVersion 1
  $ f =  \frac{1}{\frac{1}{N-1}\sum_{i=1}^{N-1}t_{i}} 
  = \frac{N-1}{\sum_{i=1}^{N-1}t_{i}}
  = \frac{N-1}{T}, $
\else
  $$ f =  \frac{1}{\frac{1}{N-1}\sum_{i=1}^{N-1}t_{i}} 
  = \frac{N-1}{\sum_{i=1}^{N-1}t_{i}}
  = \frac{N-1}{T}, $$
\fi
because $\sum_{i=1}^{N-1}t_{i} = T.$ Thus, we have an equivalent definition of throughput:

\begin{definition}
The \emph{throughput} is the ratio of the number of robots that arrive at a target region, not counting the first robot to reach it, to the arrival time of the last robot.
\label{def:throughput2}
\end{definition} 

The target area is a limited resource that must be shared between the robots. Since the velocities of the robots have an upper bound, a robot needs a minimum amount of time to reach and leave the target before letting another robot in. Let the \emph{asymptotic throughput} of the target area be its throughput as the time tends to infinity. Because any physical phenomenon is limited by the speed of light, this measure is bounded. Then, the asymptotic throughput is well suited to measure the access of a common target area as the number of robots grows.

One should expect that the asymptotic throughput depends mainly on the target size and shape, the maximum speed for robots, $v$, and the minimum distance between robots, $d$. As any bounded target region can be included in a circle of radius $s$, we will deal hereafter with only circular target regions. 

To efficiently access the target area, we identify two main cases: $s \ge d/2$ and $s < d/2$. There are targets that several robots can simultaneously reach without collisions. That is the case if the radius $s \ge d/2$. Thus, one approach is making lanes to arrive in the target region so that as many robots as possible can simultaneously arrive. After the robots arrive at the target, they must leave the target region by making curves. However, we discovered in \citep{arxivAlgorithms} that this approach does not obtain good results in our realistic simulations due to the influence of other robots, although it is theoretically the best approach if the robots could run at a constant speed and maintain a fixed distance between each other.

We are also interested in the case where $s < d / 2$, when only one robot can occupy the target area simultaneously. Making two queues and avoiding the inter-robot distance to be less than $d$ is good guidance to work efficiently. Particularly, the case $s = 0$ offers interesting insights, so we begin by discussing it. 

\subsection{Common target point: $s = 0$}

We consider the case where robots are moving in straight lines at constant linear speed $v$, 
maintaining a distance of at least $d$ between each other.
A robot has reached the target when its centre of mass is over the target. 
When $s = 0$, the target is a point.
Our first result is the optimal throughput when robots are moving in a straight line to a  target point.
It is illustrated in Figure \ref{fig:straight_line}.
In this section, we construct a solution to attain the optimal throughput.

\begin{figure}
  \centering
  \includegraphics[width=0.38\columnwidth]{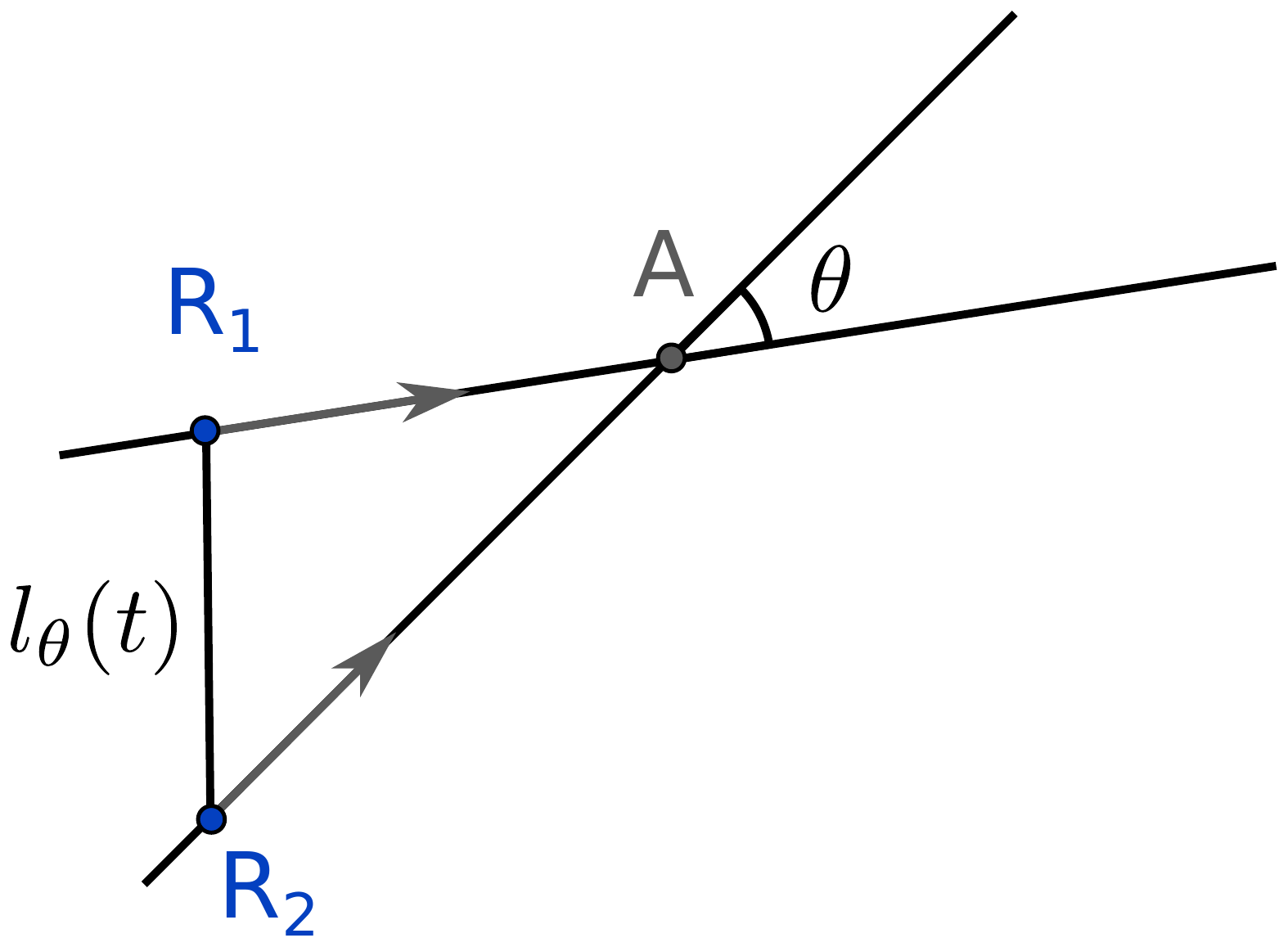}
  \caption{We consider two robots $R_1$ and $R_2$ moving in  straight lines toward a target at A. The angle between their trajectory is $\theta$. The distance between the two robots over  time is denoted by $l_\theta(t)$.} 
  \label{fig:straight_line}
\end{figure}

First, let us consider two robots, Robot 1 and Robot 2. Their trajectories are straight lines towards the target. Assume the straight-line trajectory of Robot 1 has an angle $\theta_{1}$ with the $x$-axis and the one of Robot 2 has $\theta_{2}$.
We call $\theta_{2} - \theta_{1} = \theta$ the angle between the two lines. 
The positions of the robots are described by the kinematic equation (\ref{eq:kimatic_equation_punctual_target}) below, where
$(x_1(t), y_1(t))$ and $(x_2(t), y_2(t))$ are the positions of Robot 1 and Robot 2, respectively, and
$t \in \mathbb{R}$ is an instant of time.
Without loss of generality, we set the origin of time when Robot 1 reaches the target, and the target is located at $(0, 0)$. Thus, $(x_1(0), y_1(0)) = (0, 0)$.
$\tau$ is the delay between the two arrivals at the target. Then, $(x_2(\tau), y_2(\tau)) = (0, 0)$,

\begin{equation}
\label{eq:kimatic_equation_punctual_target}
\left[
  \begin{matrix}
    x_1(t)\\ 
    y_1(t) 
  \end{matrix}
\right]=
\left[
  \begin{matrix}
    v t \cos(\theta_{1}) \\
    v t \sin(\theta_{1})
  \end{matrix}
\right]
\text{ and }
\left[
  \begin{matrix}
    x_2(t)\\ 
    y_2(t) 
  \end{matrix}
\right]=
\left[
  \begin{matrix}
    v (t - \tau) \cos(\theta_{2}) \\
    v (t - \tau) \sin(\theta_{2})
  \end{matrix}
\right]
\end{equation}

In order to find the optimal throughput, we will start with the following lemma:

\begin{lemma}
To respect a distance of at least $d$ between the two robots, the minimum delay between their arrival is $\frac{d}{v} \sqrt{\frac{2}{1 + \cos(\theta)}}$.
\label{prop:security_distance_punctual_target}
\end{lemma}
\begin{proof}
  \ifithasappendixforlemmas %
      See Online Appendix.
  \else %
  Let $l_\theta(t) = \sqrt{(x_1(t) - x_2(t))^2 + (y_1(t) - y_2(t))^2}$ be the distance between the two robots.
  The robots must maintain their minimum distance $d$ at all time:
  \begin{equation}
  \label{eq:distancerelation1}
  \forall t \in \mathbb{R}, l_\theta(t) \ge d.
  \end{equation}
  
  To avoid a collision, we have $\theta \neq \pi$, which corresponds to the case where robots face each other exactly. As a result, $\cos(\theta) \neq - 1$.
  For ease of calculation, we define $X = \tau v$, that is, the distance between Robot 1 and Robot 2 when Robot 1 reaches the target. We also define $P_\theta(t) = l_\theta(t)^2 - d^2$, so the constraint in  (\ref{eq:distancerelation1}) for minimum distance between them is expressed by
  \if\shortVersion 1
    $
    \forall t \in \mathbb{R}, l_\theta(t) \ge d
    \Leftrightarrow
    \forall t \in \mathbb{R}, P_\theta(t) \ge 0.
    $
    In addition, we have
    $
      P_\theta(t) 
        = 2 (1 - \cos(\theta)) v^{2} t^{2} - 2 X (1 - \cos(\theta))v t + X^2 - d^2,
    $
  \else
    $$
    \forall t \in \mathbb{R}, l_\theta(t) \ge d
    \Leftrightarrow
    \forall t \in \mathbb{R}, P_\theta(t) \ge 0.
    $$
    
    In addition, we have
    $$
    \begin{aligned}
      P_\theta(t) 
        &= (v t \cos(\theta_{1}) - v (t - \tau)  \cos(\theta_{2}))^2 + (v t \sin(\theta_{1}) - v (t - \tau)  \sin(\theta_{2}))^2 \\
          &\phantom{=\ }- d^2\\
    \ifexpandexplanation
        \end{aligned}
      $$
      $$
        \begin{aligned} 
          \phantom{P_\theta(t) }
          &= (v t \cos(\theta_{1}) - (v t - X)  \cos(\theta_{2}))^2 + (v t \sin(\theta_{1}) - (v t - X)  \sin(\theta_{2}))^2 - d^2\\
          &= (v t)^{2} \cos(\theta_{1})^{2} - 2 v t \cos(\theta_{1})(v t - X)  \cos(\theta_{2}) +  (v t - X)^2  \cos(\theta_{2})^2 + \\
            &\phantom{=\ \ } (v t)^{2} \sin(\theta_{1})^{2} - 2 v t \sin(\theta_{1})(v t - X)  \sin(\theta_{2}) +  (v t - X)^2  \sin(\theta_{2})^2  - d^2\\
        \end{aligned}
      $$
      $$
        \begin{aligned} 
          \phantom{P_\theta(t) }
    \fi
        &= (v t)^{2}  - 2 v t(v t - X)  (\cos(\theta_{1})\cos(\theta_{2}) + \sin(\theta_{1}) \sin(\theta_{2}))   \\
          &\phantom{=\ \ }   +  (v t - X)^2  - d^2 \\
        &= (v t)^{2}  - 2 v t(v t - X)  \cos(\theta_{2} - \theta_{1})  +  (v t - X)^2    - d^2\\
        &= (v t)^{2}  - 2 v t(v t - X)  \cos(\theta)  +  (v t - X)^2    - d^2\\
\ifexpandexplanation 
        \end{aligned}
      $$
      $$
        \begin{aligned}
          \phantom{P_\theta(t) }  
          &= (v t)^{2}  - 2 v t(v t - X)  + 2 v t(v t - X)  - 2 v t(v t - X)  \cos(\theta)  +  (v t - X)^2    - d^2\\
          &= (v t)^{2}  - 2 v t(v t - X)  +  (v t - X)^2 + 2 v t(v t - X)  - 2 v t(v t - X)  \cos(\theta)      - d^2\\
          &= (v t -  (v t - X))^2 + 2 v t(v t - X)  - 2 v t(v t - X)  \cos(\theta)      - d^2\\
          &= X^2 + 2 v t(v t - X)  - 2 v t(v t - X)  \cos(\theta)  - d^2\\
          &= 2 v t(v t - X)  - 2 v t(v t - X)  \cos(\theta) +X^2  - d^2\\
          &= 2 (v t)^{2} -2 X v t  - 2 (v t)^{2}\cos(\theta) + 2 X v t  \cos(\theta) +X^2  - d^2\\
          &= 2 (v t)^{2}  - 2 (v t)^{2}\cos(\theta) - 2X v t +2 X v t  \cos(\theta) +X^2  - d^2\\
          &= 2 (1 - \cos(\theta)) (v t)^2 - 2 X (1 - \cos(\theta))v t + X^2 - d^2\\
        \end{aligned}
      $$
      $$
        \begin{aligned} 
          \phantom{P_\theta(t) }
\fi
      &= 2 (1 - \cos(\theta)) v^{2} t^{2} - 2 X (1 - \cos(\theta))v t + X^2 - d^2,\\
    \end{aligned}
    $$
  \fi
  where we used $\cos(\theta)=\cos(\theta_{2}-\theta_{1})= \cos(\theta_{2})  \cos(\theta_{1}) + \sin(\theta_{2}) \sin(\theta_{1}) $.

  We identify two cases:
  \begin{enumerate}
  \item Case 1: $\cos(\theta) \neq 1$.
  Then $P_\theta(t)$ is a second-degree polynomial in $t$. It is of the form $at^2 + b t + c$ 
  with $a = 2 (1 - \cos(\theta)) v^2$,
  $b = -2 X (1 - \cos(\theta))v$ and
  $c = X^2 - d^2$.
  We know that $P_\theta(t)$ has $a$ with positive sign for all $t$, because $(1 - \cos(\theta)) > 0$ when $\cos(\theta) \neq 1$. Thus, as $a>0$, by second-degree polynomial inequalities properties, $P_\theta(t) \ge 0$ for all $t$ if and only if its discriminant $\Delta = b^2 - 4 a c$ is negative, that is, 
  \if\shortVersion 1
    $
    \forall t \in \mathbb{R}, P_\theta(t) \ge 0
    \Leftrightarrow
    b^2 - 4 a c  \le 0.
    $
    Thus,
    \begin{equation}
      X \ge d\sqrt{\frac{2}{1 + \cos(\theta)}}   
      \label{eq:methogology:pformula}
    \end{equation}
  \else
  
    $$
    \forall t \in \mathbb{R}, P_\theta(t) \ge 0
    \Leftrightarrow
    b^2 - 4 a c  \le 0.
    $$
    
    Thus:
    $$
    \begin{aligned}
    & 2^2 X^2 (1 - \cos(\theta))^2 v^2 - 4 \cdot 2 (1 - \cos(\theta))v^2(X^2 - d^2)  \le 0   \Leftrightarrow\\
    & (4(1 - \cos(\theta))v^{2}) ( X^2 (1 - \cos(\theta)) -  2 (X^2 - d^2))  \le 0   \Rightarrow\\
    & X^2 (1 - \cos(\theta)) - 2 (X^2 - d^2)  \le 0  \Leftrightarrow
     2 d^2 - X^2 (1 + \cos(\theta))  \le 0  \Leftrightarrow\\
    & \frac{2d^{2}}{1 + \cos(\theta)}  \le X^2  \Rightarrow 
    \end{aligned}
    $$
    
    \begin{equation}
      X \ge d\sqrt{\frac{2}{1 + \cos(\theta)}}   
      \label{eq:methogology:pformula}
    \end{equation}
  \fi
  \item Case 2: $\cos(\theta) = 1$. Then $P_{\theta}(t) = X^{2} - d^{2}$. In this case, $P_{\theta}(t) \ge 0$ for all $t$ when $X^{2} - d^{2} \ge 0 \Rightarrow X \ge d$. This is the same as using $\cos(\theta) = 1$ in (\ref{eq:methogology:pformula}).
  \end{enumerate}
  
  Hence, (\ref{eq:methogology:pformula}) gives,   for the robots to respect the minimum distance $d$ for every time  $t$,  a relation between the minimum distance, the angle between the lanes and the distance between Robot 1 and Robot 2 when Robot 1 reaches the target. The final result is obtained noticing that (\ref{eq:methogology:pformula}) is equivalent to
  \if\shortVersion 1
    $
    \tau \ge \frac{d}{v} \sqrt{\frac{2}{1 + \cos(\theta)}}.
    $
  \else
    $$
    \tau \ge \frac{d}{v} \sqrt{\frac{2}{1 + \cos(\theta)}}.
    $$
  \fi
  \fi %
\end{proof}

This result enables us to show Proposition \ref{prop:optimal_throughput_straight_line_punctual_target}.

\begin{proposition}
The optimal throughput $f$ for a point-like target ($s = 0$) is $f = \frac{v}{d}$. 
It is achieved when robots form a single line, i.e., the angle between robots trajectories must be $0$.
\label{prop:optimal_throughput_straight_line_punctual_target}
\end{proposition}

\begin{proof}
We show by induction on $N$, which is the number of robots moving towards the target. We define $\theta_{N}$ as the angle between the trajectories of Robot $N - 1$ and Robot $N$; 
$\tau_{N}$, the minimum delay between the arrival of Robot $N - 1$ and Robot $N$; and
$\Delta_{N}$, the minimum delay between the arrival of Robot 1 and Robot $N$.
We want to show the following predicate: for all $N \ge 2$, $\Delta_{N} = (N-1) d / v$ for $\theta_2 = \theta_3 = \ldots = \theta_{N} = 0$.

Base case ($N=2$):
Let $\tau_2$ be the delay between the arrival of Robot 1 and Robot 2.
From Lemma \ref{prop:security_distance_punctual_target} we have that the minimum delay between Robot 1 and Robot 2 is equal to
$\frac{d}{v} \sqrt{\frac{2}{1 + \cos(\theta_2)}}$,
which is minimised by $\theta_2 = 0$.
Then, the minimum delay between the two robots is $\tau_2 = d/v = \Delta_{2}$.

Inductive step: We suppose the predicate is true for a given $N-1 \ge 2$. 
We will show that it implies the predicate is true for $N$ robots.
As in the previous case, we conclude from Lemma \ref{prop:security_distance_punctual_target} that the minimum delay between Robot $N-1$ and Robot $N$ is equal to
$\frac{d}{v} \sqrt{\frac{2}{1 + \cos(\theta_{N})}}$, which is minimised for $\theta_{N} = 0$.
Then, the minimum delay between the two robots is $\tau_{N} = d/v$.
We have 
  \if\shortVersion 1
    $
    \begin{aligned}
    \Delta_{N} & = \Delta_{N-1} + \tau_{N}
              = (N-2)\frac{d}{v} + \frac{d}{v} 
              = (N-1)\frac{d}{v}.
    \end{aligned}
    $
  \else
    $$
    \begin{aligned}
    \Delta_{N} & = \Delta_{N-1} + \tau_{N}
              = (N-2)\frac{d}{v} + \frac{d}{v} 
              = (N-1)\frac{d}{v}.
    \end{aligned}
    $$

  \fi
Consequently, the minimum delay between Robot 1 and Robot $N$ is $\Delta_{N} = \sum_{i = 2}^{N} \tau_i = (N-1)\frac{d}{v}$ and the time of arrival of Robot $N$, for all $N$, 
is minimised for $\theta_2 = \theta_3 = \ldots = \theta_{N} = 0$. Finally, by Definition \ref{def:throughput2}, the throughput is $f = \frac{N-1}{\Delta_{N}} = \frac{v}{d}$.
\end{proof}

The insight derived from Proposition \ref{prop:optimal_throughput_straight_line_punctual_target} implies that we should increase the maximum speed of the robots or decrease the minimum distance between them to increase the throughput. It is also noted that the optimal trajectory for all the robots is to form a queue behind the target and Robot 1.
As a result, the optimal path is to create one lane to reach the target.
When we increase the angle $\theta$ between the path of a robot and the next one, 
we introduce a delay from the optimal throughput. 
For instance, Figure \ref{fig:theoretical:normalised_delay} shows the normalised delay for different angles $\theta$ (normalised by dividing $\tau$ by $\tau_{min} = d / v$) between two robots, according to Lemma \ref{prop:security_distance_punctual_target}. This figure shows that for an angle of $\pi/3$, the minimum delay is $15\%$ higher than for an angle of 0, and the minimum delay is $41\%$ higher for an angle of $\pi/2$.

\begin{figure}[t]
  \centering
  \includegraphics[width=0.8\columnwidth]{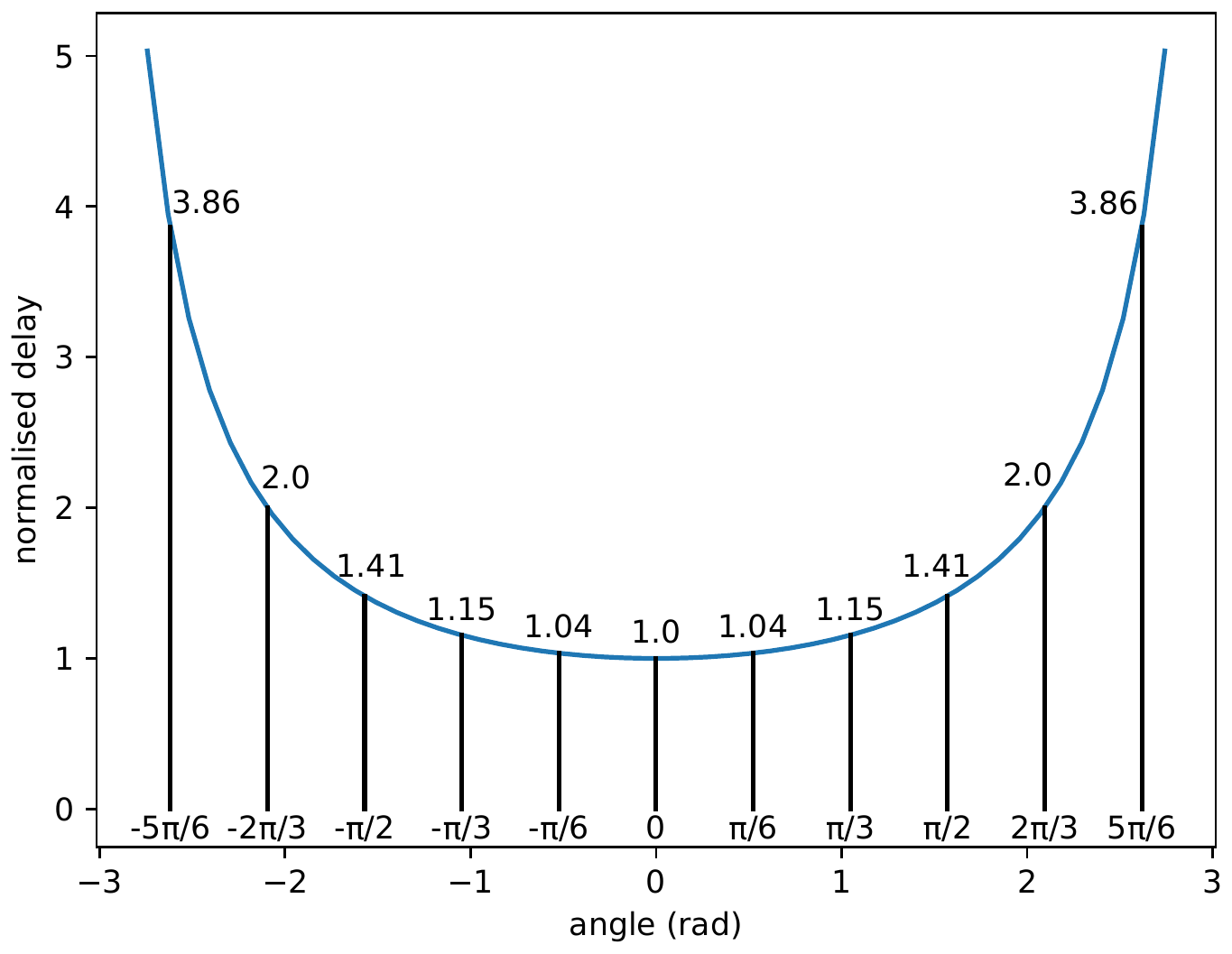}
  \caption{Normalised delay versus the angle between the trajectories of the robots.}
  \label{fig:theoretical:normalised_delay}
\end{figure}

\subsection{Small target area: $0 < s < d/2$}
\label{sec:smalltargetarrea2}

In this section, we suppose a small target area where $0 < s < d/2$, hence we cannot yet fit two lanes with a distance $d$ towards the target. The next results are based on a strategy using two \emph{parallel lanes} as close as possible to guarantee the minimum distance $d$ between robots. Figure \ref{fig:smalltargetarrea} describes these two parallel lanes. We hereafter call this  strategy \emph{compact lanes}. Proposition \ref{prop:parallel1} considers a target area with radius $0 < s \le \frac{\sqrt{3}}{4}d$, and Proposition \ref{prop:parallel2} assumes $\frac{\sqrt{3}}{4}d < s < \frac{d}{2}$.

\begin{figure}[t]
  \centering
  \includegraphics[width=0.5\columnwidth]{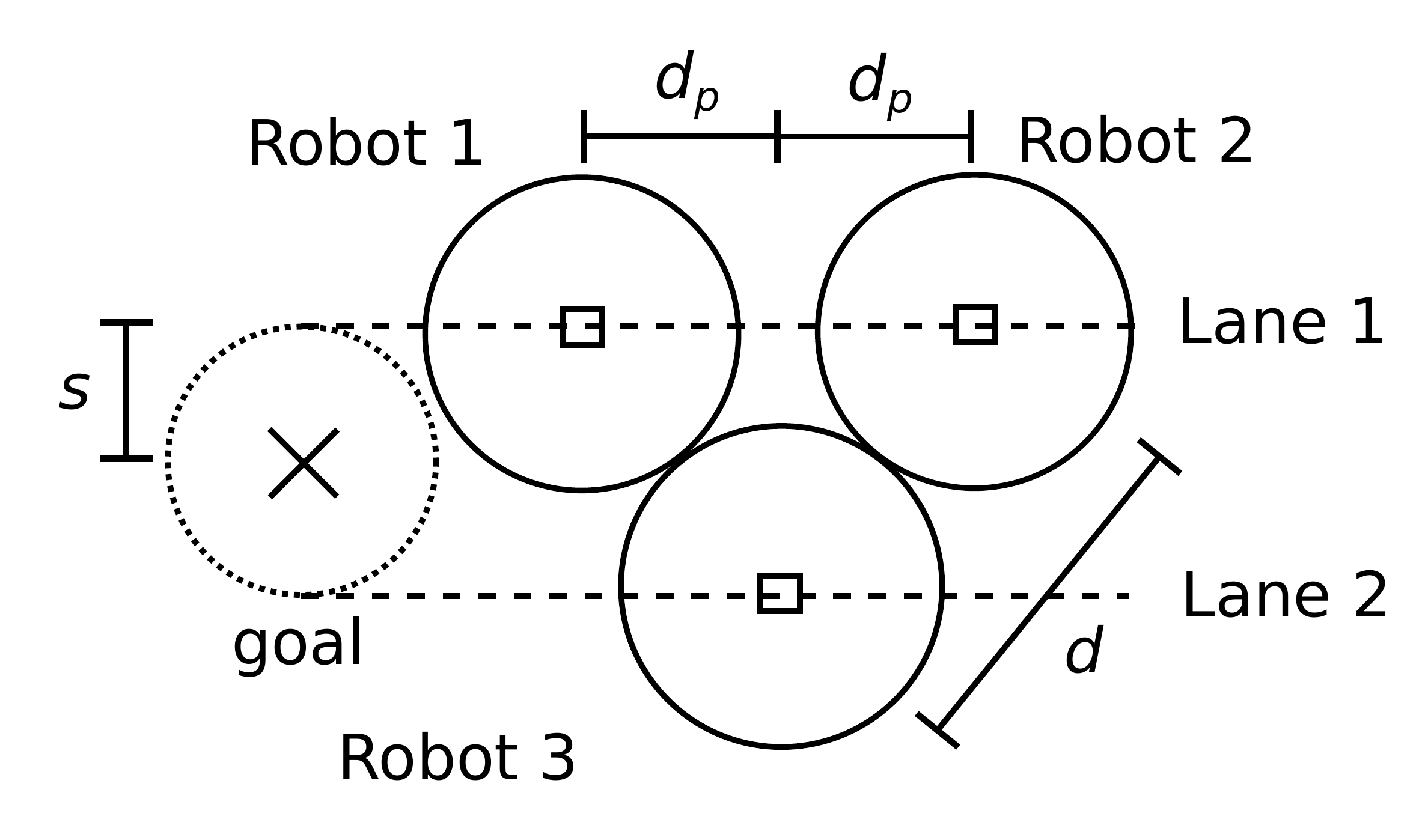}
  \caption{Two parallel robot lanes for a small target, illustrating the compact lanes strategy.}
  \label{fig:smalltargetarrea}
\end{figure}

\begin{proposition}
  Assume two parallel lanes with robots at maximum speed $v$ and maintaining a minimum distance $d$ between them. The throughput of a common target area with radius $0<s \le \frac{\sqrt{3}}{4}d$ at a given time $T$ after the first robot has reached the target area is
  \begin{equation}
    f(T) = \frac{1}{T} \left(\left \lfloor \frac{vT}{2 \sqrt{d^2 - (2 s)^2}} \right \rfloor + \left \lfloor \frac{vT}{2 \sqrt{d^2 - (2 s)^2}} + \frac{1}{2}\right \rfloor \right )
    \label{eq:giventime1}
  \end{equation}
  and is limited by
  \begin{equation}
    f = \lim_{T \to \infty} f(T) = \frac{v}{d\sqrt{1-(\frac{2s}{d})^{2}}}.
    \label{eq:limit1}
  \end{equation} 
  \label{prop:parallel1}
\end{proposition}
\begin{proof}
  Consider Figure \ref{fig:smalltargetarrea}. The distance between the lanes is $2 s$, and the minimum distance between two robots is $d$. Thus, 
  $d_p = \sqrt{d^2 - (2 s)^2}.$
  Hence, the distance between two robots in the same lane is 
  $d_e = 2 d_p = 2 \sqrt{d^2 - (2 s)^2}.$
  
  The distance between two robots in the same lane must not be less than $d$, so $d_{e} \ge d.$ This is true, because as we have that $0 < s \le \frac{\sqrt{3}}{4}d$, 
  \if\shortVersion 1
   $
        d_e =   2 \sqrt{d^2 - (2 s)^2} 
            \ge 2\sqrt{d^2 - \left(2 \frac{\sqrt{3}}{4}d \right)^2}
            = d.
   $
  \else
\ifexpandexplanation

    $$
      \begin{aligned}
        d_e &=   2 \sqrt{d^2 - (2 s)^2} 
            \ge 2\sqrt{d^2 - \left(2 \frac{\sqrt{3}}{4}d \right)^2}\\ 
            &=   2\sqrt{d^2 - \left( \frac{\sqrt{3}}{2}d \right)^2}\\
            &=   2\sqrt{d^2 - \frac{3}{4}d^{2}}\\
            &=   2\sqrt{\frac{1}{4}d^{2}}\\
            &= d.
      \end{aligned}
    $$   
\else

   $$
     d_e =   2 \sqrt{d^2 - (2 s)^2} 
        \ge 2\sqrt{d^2 - \left(2 \frac{\sqrt{3}}{4}d \right)^2}
         = d.
   $$
\fi
  \fi

  Without loss of generality, assume the first robot to reach the target area being at the top lane in Figure \ref{fig:smalltargetarrea}. The number of robots on any lane is the integer division of the size of the lane by the offset between the robots plus one (because we are including the first robot in this counting). Therefore, the number of robots for a given time $T$ in the top lane is  $N_{1}(T) = \left \lfloor \frac{vT}{d_e} + 1\right \rfloor$ and in the bottom lane is $N_{2}(T) = \left \lfloor \frac{vT - d_p}{d_e} + 1\right \rfloor = \left \lfloor \frac{vT}{d_e} + \frac{1}{2}\right \rfloor$. By Definition \ref{def:throughput2},
    \if\shortVersion 1
      $f(T) = \frac{N_{1}(T) + N_{2}(T) - 1}{T} = \frac{1}{T} \left(\left \lfloor \frac{vT}{2 \sqrt{d^2 - (2 s)^2}} \right \rfloor + \left \lfloor \frac{vT}{2 \sqrt{d^2 - (2 s)^2}} + \frac{1}{2}\right \rfloor \right ).$
    \else
      $$f(T) = \frac{N_{1}(T) + N_{2}(T) - 1}{T} = \frac{1}{T} \left(\left \lfloor \frac{vT}{2 \sqrt{d^2 - (2 s)^2}} \right \rfloor + \left \lfloor \frac{vT}{2 \sqrt{d^2 - (2 s)^2}} + \frac{1}{2}\right \rfloor \right ).$$ 
    \fi
  By the definition of the floor function, $\lfloor x \rfloor = x - frac(x)$ with $0 \le frac(x) < 1$. Thus, 
  \if\shortVersion 1
    $
        \lim_{T \to \infty} f(T) 
          = \frac{v}{d\sqrt{1-(2s/d)^{2}}},
    $
  \else
\ifexpandexplanation    
    $$
      \begin{aligned}
        \lim_{T \to \infty} f(T) 
            &= \lim_{T \to \infty} \frac{1}{T} \left(\left \lfloor \frac{vT}{ d_e} \right \rfloor + \left \lfloor \frac{vT}{d_{e}} + \frac{1}{2}\right \rfloor \right ) 
            \\
          &= \lim_{T \to \infty} \frac{1}{T} \left( \frac{vT}{ d_e} - frac\left(\frac{vT}{ d_e}\right) + \frac{vT}{d_{e}} + \frac{1}{2} - frac \left ( \frac{vT}{d_{e}} + \frac{1}{2}\right ) \right )\\ 
      \end{aligned}
    $$ 
    $$ 
      \begin{aligned} 
            &= \lim_{T \to \infty} \frac{1}{T}\left(\frac{2vT}{ d_e}\right) + \frac{1}{T} \left(  - frac\left(\frac{vT}{ d_e}\right)  + \frac{1}{2} - frac \left ( \frac{vT}{d_{e}} + \frac{1}{2}\right ) \right )\\ 
          &= \frac{2v}{d_{e}} 
            \\&= \frac{2v}{2\sqrt{d^{2}-(2s)^{2}}} 
          = \frac{v}{d\sqrt{1-(2s/d)^{2}}},
      \end{aligned}
    $$
\else 
    $$
      \begin{aligned}
        \lim_{T \to \infty} f(T) 
          &= \lim_{T \to \infty} \frac{1}{T} \left( \frac{vT}{ d_e} - frac\left(\frac{vT}{ d_e}\right) + \frac{vT}{d_{e}} + \frac{1}{2} - frac \left ( \frac{vT}{d_{e}} + \frac{1}{2}\right ) \right )\\ 
          &= \frac{2v}{d_{e}} 
          = \frac{v}{d\sqrt{1-(2s/d)^{2}}},
      \end{aligned}
    $$
\fi 
  \fi
  as $\displaystyle \lim_{T\to \infty} \frac{frac(x)}{T} = 0$, for any $x$.
\end{proof}

\begin{proposition}
  Assume two parallel lanes with robots at maximum speed $v$ and maintaining a minimum distance $d$ between them. The throughput of a common target area with radius $\frac{\sqrt{3}}{4}d <s < \frac{d}{2}$ at a given time $T$ after the first robot has reached the target area  is 
  \begin{equation}
    f(T) = \frac{1}{T} \left(\left \lfloor \frac{vT}{d} \right \rfloor + \left \lfloor \frac{vT}{d} + \frac{1}{2}\right \rfloor \right ) 
    \label{eq:giventime2}
  \end{equation}
  and is limited by
  \begin{equation}
    f = \lim_{T\to \infty} f(T) = \frac{2v}{d}. 
    \label{eq:limit2}
  \end{equation}
  \label{prop:parallel2}
\end{proposition}
\begin{proof}
  As the distance between the robots must be at least $d$ and $\frac{\sqrt{3}}{4}d < s < \frac{d}{2}$, we can assign $d_{p} = d/2$ in Figure \ref{fig:smalltargetarrea}. By doing so, two robots side by side in one lane and a robot in the other lane form an equilateral triangle with side measuring $d$, whose height has size $\frac{\sqrt{3}}{2}d$. Hence, the minimum diameter of the circular target region must be  this value, and the hypothesis says so. 
  
  Also, the radius of the target area is less than $d/2,$ implying that the three robots in Figure \ref{fig:smalltargetarrea} must stay in the equilateral triangle formation because the two lanes cannot be far by $d$ units of distance. 
  
  Thus, the throughput for a given time $T$ is calculated similarly as in Proposition \ref{prop:parallel1} resulting 
  \if\shortVersion 1
    $f(T) = \frac{1}{T} \left(\left \lfloor \frac{vT}{d} \right \rfloor + \left \lfloor \frac{vT}{d} + \frac{1}{2}\right \rfloor \right )$  and   $f = \lim_{T\to \infty} f(T) = \frac{2v}{d}.$
  \else
    $$f(T) = \frac{1}{T} \left(\left \lfloor \frac{vT}{d} \right \rfloor + \left \lfloor \frac{vT}{d} + \frac{1}{2}\right \rfloor \right ) \text{ and }  f = \lim_{T\to \infty} f(T) = \frac{2v}{d}.$$
  \fi
\end{proof}

Observe that if we use $T = k \frac{d}{v}$ for any $0 < k \in \Zeta$ in (\ref{eq:giventime2}), the compact lanes strategy can achieve the throughput of two parallel lanes of robots going in the direction of the target region when $T=k \frac{d}{v}$ for any $k \in \Zeta$ or when $T \to \infty$, even though two robots cannot reach the target region at the same time.

\subsection{Large target area: $s \geq d/2$}

We now focus on situations where more than two robots can simultaneously touch the target. 
In this section, we will present three feasible strategies. 

The simplest strategy is to consider several parallel lanes being at a distance $d$ from each other. However, it is possible to obtain higher throughput. In particular, we identify two other strategies: (a) using parallel straight line lanes that may be distanced lower than $d$; and (b) robots moving towards the target following curved trajectories. The strategy (a) uses more than two compact lanes, extending the strategy presented in the previous section. By doing this, the robots fit in a hexagonal packing arrangement moving towards the target region. The strategy (b) uses a touch and run approach. In it, robots do not cross the target area, they only reach it and return by the opposite direction using curved trajectories which respect the minimum distance $d$.

We start with the parallel lanes strategy, which has the lowest asymptotic throughput over the strategies presented in this section, for comparison with the other strategies. In particular, it will be used later as a justification for the lowest number of lanes used in the strategy (b) in (\ref{eq:Kbounds}) in Proposition \ref{prop:Kboundsrk}. Following their description and properties, a discussion comparing them is provided.

\subsubsection{Parallel lanes}

Here we consider the robots moving inside lanes. 
The lanes are straight lines, and the linear velocity $v$ of the robots is constant.
We consider the lanes being separated by a distance $d$ and each robot maintaining a distance $d$ from each other. 

\begin{proposition}
Assume a circular target region with centre at $(0,0)$ and radius $s \ge \frac{d}{2}$ and parallel lanes starting at $(s,s-(i-1)d)$ for $i\in \{1, \dots, \left\lfloor \frac{2s}{d} \right\rfloor+1\}$. At each Lane $i$, the first robot is located at the point $(s,s-(i-1)d)$ in the starting configuration. Then, the first robot to reach the target is located at $(s, s-(J-1)d)$, for 
\if\shortVersion 1
  $
    J=
      \left\lfloor\frac{s}{d}\right\rfloor + 1$, 
         if  $\left\vert s - \left\lfloor\frac{s}{d}\right\rfloor d\right\vert \le \left\vert s - \left\lceil\frac{s}{d}\right\rceil d\right\vert $, 
      otherwise
      $J=\left\lceil\frac{s}{d}\right\rceil + 1$. 
\else
  $$
    J=
    \begin{cases}
      \left\lfloor\frac{s}{d}\right\rfloor + 1,
        & \text{ if } \left\vert s - \left\lfloor\frac{s}{d}\right\rfloor d\right\vert \le \left\vert s - \left\lceil\frac{s}{d}\right\rceil d\right\vert , \\
      \left\lceil\frac{s}{d}\right\rceil + 1, 
        & \text{ otherwise.}
    \end{cases}
  $$
\fi
The throughput for a given time $T$ after the first robot reaches the target region is 
\begin{equation}
  f_{p}(T) =  \frac{1}{T}\left(\sum_{i=1}^{\left\lfloor\frac{2s}{d}\right\rfloor+1} N_{i}(T)\right) - \frac{1}{T},
  \label{eq:parallelT}
\end{equation}
for
\if\shortVersion 1
  $
    N_{i}(T) = 
      \left\lfloor \frac{vT - d_{i} + d_{J} }{d}  + 1 \right\rfloor$,
         if  $T \ge \frac{d_{i}-d_{J}}{v}$, otherwise,
      $N_{i}(T) = 0, $
  $d_{j} = s - \sqrt{s^{2} - (s-(j-1)d)^{2}}$,  and 
\else 
  $$
    N_{i}(T) = 
    \begin{cases}
      \left\lfloor \frac{vT - d_{i} + d_{J} }{d}  + 1 \right\rfloor,
        & \text{ if } T \ge \frac{d_{i}-d_{J}}{v},\\
      0, & \text{ otherwise,}
    \end{cases}
  $$
  $$d_{j} = s - \sqrt{s^{2} - (s-(j-1)d)^{2}}, \text{ and }$$ 
\fi
\begin{equation} 
f_{p} = \lim_{T \to \infty} f_{p}(T) =  \left\lfloor \frac{2s}{d} + 1 \right\rfloor  \frac{v}{d}.
\label{eq:parallelLimit}
\end{equation}
\label{methodology:bigtarget:independent_straight:proof}
\end{proposition}

\begin{proof}
  \begin{figure}[t]
    \centering
    \includegraphics[width=0.65\columnwidth]{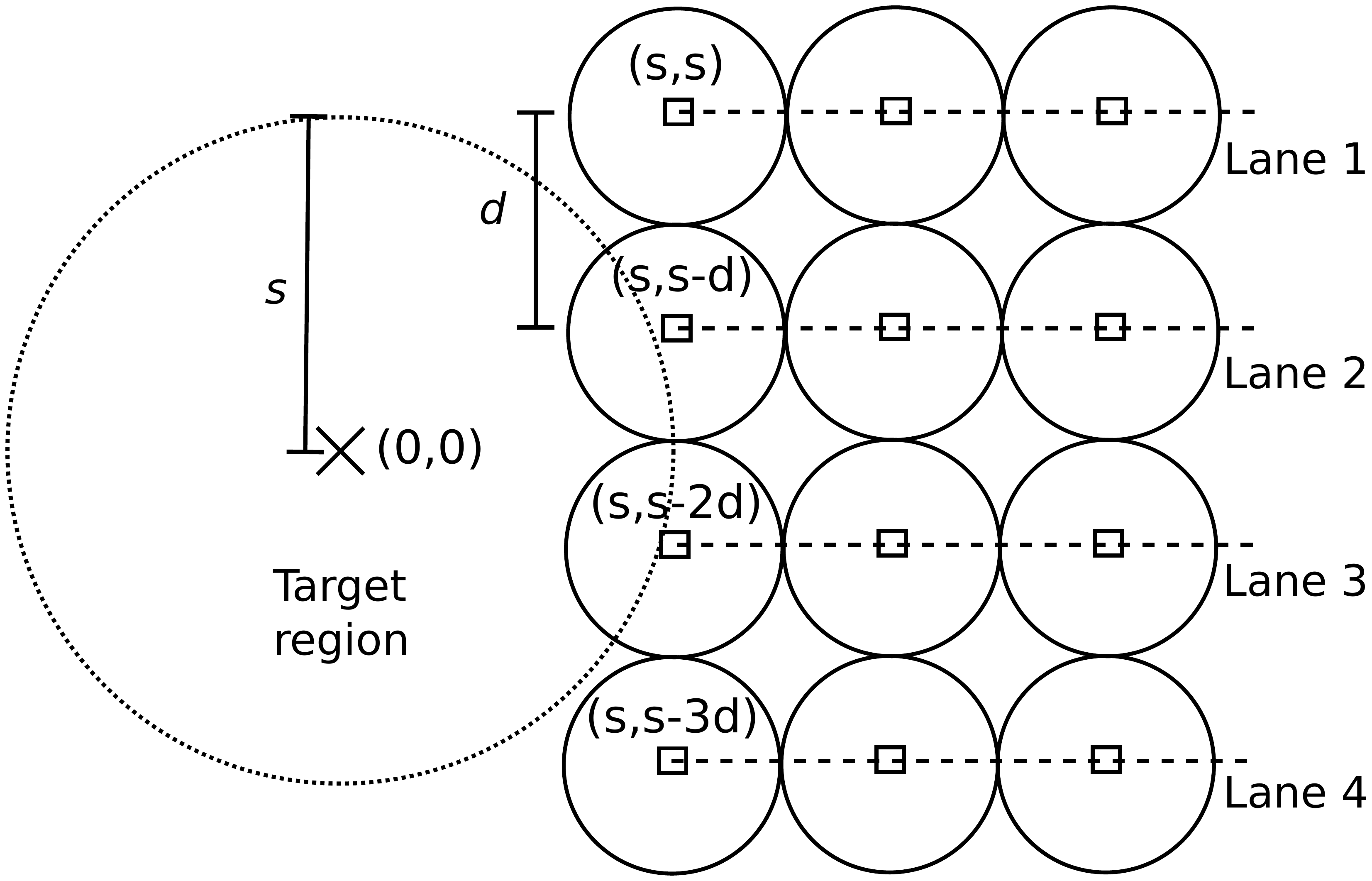}
    \caption{Example of the parallel lanes strategy. Robots and parallel lanes are distant by $d$. The first lane, Lane 1, is at the top. The first robot of each lane is located at $(s,s-(i-1)d)$ for the Lane $i$.}
    \label{fig:parallellanes}
  \end{figure} 
  
  When robots move in straight lines in a single lane, we know from Proposition \ref{prop:optimal_throughput_straight_line_punctual_target} that the optimal throughput is $\frac{v}{d}$. Since $s \geq \frac{d}{2}$, we can have multiple straight line lanes that are parallel to each other (Figure \ref{fig:parallellanes}). 
  
  \begin{figure}[t]
    \centering
    \includegraphics[width=0.55\columnwidth]{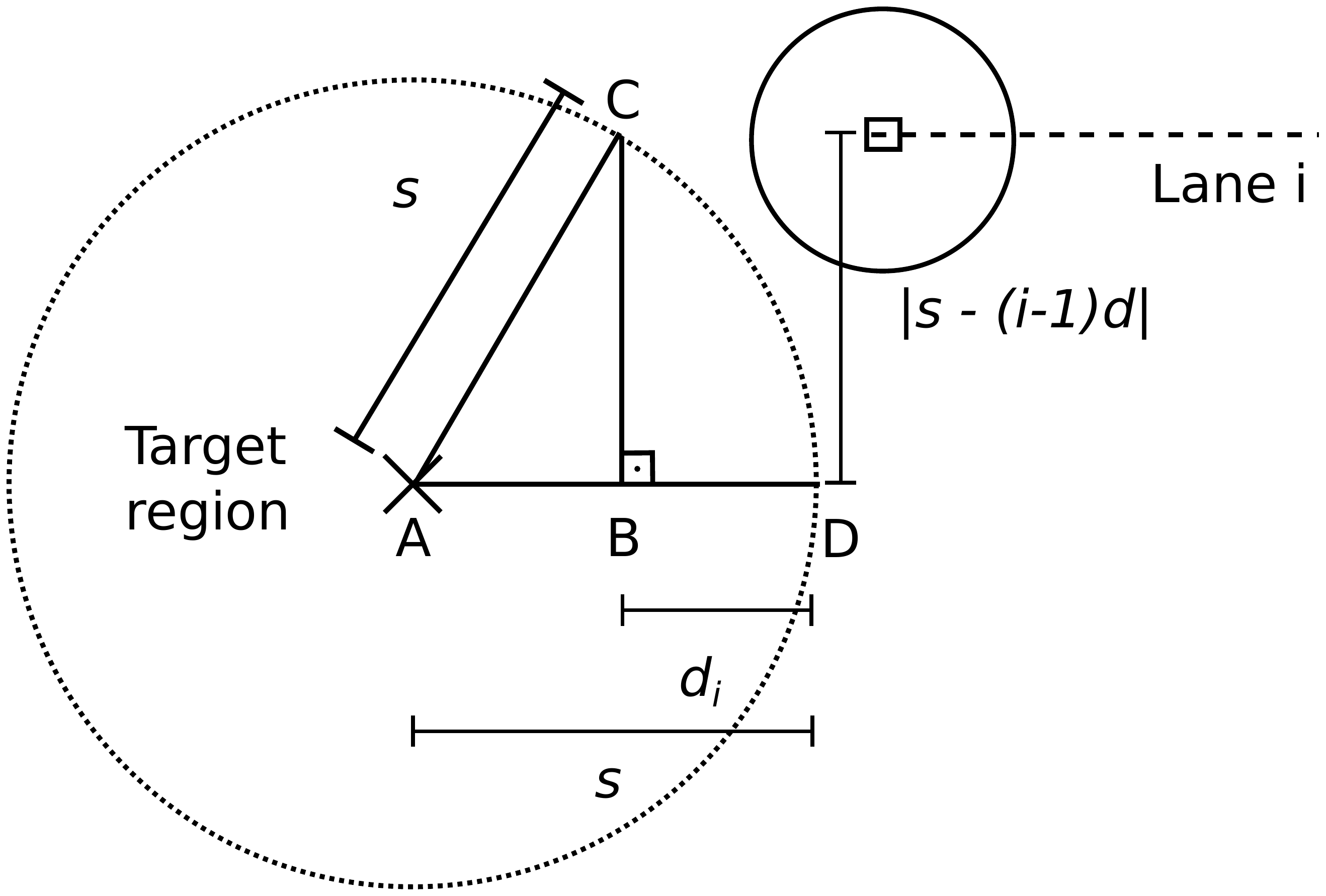}
    \caption{The distance from the target region to a robot at the beginning of the Lane $i$ is equal to $d_{i}$ (represented by $\overline{BD}$).}
    \label{fig:RobotFarCentre}
  \end{figure} 
  
  As the robots are going to a circular target region, robots next to  the centre reach the region in shorter time than the others. The first robot of each lane must run an additional distance $d_{i}$ from the beginning of its lane, which is related to its y-coordinate.  Figure \ref{fig:RobotFarCentre} illustrates this distance for a robot in Lane $i$. The right triangle $ABC$ has hypotenuse $\overline{AC}$ measuring $s$, so the horizontal cathetus $\overline{AB}$ measures $\sqrt{s^{2} - (s-(i-1)d)^{2}}$. Consequently, the  robot in the Lane $i$ needs to walk an additional distance represented by $\overline{BD}$, which  has $d_{i} = \vert \overline{BD}\vert = s - \vert \overline{AB}\vert = s - \sqrt{s^{2} - (s-(i-1)d)^{2}}$ units of length. 
  
  This distance is minimised when $\vert \overline{BD}\vert = 0$, that would happen if $i = \frac{s}{d}+1$. However, $i$ must be integer, so $\vert \overline{BD}\vert $  is minimised for an integer $J$ that minimises $d_{J}$.  If $\frac{s}{d}\notin \Zeta$, the two nearest integers  are $\left\lfloor\frac{s}{d}\right\rfloor$ and $\left\lceil\frac{s}{d}\right\rceil$. Thus, if $J = \left\lfloor\frac{s}{d}\right\rfloor + 1$ then, equivalently,
  \if\shortVersion 1
    $d_{\left\lfloor\frac{s}{d}\right\rfloor+1} \le d_{\left\lceil\frac{s}{d}\right\rceil+1} \LR$ 
    $\left\vert s - \left\lfloor\frac{s}{d}\right\rfloor d\right\vert \le \left\vert s - \left\lceil\frac{s}{d}\right\rceil d\right\vert $. Thus,
    $
      J= \left\lfloor\frac{s}{d}\right\rfloor + 1,
    $        
    if  $\left\vert s - \left\lfloor\frac{s}{d}\right\rfloor d\right\vert \le \left\vert s - \left\lceil\frac{s}{d}\right\rceil d\right\vert $, otherwise, $J=
        \left\lceil\frac{s}{d}\right\rceil + 1.$
  \else  
    $d_{\left\lfloor\frac{s}{d}\right\rfloor+1} \le d_{\left\lceil\frac{s}{d}\right\rceil+1} \LR$ 
    $s - \sqrt{s^{2} - \left(s - \left\lfloor\frac{s}{d}\right\rfloor d\right)^{2}} \le$ $ s - \sqrt{s^{2} -\left(s - \left\lceil\frac{s}{d}\right\rceil d\right)^{2}}  \LR $
    $ \sqrt{s^{2} -\left(s - \left\lceil\frac{s}{d}\right\rceil d\right)^{2}} \le \sqrt{s^{2} - \left(s - \left\lfloor\frac{s}{d}\right\rfloor d\right)^{2}}\LR $ 
    $ s^{2} -\left(s - \left\lceil\frac{s}{d}\right\rceil d\right)^{2} \le s^{2} - \left(s - \left\lfloor\frac{s}{d}\right\rfloor d\right)^{2}\LR $ 
    $\left(s - \left\lfloor\frac{s}{d}\right\rfloor d\right)^{2} \le \left(s - \left\lceil\frac{s}{d}\right\rceil d\right)^{2} \LR $
    $\left\vert s - \left\lfloor\frac{s}{d}\right\rfloor d\right\vert \le \left\vert s - \left\lceil\frac{s}{d}\right\rceil d\right\vert $. Thus,
    $$ 
      J=
      \begin{cases}
        \left\lfloor\frac{s}{d}\right\rfloor + 1,
          & \text{ if } \left\vert s - \left\lfloor\frac{s}{d}\right\rfloor d\right\vert \le \left\vert s - \left\lceil\frac{s}{d}\right\rceil d\right\vert , \\
        \left\lceil\frac{s}{d}\right\rceil + 1, 
          & \text{ otherwise.}
      \end{cases}
    $$
  \fi
  
  Let $N(T)$ be the number of robots that arrive at the target region until a given time $T$ after the first robot has reached it. Thus,
  \if\shortVersion 1
    $
      N(T) = \sum_{i=1}^{\left\lfloor\frac{2s}{d}\right\rfloor+1} N_{i}(T),
    $
  \else
    $$
      N(T) = \sum_{i=1}^{\left\lfloor\frac{2s}{d}\right\rfloor+1} N_{i}(T),
    $$
  \fi
  for the number of robots at Lane $i$, $N_{i}(T)$, that arrived at the target region by  time $T$. As every robot has the same linear velocity and started at the same x-coordinate, when the first robot at Lane $J$ reaches the target region, all robots have run $d_{J}$ units of length. Hence, at each Lane $i$, instead of running an additional $d_{i}$ to reach the target region, they need to run $d_{i} - d_{J}$. Consequently,
  \if\shortVersion 1 
    $ N_{i}(T) = 
      \left\lfloor \frac{vT - (d_{i} - d_{J}) }{d}  + 1 \right\rfloor, $
        if  $T \ge \frac{d_{i}-d_{J}}{v},$ otherwise, 
      $N_{i}(T) = 0,$
  \else
    $$
      N_{i}(T) = 
      \begin{cases}
        \left\lfloor \frac{vT - (d_{i} - d_{J}) }{d}  + 1 \right\rfloor,
          & \text{ if } T \ge \frac{d_{i}-d_{J}}{v},\\
        0, & \text{ otherwise,}
      \end{cases}
    $$ 
  \fi
  and, by Definition \ref{def:throughput2},
  \if\shortVersion 1 
    $f_{p}(T) = \frac{N(T)-1}{T} = \frac{1}{T}\left(\sum_{i=1}^{\left\lfloor\frac{2s}{d}\right\rfloor+1} N_{i}(T)\right) - \frac{1}{T}.$
  \else
    $$f_{p}(T) = \frac{N(T)-1}{T} = \frac{1}{T}\left(\sum_{i=1}^{\left\lfloor\frac{2s}{d}\right\rfloor+1} N_{i}(T)\right) - \frac{1}{T}.$$
  \fi
  Also, 
  \if\shortVersion 1 
    $
        f_{p} 
          = \lim _{ T \to \infty} f_{p}(T)
          = \left\lfloor\frac{2s}{d} + 1\right\rfloor\frac{v}{d},
    $
  \else
    $$
      \begin{aligned}
        f_{p} 
          &= \lim _{ T \to \infty} f_{p}(T)
          = \lim _{ T \to \infty} \left( \frac{1}{T}\left(\sum_{i=1}^{\left\lfloor\frac{2s}{d}\right\rfloor+1} N_{i}(T)\right) - \frac{1}{T} \right)
          \\& = \lim _{ T \to \infty} \frac{1}{T}\sum_{i=1}^{\left\lfloor\frac{2s}{d}\right\rfloor+1} N_{i}(T) 
          = \lim _{ T \to \infty} \sum_{i=1}^{\left\lfloor\frac{2s}{d}\right\rfloor+1} \frac{N_{i}(T)}{T} 
          \\&= \lim _{ T \to \infty} \sum_{i=1}^{\left\lfloor\frac{2s}{d}\right\rfloor+1} \frac{1}{T}\left\lfloor \frac{vT - d_{i} + d_{J} }{d}  + 1 \right\rfloor
            \hspace{1.9cm} \left[\text{as } T\to \infty \Rightarrow T \ge \frac{d_{i}-d_{J}}{v}\right]
          \\&= \lim _{ T \to \infty} \sum_{i=1}^{\left\lfloor\frac{2s}{d}\right\rfloor+1} \frac{1}{T}\left( \frac{vT - d_{i} + d_{J} }{d}  + 1 - frac\left(\frac{vT - d_{i} + d_{J} }{d}  + 1\right)\right)
          \\&= \lim _{ T \to \infty} \sum_{i=1}^{\left\lfloor\frac{2s}{d}\right\rfloor+1} \left( \frac{v}{d} - \frac{d_{i} - d_{J} }{d T}  + \frac{1}{T}\right) - \lim _{ T \to \infty} \frac{1}{T} \sum_{i=1}^{\left\lfloor\frac{2s}{d}\right\rfloor+1} frac\left(\frac{vT - d_{i} + d_{J} }{d}  + 1\right)
          \\&= \left\lfloor\frac{2s}{d} + 1\right\rfloor\frac{v}{d},
      \end{aligned}
    $$ 
  \fi
  as $frac$ and $d_{i}$ are bounded for every $i$ due to $0 \le d_{i} \le s$ and $0 \le frac(x) < 1$ for any $x$.
\end{proof}

\subsubsection{Hexagonal packing}

By extending the compact lanes to more than two lanes, the robots will be packed in a hexagonal formation. An illustration of this strategy is shown in Figure \ref{fig:targethexagonalpacking}. As we can see, robots from different lanes are still able to move towards the target keeping a distance $d$ from each other, even though the lanes have a distance lower than $d$. 

We first compute an upper bound of the asymptotic throughput for the \emph{hexagonal packing} strategy, then we calculate the throughput for a given time using this strategy.

\begin{figure}
  \centering
  \includegraphics[width=0.65\columnwidth]{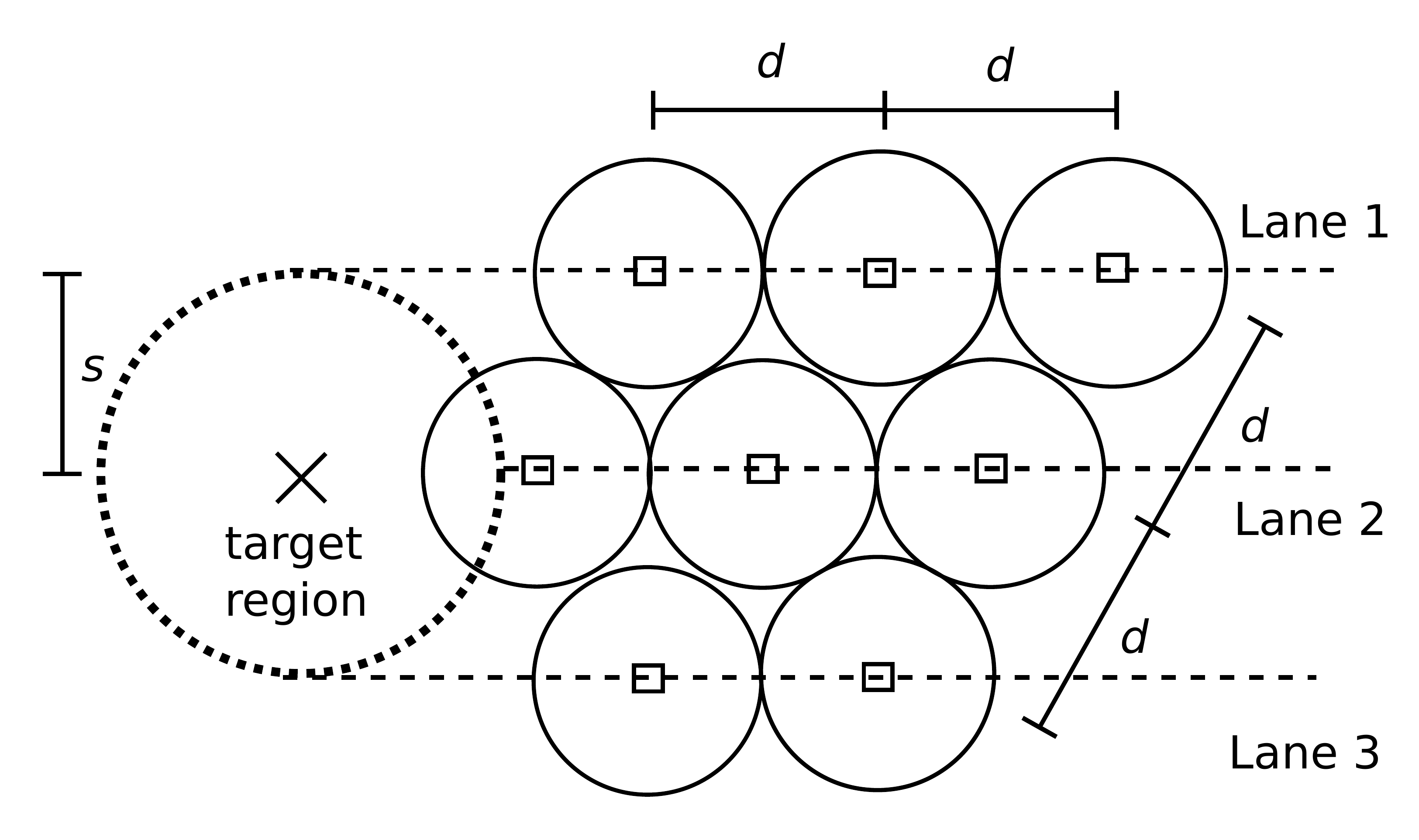}
  \caption{Robot lanes for hexagonal packing.}
  \label{fig:targethexagonalpacking}
\end{figure}

\begin{proposition}
Assume robots moving at speed $v$, going to 
  a circular target of radius $s$.
The upper bound of the asymptotic throughput for the hexagonal packing strategy is 
\begin{equation}
f_{h}^{max} = 
\frac{2}{\sqrt{3}}
\left(
\frac{2 s}{d} + 1
\right)
\frac{v}{d}.
\label{eq:limitnoangle}
\end{equation}
\label{prop:triangularthroughput}
\end{proposition}
\begin{proof}
Without loss of generality, we consider the target at the origin of a coordinate system and the robots are moving parallel to the $x$-axis. By Definition \ref{def:throughput2}, the throughput considers the number of robots that cross the target during a unit of time, after the first robot has reached it. We evaluate that number of robots, $N_{T}$, during a time $T$. As a result, computing the maximum throughput is reduced to finding the maximum number of robots (their centre of mass) that can fit in a rectangle of width $w(T) = v T$ and height $h = 2 s$ (Figure \ref{fig:recthexagonalpackingp}).

\begin{figure}[t]
  \centering
  \includegraphics[width=0.9\columnwidth]{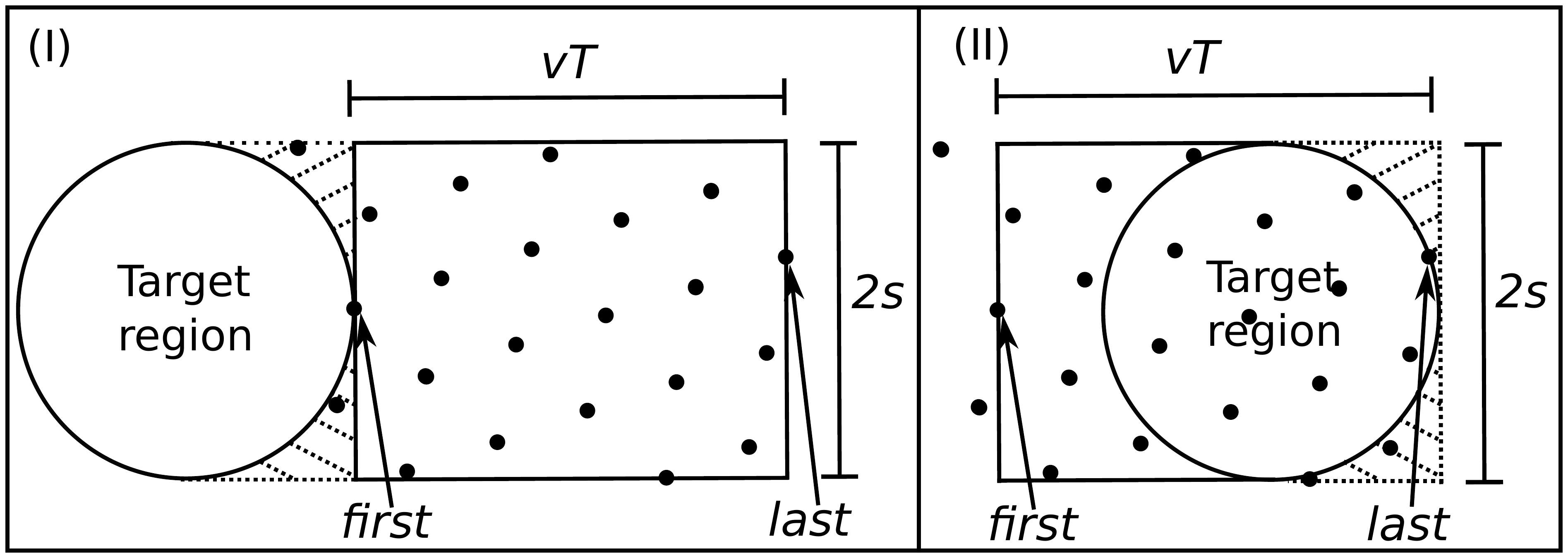}
  \caption{Robots are represented by black dots and they are in hexagonal formation and distant by $d$. In (I), only the first robot reached the target and, in (II), all robots that are not in the dashed area arrived in the target region before the last robot. The robots on the right dashed area should not be counted because their arrival time is greater than the arrival time of the last robot, hence they arrive after the considered time frame $T$. That is,  all robots on the dashed area in (I) should be counted as part of the number of robots that reached the target region in the time between the first and the last robot, while the robots on the dashed area in (II) should not. As these dashed areas have the same value, we consider in this proof the number of robots inside a rectangle $vT \times 2s$. Then, the dashed area used for counting in (II) replaces the unconsidered robots in the dashed area in (I). As we are concerned with $T \to \infty$, any possible difference between the number of robots on the dashed areas of either side due to the configuration of the hexagonal packing is negligible.}
  \label{fig:recthexagonalpackingp}
\end{figure}

Since we have the constraint that robots must be at a minimum distance $d$ from each other, we consider discs of radius $d/2$ as reserved areas for each robot, and any two reserved areas must not intersect. Therefore, the problem is equivalent to finding the optimal arrangement of circles of radius $d/2$ in a rectangle of width $W(T) = w(T) + 2 \frac{d}{2}$ and height $H=h + 2 \frac{d}{2}$. This formulation is a variant of the \textit{circle packing} problem, 
which is already well studied\footnote{See \url{http://packomania.com/} 
and \url{http://hydra.nat.uni-magdeburg.de/packing/crc_var/crc.html} for an informal introduction.}. The term $2 \frac{d}{2}$ was added because the circle packing problem deals with full circles, not their centres.

The optimal surface occupied by the circles divided by the rectangle area was proven to be $\pi \sqrt{3} / 6$ 
in the case of hexagonal packing over an infinite area \citep{chang2010simple}. Thus, the total area occupied by the circles representing the reserved areas of the robots is given by $\left(\pi \sqrt{3} / 6\right) H   W(T)$. Hence, the maximum number of robots $N_{T}$ that can fit inside the $H W(T)$ area is bounded by  $N_{opt}(T) \ge N_{T}$, for 
\if\shortVersion 1
  $
  N_{opt}(T) = 
  \left \lfloor
  \frac{\left(\pi \sqrt{3} / 6\right) H   W(T)   }{\pi d^2 / 4} 
  \right \rfloor = 
  \left \lfloor
  \frac{2H   W(T)   }{\sqrt{3} d^2 }
  \right \rfloor.
  $
\else
  $$
  N_{opt}(T) = 
  \left \lfloor
  \frac{\left(\pi \sqrt{3} / 6\right) H   W(T)   }{\pi d^2 / 4} 
  \right \rfloor = 
  \left \lfloor
  \frac{2H   W(T)   }{\sqrt{3} d^2 }
  \right \rfloor.
  $$
\fi
By Definition \ref{def:throughput2}, the maximum throughput is
\if\shortVersion 1
  $
  f_{h}^{max}(T) =  \frac{N_{opt}(T)-1}{T} = \frac{\left \lfloor
  \frac{2H   W(T)   }{\sqrt{3} d^2 }
  \right \rfloor-1}{T}.
  $
\else
  $$
  f_{h}^{max}(T) =  \frac{N_{opt}(T)-1}{T} = \frac{\left \lfloor
  \frac{2H   W(T)   }{\sqrt{3} d^2 }
  \right \rfloor-1}{T}.
  $$
\fi
As for any $x$, $\lfloor x \rfloor = x - frac(x)$ and $0 \le frac(x) < 1$, the upper bound of the asymptotic throughput is 
\if\shortVersion 1 
  $
  f_{h}^{max} = \lim_{T\to \infty} f_{h}^{max}(T)
    = \frac{2 }{\sqrt{3}}\left(\frac{2s}{d}+1\right)\frac{v}{d}.
  $
\else
  $$
  \begin{aligned}
  f_{h}^{max} &= \lim_{T\to \infty} f_{h}^{max}(T)
    = \lim_{T\to \infty} \frac{\left \lfloor \frac{2H   W(T)   }{\sqrt{3} d^2 } \right \rfloor-1}{T}
    = \lim_{T\to \infty} \frac{2 H   W(T)   }{\sqrt{3} d^2 T}\\
\ifexpandexplanation
    &= \lim_{T\to \infty}  \frac{2(h+d)(W(T)+d) }{\sqrt{3} d^2 T}\\
\fi
    &= \lim_{T\to \infty}  \frac{2(2s+d)(v T+d)  }{\sqrt{3} d^2 T}
    = \lim_{T\to \infty}  \frac{2 }{\sqrt{3}}\frac{2s+d}{d^2}\frac{v T+d}{T}\\
\ifexpandexplanation
    &= \lim_{T\to \infty}  \frac{2 }{\sqrt{3}}\left(\frac{2s}{d^2 }+\frac{d}{d^2 }\right)\left(\frac{v T}{ T} + \frac{d}{T}\right)\\
\fi
    &= \lim_{T\to \infty}  \frac{2 }{\sqrt{3}}\left(\frac{2s}{d^2 }+\frac{1}{d }\right)\left(v + \frac{d}{T}\right)
    = \frac{2 }{\sqrt{3}}\left(\frac{2s}{d^2 }+\frac{1}{d }\right)v
    = \frac{2 }{\sqrt{3}}\left(\frac{2s}{d}+1\right)\frac{v}{d}.\\
  \end{aligned}
  $$
\fi
\end{proof}

Proposition \ref{prop:triangularthroughput} presents an upper bound of the asymptotic throughput using hexagonal packing, but it did not tell us which is the best placement of the robots inside a corridor, since the hexagonal formation can be rotated by different angles. Hence, we are going to present results about the throughput considering the placement of the hexagonal packing inside a corridor of robots going to the target region. First, however, we will need the following definition:

\begin{definition}
  The hexagonal packing angle $\theta$ is the angle formed by the $x$-axis and the line formed by any robot at position $(x,y)$ and its neighbour at $(x + d \cos(\theta), y + d \sin(\theta))$ under the target region reference frame.
\end{definition}

Observe that any robot at $(x,y)$ under the hexagonal packing has at most six neighbours located at $\big(x + d \cos\big(\theta\big), y + d \sin\big(\theta\big)\big),$ $\big(x + d \cos\big(\theta+\frac{\pi}{3}\big), y + d \sin\big(\theta+\frac{\pi}{3}\big)\big),$ $\dots,$ $\big(x + d \cos\big(\theta+\frac{5\pi}{3}\big), y + d \sin\big(\theta+\frac{5\pi}{3}\big)\big)$ (Figure \ref{fig:hexangles}). If $\theta = \frac{\pi}{3}$, putting this value in the previous series results that the first neighbour robot would be at $(x + d \cos(\pi/3), y + d \sin(\pi/3)),$ and the last neighbour robot at $\left(x + d \cos\left(0\right), y + d \sin\left(0\right)\right)$. This is the same result if $\theta = 0$ was used. Consequently, due to this periodicity, we assume hexagonal packing angles in $\lbrack 0,\frac{\pi}{3} \rparen.$ 

\begin{figure}
  \centering
  \includegraphics[width=\columnwidth]{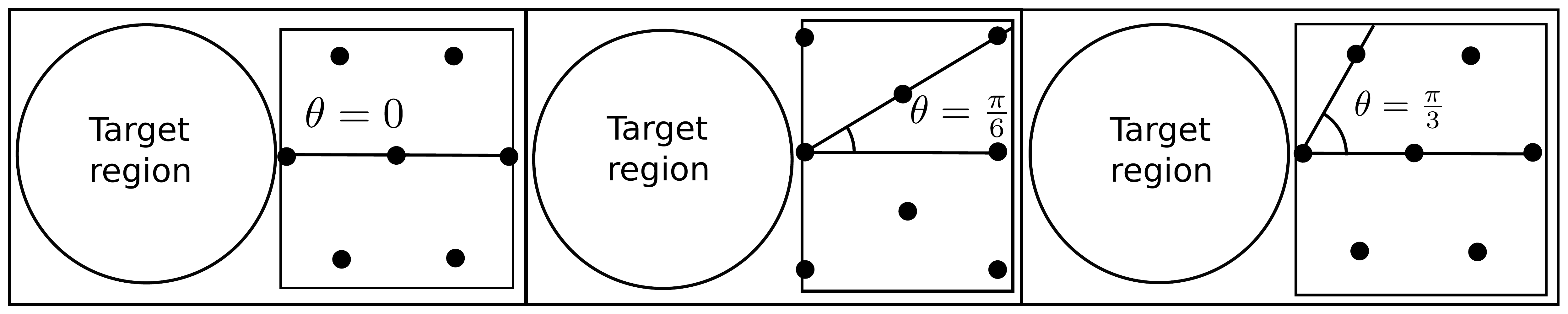}\\
  \caption{Example of hexagonal packing with different angles. The robots are the black dots. }
  \label{fig:hexangles}
\end{figure}

The next proposition states the bounds of the throughput in the limit towards the infinity for hexagonal packing using an arbitrary, but fixed, hexagonal packing angle $\theta$. We consider a fixed $\theta$ because normally in a robotic swarm the robots rely on local sensing. For getting the maximum number of robots inside the corridor, all robots should know the size of the corridor and communicate by local-ranged message sending. It would take time to send information, and for all robots to adjust their orientation each time a new robot joins the swarm when using this local sensing approach.

In other words, if the corridor where the robots are going in the direction of the target is increasing over time, then $\theta$ should change over time for the optimal throughput. However, in practice, changing the hexagonal packing angle implies all robots must turn to a hexagonal packing angle $\theta^{*}$ depending on the size of the new rectangle based on the added robots to it in order to maximise the number of robots inside the corridor. In addition to the time to send messages with this parameter, more time would be needed for every robot to adapt to the updated computed $\theta^{*}$ because the turning velocity of the robots is finite. Therefore, we do not handle this adjustable scenario in this paper.

\begin{proposition}
  Assume the robots using hexagonal formation coming to a circular target area with radius $s$ such that the first robot to reach it was at time $0$ at $(x_{0},y_{0}) = (w, 0)$, for any $w \ge s$. For a given time $T$, the robots are going to the target at linear speed $v$, keeping a minimum distance $d$ between neighbours ($0 < d \le 2 s$), using fixed hexagonal packing angle $\theta \in \lbrack 0,\pi/3 \rparen$. The throughput for a given time is given by

  \begin{equation}
    \begin{aligned}
      f_{h}(T,\theta) = & \frac{1}{T} \sum_{x_{h}=- n_{l}^{-}}^{n_{l}^{+}-1} \left( \lfloor Y_{2}^{R}(x_{h}) \rfloor - \lceil Y_{1}^{R}(x_{h}) \rceil + 1 \right)+\\
      &\frac{1}{T}  \sum_{x_{h} = B}^{U}\left(\lfloor Y_{2}^{S}(x_{h}) \rfloor - \lceil Y_{1}^{S}(x_{h}) \rceil + 1 \right) - \frac{1}{T},
    \end{aligned}
    \label{eq:hexthroughput}
  \end{equation}
  for $\left\lfloor Y_{2}^{R}(x_{h}) \right\rfloor \ge \left \lceil Y_{1}^{R}(x_{h}) \right \rceil $ and $\left\lfloor Y_{2}^{S}(x_{h}) \right\rfloor \ge \left \lceil Y_{1}^{S}(x_{h}) \right \rceil $  (if for some $x_{h}$, either of these conditions is false, we assume the respective summand for this  $x_{h}$ being zero). Also,
  \if\shortVersion 1
    $
       n_{l}^{-} = 
         \left\lfloor\frac{2s\sin\left(\vert \pi/6 - \theta\vert \right)}{\sqrt{3}d}\right \rfloor,
    $
    $
        n_{l}^{+} = 
          \left\lfloor\frac{2 (vT - s)\cos(\pi/6 - \theta)  + 2s\sin(\vert \pi/6 - \theta\vert ) }{\sqrt{3}d}+1\right\rfloor,
    $
  \else
    $$
       n_{l}^{-} = 
         \left\lfloor\frac{2s\sin\left(\vert \pi/6 - \theta\vert \right)}{\sqrt{3}d}\right \rfloor,
    $$
    $$
        n_{l}^{+} = 
          \left\lfloor\frac{2 (vT - s)\cos(\pi/6 - \theta)  + 2s\sin(\vert \pi/6 - \theta\vert ) }{\sqrt{3}d}+1\right\rfloor,
    $$
  \fi
  $$
    Y_{1}^{R}(x_{h}) = 
        \left\{
        \begin{array}{>{\displaystyle}c>{\displaystyle}l}
          \max\left(\frac{ \sin(\frac{\pi}{3}-\theta) x_h -\frac{s}{d}}{\cos\left(\theta -\frac{\pi}{6}\right)},
          \frac{-\cos(\frac{\pi}{3}-\theta)x_{h}}{\sin\left(\frac{\pi}{6}-\theta\right)}\right),
            & \text{ if } \theta < \pi/6, \\
          \max\left(\frac{ \sin(\frac{\pi}{3}-\theta) x_h -\frac{s}{d}}{\cos\left(\theta -\frac{\pi}{6}\right)},
          \frac{\frac{vT-s}{d} - \cos(\frac{\pi}{3}-\theta)x_{h}}{\sin\left(\frac{\pi}{6}-\theta\right)} \right),
          & \text{ if } \theta > \pi/6,\\
           \frac{x_{h}}{2}-\frac{s}{d},
            & \text{ if } \theta = \pi/6,
        \end{array}
        \right.
  $$
  $$
    Y_{2}^{R}(x_{h}) = 
        \left\{
        \begin{array}{>{\displaystyle}c>{\displaystyle}l}
          \min\left(\frac{ \sin(\frac{\pi}{3}-\theta) x_h +\frac{s}{d}}{\cos\left(\theta -\frac{\pi}{6}\right)},
          \frac{\frac{vT-s}{d} - \cos(\frac{\pi}{3}-\theta)x_{h}}{\sin\left(\frac{\pi}{6}-\theta\right)} \right),
            & \text{ if } \theta < \pi/6, \\
          \min\left(\frac{ \sin(\frac{\pi}{3}-\theta) x_h +\frac{s}{d}}{\cos\left(\theta -\frac{\pi}{6}\right)},
          \frac{-\cos(\frac{\pi}{3}-\theta)x_{h}}{\sin\left(\frac{\pi}{6}-\theta\right)}\right),
          & \text{ if } \theta > \pi/6,\\
           \frac{x_{h}}{2}+\frac{s}{d},
            & \text{ if } \theta = \pi/6,
        \end{array}
        \right.
  $$
  $$
    B = \left\{
    \begin{array}{>{\displaystyle}c>{\displaystyle}l}
         \left\lceil\frac{2( \sin(\pi/3-\theta)(c_{x} - l_{x})  + \cos(\pi/3-\theta)(y_{0} - l_{y} -s) )}{\sqrt{3}d}\right\rceil,
           & \text{ if } T > \frac{s}{v},
         \\
         \left\lceil-\frac{2\sqrt{2svT - (vT)^{2}}}{\sqrt{3}d}\sin\left(\theta + \frac{\pi}{6}\right)\right\rceil,
           & \text{ otherwise, }
    \end{array}
    \right.
  $$
  for $c_{x} =  x_{0} + vT - s$ and 
  \if\shortVersion 1
    $
        (l_{x},l_{y}) = 
          \argmin_{(x,y) \in Z}{\vert vT - s + x_{0} - x\vert + \vert y_{0} - y\vert },
         $   
          if  $T > \frac{s}{v}$, 
          otherwise,
          $(l_{x},l_{y}) = (x_{0},y_{0})$,
  \else
    $$
        (l_{x},l_{y}) =  \left\{
        \begin{array}{>{\displaystyle}cl}
          \argmin_{(x,y) \in Z}{\vert vT - s + x_{0} - x\vert + \vert y_{0} - y\vert },
            & \text{ if } T > \frac{s}{v}, \\
          (x_{0},y_{0}),
            & \text{ otherwise, }
        \end{array}
        \right. 
    $$
  \fi
  where $Z$ is the set of robot positions inside the rectangle measuring $vT - s \times 2s$ for $vT - s > 0$. If  $T > \frac{s}{v}$  or $\arctan\left( \frac{\frac{s}{2} - \sin(\theta) (vT - s) }{\frac{\sqrt{3}s}{2} + \cos(\theta) (vT - s)} \right) <  \frac{\pi}{2} - \theta$,
  \if\shortVersion 1
    $
      U = \left \lfloor \frac{2(\sin(\pi/3-\theta)(c_{x} - l_{x}) + \cos(\pi/3-\theta)(y_{0} - l_{y}) + s)}{\sqrt{3}d} \right \rfloor,
    $
  \else
    $$
      U = \left \lfloor \frac{2(\sin(\pi/3-\theta)(c_{x} - l_{x}) + \cos(\pi/3-\theta)(y_{0} - l_{y}) + s)}{\sqrt{3}d} \right \rfloor,
    $$
  \fi
  otherwise,
  \if\shortVersion 1
    $
        U = \left \lfloor \frac{2\sqrt{2svT - (vT)^{2}}}{\sqrt{3}d}\cos\left(\theta-\frac{\pi}{3}\right) \right\rfloor.
    $
  \else
    $$
        U = \left \lfloor \frac{2\sqrt{2svT - (vT)^{2}}}{\sqrt{3}d}\cos\left(\theta-\frac{\pi}{3}\right) \right\rfloor.
    $$
  \fi
  Also,
  \if\shortVersion 1
    $
        Y_{1}^{S}(x_{h}) = 
         \frac{d {x_{h}}- {C_{-\theta,x}} + \sqrt{3} C_{-\theta,y}   - \sqrt{\Delta(x_{h})}}{2  d}  
    $ and
  \else
    $$
        Y_{1}^{S}(x_{h}) = 
         \frac{d {x_{h}}- {C_{-\theta,x}} + \sqrt{3} C_{-\theta,y}   - \sqrt{\Delta(x_{h})}}{2  d} \text{ and } 
    $$
  \fi
  \begin{equation}
    Y_{2}^{S}(x_{h}) = \left\{
    \begin{array}{>{\displaystyle}c>{\displaystyle}l}
      \min(L(x_{h}),C_{2}(x_{h})) - 1, 
       & \text{ if } \min(L(x_{h}),C_{2}(x_{h})) \\ 
       & \phantom{if} = \lfloor L(x_{h}) \rfloor \text{ and } T > \frac{s}{v},\\
      \min(L(x_{h}),C_{2}(x_{h})), 
       & \text{ otherwise, } 
    \end{array}
    \right.
    \label{eq:whereIusedepsilon}
  \end{equation}
  \if\shortVersion 1
    $
      C_{-\theta} = \left[
            \begin{array}{cc}
              \cos(-\theta) & -\sin(-\theta)\\
              \sin(-\theta) & \cos(-\theta)\\
            \end{array}
            \right]
            \left[
            \begin{array}{c}
              c_{x} - l_{x}\\
              y_{0} - l_{y}\\
            \end{array}
            \right],
    $
    $
      \Delta(x_{h}) =   4  s^{2} - \left(\sqrt{3}  {\left(d {x_{h}} -{C_{-\theta,x}} \right)} - C_{-\theta,y}\right)^{2},
    $
    $
      C_{2}(x_{h}) = 
       \frac{d {x_{h}}- {C_{-\theta,x}}  + \sqrt{3} C_{-\theta,y}  + \sqrt{\Delta(x_{h})}}{2  d}, 
    $
  \else
    $$
      C_{-\theta} = \left[
            \begin{array}{cc}
              \cos(-\theta) & -\sin(-\theta)\\
              \sin(-\theta) & \cos(-\theta)\\
            \end{array}
            \right]
            \left[
            \begin{array}{c}
              c_{x} - l_{x}\\
              y_{0} - l_{y}\\
            \end{array}
            \right],
    $$
    $$
      \Delta(x_{h}) =   4  s^{2} - \left(\sqrt{3}  {\left(d {x_{h}} -{C_{-\theta,x}} \right)} - C_{-\theta,y}\right)^{2},
    $$
    $$
      C_{2}(x_{h}) = 
       \frac{d {x_{h}}- {C_{-\theta,x}}  + \sqrt{3} C_{-\theta,y}  + \sqrt{\Delta(x_{h})}}{2  d}, 
    $$
  \fi
  \if\shortVersion 1
    $
      L(x_{h}) = 
        \frac{\sin\left(\frac{\pi}{2} - \theta\right)(d x_{h}  - C_{-\theta,x}) + \cos\left(\frac{\pi}{2} - \theta\right)C_{-\theta,y}}{d \sin\left(\frac{5\pi}{6}-\theta\right)},
     $      if  $T > \frac{s}{v}$, otherwise
     $ L(x_{h}) =
        \frac{\sin\left(\frac{\pi}{2}-\theta\right)  x_{h}}{\sin\left( \frac{5\pi}{6}-\theta\right)}$,
  \else
    $$
      L(x_{h}) = \left\{ 
      \begin{array}{>{\displaystyle}c>{\displaystyle}l}
        \frac{\sin\left(\frac{\pi}{2} - \theta\right)(d x_{h}  - C_{-\theta,x}) + \cos\left(\frac{\pi}{2} - \theta\right)C_{-\theta,y}}{d \sin\left(\frac{5\pi}{6}-\theta\right)},
          & \text{ if } T > \frac{s}{v}, \\
        \frac{\sin\left(\frac{\pi}{2}-\theta\right)  x_{h}}{\sin\left( \frac{5\pi}{6}-\theta\right)},
          & \text{ otherwise,}\\
      \end{array}
      \right.
    $$
  \fi
    and
  \begin{equation}
   \begin{aligned}
      \lim_{T \to \infty}  f_{h}(T,\theta) \in&
         \left(\frac{4vs}{\sqrt{3}d^{2}} - \frac{2 v  \cos(\theta -\pi/6)}{\sqrt{3}d}, \frac{4vs}{\sqrt{3}d^{2}} +  \frac{2 v  \cos(\theta-\pi/6)}{\sqrt{3}d}\right].
    \end{aligned}
    \label{eq:hexthroughputbounds}
  \end{equation}
  \label{prop:hexthroughputbounds}
\end{proposition}
\begin{proof}
  We are concerned about the throughput of the target region for a given time and hexagonal packing angle $\theta$, $f_{h}(T,\theta) = \frac{N(T,\theta)-1}{T}$, where $N(T,\theta)$ denotes the number of robots which arrived at the target region. Figure \ref{fig:hexnumrobots} illustrates the arrival of the robots on the target region.
  As this region has a circular shape, not all robots at the distance $vT$ arrive at target region by the time $T$. Thence, the number of robots in hexagonal packing are divided into the number of robots located inside a rectangle, $N_{R}$, and of robots inside a semicircle $N_{S}$ (Figure \ref{fig:hexnumrobots} (III)). That is, $N(T,\theta) = N_{S}(T,\theta) + N_{R}(T,\theta)$ and $N_{R} = 0$ whenever $vT \le s$.
  
  \begin{figure}[t!]
    \centering
    \includegraphics[width=\columnwidth]{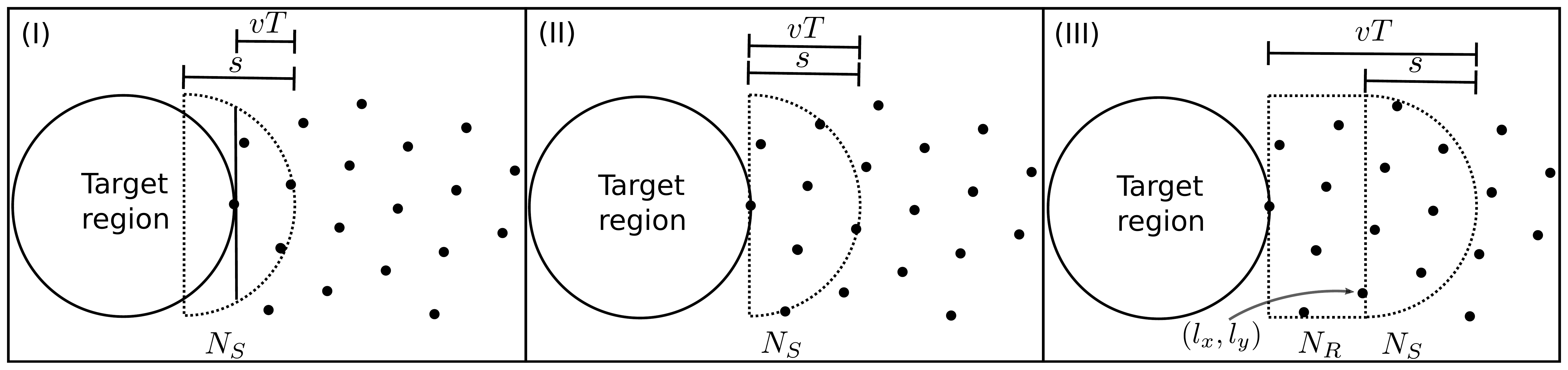}
    \caption{(I) When the robots -- here represented by black dots -- in hexagonal packing begin to arrive at the target region, only the robots inside a part of the semicircle are counted. (II) We consider the first robot to reach the target region being at $(x_{0},y_{0})$  at time $0$. As $T$ grows, this continues until $vT = s$. (III) When $vT > s$, the robots are counted on two regions: a rectangular, $N_{R}$, and a semicircular, $N_{S}$. When $vT > s$, the semicircular region counting starts after the last robot on the rectangular region located at $(l_{x},l_{y})$.}
    \label{fig:hexnumrobots}
  \end{figure}

  \begin{figure}[t!]
    \centering
    \includegraphics[width=0.7\columnwidth]{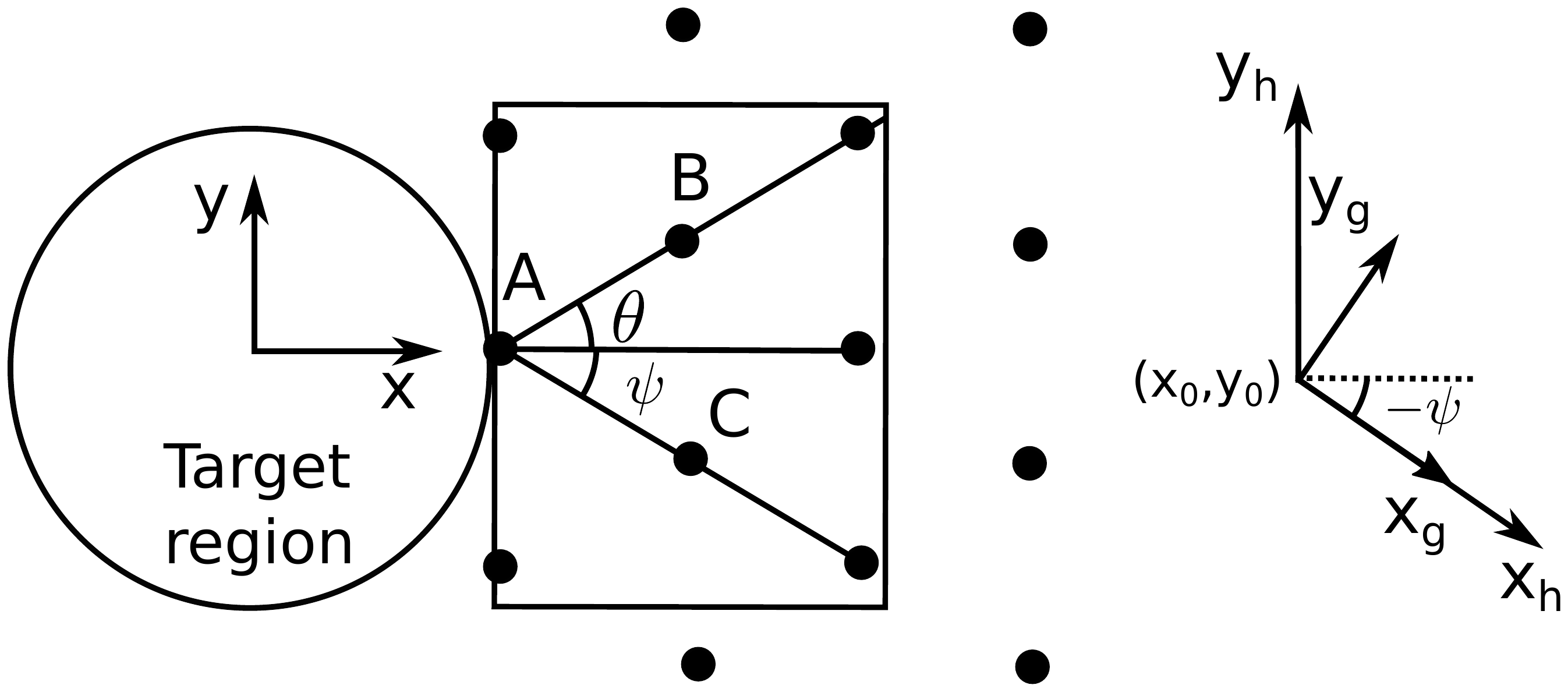}
    \caption{The reference frames used in this proof: the usual Euclidean space $(x,y)$ in relation to the target region and the rectangle region formed by robots in hexagonal packing going to it; the coordinate space $(x_{g},y_{g})$, formed by the usual space after a translation to the first robot to reach the target region at $(x_{0},y_{0})$, followed by a rotation by $-\psi$; the coordinate space $(x_{h},y_{h})$, a hexagonal grid coordinate space made after this transformation  and a linear transformation $H$. Robots are represented by the black dots and they are on hexagonal formation. Each neighbour of a robot is distant by $d$, so $\bigtriangleup ABC$ is equilateral. Thus, $\theta + \psi = \pi/3$.}
    \label{fig:referencepsi}
  \end{figure}
  
  \begin{figure}[t!]
    \centering
    \includegraphics[width=\columnwidth]{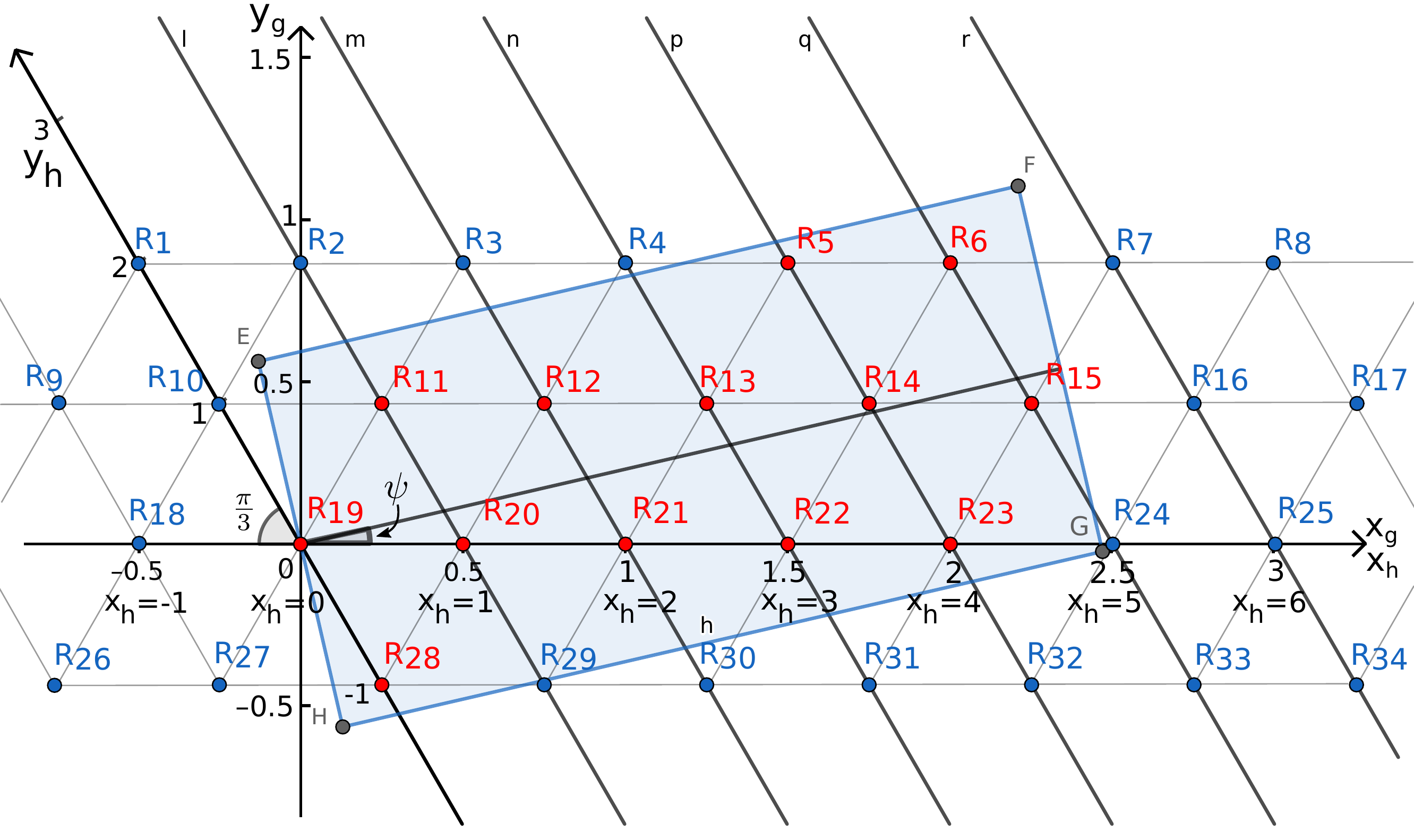}
    \caption{Robots in hexagonal packing formation, and the corresponding rectangular corridor which will reach the target region. Robots are located in the $l$, $m$, $n$, $p$, $q$ and $r$ lines, which are parallel to the $y_h$-axis. In this example, $\psi = 0.227$ and the distance between all robots is $d = 0.5$. The distance between those parallel-to-$y_{h}$ lines is $\sqrt{3}d / 2$. The robots inside the rectangle EFGH are counted and are indicated by red points, while blue points are robots outside the rectangle. Although the $x_{h}$-axis coincides with the $x_{g}$-axis, $x_{h}$ is scaled by $d$.}
    \label{fig:rectangle_robots}
  \end{figure}
  
  \begin{figure}[t!]
    \centering
    \includegraphics[width=0.8\columnwidth]{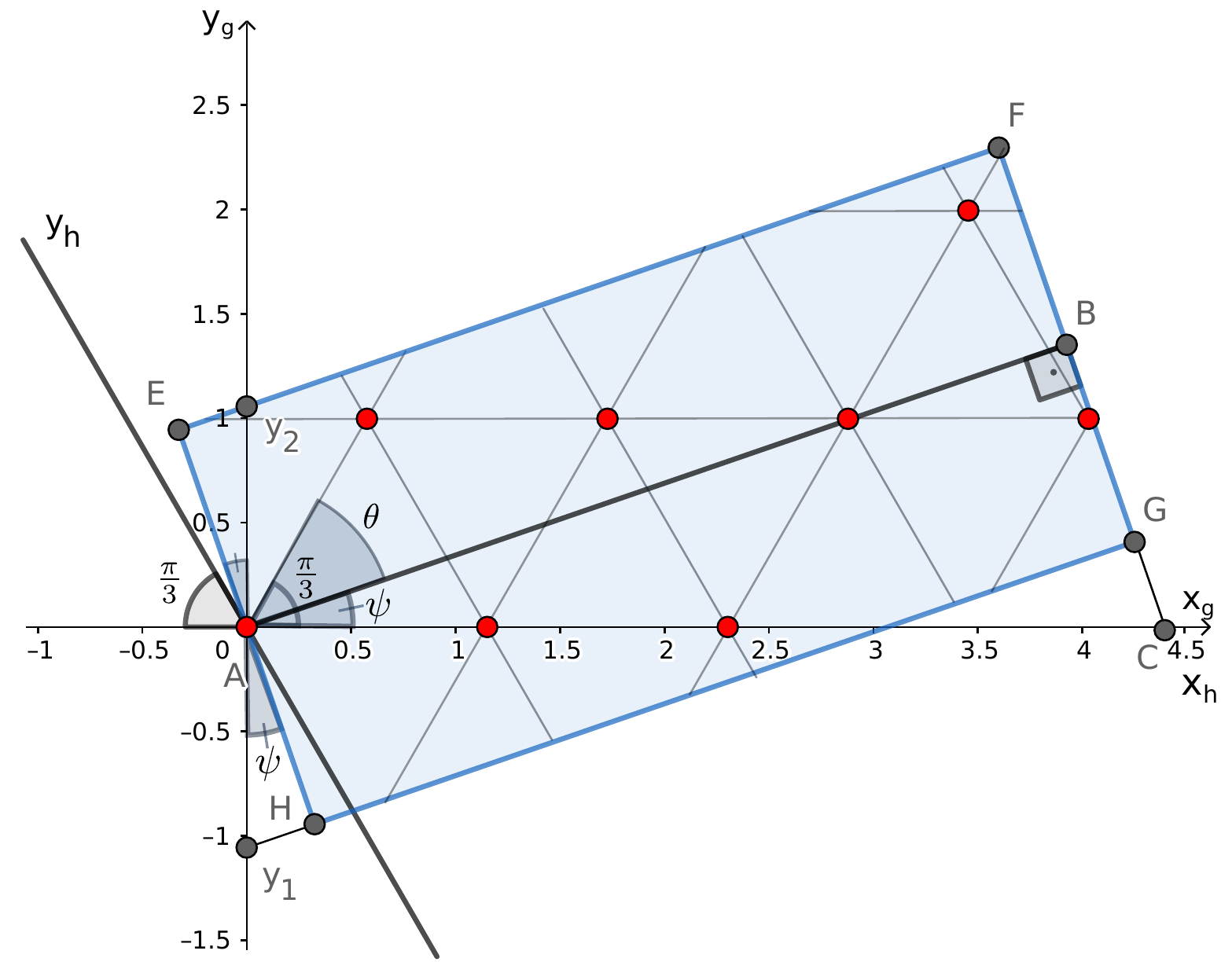}
    \caption{The problem involves the rectangle EFGH in a hexagonal grid (grey lines inside the rectangle) of robots (the red dots). The $x_h$-axis is horizontal and coincides with the $x$-axis. The $y_h$-axis forms a $2 \pi / 3$ angle with it. $\overline{EH}$ and $\overline{AB}$ have length $2 s$ and $vT-s$, respectively. In this example, $\psi = 19\pi/180$ and $s=1$. The angles marked with a line are equal to $\psi$, because the angle formed by $\protect \overrightarrow{y_{2}A}$ and the $x$-axis is right, as well as $\widehat{EAB}$. Accordingly, $\widehat{y_{2}AB} = \pi/2 - \psi$ implies that $\protect \widehat{EAy_{2}} = \psi$.}
    \label{fig:rectangle_problem}
  \end{figure}
  
  This proof is divided in lemmas for helping the construction of the equation to compute $N_{R}(T,\theta)$ and $N_{S}(T,\theta)$ as well for calculating $\lim_{T\to \infty} f_{h}(T,\theta)$.  Before presenting them, we discuss a coordinate space transformation which will be used to count the robots for $N_{R}$ and $N_{S}$. This transformation was inspired by \citep{redblobgames}.

  Figure \ref{fig:referencepsi} shows the coordinate spaces used in this proof. Let $\psi = \pi/3 - \theta$ (because the angle of the equilateral triangle formed by neighbours is $\pi/3$, as explained in Figure \ref{fig:referencepsi}). Accordingly, $\psi \in \lbrack 0,\pi/3 \rparen$ too.
  The usual Euclidean coordinate space which represents the location of all robots is denoted here by $(x,y)$ coordinates. The next coordinate space is denoted by $(x_{g},y_{g})$, and it is the result of a translation of the usual Euclidean coordinate space by the position of the first robot to reach the target region at $(x_{0},y_{0})$, then a rotation of $-\psi$, that is,
  \if\shortVersion 1
    $
       \left[
      \begin{array}{c}
        x_{g} \\
        y_{g}
      \end{array}
      \right]
      = \left[
      \begin{array}{cc}
        \cos(-\psi) & -\sin(-\psi)\\
        \sin(-\psi) & \cos(-\psi)\\
      \end{array}
      \right]
      \left[
      \begin{array}{c}
        x-x_{0}\\
        y-y_{0}\\
      \end{array}
      \right].
    $
  \else
    $$
       \left[
      \begin{array}{c}
        x_{g} \\
        y_{g}
      \end{array}
      \right]
      = \left[
      \begin{array}{cc}
        \cos(-\psi) & -\sin(-\psi)\\
        \sin(-\psi) & \cos(-\psi)\\
      \end{array}
      \right]
      \left[
      \begin{array}{c}
        x-x_{0}\\
        y-y_{0}\\
      \end{array}
      \right].
    $$
  \fi
  The last coordinate space is denoted by $(x_{h},y_{h})$ and it is intended to represent a hexagonal grid such that the position of each robot is  an integer pair.
  
  Figure \ref{fig:rectangle_robots} shows the location of robots with respect to that hexagonal grid.
  Let $(x_{h}, y_{h}) \in \mathbb{Z}^2$ be the hexagonal coordinates of a robot in this hexagonal grid space. In this figure, there is an integer grid in grey -- the horizontal lines correspond to fixed integer $y_{h}$ values and the inclined ones, $x_{h}$ values. For example, in Figure \ref{fig:rectangle_robots} robots $R_{10}$, $R_{11}$ and $R_{20}$ respectively are at $(0,1)$, $(1,1)$ and $(1,0)$ at $(x_{h}, y_{h})$ coordinate system, which is equivalent to $\left(-1/4,\sqrt{3}/4\right)$, $\left(1/4,\sqrt{3}/4\right)$ and $\left(1/2,0\right)$ on the usual two dimensional coordinate system with origin at $(x_{0},y_{0})$. 
  
  We get the linear transformation $H$ from a point $(x_{h},y_{h})$ to  $(x_{g},y_{g})$ basis by knowing the result of this transformation for the standard vectors $(1,0)$ and $(0,1)$. Observing Figure \ref{fig:rectangle_robots} and having that the angle between the $x$-axis and $y_{h}$-axis is by definition $2\pi/3$, we get the following mappings $(x_{h},y_{h}) \mapsto (x_{g},y_{g})$: $(1,0) \mapsto (d,0)$ and $(0,1) \mapsto (d\cos\left(2\pi/3\right),d\sin\left(2\pi/3\right)) = (-\frac{d}{2}, \frac{\sqrt{3}d}{2})$ (in Figure \ref{fig:rectangle_robots} these two mappings are represented by robots $R_{20}$ and $R_{10}$, respectively, with $d=0.5$). Then,
  
  \begin{equation}
    \left[
    \begin{array}{c}
      x_{g} \\
      y_{g}
    \end{array}
    \right]
    = {\left[
    \begin{array}{cc}
      H\left(\left[
      \begin{aligned}
        1 \\
        0
      \end{aligned}
      \right]\right) & 
      H\left(\left[
      \begin{aligned}
        0 \\
        1
      \end{aligned}
      \right]\right) 
    \end{array}
    \right]}
    \left[
    \begin{array}{c}
      x_{h}\\
      y_{h}
    \end{array}
    \right] 
    = \left[
    \begin{array}{cc}
      d & -\frac{d}{2}\\
      0 & \frac{\sqrt{3}d}{2}\\ 
    \end{array}
    \right]
    \left[
    \begin{array}{c}
      x_{h}\\
      y_{h}
    \end{array}
    \right].
    \label{eq:xhyh2xgyg}
  \end{equation}

  Counting the robots inside the rectangle is the same as counting the number of integer hexagonal coordinate points lying inside it.
  Figure \ref{fig:rectangle_problem} shows the rectangular part with some robots in hexagonal packing, where the robots are the red dots and the hexagonal packing is guided by the grey lines inside the rectangle, based on the value of the angle $\psi$. The rectangle is of width $vT - s$ and of height $2s$. The reference frame of the hexagonal grid is rotated in relation to the target region (Figure \ref{fig:referencepsi}). From Figure \ref{fig:rectangle_problem}, we have 
  \begin{equation}
    2 s = (y_{2} - y_{1})\cos(\psi),\ y_{2} = \frac{s}{\cos(\psi)} \text{ and } y_{1} = -\frac{s}{\cos(\psi)}.
    \label{eq:2sy2y1}
  \end{equation}

  We consider a robot with coordinates $(x_g, y_g)$. The four sides of the rectangle EFGH, $\overline{HG}$, $\overline{EF}$, $\overline{EH}$ and $\overline{FG}$, have the following equations of line: $y_{g} = y_{1} + \tan(\psi)x_{g}$, $y_{g} = y_2 + x_{g} \tan(\psi)$, $y_{g} = \tan\left(\psi+\frac{\pi}{2}\right)x_{g}$ and $y_{g} = \tan(\psi + \frac{\pi}{2}) \left(x_{g} - \frac{v T-s}{\cos(\psi)} \right) $, respectively. The term $\frac{vT-s}{\cos(\psi)}$ in the last equation arises because of the length of $\overline{AC}$, which is the hypotenuse of $\bigtriangleup ABC$ whose side $\overline{AB}$ measures $vT$. Knowing that $\tan\left(\psi+\frac{\pi}{2}\right) = -\cot(\psi)$, the equations below are all true for a robot at $(x_{g},y_{g})$ to be inside or on the boundary of the previously defined rectangle,
  \begin{equation}
    \begin{aligned}
     &y_g \ge y_1 + x_g \tan( \psi), \,
     y_g \le y_2 + x_g \tan (\psi), \, 
     - x_g \le \tan(\psi) y_g,\text{ and } \\
     &- \left(x_g - \frac{vT-s}{\cos\psi} \right) \ge \tan(\psi) y_g.
    \end{aligned}
    \label{eq:insiderectangleEFGH}
  \end{equation}
  
  Now we take the minimum and maximum $y_{h}$ value for each parallel-to-$y_{h}$ line depending on the $x_{h}$ value. Using (\ref{eq:xhyh2xgyg}) for converting (\ref{eq:insiderectangleEFGH}) to $x_{h}$ and $y_{h}$ coordinate system, i.e., hexagonal coordinates, we have for  $\overline{HG}$ and $\overline{EF}$
  \if\shortVersion 1
    $  \left(\frac{\sqrt{3}}{2} + \frac{1}{2}\tan(\psi) \right) y_h  - \tan(\psi) x_h  \ge  \frac{y_1}{d}$ and
    $ \left(\frac{\sqrt{3}}{2} + \frac{1}{2}\tan(\psi) \right) y_h - \tan(\psi) x_h  \le  \frac{y_2}{d}.$
  \else
    $$  \left(\frac{\sqrt{3}}{2} + \frac{1}{2}\tan(\psi) \right) y_h  - \tan(\psi) x_h  \ge  \frac{y_1}{d} \text{ and}$$
    $$ \left(\frac{\sqrt{3}}{2} + \frac{1}{2}\tan(\psi) \right) y_h - \tan(\psi) x_h  \le  \frac{y_2}{d}.$$
  \fi
  Hence,
  \if\shortVersion 1
    $\frac{y_1}{d} \le \left(\frac{\sqrt{3}}{2} + \frac{1}{2}\tan\psi \right) y_h - \tan(\psi) x_h \le  \frac{y_2}{d} \Leftrightarrow$
  \else
    $$\frac{y_1}{d} \le \left(\frac{\sqrt{3}}{2} + \frac{1}{2}\tan\psi \right) y_h - \tan(\psi) x_h \le  \frac{y_2}{d} \Leftrightarrow$$
  \fi
  \begin{equation}
    \frac{\frac{2y_1}{d} + 2\tan(\psi) x_h}{\sqrt{3} + \tan(\psi)}  \le  y_h \le
    \frac{\frac{2y_2}{d} + 2\tan(\psi) x_h}{\sqrt{3} + \tan(\psi)}.
  \label{eq:boundsyh1}
  \end{equation}
  
  Analogously, but considering $\overline{EH}$ and $\overline{FG}$,
  \begin{equation}
  -  x_{h}   \le \left(\tan(\psi) \frac{\sqrt{3}}{2} - \frac{1}{2}\right)y_{h} \text{ and } 
    \left(\tan(\psi) \frac{\sqrt{3}}{2} - \frac{1}{2}\right)y_{h} \le \frac{v T-s}{d\cos(\psi)} - x_{h}.
  \label{eq:xhvTcos1}
  \end{equation}

  Based on the sign of $\left(\tan(\psi) \frac{\sqrt{3}}{2} - \frac{1}{2}\right)$ and excluding the null case (when $\psi=\pi/6$), we have two different inequalities over $y_{h}$. Assuming $\psi \in \lbrack 0,\pi/3 \rparen$, we have  $\left(\tan(\psi) \frac{\sqrt{3}}{2} - \frac{1}{2}\right) > 0 \LR \tan(\psi) \frac{\sqrt{3}}{2} > \frac{1}{2} \LR \tan(\psi) > \frac{1}{\sqrt{3}} \LR \psi > \pi/6$. Thus, from (\ref{eq:xhvTcos1}), 
\ifexpandexplanation
  \begin{equation}
    \begin{aligned}
      \frac{-x_{h}}{\frac{\sqrt{3} \tan(\psi) - 1}{2}}   \le y_{h} \le \frac{\frac{v T-s}{d\cos(\psi)} - x_{h}}{ \frac{\sqrt{3} \tan(\psi) - 1}{2}}, & \text{ if } \psi > \pi/6, \\
      \frac{\frac{v T-s}{d\cos(\psi)} - x_{h}}{\frac{\sqrt{3} \tan(\psi) - 1}{2}}   \le y_{h} \le \frac{-x_{h}}{\frac{\sqrt{3} \tan(\psi) - 1}{2}}, & \text{ if } \psi < \pi/6.
    \end{aligned}
  \end{equation}
\fi
  \begin{equation}
    \begin{aligned}
      \frac{-2x_{h}}{\sqrt{3} \tan(\psi) - 1}   \le y_{h} \le \frac{\frac{2(v T-s)}{d\cos(\psi)} - 2x_{h}}{\sqrt{3} \tan(\psi) - 1}, & \text{ if } \psi > \pi/6, \\
      \frac{\frac{2(v T-s)}{d\cos(\psi)} - 2x_{h}}{\sqrt{3} \tan(\psi) - 1}   \le y_{h} \le \frac{-2x_{h}}{\sqrt{3} \tan(\psi) - 1}, & \text{ if } \psi < \pi/6.
    \end{aligned}
    \label{eq:boundsyh2}
  \end{equation}
  
  We have that (\ref{eq:boundsyh1}) and (\ref{eq:boundsyh2}) restrict the value of $y_{h}$ depending on the value of $x_{h}$ by the relation
  \begin{equation}
    \begin{aligned}
      \max\left(\frac{\frac{2y_1}{d} + 2\tan(\psi) x_h}{\sqrt{3} + \tan(\psi)},
      \frac{-2 x_{h}}{\sqrt{3} \tan(\psi) - 1}\right) 
        \le y_{h} 
      \\
      \le
      \min\left(\frac{\frac{2y_2}{d} + 2\tan(\psi) x_h}{\sqrt{3} + \tan(\psi)},
      \frac{\frac{2(v T-s)}{d\cos(\psi)} - 2x_{h}}{\sqrt{3} \tan(\psi) - 1}\right), 
        & \text{ if } \psi > \pi/6, \\
      \max\left(\frac{\frac{2y_1}{d} + 2\tan(\psi) x_h}{\sqrt{3} + \tan(\psi)},
      \frac{\frac{2(v T-s)}{d\cos(\psi)} - 2x_{h}}{\sqrt{3} \tan(\psi) - 1} \right)
        \le y_{h} 
      \\
      \le 
      \min\left(\frac{\frac{2y_2}{d} + 2\tan(\psi) x_h}{\sqrt{3} + \tan(\psi)},
      \frac{-2x_{h}}{\sqrt{3} \tan(\psi) - 1}\right), 
        & \text{ if } \psi < \pi/6.
    \end{aligned}
    \label{eq:boundsyhminmaxreal}
  \end{equation}
  
  Using hexagonal coordinates the position of each robot is represented by a pair of integers. Then, assuming $x_{h}$ and $y_{h}$ integers, (\ref{eq:boundsyhminmaxreal}) becomes $\lceil Y_{1}^{R}(x_{h}) \rceil \le y_{h} \le \lfloor Y_{2}^{R}(x_{h}) \rfloor,$ for   
  \if\shortVersion 1
    \begin{equation}
        Y_{1}^{R}(x_{h}) 
          =
          \left\{
          \begin{array}{>{\displaystyle}c>{\displaystyle}l}
            \max\left(\frac{\sin(\psi) x_h - \frac{s}{d}}{\cos\left(\frac{\pi}{6}-\psi\right)},
            \frac{-\cos(\psi)x_{h}}{\sin\left(\psi-\frac{\pi}{6}\right)}\right),
              & \text{ if } \psi > \pi/6, \\
            \max\left(\frac{ \sin(\psi) x_h -\frac{s}{d}}{\cos\left(\frac{\pi}{6}-\psi\right)},
            \frac{\frac{v T-s}{d} - \cos(\psi)x_{h}}{\sin\left(\psi-\frac{\pi}{6}\right)} \right),
            & \text{ if } \psi < \pi/6,\\
            \frac{x_{h}}{2}-\frac{s}{d},
              & \text{ if } \psi = \pi/6,
          \end{array}
          \right.
      \label{eq:y1xh}
    \end{equation}
    \begin{equation}
      Y_{2}^{R}(x_{h}) 
        =
          \left\{
          \begin{array}{>{\displaystyle}c>{\displaystyle}l}
            \min\left(\frac{\sin(\psi) x_h + \frac{s}{d}}{\cos\left(\frac{\pi}{6}-\psi\right)},
            \frac{\frac{v T-s}{d} - \cos(\psi)x_{h}}{\sin\left(\psi-\frac{\pi}{6}\right)} \right),
              & \text{ if } \psi > \pi/6, \\
            \min\left(\frac{ \sin(\psi) x_h +\frac{s}{d}}{\cos\left(\frac{\pi}{6}-\psi\right)},
            \frac{-\cos(\psi)x_{h}}{\sin\left(\psi-\frac{\pi}{6}\right)}\right),
            & \text{ if } \psi < \pi/6,\\
             \frac{x_{h}}{2}+\frac{s}{d},
              & \text{ if } \psi = \pi/6.
          \end{array}
          \right.
      \label{eq:y2xh}
    \end{equation}  
  \else 
\ifexpandexplanation 
    $$
      \begin{aligned}
        Y_{1}^{R}(x_{h}) &= 
          \left\{
          \begin{array}{>{\displaystyle}c>{\displaystyle}l}
            \max\left(\frac{\frac{2y_1}{d} + 2\tan(\psi) x_h}{\sqrt{3} + \tan(\psi)},
            \frac{-2x_{h}}{\sqrt{3} \tan(\psi) - 1}\right),
              & \text{ if } \psi > \pi/6, \\
            \max\left(\frac{\frac{2y_1}{d} + 2\tan(\psi) x_h}{\sqrt{3} + \tan(\psi)},
            \frac{\frac{2 (v T-s)}{d\cos(\psi)} - 2x_{h}}{\sqrt{3} \tan(\psi) - 1} \right),
            & \text{ if } \psi < \pi/6,\\
            \frac{\sqrt{3}y_{1} + d x_{h}}{2d},
              & \text{ if } \psi = \pi/6,
          \end{array}
          \right.
          \\ 
      \end{aligned}
    $$
    $$
      \begin{aligned}
          &\left\{
          \begin{array}{>{\displaystyle}c>{\displaystyle}l}
            \max\left(\frac{\frac{2y_1}{d} + 2\tan(\psi) x_h}{{\sqrt{3} + \tan(\psi)}},
            \frac{-2x_{h}}{\sqrt{3} \tan(\psi) - 1}\right),
              & \text{ if } \psi > \pi/6, \\
            \max\left(\frac{\frac{2y_1}{d} + 2\tan(\psi) x_h}{\sqrt{3} + \tan(\psi)},
            \frac{\frac{2(v T-s)}{d\cos(\psi)} - 2x_{h}}{\sqrt{3} \tan(\psi) - 1} \right),
            & \text{ if } \psi < \pi/6,\\
            \frac{\frac{-\sqrt{3}s}{\cos(\pi/6)} + d x_{h}}{2d},
              & \text{ if } \psi = \pi/6,
          \end{array}
          \right.
          \\
          &=
          \left\{
          \begin{array}{>{\displaystyle}c>{\displaystyle}l}
            \max\left(\frac{\frac{2y_1\cos(\psi)}{d} + 2\sin(\psi) x_h}{{\sqrt{3}\cos(\psi) + \sin(\psi)}},
            \frac{-2x_{h}\cos(\psi)}{\sqrt{3} \sin(\psi) - \cos(\psi)}\right),
              & \text{ if } \psi > \pi/6, \\
            \max\left(\frac{\frac{2y_1\cos(\psi)}{d} + 2\sin(\psi) x_h}{\sqrt{3}\cos(\psi) + \sin(\psi)},
            \frac{\frac{2(v T-s)}{d} - 2x_{h}\cos(\psi)}{\sqrt{3} \sin(\psi) - \cos(\psi)} \right),
            & \text{ if } \psi < \pi/6,\\
            \frac{\frac{-2\sqrt{3}s}{\sqrt{3}} + d x_{h}}{2d},
              & \text{ if } \psi = \pi/6,
          \end{array}
          \right.
          \\
          &=
          \left\{
          \begin{array}{>{\displaystyle}c>{\displaystyle}l}
            \max\left(\frac{\frac{-2s}{d} + 2\sin(\psi) x_h}{{\sqrt{3}\cos(\psi) + \sin(\psi)}},
            \frac{-2x_{h}\cos(\psi)}{\sqrt{3} \sin(\psi) - \cos(\psi)}\right),
              & \text{ if } \psi > \pi/6, \\
            \max\left(\frac{\frac{-2s}{d} + 2\sin(\psi) x_h}{\sqrt{3}\cos(\psi) + \sin(\psi)},
            \frac{\frac{2(v T-s)}{d} - 2x_{h}\cos(\psi)}{\sqrt{3} \sin(\psi) - \cos(\psi)} \right),
            & \text{ if } \psi < \pi/6,\\
            \frac{-2s + d x_{h}}{2d},
              & \text{ if } \psi = \pi/6,
          \end{array}
          \right.
          \\
      \end{aligned}
    $$
    $$
      \begin{aligned}
          &=
          \left\{
          \begin{array}{>{\displaystyle}c>{\displaystyle}l}
            \max\left(\frac{\sin(\psi) x_h - \frac{s}{d}}{\cos\left(\frac{\pi}{6}-\psi\right)},
            \frac{-\cos(\psi)x_{h}}{\sin\left(\psi-\frac{\pi}{6}\right)}\right),
              & \text{ if } \psi > \pi/6, \\
            \max\left(\frac{ \sin(\psi) x_h -\frac{s}{d}}{\cos\left(\frac{\pi}{6}-\psi\right)},
            \frac{\frac{v T-s}{d} - \cos(\psi)x_{h}}{\sin\left(\psi-\frac{\pi}{6}\right)} \right),
            & \text{ if } \psi < \pi/6,\\
            \frac{x_{h}}{2}-\frac{s}{d},
              & \text{ if } \psi = \pi/6,
          \end{array}
          \right.
          \\
      \end{aligned}
    $$
    \begin{equation}
      \begin{aligned}
          &=
          \left\{
          \begin{array}{>{\displaystyle}c>{\displaystyle}l}
            \max\left(\frac{d\sin(\psi) x_h - s}{d\cos\left(\frac{\pi}{6}-\psi\right)},
            \frac{-\cos(\psi)x_{h}}{\sin\left(\psi-\frac{\pi}{6}\right)}\right),
              & \text{ if } \psi > \pi/6, \\
            \max\left(\frac{ d\sin(\psi) x_h -s}{d\cos\left(\frac{\pi}{6}-\psi\right)},
            \frac{v T-s - d\cos(\psi)x_{h}}{d\sin\left(\psi-\frac{\pi}{6}\right)} \right),
            & \text{ if } \psi < \pi/6,\\
             \frac{x_{h}}{2}-\frac{s}{d},
              & \text{ if } \psi = \pi/6,
          \end{array}
          \right.
      \end{aligned}
      \label{eq:y1xh}
    \end{equation}
    \begin{equation}
      \begin{aligned}
      Y_{2}^{R}(x_{h}) 
        &= 
        \left\{    
        \begin{array}{>{\displaystyle}c>{\displaystyle}l}
          \min\left(\frac{\frac{2y_2}{d} + 2\tan(\psi) x_h}{\sqrt{3} + \tan(\psi)} ,
          \frac{\frac{2(v T-s)}{d\cos(\psi)} - 2x_{h}}{\sqrt{3} \tan(\psi) - 1}\right),
            & \text{ if } \psi > \pi/6, \\
          \min\left(\frac{\frac{2y_2}{d} + 2\tan(\psi) x_h}{\sqrt{3} + \tan(\psi)},
          \frac{-2x_{h}}{\sqrt{3} \tan(\psi) - 1}\right), 
            & \text{ if } \psi < \pi/6,\\
          \frac{\sqrt{3}y_{2} + d x_{h}}{2d},
            & \text{ if } \psi = \pi/6,
        \end{array}
        \right. 
        \\
        &=
          \left\{
          \begin{array}{>{\displaystyle}c>{\displaystyle}l}
            \min\left(\frac{\sin(\psi) x_h + \frac{s}{d}}{\cos\left(\frac{\pi}{6}-\psi\right)},
            \frac{\frac{v T-s}{d} - \cos(\psi)x_{h}}{\sin\left(\psi-\frac{\pi}{6}\right)} \right),
              & \text{ if } \psi > \pi/6, \\
            \min\left(\frac{ \sin(\psi) x_h +\frac{s}{d}}{\cos\left(\frac{\pi}{6}-\psi\right)},
            \frac{-\cos(\psi)x_{h}}{\sin\left(\psi-\frac{\pi}{6}\right)}\right),
            & \text{ if } \psi < \pi/6,\\
             \frac{x_{h}}{2}+\frac{s}{d},
              & \text{ if } \psi = \pi/6.
          \end{array}
          \right.
        \\
      \end{aligned}
      \label{eq:y2xh}
    \end{equation}
\else 
    \begin{equation}
      \begin{aligned}
        Y_{1}^{R}(x_{h}) &= 
          \left\{
          \begin{array}{>{\displaystyle}c>{\displaystyle}l}
            \max\left(\frac{\frac{2y_1}{d} + 2\tan(\psi) x_h}{\sqrt{3} + \tan(\psi)},
            \frac{-2x_{h}}{\sqrt{3} \tan(\psi) - 1}\right),
              & \text{ if } \psi > \pi/6, \\
            \max\left(\frac{\frac{2y_1}{d} + 2\tan(\psi) x_h}{\sqrt{3} + \tan(\psi)},
            \frac{\frac{2 (v T-s)}{d\cos(\psi)} - 2x_{h}}{\sqrt{3} \tan(\psi) - 1} \right),
            & \text{ if } \psi < \pi/6,\\
            \frac{\sqrt{3}y_{1} + d x_{h}}{2d},
              & \text{ if } \psi = \pi/6,
          \end{array}
          \right.
          \\ 
          &=
          \left\{
          \begin{array}{>{\displaystyle}c>{\displaystyle}l}
            \max\left(\frac{\sin(\psi) x_h - \frac{s}{d}}{\cos\left(\frac{\pi}{6}-\psi\right)},
            \frac{-\cos(\psi)x_{h}}{\sin\left(\psi-\frac{\pi}{6}\right)}\right),
              & \text{ if } \psi > \pi/6, \\
            \max\left(\frac{ \sin(\psi) x_h -\frac{s}{d}}{\cos\left(\frac{\pi}{6}-\psi\right)},
            \frac{\frac{v T-s}{d} - \cos(\psi)x_{h}}{\sin\left(\psi-\frac{\pi}{6}\right)} \right),
            & \text{ if } \psi < \pi/6,\\
            \frac{x_{h}}{2}-\frac{s}{d},
              & \text{ if } \psi = \pi/6,
          \end{array}
          \right.
      \end{aligned}
      \label{eq:y1xh}
    \end{equation}
    \begin{equation}
      \begin{aligned}
      Y_{2}^{R}(x_{h}) 
        &= 
        \left\{    
        \begin{array}{>{\displaystyle}c>{\displaystyle}l}
          \min\left(\frac{\frac{2y_2}{d} + 2\tan(\psi) x_h}{\sqrt{3} + \tan(\psi)} ,
          \frac{\frac{2(v T-s)}{d\cos(\psi)} - 2x_{h}}{\sqrt{3} \tan(\psi) - 1}\right),
            & \text{ if } \psi > \pi/6, \\
          \min\left(\frac{\frac{2y_2}{d} + 2\tan(\psi) x_h}{\sqrt{3} + \tan(\psi)},
          \frac{-2x_{h}}{\sqrt{3} \tan(\psi) - 1}\right), 
            & \text{ if } \psi < \pi/6,\\
          \frac{\sqrt{3}y_{2} + d x_{h}}{2d},
            & \text{ if } \psi = \pi/6,
        \end{array}
        \right. 
        \\
        &=
          \left\{
          \begin{array}{>{\displaystyle}c>{\displaystyle}l}
            \min\left(\frac{\sin(\psi) x_h + \frac{s}{d}}{\cos\left(\frac{\pi}{6}-\psi\right)},
            \frac{\frac{v T-s}{d} - \cos(\psi)x_{h}}{\sin\left(\psi-\frac{\pi}{6}\right)} \right),
              & \text{ if } \psi > \pi/6, \\
            \min\left(\frac{ \sin(\psi) x_h +\frac{s}{d}}{\cos\left(\frac{\pi}{6}-\psi\right)},
            \frac{-\cos(\psi)x_{h}}{\sin\left(\psi-\frac{\pi}{6}\right)}\right),
            & \text{ if } \psi < \pi/6,\\
             \frac{x_{h}}{2}+\frac{s}{d},
              & \text{ if } \psi = \pi/6.
          \end{array}
          \right.
        \\
      \end{aligned}
      \label{eq:y2xh}
    \end{equation}
\fi
  \fi
  We simplified above using (\ref{eq:2sy2y1}), $\cos\big(\frac{\pi}{6} - \psi\big) = \frac{\sqrt{3}}{2}\cos(\psi)+\frac{1}{2}\sin(\psi)$ and $\sin\big(\psi-\frac{\pi}{6}\big) = \frac{\sqrt{3}}{2}\sin(\psi)-\frac{1}{2}\cos(\psi)$.

  Now we get the possible integer values for the $x_{h}$-axis which are inside the rectangle EFGH, that is, we count the number of lines parallel with the $y_h$-axis that intersect the rectangle for $x_{h}$ integer values. Let $n_{l}$ be the number of such parallel lines. We consider $n_{l} = n_{l}^{-} + n_{l}^{+}$, such that $n_{l}^{-}$ is the number of lines parallel to the $y_{h}$-axis whose intersection with the $x_{h}$-axis is a point $(i,0)$ for $i<0$ and $i\in \Zeta$, and $n_{l}^{+}$ is similar but for non-negative integer $i$. For example, in Figure \ref{fig:rectangle_robots} we have $n_{l}^{-} = 0$ and $n_{l}^{+} = 6$ (we marked below the values of the points over the $x$-axis, the equivalent over $x_{h}$-axis in order to aid enumerating them). Note that the point $(i, 0)$ may be outside of the rectangle, but it will still be counted if there are integer $(i, y_h)$ coordinates inside the rectangle. The next lemma shows how to compute $n_{l}^{+}$ and $n_{l}^{-}$ to aid in this proof development.

  \begin{lemma}
    On the $(x_{h},y_{h})$ coordinate system, the integer values for $x_{h}$ robot coordinates inside the rectangle EFGH are in the set $\{-n_{l}^{-}, \dots, n_{l}^{+}-1\}$ with
    \begin{equation}
      n_{l}^{+} = 
        \left\lfloor\frac{2 (vT-s) \cos(\psi - \pi/6) + 2s\sin(\vert \psi - \pi/6\vert ) }{\sqrt{3}d} + 1\right\rfloor, 
        \label{eq:nlp}
    \end{equation}
    and
      \begin{equation}
        n_{l}^{-} = 
        \left\lfloor\frac{2s\sin\left(\left\vert \psi - \pi/6\right\vert \right)}{\sqrt{3}d}\right\rfloor.
        \label{eq:nlm}
      \end{equation}
    \label{lemma:nlmnlp}
  \end{lemma}
  \begin{proof}
  \ifithasappendixforlemmas %
    See Online Appendix.
  \else %
  \begin{figure}[t]
    \centering
    \includegraphics[width=0.9\columnwidth]{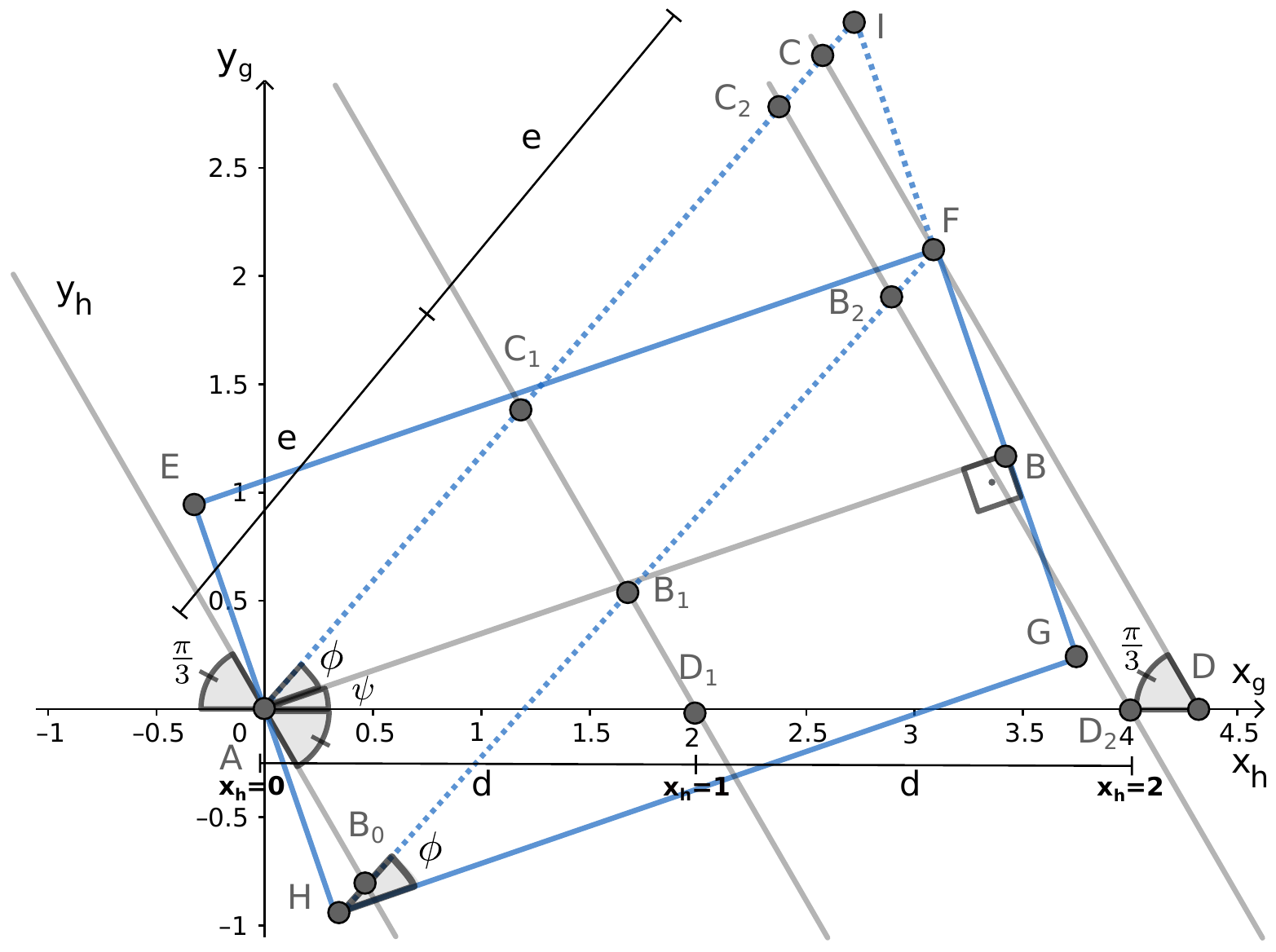}
    \caption{The goal is to count how many points named $B_{i}$ lie in the diagonal $\overline{HF}$. $B_{i}$ is the intersection of a parallel-to-$y_{h}$ line on a $x_{h}$ integer coordinate and the diagonal $\overline{HF}$. The triangles $AD_{i}C_{i}$ for any $i \in \{1,2\}$ and $ADC$ are similar. $\overline{AD_{i}}$ and $\overline{AC_{i}}$  have distance $i \cdot d$ and $i \cdot e$, respectively. In this example, $d = 2$ and there are three points lying over $\overline{HF}$.}
    \label{fig:rectangle_triangles}
  \end{figure}

 For getting $n_{l}^{+}$, we count how many parallel-to-$y_{h}$ lines when projected over the $x$-axis are distant from each other by $d$ on this axis and are inside the rectangle. These lines must intersect the diagonal $\overline{HF}$ of the rectangle, but commencing from the intersection between the $y_{h}$-axis and the diagonal (i.e., from $B_{0}$ in the Figure \ref{fig:rectangle_triangles}). Let $\phi = \arctan\left(\frac{2s}{vT-s}\right)$ be the  angle of the diagonal in relation to the rectangle base. We have two cases depending on the value of $\psi$.
  
  \begin{itemize} 
    \item Case $\psi \le \frac{\pi}{6}$: from Figure \ref{fig:rectangle_triangles}, every line parallel to $y_{h}$ is distant by $d$ on the projection onto the $x$-axis. The triangles $AD_{i}C_{i}$ for any $i \in \{1,\dots,n_{l}^{+}-1\}$ and $ADC$ are similar, $\vert \overline{AD_{1}}\vert = d$ and $\vert \overline{AC_{1}}\vert = e$, whose value is unknown for the moment. $\bigtriangleup ADC$ has angles $\widehat{CAD} = \psi + \phi$, $\widehat{ADC} = \pi/3$ and $\widehat{ACD} = \pi - \widehat{CAD} - \widehat{ADC} = 2\pi/3 - \psi - \phi$. As for every $i$, $\bigtriangleup AD_{i}C_{i} \sim \bigtriangleup ADC$, we have that 
    \begin{equation}
      \frac{\vert \overline{AC}\vert }{\vert \overline{AC_{1}}\vert } = \frac{\vert \overline{AD}\vert }{\vert \overline{AD_{1}}\vert } \LR \frac{\vert \overline{AC}\vert }{e}  = \frac{\vert \overline{AD}\vert }{d}.
      \label{eq:acac1}
    \end{equation} 
    
    As AHFI is a parallelogram, $\vert \overline{HF}\vert = \vert \overline{AI}\vert $ and $\vert \overline{FI}\vert = \vert \overline{AH}\vert = s$, then $\vert \overline{BI}\vert = 2s$. Thus, $\vert \overline{AI}\vert = \sqrt{(2s)^{2} + (vT-s)^{2}}$, because $\bigtriangleup ABI$ is right-angled.  
    Also, by the law of sines, we get $ \frac{\vert \overline{AD}\vert }{\sin(\widehat{ACD})} = \frac{\vert \overline{AC}\vert }{\sin(\widehat{ADC})} \LR $
    \begin{equation}
      \begin{aligned}
        \vert \overline{AD}\vert  &= \vert \overline{AC}\vert \frac{\sin(\widehat{ACD})}{\sin(\widehat{ADC})} 
          = \left(\vert \overline{AI}\vert - \vert \overline{CI}\vert \right)\frac{\sin(\widehat{ACD})}{\sin(\widehat{ADC})} \\
          &= \left(\vert \overline{AI}\vert - \vert \overline{CI}\vert \right)\frac{\sin(2\pi/3 - \psi - \phi)}{\sin(\pi/3)}.
      \end{aligned}
      \label{eq:adacsin1}
    \end{equation}
    $AB_{0}FC$ is a parallelogram as well, so $\vert \overline{AC}\vert = \vert \overline{B_{0}F}\vert = \vert \overline{HF}\vert - \vert \overline{HB_{0}}\vert $ and $\vert \overline{CI}\vert = \vert \overline{HB_{0}}\vert $.

    The $\bigtriangleup AB_{0}H$ has angles $\widehat{HAB_{0}} = \widehat{HAB} - \widehat{B_{0}AB} = \widehat{HAB} - (\widehat{B_{0}AD} + \widehat{DAB}) = \pi/2 - (\pi/3 + \psi) = \pi/6 - \psi$, $\widehat{AHB_{0}} = \widehat{AHG} - \widehat{FHG} = \pi/2 - \phi$ and $\widehat{HB_{0}A} = \pi - \widehat{HAB_{0}} - \widehat{AHB_{0}} = \pi/3 + \psi  + \phi$. By the law of sines, we have $\vert \overline{HB_{0}}\vert =\frac{\sin(\widehat{HAB_{0}})\vert \overline{AH}\vert }{\sin(\widehat{HB_{0}A})} = \frac{\sin(\pi/6 - \psi)s}{\sin(\pi/3 + \psi  + \phi)}$. Hence,
    \if\shortVersion 1
     from (\ref{eq:adacsin1}), 
     $\vert \overline{AD}\vert = 
       \left(\vert \overline{AI}\vert - \vert \overline{CI}\vert \right)\frac{\sin(2\pi/3 - \psi - \phi)}{\sin(\pi/3)},$ thus
     \begin{equation}
       \vert \overline{AD}\vert = \frac{2\cos(\pi/6 - \psi)   (vT-s)   +  {2s\sin(\pi/6 - \psi)}}{\sqrt{3}} 
       \label{eq:adacsin}
     \end{equation}
    \else
      $$
        \begin{aligned}
            &\vert \overline{AD}\vert 
            = \left(\vert \overline{AI}\vert - \vert \overline{CI}\vert \right)\frac{\sin(2\pi/3 - \psi - \phi)}{\sin(\pi/3)} 
             \hspace*{28mm} [\text{from (\ref{eq:adacsin1})}]
            \\
            &= \left(\sqrt{(2s)^{2} + (vT-s)^{2}} - \frac{s\sin(\pi/6 - \psi)}{\sin(\pi/3 + \psi  + \phi)}\right)\frac{\sin(2\pi/3 - \psi - \phi)}{\sin(\pi/3)}\\
            &= \sqrt{(2s)^{2} + (vT-s)^{2}}\frac{\sin(2\pi/3 - \psi - \phi)}{\sin(\pi/3)} - \frac{s\sin(\pi/6 - \psi)}{\sin(\pi/3)}\\ 
\ifexpandexplanation
              &= 2\sqrt{(2s)^{2} + (vT-s)^{2}}\frac{\sin(2\pi/3 - \psi - \phi)}{\sqrt{3}} - \frac{2s\sin(\pi/6 - \psi)}{\sqrt{3}}\\ 
\fi
            &= \frac{2\sqrt{(2s)^{2} + (vT-s)^{2}}{\sin\left(2\pi/3 - \psi - \phi\right)} - {2s\sin(\pi/6 - \psi)}}{\sqrt{3}}\\
            &= \frac{2\sqrt{(2s)^{2} + (vT-s)^{2}}\left(\sin(\frac{2\pi}{3} - \psi)\cos(\phi) - \cos\left(\frac{2\pi}{3} - \psi\right)\sin(\phi)\right)}{\sqrt{3}}\\ 
            &\phantom{=} \ - \frac{2s\sin(\frac{\pi}{6} - \psi)}{\sqrt{3}}\\
            &= \frac{2\sqrt{(2s)^{2} + (vT-s)^{2}}\left(        \frac{\sin(2\pi/3 - \psi)(vT-s)}{\sqrt{(2s)^{2} + (vT-s)^{2}}}        -  \frac{2s\cos(2\pi/3 - \psi)}{\sqrt{(2s)^{2} + (vT-s)^{2}}}        \right)}{\sqrt{3}} \\
            &\phantom{=} \ - \frac{2s\sin(\pi/6 - \psi)}{\sqrt{3}}\\
        \end{aligned}
      $$        
      \begin{align}
          &= \frac{2\left(\sin(2\pi/3 - \psi)        (vT-s)       - 2s\cos(2\pi/3 - \psi)                \right) - {2s\sin(\pi/6 - \psi)}}{\sqrt{3}}\nonumber\\
\ifexpandexplanation
          &= \frac{2\left(\sin(2\pi/3 - \psi)        (vT-s)       + 2s\sin(\pi/6 - \psi)        \right) - {2s\sin(\pi/6 - \psi)}}{\sqrt{3}}\nonumber\\
          &= \frac{2\sin(2\pi/3 - \psi)   (vT-s)   + 4s \sin(\pi/6 - \psi)   - {2s\sin(\pi/6 - \psi)}}{\sqrt{3}} \nonumber\\
          &= \frac{2\sin(2\pi/3 - \psi)   (vT-s)   +  {2s\sin(\pi/6 - \psi)}}{\sqrt{3}} \nonumber\\
\fi
          &= \frac{2\cos(\pi/6 - \psi)   (vT-s)   +  {2s\sin(\pi/6 - \psi)}}{\sqrt{3}} \label{eq:adacsin}
      \end{align}
    \fi
    Above we used $\sin(2\pi/3 - \psi) = \cos(\pi/6 - \psi)$, $\cos(2\pi/3 - \psi) = - \sin(\pi/6 - \psi)$, $\sin(2\pi/3 - \psi -\phi) = \sin(\pi/3 + \psi +\phi)$, $\sin(2\pi/3 - \psi - \phi) = \sin(2\pi/3 - \psi)\cos(\phi) - \cos(2\pi/3 - \psi)\sin(\phi)$, $\sin(\arctan(y/x)) = \frac{y}{\sqrt{x^{2}+y^{2}}}$, and $\cos($ $\arctan(y/x) ) = \frac{x}{\sqrt{x^{2}+y^{2}}}$.
    
    Therefore, the number of lines parallel to the $y_{h}$-axis intersecting $\overline{B_{0}F}$ for integer $x_{h}$ values is
    \if\shortVersion 1
      $
          n_{l}^{+} =  \left\lfloor\frac{\vert \overline{B_{0}F}\vert }{e} + 1\right\rfloor 
            = \big\lfloor\frac{2\cos(\pi/6 - \psi)   (vT-s)   +  {2s\sin(\pi/6 - \psi)}}{\sqrt{3}d} + 1\big\rfloor 
      $
      by using (\ref{eq:acac1}) and (\ref{eq:adacsin}).
    \else
      $$
        \begin{aligned}
          n_{l}^{+}
            &=  \left\lfloor\frac{\vert \overline{B_{0}F}\vert }{e} + 1\right\rfloor 
\ifexpandexplanation
            =  \left\lfloor\frac{\vert \overline{HF}\vert - \vert \overline{HB_{0}}\vert }{e} + 1\right\rfloor 
\fi
            =  \left\lfloor\frac{\vert \overline{AC}\vert }{e} + 1\right\rfloor 
            = \left\lfloor\frac{\vert \overline{AD}\vert }{d} + 1\right\rfloor
            &[\text{from (\ref{eq:acac1})}]
            \\
            &= \left\lfloor\frac{2\cos(\pi/6 - \psi)   (vT-s)   +  {2s\sin(\pi/6 - \psi)}}{\sqrt{3}d} + 1\right\rfloor 
             &[\text{from (\ref{eq:adacsin})}]
            \\
        \end{aligned}
      $$
    \fi
    \item Case $\psi > \frac{\pi}{6}$: Figure \ref{fig:rectangle_triangles2} shows this case. Observe that when $\psi > \frac{\pi}{6}$, $\overline{EA}$ is on the left side of the $y_{h}$-axis. Also, note that we are considering now the diagonal $\overline{EG}$, because the $y_{h}$-axis does not intersect the diagonal $\overline{HF}$ for these values of $\psi$. Then, we have to consider  $\overline{B_{0}G}$ to count  $n_{l}^{+}$.  Additionally, $\vert B_{0}G\vert = \vert AC\vert $, due to the $AB_{0}GC$ parallelogram properties. As in the previous case, for $i \in \{1, \dots, n_{l}^{+}-1\},\bigtriangleup AD_{i}C_{i} \sim \bigtriangleup ADC$, $\widehat{CAD} = \widehat{BAD} - \widehat{BAC} = \psi - \phi$, $\widehat{ADC} = \pi/3$, $\widehat{ACD} = \pi - \widehat{CAD} - \widehat{ADC} = 2\pi/3 - \psi + \phi$, and $\frac{\vert \overline{B_{0}G}\vert }{e} = \frac{\vert \overline{AC}\vert }{e} = \frac{\vert \overline{AD}\vert }{d}$, by the similarity of these triangles as we showed in the previous case. Also, we have $\widehat{EAB_{0}} = \widehat{DAE} - \widehat{DAB_{0}} = \psi + \pi/2 - 2\pi/3 = \psi - \pi/6$, $\widehat{B_{0}EA} = \widehat{FEA} - \widehat{FEB_{0}} = \pi/2-\phi$, $\widehat{EB_{0}A} = \pi - \widehat{B_{0}EA} - \widehat{EAB_{0}} = \pi - (\pi/2-\phi) - (\psi - \pi/6)  = 2\pi/3 + \phi - \psi $. Thus, by the law of sines, $\frac{\vert \overline{B_{0}E}\vert }{\sin(\widehat{EAB_{0}})} = \frac{\vert \overline{EA}\vert }{\sin(\widehat{EB_{0}A})} \LR \vert \overline{B_{0}E}\vert = \frac{s\sin(\widehat{EAB_{0}})}{\sin(\widehat{EB_{0}A})} = \frac{s\sin( \psi - \pi/6)}{\sin(2\pi/3 + \phi - \psi)}$.  $EAIG$ and $B_{0}ACG$ are parallelograms  sharing the points G and A, so $\vert \overline{B_{0}E}\vert = \vert \overline{CI}\vert $. By following similar steps as before, we get
    \if\shortVersion 1
     $ n_{l}^{+} = \left\lfloor\frac{2\cos(\pi/6 - \psi)(vT-s) + 2s\sin(\psi - \pi/6)}{\sqrt{3}d} + 1\right\rfloor.$
    \else
\ifexpandexplanation 
      $$
        \begin{aligned}
        &n_{l}^{+}
          =  \left\lfloor\frac{\vert \overline{B_{0}G}\vert }{e} + 1\right\rfloor 
          =  \left\lfloor\frac{\vert \overline{AC}\vert }{e} + 1\right\rfloor 
          = \left\lfloor\frac{\vert \overline{AD}\vert }{d} + 1\right\rfloor\\
        \end{aligned}
      $$
      $$
        \begin{aligned}
          & 
          = \left\lfloor\frac{\left(\vert \overline{AI}\vert - \vert \overline{CI}\vert \right)\frac{\sin(2\pi/3 - \psi + \phi)}{\sin(\pi/3)}}{d} + 1\right\rfloor\\
          &= \left\lfloor\frac{\left(\sqrt{(2s)^{2} + (vT-s)^{2}} - \frac{s\sin(\psi - \pi/6)}{\sin(2\pi/3 - \psi  + \phi)}\right)\frac{\sin(2\pi/3 - \psi + \phi)}{\sin(\pi/3)}}{d} + 1\right\rfloor\\
          &= \left\lfloor\frac{\sqrt{(2s)^{2} + (vT-s)^{2}}\frac{\sin(2\pi/3 - \psi + \phi)}{\sin(\pi/3)} - \frac{s\sin(\psi - \pi/6)}{\sin(\pi/3)}}{d} + 1\right\rfloor\\
        \end{aligned}
      $$
      $$
        \begin{aligned}
          &= \left\lfloor\frac{2\sqrt{(2s)^{2} + (vT-s)^{2}}{\sin(2\pi/3 - \psi + \phi)} - {2s\sin(\psi - \pi/6)}}{\sqrt{3}d} + 1\right\rfloor\\
          &= \Bigg\lfloor\frac{2\sqrt{(2s)^{2} + (vT-s)^{2}}( \sin(2\pi/3 - \psi)\cos(\phi) + \cos(2\pi/3 - \psi)\sin(\phi) )}{\sqrt{3}d}  \\
            &\phantom{=}\ -\frac{{2s\sin(\psi - \pi/6)}}{\sqrt{3}d} + 1\Bigg\rfloor\\
          &= \left\lfloor\frac{2\sin(2\pi/3 - \psi)(vT-s) + 4s\cos(2\pi/3 - \psi)  - {2s\sin(\psi - \pi/6)}}{\sqrt{3}d} + 1\right\rfloor\\
        \end{aligned}
      $$
      $$
        \begin{aligned}
          &= \left\lfloor\frac{2\sin(2\pi/3 - \psi)(vT-s) + 4s\sin(\psi - \pi/6)  - {2s\sin(\psi - \pi/6)}}{\sqrt{3}d} + 1\right\rfloor\\
          &= \left\lfloor\frac{2\sin(2\pi/3 - \psi)(vT-s) + 2s\sin(\psi - \pi/6)}{\sqrt{3}d} + 1\right\rfloor \\ 
          &= \left\lfloor\frac{2\cos(\pi/6 - \psi)(vT-s) + 2s\sin(\psi - \pi/6)}{\sqrt{3}d} + 1\right\rfloor.\\
        \end{aligned}
      $$
\else
     $$
      \begin{aligned}
        &n_{l}^{+}
          = \left\lfloor\frac{2\cos(\pi/6 - \psi)(vT-s) + 2s\sin(\psi - \pi/6)}{\sqrt{3}d} + 1\right\rfloor.\\
      \end{aligned}
      $$
\fi
    \fi
    We used this time $\sin(2\pi/3 - \psi) =  \cos(\psi - \pi/6)$ and $\cos(2\pi/3 - \psi) =  \sin(\psi - \pi/6)$.

  \end{itemize}

  For the final result on (\ref{eq:nlp}), we simplified using the fact that when $\psi \le \pi/6$, $\sin(\vert \psi - \pi/6\vert ) = \sin(\pi/6 - \psi)$, otherwise, $\sin(\vert \psi - \pi/6\vert ) = \sin(\psi - \pi/6).$
  
  \begin{figure}[t]
    \centering
    \includegraphics[width=0.85\columnwidth]{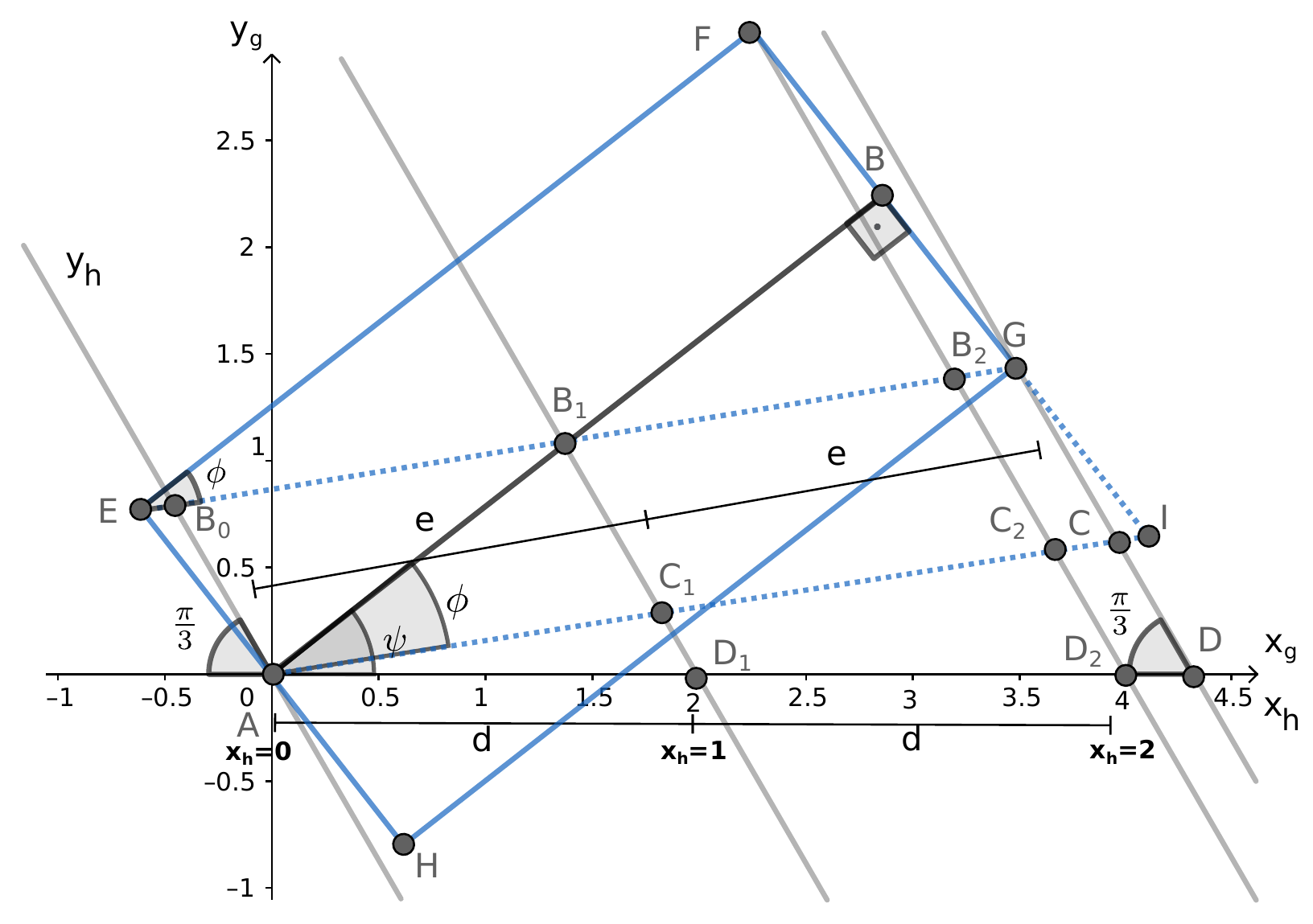}
    \caption{The goal is to count how many points named $B_{i}$ lie in the diagonal $\overline{EG}$. The triangles $AD_{i}C_{i}$ for any $i \in \{1,2\}$ and $ADC$ are similar. $\overline{AD_{i}}$ and $\overline{AC_{i}}$  have distance $i\cdot d$ and $i\cdot e$, respectively. In this example, $d = 2$ and there are three points lying over $\overline{EG}$. }
    \label{fig:rectangle_triangles2}
  \end{figure}

  \begin{figure}[t]
    \centering
    \includegraphics[width=0.85\columnwidth]{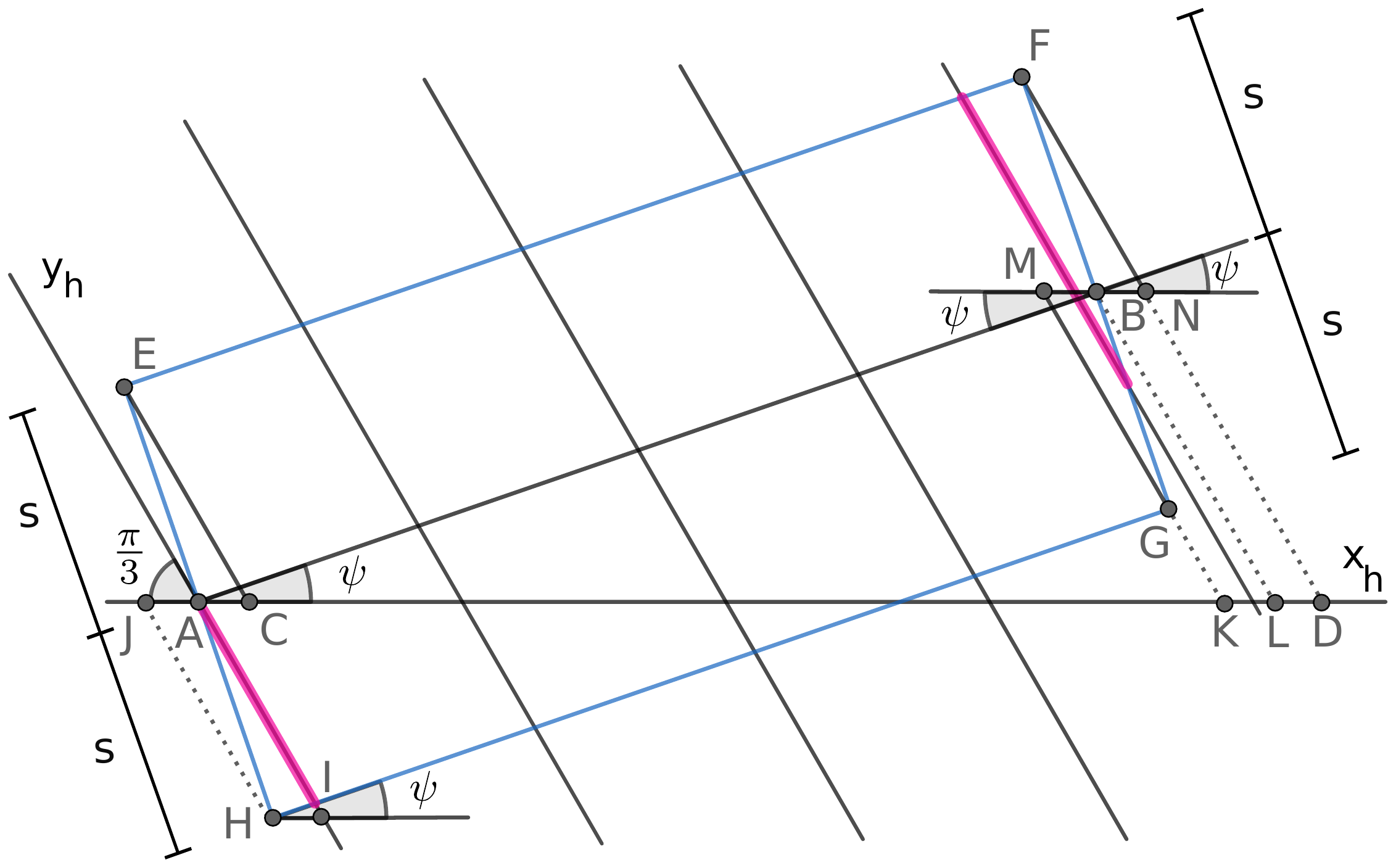}
    \caption{In this example $\psi \le \pi/6$. The pink line on the left side is an example of one satisfying Lemma \ref{lemma:Interval1}, while the one on the right side, Lemma \ref{lemma:endcase}. The triangles ACE, HIA, BMG and BNF are congruent, because their respective angles are equal -- due to parallelism -- and $\vert \overline{EA}\vert =\vert \overline{AH}\vert =\vert \overline{GB}\vert =\vert \overline{FB}\vert =s$. In this example, except for $\protect\overleftrightarrow{JH}, \protect\overleftrightarrow{EC}, \protect\overleftrightarrow{MG},  \protect\overleftrightarrow{BL}$ and $\protect \overleftrightarrow{FD}$, the lines  parallel-to-$y_{h}$ are distant by $d$  on the projection over the $x$-axis and can have robots on them. }
    \label{fig:not_full_lines}
  \end{figure}
  
  \begin{figure}[t]
    \centering
    \includegraphics[width=0.7\columnwidth]{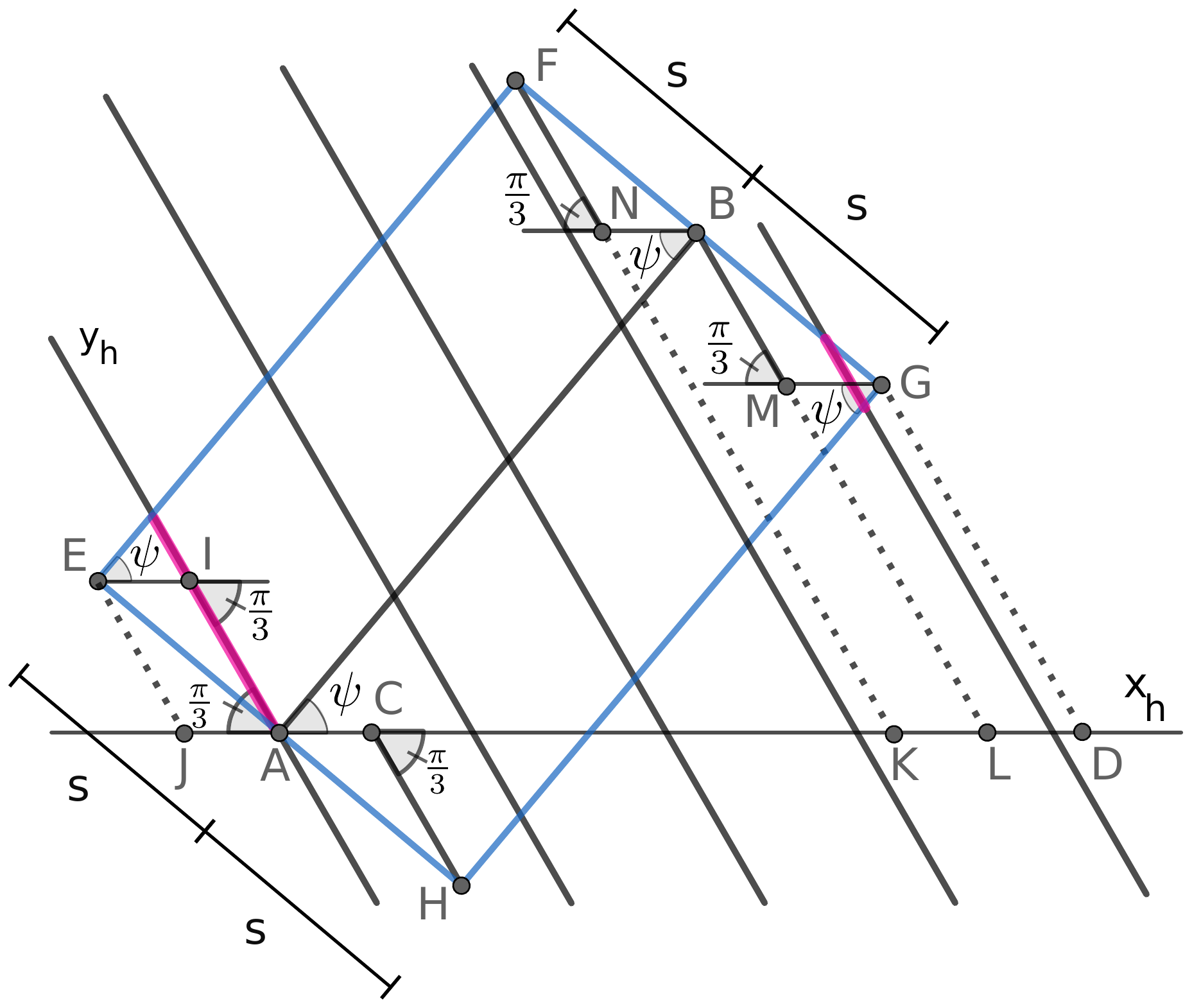}
    \caption{In this example $\psi > \pi/6$, the side EH has an angle greater than zero with the $y_{h}$-axis. The pink line on the left side is an example of one satisfying Lemma \ref{lemma:Interval1}, while the one on the right side, Lemma \ref{lemma:endcase}.  The triangles AIE, HCA, FNB and BMG are congruent, because their respective angles are equal -- due to parallelism -- and $\vert \overline{EA}\vert =\vert \overline{AH}\vert =\vert \overline{GB}\vert =\vert \overline{FB}\vert =s$. Except for $\protect\overleftrightarrow{EJ}, \protect\overleftrightarrow{CH}, \protect\overleftrightarrow{FK}, \protect\overleftrightarrow{BL}$ and $\protect \overleftrightarrow{GD}$, the lines  parallel-to-$y_{h}$ are distant by $d$ on the projection over the $x$-axis and can have robots on them.  }
    \label{fig:not_full_lines2}
  \end{figure}
  
  For $n_{l}^{-}$, we also calculate how many lines parallel to the $y_{h}$-axis projected over the $x$-axis are distant from each other by $d$ on this axis and are inside the rectangle. However, we consider only those on the left side of the point $A$, i.e., commencing from the one whose intersection with the $x$-axis is at $(-d,0)$, equivalently, $(-1,0)$ on the $(x_{h},y_{h})$ coordinate system. We also have two cases here.
    
    \begin{itemize}
      \item Case $\psi \le \pi/6$:  
      Figure \ref{fig:not_full_lines} shows the $\bigtriangleup HIA$ on the left side of the rectangle EFGH. As the robots are over the parallel-to-$y_{h}$ lines distant by $d$ on the projection over the $x$-axis, we want to know how many parallel lines intersect $\overline{HI}$ (equivalently, how many such lines intersect $\overline{JA}$ due to parallelism), excluding $\overleftrightarrow{AI}$ (because it was already counted on $n_{l}^{+}$). Thus,
      \if\shortVersion 1
        $n_{l}^{-} = \left\lfloor\frac{\vert HI\vert }{d}\right\rfloor.$
      \else
        $$n_{l}^{-} = \left\lfloor\frac{\vert HI\vert }{d}\right\rfloor.$$
      
      \fi
      We have $\vert \overline{AH}\vert = s, \widehat{H} = \pi/2 + \psi, \widehat{I} = \pi/3$ and $\widehat{A} = \pi - \widehat{I} - \widehat{H} = \pi - \pi/3 - (\pi/2 + \psi) = \pi/6 - \psi$. By the law of sines on the angles opposite to the sides $\overline{AH}$ and $\overline{HI}$, results the following 
      \begin{equation}
        \begin{aligned}
        \vert HI\vert 
          = \frac{\vert AH\vert \sin(\widehat{A})}{\sin(\widehat{I})} 
          = \frac{s\sin\left(\frac{\pi}{6} - \psi\right)}{\sin\left(\frac{\pi}{3}\right)} 
          = \frac{2s\sin\left(\frac{\pi}{6} - \psi\right)}{\sqrt{3}}.
        \end{aligned}
        \label{eq:hisize}
      \end{equation}
      Thus, 
      \if\shortVersion 1
        $
        n_{l}^{-}  = \left \lfloor\frac{2s\sin\left(\frac{\pi}{6} - \psi\right)}{\sqrt{3}d}  \right \rfloor.
        $
      \else
        $$
        n_{l}^{-}  = \left \lfloor\frac{2s\sin\left(\frac{\pi}{6} - \psi\right)}{\sqrt{3}d}  \right \rfloor.
        $$
      \fi
      \item Case $\psi > \pi/6$. Figure \ref{fig:not_full_lines2} illustrates this case. The reasoning is similar to the previous case, but now we use $\bigtriangleup EIA$. Then, $\vert \overline{EA}\vert = s, \widehat{E} =  \pi/2 - \psi$, $\widehat{I} = 2\pi/3$ and $\widehat{A} =  
       \pi - \widehat{I} - \widehat{E} 
       = \pi - 2\pi/3 - (\pi/2 - \psi) 
       =  \psi - \pi/6$. Consequently, 
      \if\shortVersion 1 
        $
            n_{l}^{-} 
              = \left\lfloor\frac{\vert \overline{EI}\vert }{d} \right\rfloor   
              = \left\lfloor\frac{2s\sin\left(\psi - \frac{\pi}{6}\right)}{\sqrt{3}d}\right\rfloor.  
        $
      \else
        $$
          \begin{aligned}
            n_{l}^{-} 
              = \left\lfloor\frac{\vert \overline{EI}\vert }{d} \right\rfloor   
              = \left\lfloor\frac{2s\sin\left(\psi - \frac{\pi}{6}\right)}{\sqrt{3}d}\right\rfloor.  \\
          \end{aligned}
        $$
      \fi
    \end{itemize}
    
    For the final result in (\ref{eq:nlm}), we use the absolute value inside the sine function to combine both cases.
    \fi %
  \end{proof}
  
  By the previous lemma, we have calculated the interval of integer $x_{h}$ values needed for counting the robots inside the rectangle. In the next lemma, we get the equation for the number of robots at the rectangular part ($N_{R}$) ranging from these integer $x_{h}$ values. Although the proposition we are now proving gives the throughput in terms of $\theta$, we are first going to calculate this number in terms of $\psi$.
  
  \begin{lemma}
    For $\psi \in \lbrack 0,\pi/3 \rparen$,
    \if\shortVersion 1
      $
        N_{R}(T,\psi) = 
          \sum_{x_{h}=-n_{l}^{-}}^{n_{l}^{+}-1}\left( \lfloor Y_{2}^{R}(x_{h}) \rfloor - \lceil Y_{1}^{R}(x_{h}) \rceil + 1\right).
      $
    \else
      $$
        N_{R}(T,\psi) = 
          \sum_{x_{h}=-n_{l}^{-}}^{n_{l}^{+}-1}\left( \lfloor Y_{2}^{R}(x_{h}) \rfloor - \lceil Y_{1}^{R}(x_{h}) \rceil + 1\right).
      $$
    \fi
    If for some $x_{h}$ $\left\lfloor Y_{2}^{R}(x_{h}) \right\rfloor < \left \lceil Y_{1}^{R}(x_{h}) \right \rceil $, we assume the respective summand for this $x_{h}$ being zero.
    \label{lemma:NR}
  \end{lemma}
  \begin{proof}
  \ifithasappendixforlemmas %
      See Online Appendix.
  \else %
  By the previous lemma and knowing that the positions of the robots are integer coordinates over the hexagonal grid coordinate space, 
  \if\shortVersion 1
    $
      N_{R}(T,\psi) 
      = \sum_{x_{h}=-n_{l}^{-}}^{n_{l}^{+}-1} \left( \lfloor Y_{2}^{R}(x_{h}) \rfloor - \lceil Y_{1}^{R}(x_{h}) \rceil + 1 \right).
    $
  \else
    $$
      \begin{aligned}
        N_{R}(T,\psi) 
        &= \sum_{x_{h}=-n_{l}^{-}}^{n_{l}^{+}-1}\sum_{y_{h}=\lceil Y_{1}^{R}(x_{h}) \rceil}^{\lfloor Y_{2}^{R}(x_{h}) \rfloor}1
        = \sum_{x_{h}=-n_{l}^{-}}^{n_{l}^{+}-1} \left( \lfloor Y_{2}^{R}(x_{h}) \rfloor - \lceil Y_{1}^{R}(x_{h}) \rceil + 1 \right).
      \end{aligned}
    $$
  \fi
  since (\ref{eq:y1xh}) and (\ref{eq:y2xh}) give the minimum ($Y_{1}^{R}$) and maximum ($Y_{2}^{R}$) $y_{h}$ coordinates for a given $x_{h}$ value such that the robot is inside the rectangle.  Note that the last summation can only be used when $\left\lfloor Y_{2}^{R}(x_{h}) \right\rfloor \ge \left \lceil Y_{1}^{R}(x_{h}) \right \rceil $, otherwise a negative number of robots would be accounted.   
  \fi %
  \end{proof}
  
  In special, for $\psi = \pi/6$, by (\ref{eq:y1xh}) and (\ref{eq:y2xh}),
  \begin{equation}
    N(T,\pi/6) = \sum_{x_{h}=0}^{
    \left\lfloor\frac{2  (vT-s) }{\sqrt{3}d} \right\rfloor
    } \left(\left \lfloor \frac{\sqrt{3}y_{2} + d x_{h}}{2d} \right \rfloor - \left \lceil \frac{\sqrt{3}y_{1} + d x_{h}}{2d} \right \rceil + 1 \right). 
    \label{eq:30degreescasepsi}
  \end{equation}

  If $\psi \neq \pi/6$, each parallel-to-$y_{h}$-axis  line intersects  two segments of the rectangle EFGH. The $y_{h}$-components of the two intersections of a rectangle side and such lines are the values of $Y_{1}^{R}(x_{h})$ and  $Y_{2}^{R}(x_{h})$ for a given $x_{h}$. Hence, the set of $x_{h}$ integer values $\{-n_{l}^{-}, \dots, n_{l}^{+}-1\}$ will be cut in disjoint subsets based on the $\max$ and $\min$ outcomes of (\ref{eq:y1xh}) and (\ref{eq:y2xh}). That is, $Y_{1}^{R}(x_{h})$ and $Y_{2}^{R}(x_{h})$, respectively; equivalently, which two sides of the rectangle the parallel-to-$y_{h}$-axis line corresponding to $(x_{h},0)$ intersects. The following lemmas describe each subset: $\{- n_{l}^{-}, \dots, n_{l}^{-}\}$ in Lemma \ref{lemma:Interval1}; $\{n_{l}^{-} + 1, \dots, K'-1 \}$ in Lemma \ref{lemma:MiddleInterval}; $\{K', \dots, n_{l}^{+}-1\}$ in Lemma \ref{lemma:endcase}, for an integer $K'$ defined later.

  \begin{lemma}
    Consider parallel-to-$y_{h}$-axis lines inside the rectangle EFGH intersecting the $x_{h}$-axis at $(x_{h},0)$, for $x_{h} \in \Zeta$. The two following statements  are equivalent:
    \begin{itemize}
      \item[(I)] 
        If $\psi < \pi/6$,
        \begin{equation}
          Y_{1}^{R}(x_{h}) = 
             \frac{\frac{2y_1}{d} + 2\tan(\psi) x_h}{\sqrt{3} + \tan(\psi)} \text{ and }
          Y_{2}^{R}(x_{h}) = 
            \frac{-2 x_{h}}{\sqrt{3} \tan(\psi) - 1},
          \label{eq:eq48}
        \end{equation}
        and, if $\psi > \pi/6$,
        \begin{equation}
          Y_{1}^{R}(x_{h}) = 
            \frac{-2 x_{h}}{\sqrt{3} \tan(\psi) - 1} \text{ and }  
          Y_{2}^{R}(x_{h}) = 
           \frac{\frac{2y_2}{d} + 2\tan(\psi) x_h}{\sqrt{3} + \tan(\psi)}.
        \end{equation}
      \item[(II)] 
         $x_{h} \in \{- n_{l}^{-}, \dots,   n_{l}^{-}\}$.
    \end{itemize}
    \label{lemma:Interval1}
  \end{lemma}
  \begin{proof} 
    \ifithasappendixforlemmas %
      See Online Appendix.
    \else %
       $(I) \Rightarrow (II):$ Let $\psi < \pi/6$. By (\ref{eq:y1xh}) and (\ref{eq:y2xh}), (\ref{eq:eq48}) is equivalent to 
    \if\shortVersion 1
      $
      \frac{ \sin(\psi) x_h -\frac{s}{d}}{\cos\left(\frac{\pi}{6}-\psi\right)} \ge \frac{\frac{v T-s}{d} - \cos(\psi)x_{h}}{\sin\left(\psi-\frac{\pi}{6}\right)}
      $ and $
      \frac{ \sin(\psi) x_h +\frac{s}{d}}{\cos\left(\frac{\pi}{6}-\psi\right)} \ge \frac{-\cos(\psi)x_{h}}{\sin\left(\psi-\frac{\pi}{6}\right)}.
      $
    \else
      $$
      \frac{ \sin(\psi) x_h -\frac{s}{d}}{\cos\left(\frac{\pi}{6}-\psi\right)} \ge \frac{\frac{v T-s}{d} - \cos(\psi)x_{h}}{\sin\left(\psi-\frac{\pi}{6}\right)}
      \text{ and }
      \frac{ \sin(\psi) x_h +\frac{s}{d}}{\cos\left(\frac{\pi}{6}-\psi\right)} \ge \frac{-\cos(\psi)x_{h}}{\sin\left(\psi-\frac{\pi}{6}\right)}.
      $$
    \fi
    From the second inequality, we have 
    \if\shortVersion 1
      $
            \frac{ \sin(\psi) x_h +\frac{s}{d}}{\cos\left(\frac{\pi}{6}-\psi\right)}
            \ge
            \frac{- \cos(\psi)x_{h}}{\sin\left(\psi-\frac{\pi}{6}\right)}
        \LR 
           x_{h}
          \le
          \frac{- 2s\sin\left(\psi-\frac{\pi}{6}\right)}{\sqrt{3}d}
          = \frac{2s\sin\left(\frac{\pi}{6}-\psi\right)}{\sqrt{3}d}.
      $
    \else
      $$
        \begin{aligned}
          &\phantom{\LR} 
            \frac{ \sin(\psi) x_h +\frac{s}{d}}{\cos\left(\frac{\pi}{6}-\psi\right)}
            \ge
            \frac{- \cos(\psi)x_{h}}{\sin\left(\psi-\frac{\pi}{6}\right)}
          \\ 
          &\LR
            \left(\frac{ \sin(\psi) }{\cos\left(\frac{\pi}{6}-\psi\right)} + \frac{ \cos(\psi)}{\sin\left(\psi-\frac{\pi}{6}\right)}\right) x_{h}
            \ge
            -\frac{s}{d\cos\left(\frac{\pi}{6}-\psi\right)}
          \\
          &\LR
            \left( \frac{\sin(\psi) \sin\left(\psi-\frac{\pi}{6}\right) +  \cos(\psi)\cos\left(\frac{\pi}{6}-\psi\right)}{\cos\left(\frac{\pi}{6}-\psi\right)\sin\left(\psi-\frac{\pi}{6}\right)}\right) x_{h}
            \ge
            - \frac{s\sin\left(\psi-\frac{\pi}{6}\right)}{d\cos\left(\frac{\pi}{6}-\psi\right)\sin\left(\psi-\frac{\pi}{6}\right)}
          \\ 
        \end{aligned}
      $$
      $$
        \begin{aligned}
\ifexpandexplanation 
        &\LR
          \left( \frac{\sin(\psi) \sin\left(\psi-\frac{\pi}{6}\right) +  \cos(\psi)\cos\left(\psi-\frac{\pi}{6}\right)}{\cos\left(\frac{\pi}{6}-\psi\right)\sin\left(\psi-\frac{\pi}{6}\right)}\right) x_{h}
          \ge
        - \frac{s\sin\left(\psi-\frac{\pi}{6}\right)}{d\cos\left(\frac{\pi}{6}-\psi\right)\sin\left(\psi-\frac{\pi}{6}\right)}
        \\ 
        &\LR
          \left( \frac{\cos(\psi - (\psi-\frac{\pi}{6}))}{\cos\left(\frac{\pi}{6}-\psi\right)\sin\left(\psi-\frac{\pi}{6}\right)}\right) x_{h}
          \ge
         - \frac{s\sin\left(\psi-\frac{\pi}{6}\right)}{d\cos\left(\frac{\pi}{6}-\psi\right)\sin\left(\psi-\frac{\pi}{6}\right)}
        \\
        &\LR
          \left( \frac{\cos(\psi - \psi + \frac{\pi}{6})}{\cos\left(\frac{\pi}{6}-\psi\right)\sin\left(\psi-\frac{\pi}{6}\right)}\right) x_{h}
          \ge
         - \frac{s\sin\left(\psi-\frac{\pi}{6}\right)}{d\cos\left(\frac{\pi}{6}-\psi\right)\sin\left(\psi-\frac{\pi}{6}\right)}
        \\
\fi
        &\LR
          \left( \frac{\cos(\frac{\pi}{6})}{\cos\left(\frac{\pi}{6}-\psi\right)\sin\left(\psi-\frac{\pi}{6}\right)}\right) x_{h}
          \ge
         - \frac{s\sin\left(\psi-\frac{\pi}{6}\right)}{d\cos\left(\frac{\pi}{6}-\psi\right)\sin\left(\psi-\frac{\pi}{6}\right)}
        \\
        &\LR
          \left( \frac{\frac{\sqrt{3}}{2}}{\cos\left(\frac{\pi}{6}-\psi\right)\sin\left(\psi-\frac{\pi}{6}\right)}\right) x_{h}
          \ge
          - \frac{s\sin\left(\psi-\frac{\pi}{6}\right)}{d\cos\left(\frac{\pi}{6}-\psi\right)\sin\left(\psi-\frac{\pi}{6}\right)}
        \\
\ifexpandexplanation
        &\LR
           x_{h}
          \le
          \frac{- s\sin\left(\psi-\frac{\pi}{6}\right)}{\frac{\sqrt{3}}{2}d}
        \\
\fi
        &\LR 
           x_{h}
          \le
          \frac{- 2s\sin\left(\psi-\frac{\pi}{6}\right)}{\sqrt{3}d}
          = \frac{2s\sin\left(\frac{\pi}{6}-\psi\right)}{\sqrt{3}d}.
      \end{aligned}
    $$
  \fi
  The change of inequality sign above is due to $\cos\left(\frac{\pi}{6}-\psi\right)\sin\left(\psi-\frac{\pi}{6}\right) < 0$ for $\psi < \pi/6$.  As $x_{h} \in \Zeta$, 
  \if\shortVersion 1
    $
        x_{h} \le  \left\lfloor\frac{2s\sin\left(\frac{\pi}{6} - \psi \right)}{\sqrt{3}d}\right\rfloor = n_{l}^{-}.
    $
  \else
    $$
        x_{h} \le  \left\lfloor\frac{2s\sin\left(\frac{\pi}{6} - \psi \right)}{\sqrt{3}d}\right\rfloor = n_{l}^{-}.
    $$
  \fi
  The lower value on $x_{h}$ is obtained by Lemma \ref{lemma:nlmnlp}, as to be inside the rectangle EFGH $x_{h} \ge -n_{l}^{-}$. For $\psi > \pi/6$, we obtain the same result by a similar reasoning, but without changing the inequality sign since in this case $\cos\left(\frac{\pi}{6}-\psi\right)\sin\left(\psi-\frac{\pi}{6}\right) > 0$.

    $(II) \Rightarrow (I):$ 
    From (\ref{eq:boundsyh1}), (\ref{eq:xhvTcos1}) (i.e., the line equations for $\overleftrightarrow{HG}$, $\overleftrightarrow{EH}$ and $\overleftrightarrow{EF}$), (\ref{eq:y1xh}) and (\ref{eq:y2xh}) (i.e., the definitions of $Y_{1}^{R}$ and $Y_{2}^{R}$), we have, if $\psi < \pi/6$, 
    \if\shortVersion 1 
      $(x_{h},Y_{1}^{R}(x_{h})) \in \overleftrightarrow{HG} \LR Y_{1}^{R}(x_{h}) =  \frac{\frac{2y_1}{d} + 2\tan(\psi) x_h}{\sqrt{3} + \tan(\psi)},$
      $(x_{h},Y_{2}^{R}(x_{h})) \in \overleftrightarrow{EH} \LR Y_{2}^{R}(x_{h}) =  \frac{-2 x_{h}}{\sqrt{3} \tan(\psi) - 1},$
      and, if $\psi > \pi/6$,
      $(x_{h},Y_{1}^{R}(x_{h})) \in \overleftrightarrow{EH} \LR Y_{1}^{R}(x_{h}) = \frac{-2 x_{h}}{\sqrt{3} \tan(\psi) - 1} ,$
      $(x_{h},Y_{2}^{R}(x_{h})) \in \overleftrightarrow{EF} \LR Y_{2}^{R}(x_{h}) =  \frac{\frac{2y_2}{d} + 2\tan(\psi) x_h}{\sqrt{3} + \tan(\psi)}.$
    \else
      $$(x_{h},Y_{1}^{R}(x_{h})) \in \overleftrightarrow{HG} \LR Y_{1}^{R}(x_{h}) =  \frac{\frac{2y_1}{d} + 2\tan(\psi) x_h}{\sqrt{3} + \tan(\psi)},$$
      $$(x_{h},Y_{2}^{R}(x_{h})) \in \overleftrightarrow{EH} \LR Y_{2}^{R}(x_{h}) =  \frac{-2 x_{h}}{\sqrt{3} \tan(\psi) - 1},$$
      and, if $\psi > \pi/6$,
      $$(x_{h},Y_{1}^{R}(x_{h})) \in \overleftrightarrow{EH} \LR Y_{1}^{R}(x_{h}) = \frac{-2 x_{h}}{\sqrt{3} \tan(\psi) - 1} ,$$
      $$(x_{h},Y_{2}^{R}(x_{h})) \in \overleftrightarrow{EF} \LR Y_{2}^{R}(x_{h}) =  \frac{\frac{2y_2}{d} + 2\tan(\psi) x_h}{\sqrt{3} + \tan(\psi)}.$$
      
    \fi
    Then, we prove this part by showing that for all $x_{h} \in \{- n_{l}^{-},\dots,n_{l}^{-}\}$, the line parallel to the $y_{h}$-axis intercepting the point $(x_{h},0)$ intercepts both sides $\overline{EH}$ and $\overline{HG}$ (and no other), if $\psi < \pi/6$ (Figure \ref{fig:not_full_lines}), and,  if $\psi > \pi/6$, both sides $\overline{EH}$ and $\overline{EF}$ (and no other) (Figure \ref{fig:not_full_lines2}).
    
  \begin{itemize}
    \item Case $\psi < \pi/6$:  
    Figure \ref{fig:not_full_lines} shows the triangles HIA, ACE and BMG inside the rectangle EFGH. As the robots are over the parallel lines to the $y_{h}$-axis, which are distant by $d$ when projected over the $x$-axis, we want to know how many such parallel lines intersect $\overline{HI}$ (equivalently, how many such lines intersect $\overline{JA}$ due to parallelism) or $\overline{AC}$.
    For such parallel lines that intersect $\overline{HI}$,  Lemma \ref{lemma:nlmnlp} showed that for every $x_{h} \in \{-n_{l}^{-},\dots,-1\}$  the line parallel to $y_{h}$-axis intersecting $(x_{h},0)$ is inside the rectangle. Also, these lines intersect the sides $\overline{EH}$ and $\overline{HG}$, as any line parallel to $\overleftrightarrow{AI}$ which is on its left side intersects the sides $\overline{EH}$ and $\overline{HG}$ if it is inside the rectangle.
    For the case where such parallel lines intersect  $\overline{AC}$, we need to know the maximum integer value, $M$, such that these parallel lines still intersect the sides $\overline{EH}$ and $\overline{HG}$ for any $x_{h} \in \{0, \dots, M\}$. Starting from point $A$ (that is, when $x_{h} = 0$), we have
    \if\shortVersion 1
      $M= \left\lfloor\frac{\vert \overline{AC}\vert }{d}\right\rfloor.$
    \else
      $$M= \left\lfloor\frac{\vert \overline{AC}\vert }{d}\right\rfloor.$$
    \fi
    We have $\vert \overline{AH}\vert = \vert \overline{EA}\vert = s$, $\overleftrightarrow{AI} \parallel \overleftrightarrow{EC}$, $\overleftrightarrow{HI} \parallel \overleftrightarrow{AC}$, and $\overleftrightarrow{AH} \parallel \overleftrightarrow{AE}$ (as $E$, $A$ and $H$ are collinear), then $\widehat{IHA} = \widehat{CAE}, \widehat{AIH} =  \widehat{ECA}$, and $\widehat{HAI} = \widehat{AEC}$. Thus, $\bigtriangleup HIA \cong \bigtriangleup ACE$, then $\vert \overline{AC}\vert = \vert \overline{HI}\vert $, whose value has been previously calculated in Lemma \ref{lemma:nlmnlp}, leading to
    \if\shortVersion 1
      $
      M = \left \lfloor\frac{2s\sin\left(\frac{\pi}{6} - \psi\right)}{\sqrt{3}d}  \right \rfloor = n_{l}^{-}.
      $
    \else
      $$
      M = \left \lfloor\frac{2s\sin\left(\frac{\pi}{6} - \psi\right)}{\sqrt{3}d}  \right \rfloor = n_{l}^{-}.
      $$
    \fi
    Hence, for any $x_{h} \in \{0, \dots, n_{l}^{-}\}$, those parallel lines intersect the sides $\overline{EH}$ and $\overline{HG}$. 
    \item Case $\psi > \pi/6$: Figure \ref{fig:not_full_lines2} illustrates this case. The reasoning is similar to the previous case, but using that $\bigtriangleup AIE \cong \bigtriangleup HCA  $. As the value for $\vert \overline{EI}\vert /d$ also has been calculated in Lemma \ref{lemma:nlmnlp} for this figure, then 
    \if\shortVersion 1
      $
        M 
          = \left\lfloor\frac{2s\sin\left(\psi - \frac{\pi}{6}\right)}{\sqrt{3}d}\right\rfloor = n_{l}^{-}.
      $
    \else
      $$
        M 
          = \left\lfloor\frac{2s\sin\left(\psi - \frac{\pi}{6}\right)}{\sqrt{3}d}\right\rfloor = n_{l}^{-}.
      $$
    \fi 
    Consequently, for any $x_{h} \in \{-n_{l}^{-}, \dots, n_{l}^{-}\}$, the parallel-to-$y_{h}$-axis line at $(x_{h},0)$  intersects the sides $\overline{EH}$ and $\overline{EF}$ in this case. \qed     
  \end{itemize}  \renewcommand{\qedsymbol}{}
  \fi %
  \end{proof}
  
The next lemma will define the integer $K'$ mentioned before. This number will be compared with the integer $x_{h}$ coordinate of the point $(n_{l}^{+}-1,0)$ intersected by the rightmost parallel-to-$y_{h}$-axis line inside the rectangle EFGH. Assuming $\theta \neq \pi/6$, if this rightmost line  intersects a point on the $x_{h}$-axis with an integer coordinate less than $K'$, then no parallel-to-$y_{h}$-axis line intersects the rectangle right side $\overline{FG}$. However, if the intersection point coordinate is greater than or equal to $K'$, then at least one parallel line crosses  $\overline{FG}$.

  \begin{lemma}
    Consider parallel-to-$y_{h}$-axis lines inside the rectangle EFGH intersecting the $x_{h}$-axis at $(x_{h},0)$, for $x_{h} \in \Zeta$, and  $K' = \left\lceil\frac{2 (vT-s) \cos(\psi - \pi/6) - 2s\sin(\vert \psi - \pi/6\vert ) }{\sqrt{3}d}\right\rceil$. Then, the two statements below are equivalent: 
    \begin{itemize}
       \item[(I)] If $\psi < \pi/6$
          \begin{equation}
            Y_{1}^{R}(x_{h}) = 
              \frac{\frac{2 (v T-s)}{d \cos(\psi)}-2 x_{h}}{\sqrt{3} \tan(\psi) - 1}\text{ and } 
            Y_{2}^{R}(x_{h}) = 
              \frac{\frac{2y_2}{d} + 2\tan(\psi) x_h}{\sqrt{3} + \tan(\psi)},
            \label{eq:endcase1}
          \end{equation}
          and, if $\psi > \pi/6$
          \begin{equation}
            Y_{1}^{R}(x_{h}) = 
              \frac{\frac{2y_1}{d} + 2\tan(\psi) x_h}{\sqrt{3} + \tan(\psi)} \text{ and }
            Y_{2}^{R}(x_{h}) = 
              \frac{\frac{2 (v T-s)}{d \cos(\psi)}-2 x_{h}}{\sqrt{3} \tan(\psi) - 1}.
            \label{eq:endcase2}  
          \end{equation}
        \item[(II)] $x_{h} \in \{ K', \dots, n_{l}^{+}-1\}$.
    \end{itemize}
    \label{lemma:endcase}
  \end{lemma}
  \begin{proof} 
    \ifithasappendixforlemmas %
      See Online Appendix.
    \else %
    $(I) \Rightarrow (II)$: By contrapositive, assume $x_{h} \notin  \{ K', \dots, n_{l}^{+} - 1\}$. By Lemma \ref{lemma:nlmnlp}, there is no $x_{h} > n_{l}^{+} - 1$, so $x_{h} < K'$.  
    For the case of $\psi < \pi/6$, observe in Figure \ref{fig:not_full_lines} the point $K$ on the $x_{h}$-axis. This point corresponds to the intersection of $\overleftrightarrow{MG}$ on the $x_{h}$-axis, which is the first parallel-to-$y_{h}$-axis crossing the rectangle right side $\overline{FG}$.  The point D on the $x_{h}$-axis is the projection of the point F on this axis. By (\ref{eq:adacsin}), $\vert \overline{AD}\vert = \frac{2\cos(\pi/6 - \psi)   (vT-s)   +  {2s\sin(\vert \psi - \pi/6\vert )}}{\sqrt{3}}$. Because of the parallelism, we have $\vert \overline{MN}\vert = \vert \overline{KD}\vert $. Due to the congruence of triangles ACE, HIA, BMG and BNF and (\ref{eq:hisize}), $\vert \overline{BM}\vert = \vert \overline{BN}\vert = \vert \overline{HI}\vert = \frac{2s\sin\left(\vert \psi-\pi/6\vert \right)}{\sqrt{3}}$. Thus, $\vert \overline{KD}\vert = \vert \overline{MN}\vert = \vert \overline{BM}\vert + \vert \overline{BN}\vert = \frac{4s\sin\left(\vert \psi - \pi/6 \vert \right)}{\sqrt{3}}$. Since $\vert \overline{AK}\vert = \vert \overline{AD}\vert - \vert \overline{KD}\vert = \frac{2\cos(\pi/6 - \psi)   (vT -s)  -  {2s\sin(\vert \psi - \pi/6\vert )}}{\sqrt{3}}$, the point K is located on the $(x_{h},y_{h})$ coordinate space  at $\bigg( \frac{2\cos(\pi/6 - \psi)   (vT-s)}{\sqrt{3}d}   -  \frac{2s\sin(\vert \psi - \pi/6\vert )}{\sqrt{3}d} , 0 \bigg)$, as $K$ is on the $x$-axis and to convert it to $(x_{h},y_{h})$ coordinate space we only need to divide the $x$-coordinate by $d$. On the $x_{h}$-axis, the nearest point on the right of $K$  with integer $x_{h}$ is $(\left\lceil K \right\rceil,0)=(K',0)$. As we assumed $x_{h} < K'$, no parallel-to-$y_{h}$-axis crossing a integer $(x_{h},0)$ point inside the rectangle intersects $\overline{FG}$. Thus, no such parallel line has $Y_{1}^{R}(x_{h}) = \frac{\frac{2 (v T-s)}{d \cos(\psi)}-2 x_{h}}{\sqrt{3} \tan(\psi) - 1}$, which is the $y_{h}$-coordinate of the intersection of this line with $\overleftrightarrow{FG}$. 
    
    In the case of $\psi > \pi/6$, using a similar argument in the Figure \ref{fig:not_full_lines2} concludes the desired result, but here we use $\vert \overline{NB}\vert + \vert \overline{MG}\vert = \vert \overline{KD}\vert $ and the congruence of triangles AIE, HCA, FNB and BMG. As we assumed $x_{h} < K'$, no parallel-to-$y_{h}$-axis intersecting a integer point $(x_{h},0)$  inside the rectangle crosses $\overline{FG}$, so for such line $Y_{2}^{R}(x_{h}) \neq \frac{\frac{2 (v T-s)}{d \cos(\psi)}-2 x_{h}}{\sqrt{3} \tan(\psi) - 1}$.

    $(II) \Rightarrow (I):$ If $x_{h} \in \{K', \dots, n_{l}^{+}-1\}$ then the lines parallel-to-$y_{h}$-axis inside the rectangle intersecting the $x_{h}$-axis at $(x_{h}, 0)$ are  on the right of point K or intersecting it. Hence, these lines intersect $\overline{EF}$ and $\overline{FG}$, if $\psi < \pi/6$. By applying (\ref{eq:boundsyh1}), (\ref{eq:xhvTcos1}) (for the line equations for $\overleftrightarrow{EF}$ and $\overleftrightarrow{FG}$), (\ref{eq:y1xh}) and (\ref{eq:y2xh}) (for the definitions of $Y_{1}^{R}$ and $Y_{2}^{R}$), we have (\ref{eq:endcase1}).  A similar argument is used in the case of $\psi > \pi/6$, but for $\overline{FG}$ and $\overline{HG}$ intersections, yielding  (\ref{eq:endcase2}).     
    \fi %
  \end{proof}

  The lemma below characterises when a parallel-to-$y_{h}$-axis line touches only the sides EH and FG of the rectangle. Intuitively, 
  if this happens we have a rectangle with a small width. Thus, on rectangles with a large width, no such lines are crossing the sides EH and FG, for $\psi \neq \pi/6$. We will use this lemma on the Lemma \ref{lemma:MiddleInterval}, for completing the disjoint subsets based on the possible $\max$ and $\min$ outcomes of $Y_{1}^{R}$ and $Y_{2}^{R}$.
  
  \begin{lemma}
    If $vT -s > 2s\tan(\vert \psi - \frac{\pi}{6}\vert )$, then
    there is not a $x_{h} \in \{-n_{l}^{-},\dots , n_{l}^{+}-1\}$ such that,  
    \if\shortVersion 1
      $
          Y_{1}^{R}(x_{h}) = 
            \frac{\frac{2 (v T-s)}{d \cos(\psi)}-2 x_{h}}{\sqrt{3} \tan(\psi) - 1} 
      $
      and
      $
          Y_{2}^{R}(x_{h}) = 
            \frac{-2 x_{h}}{\sqrt{3} \tan(\psi) - 1},
          \text{ if } \psi<\pi/6;
      $
      $
          Y_{1}^{R}(x_{h}) = 
            \frac{-2 x_{h}}{\sqrt{3} \tan(\psi) - 1} 
      $ 
      and 
      $
          Y_{2}^{R}(x_{h}) = 
            \frac{\frac{2 (v T-s)}{d \cos(\psi)}-2 x_{h}}{\sqrt{3} \tan(\psi) - 1},
            \text{ if } \psi > \pi/6.
      $
    \else
      $$
          Y_{1}^{R}(x_{h}) = 
            \frac{\frac{2 (v T-s)}{d \cos(\psi)}-2 x_{h}}{\sqrt{3} \tan(\psi) - 1} 
            \text{ and }
          Y_{2}^{R}(x_{h}) = 
            \frac{-2 x_{h}}{\sqrt{3} \tan(\psi) - 1},
          \text{ if } \psi<\pi/6;
      $$
      $$
          Y_{1}^{R}(x_{h}) = 
            \frac{-2 x_{h}}{\sqrt{3} \tan(\psi) - 1} 
            \text{ and }
          Y_{2}^{R}(x_{h}) = 
            \frac{\frac{2 (v T-s)}{d \cos(\psi)}-2 x_{h}}{\sqrt{3} \tan(\psi) - 1},
            \text{ if } \psi > \pi/6.
      $$
    \fi
    \label{lemma:excludeCase}
  \end{lemma}
  \begin{proof} 
    \ifithasappendixforlemmas %
        See Online Appendix.
    \else %
    This proof is by contrapositive. Assume $\psi < \pi/6$. By  (\ref{eq:y1xh}) and (\ref{eq:y2xh}), we have a $x_{h}$ such that 
    \if\shortVersion 1
      $
        \frac{\frac{y_1}{d} + \tan(\psi) x_h}{\frac{\sqrt{3} + \tan(\psi)}{2}} \le
        \frac{\frac{v T-s}{d\cos(\psi)} - x_{h}}{ \frac{\sqrt{3} \tan(\psi) - 1}{2}}
      $ and 
      $
        \frac{\frac{y_2}{d} + \tan(\psi) x_h}{{\frac{\sqrt{3} + \tan(\psi)}{2}}} \ge
         \frac{-x_{h}}{\frac{\sqrt{3} \tan(\psi) - 1}{2}}.
      $
    \else
      $$
        \frac{\frac{y_1}{d} + \tan(\psi) x_h}{\frac{\sqrt{3} + \tan(\psi)}{2}} \le
        \frac{\frac{v T-s}{d\cos(\psi)} - x_{h}}{ \frac{\sqrt{3} \tan(\psi) - 1}{2}}
      \text{ and }
        \frac{\frac{y_2}{d} + \tan(\psi) x_h}{{\frac{\sqrt{3} + \tan(\psi)}{2}}} \ge
         \frac{-x_{h}}{\frac{\sqrt{3} \tan(\psi) - 1}{2}}.
      $$
    
    \fi
    Since $\frac{\sqrt{3} \tan(\psi) - 1}{2} < 0$, the signs of inequalities change, then we have the following implication
    \if\shortVersion 1
      $
           \frac{\frac{y_1}{d}\frac{\sqrt{3} \tan(\psi) - 1}{2}}{\frac{\sqrt{3} + \tan(\psi)}{2}} - \frac{v T-s}{d\cos(\psi)} \ge  - x_{h} - \frac{\tan(\psi) x_h \frac{\sqrt{3} \tan(\psi) - 1}{2}}{\frac{\sqrt{3} + \tan(\psi)}{2}} 
      $ and 
      $
          \frac{\frac{y_2}{d}\frac{\sqrt{3} \tan(\psi) - 1}{2}}{\frac{\sqrt{3} + \tan(\psi)}{2}}\le  -x_{h} - \frac{\tan(\psi) x_h \frac{\sqrt{3} \tan(\psi) - 1}{2}}{\frac{\sqrt{3} + \tan(\psi)}{2}} 
          \Rightarrow
          \frac{\frac{y_2}{d}\frac{\sqrt{3} \tan(\psi) - 1}{2}}{\frac{\sqrt{3} + \tan(\psi)}{2}}\le 
           \frac{\frac{y_1}{d}\frac{\sqrt{3} \tan(\psi) - 1}{2}}{\frac{\sqrt{3} + \tan(\psi)}{2}} - \frac{v T-s}{d\cos(\psi)},
      $
    \else
      $$
        \begin{aligned}
           &\phantom{\LR} \frac{\frac{y_1}{d}\frac{\sqrt{3} \tan(\psi) - 1}{2}}{\frac{\sqrt{3} + \tan(\psi)}{2}} - \frac{v T-s}{d\cos(\psi)} \ge  - x_{h} - \frac{\tan(\psi) x_h \frac{\sqrt{3} \tan(\psi) - 1}{2}}{\frac{\sqrt{3} + \tan(\psi)}{2}} 
               \text{ and }
          \\
          &\phantom{\LR} \frac{\frac{y_2}{d}\frac{\sqrt{3} \tan(\psi) - 1}{2}}{\frac{\sqrt{3} + \tan(\psi)}{2}}\le  -x_{h} - \frac{\tan(\psi) x_h \frac{\sqrt{3} \tan(\psi) - 1}{2}}{\frac{\sqrt{3} + \tan(\psi)}{2}} 
          \\
          &\Rightarrow
          \frac{\frac{y_2}{d}\frac{\sqrt{3} \tan(\psi) - 1}{2}}{\frac{\sqrt{3} + \tan(\psi)}{2}}\le 
           \frac{\frac{y_1}{d}\frac{\sqrt{3} \tan(\psi) - 1}{2}}{\frac{\sqrt{3} + \tan(\psi)}{2}} - \frac{v T-s}{d\cos(\psi)},
        \end{aligned}
      $$
    \fi
    by the transitivity of $\le$ under the real numbers. Also, we have the following equivalences
    \if\shortVersion 1
      $ 
        \frac{\frac{y_2}{d}\frac{\sqrt{3} \tan(\psi) - 1}{2}}{\frac{\sqrt{3} + \tan(\psi)}{2}}\le 
        \frac{\frac{y_1}{d}\frac{\sqrt{3} \tan(\psi) - 1}{2}}{\frac{\sqrt{3} + \tan(\psi)}{2}} - \frac{v T-s}{d\cos(\psi)}
        \LR
        v T-s
        \le 
        2s\tan(\pi/6 - \psi)
        ,
      $
      due to (\ref{eq:2sy2y1}), the equalities $\tan(a+b)=\frac{\tan(a)+\tan(b)}{1-\tan(a)\tan(b)}$, $\cot(a) = -\tan(a + \pi/2)$ and $-\tan(\pi-a) = \tan(a)$ for any real $a$ and $b$.
    \else
      $$
        \begin{aligned}
          &\phantom{\LR} \frac{\frac{y_2}{d}\frac{\sqrt{3} \tan(\psi) - 1}{2}}{\frac{\sqrt{3} + \tan(\psi)}{2}}\le 
           \frac{\frac{y_1}{d}\frac{\sqrt{3} \tan(\psi) - 1}{2}}{\frac{\sqrt{3} + \tan(\psi)}{2}} - \frac{v T-s}{d\cos(\psi)}
          \\
\ifexpandexplanation
          &\LR
          \frac{v T-s}{d\cos(\psi)}
          \le 
           \frac{\frac{y_1}{d}\frac{\sqrt{3} \tan(\psi) - 1}{2}}{\frac{\sqrt{3} + \tan(\psi)}{2}} -  \frac{\frac{y_2}{d}\frac{\sqrt{3} \tan(\psi) - 1}{2}}{\frac{\sqrt{3} + \tan(\psi)}{2}} 
          \\
\fi
          &\LR
          \frac{v T-s}{d\cos(\psi)}
          \le 
           \frac{y_1-y_2}{d} \frac{{\sqrt{3} \tan(\psi) - 1}}{{\sqrt{3} + \tan(\psi)}}  \\
          &\LR
          \frac{v T-s}{d\cos(\psi)}
          \le 
           -\frac{2s}{d\cos(\psi)} \frac{{\sqrt{3} \tan(\psi) - 1}}{{\sqrt{3} + \tan(\psi)}}  \hspace{2.5cm} [\text{By (\ref{eq:2sy2y1})}]
          \\
          & \LR
          \frac{v T-s}{2s}
          \le 
           \frac{{1- \sqrt{3} \tan(\psi)}}{{\sqrt{3} + \tan(\psi)}}  \LR
          \frac{v T-s}{2s}
          \le 
           \frac{1}{\tan(\pi/3+\psi)} \\&  \LR
          \frac{v T-s}{2s}
          \le 
           \cot(\pi/3+\psi)  
          \LR 
          \frac{v T-s}{2s}
          \le 
           -\tan(\psi+5\pi/6)  
          \\
          & \LR
          v T-s
          \le 
           2s\tan(\pi/6 - \psi)
           .
        \end{aligned}
      $$
      Above we used the equalities $\tan(a+b)=\frac{\tan(a)+\tan(b)}{1-\tan(a)\tan(b)}$, $\cot(a) = -\tan(a + \pi/2)$ and $-\tan(\pi-a) = \tan(a)$ for any real $a$ and $b$.
    \fi
    
    For the case $\psi > \pi/6$, using similar arguments we get the same result, but we do not change the signs of inequalities due to $\frac{\sqrt{3}\tan(\psi) - 1}{2} > 0$ in this case. The conclusion is reached after we combine the two cases using absolute values inside the tangent.
    \fi %
  \end{proof}
  
  The next lemma completes the properties of $N_{R}(T,\psi)$ that are useful for calculating its limit when T tends to infinity.
  
  \begin{lemma}
    Let $K'=\left\lceil\frac{2 (vT-s) \cos(\psi - \pi/6) - 2s\sin(\vert \psi - \pi/6\vert ) }{\sqrt{3}d}\right\rceil$. If $vT-s > 2s \tan(\vert \psi - \pi/6\vert )$, then  $x_{h} \in \{n_{l}^{-} + 1, \dots, K' -1\}$ if and only if
    \if\shortVersion 1
      $
          Y_{1}^{R}(x_{h}) = 
            \frac{\frac{y_1}{d} + \tan(\psi) x_h}{\frac{\sqrt{3} + \tan(\psi)}{2}} 
       $  and $
          Y_{2}^{R}(x_{h}) = 
            \frac{\frac{y_2}{d} + \tan(\psi) x_h}{{\frac{\sqrt{3} + \tan(\psi)}{2}}}.
      $
    \else
      $$
        \begin{aligned}
          Y_{1}^{R}(x_{h}) = 
            \frac{\frac{y_1}{d} + \tan(\psi) x_h}{\frac{\sqrt{3} + \tan(\psi)}{2}} 
            \text{ and }
          Y_{2}^{R}(x_{h}) = 
            \frac{\frac{y_2}{d} + \tan(\psi) x_h}{{\frac{\sqrt{3} + \tan(\psi)}{2}}}.
        \end{aligned}
      $$ 
    \fi
    \label{lemma:MiddleInterval}
  \end{lemma}
  \begin{proof} 
    \ifithasappendixforlemmas %
      See Online Appendix.
    \else %
    Excluding the case when $\psi = \pi/6$,  (\ref{eq:y1xh}) and (\ref{eq:y2xh}) give four combinations of possible outcomes for the values of $Y_{1}^{R}(x_{h})$ and $Y_{2}^{R}(x_{h})$ based on the results of $\min$ and $\max$. When $vT-s > 2s \tan(\vert \psi - \pi/6\vert )$, by Lemma \ref{lemma:excludeCase}, we do not have the case when they are on the sides $EH$ and $FG$. For the given values of $x_{h}$ on the hypothesis, neither Lemma \ref{lemma:Interval1} nor Lemma \ref{lemma:endcase} applies, excluding other two combinations of results for $Y_{1}^{R}(x_{h})$ and $Y_{2}^{R}(x_{h})$. Finally, Lemma \ref{lemma:nlmnlp} shows that every parallel-to-$y_{h}$-axis line crosses the $x_{h}$-axis at $(x_{h},0)$ for $x_{h} \in \{-n_{l}^{-}, \dots, n_{l}^{+}-1\}$,  so the remaining combination yields the desired equivalence.
    \fi %
  \end{proof}
  
  Now we present the calculation of $N_{S}(T,\theta)$. Here we are using $\theta$ instead of $\psi = \pi/3 - \theta$ for  easiness of presentation.  We denote $(l_{x},l_{y})$ the position of the last robot inside a rectangle of width $vT - s$ and height $2s$ whose left side is at $(x_{0},y_{0}).$ Here \emph{last} means the robot with highest  $x$ coordinate value. However, if two robots have the same $x$ coordinate value, we take the robot whose $y$ coordinate is nearer to $y_{0}$. Let $Z$ be the set of robot positions inside the rectangle above for $vT - s > 0$.
  
  \begin{lemma}
    Let  $c_{x} = x_{0} + vT - s$, and 
    \if\shortVersion 1
      $
      (l_{x},l_{y}) = 
          \argmin_{(x,y) \in Z}{\vert vT - s + x_{0} - x\vert + \vert y_{0} - y\vert } 
       $  if  $T > \frac{s}{v}$, otherwise,
          $(l_{x},l_{y}) = (x_{0},y_{0}).$
      Then,
      $
        N_{S}(T,\theta) = \sum_{x_{h} = B}^{U}\left(\lfloor Y_{2}^{S}(x_{h}) \rfloor - \lceil Y_{1}^{S}(x_{h}) \rceil + 1\right),
      $
    \else
      $$
      (l_{x},l_{y}) = 
        \begin{cases}
          \argmin_{(x,y) \in Z}{\vert vT - s + x_{0} - x\vert + \vert y_{0} - y\vert } 
            & \text{ if } T > \frac{s}{v}, \\
          (x_{0},y_{0}) 
          & \text{ otherwise. }
        \end{cases}
      $$
      Then,
      $$
        N_{S}(T,\theta) = \sum_{x_{h} = B}^{U}\left(\lfloor Y_{2}^{S}(x_{h}) \rfloor - \lceil Y_{1}^{S}(x_{h}) \rceil + 1\right),
      $$
    \fi
    for $\left\lfloor Y_{2}^{S}(x_{h}) \right\rfloor \ge \left \lceil Y_{1}^{S}(x_{h}) \right \rceil $ (if for some $x_{h}$ $\left\lfloor Y_{2}^{S}(x_{h}) \right\rfloor < \left \lceil Y_{1}^{S}(x_{h}) \right \rceil $, we assume the respective summand for this $x_{h}$ being zero),
    \begin{equation}
      B = \left\{
      \begin{array}{>{\displaystyle}c>{\displaystyle}l}
           \left\lceil\frac{2( \sin(\pi/3-\theta)(c_{x} - l_{x})  + \cos(\pi/3-\theta)(y_{0} - l_{y} -s) )}{\sqrt{3}d}\right\rceil,
             & \text{ if } T > \frac{s}{v},
           \\
           \left\lceil-\frac{2\sqrt{2svT - (vT)^{2}}}{\sqrt{3}d}\sin\left(\theta + \frac{\pi}{6}\right)\right\rceil,
             & \text{ otherwise, }
      \end{array}
      \right. 
      \label{eq:BcasesSC}
    \end{equation}
    if  $T > \frac{s}{v}$  or $\arctan\left( \frac{\frac{s}{2} - \sin(\theta) (vT - s) }{\frac{\sqrt{3}s}{2} + \cos(\theta) (vT - s)} \right) \le  \frac{\pi}{2} - \theta$,
    \begin{equation}
      U = \left \lfloor \frac{2(\sin(\pi/3-\theta)(c_{x} - l_{x}) + \cos(\pi/3-\theta)(y_{0} - l_{y}) + s)}{\sqrt{3}d} \right \rfloor, 
      \label{eq:eq58}
    \end{equation}
    otherwise, 
    \begin{equation}
      U = \left \lfloor \frac{2\sqrt{2svT - (vT)^{2}}}{\sqrt{3}d}\cos\left(\theta-\frac{\pi}{3}\right) \right\rfloor.
      \label{eq:eq59}
    \end{equation}
    Also,
    \if\shortVersion 1
      $
        Y_{1}^{S}(x_{h}) = 
         \frac{d {x_{h}}- {C_{-\theta,x}} + \sqrt{3} C_{-\theta,y}   - \sqrt{\Delta(x_{h})}}{2  d}, 
      $
      $
        Y_{2}^{S}(x_{h}) = 
          \min(L(x_{h}),C_{2}(x_{h})) - 1, 
        $    if  $\min(L(x_{h}),C_{2}(x_{h})) = \lfloor L(x_{h})\rfloor$ 
            and  $T > \frac{s}{v}$, otherwise,
          $ Y_{2}^{S}(x_{h}) = \min(L(x_{h}),C_{2}(x_{h}))$,
      for
      $
        C_{-\theta} = \big[
              \begin{array}{cc}
                \cos(-\theta) & -\sin(-\theta)\\
                \sin(-\theta) & \cos(-\theta)\\
              \end{array}
              \big]
              \big[
              \begin{array}{c}
                c_{x} - l_{x}\\
                y_{0} - l_{y}\\
              \end{array}
              \big],
      $
      $
        \Delta(x_{h}) =   4  s^{2} - \big(\sqrt{3}  {\big(d {x_{h}} -{C_{-\theta,x}} \big)} - C_{-\theta,y}\big)^{2},
      $
      $
        C_{2}(x_{h}) = 
         \frac{d {x_{h}}- {C_{-\theta,x}}  + \sqrt{3} C_{-\theta,y}  + \sqrt{\Delta(x_{h})}}{2  d}, 
      $
      and
      $
        L(x_{h}) = 
          \frac{\sin\big(\frac{\pi}{2} - \theta\big)(d x_{h}  - C_{-\theta,x}) + \cos\big(\frac{\pi}{2} - \theta\big)C_{-\theta,y}}{d \sin\big(\frac{5\pi}{6}-\theta\big)}$,
             if  $T > \frac{s}{v}$, otherwise,
          $L(x_{h}) = \frac{\sin\big(\frac{\pi}{2}-\theta\big)  x_{h}}{\sin\big( \frac{5\pi}{6}-\theta\big)}$.
    \else
      $$
        Y_{1}^{S}(x_{h}) = 
         \frac{d {x_{h}}- {C_{-\theta,x}} + \sqrt{3} C_{-\theta,y}   - \sqrt{\Delta(x_{h})}}{2  d}, 
      $$
      $$
        Y_{2}^{S}(x_{h}) = \left\{
        \begin{array}{>{\displaystyle}c>{\displaystyle}l}
          \min(L(x_{h}),C_{2}(x_{h})) - 1, 
           & \text{ if } \min(L(x_{h}),C_{2}(x_{h})) = \lfloor L(x_{h})\rfloor \\ 
           & \phantom{if} \text{ and } T > \frac{s}{v},\\
          \min(L(x_{h}),C_{2}(x_{h})), 
           & \text{ otherwise, } 
        \end{array}
        \right.
      $$
      for
      $$
        C_{-\theta} = \left[
              \begin{array}{cc}
                \cos(-\theta) & -\sin(-\theta)\\
                \sin(-\theta) & \cos(-\theta)\\
              \end{array}
              \right]
              \left[
              \begin{array}{c}
                c_{x} - l_{x}\\
                y_{0} - l_{y}\\
              \end{array}
              \right],
      $$
      $$
        \Delta(x_{h}) =   4  s^{2} - \left(\sqrt{3}  {\left(d {x_{h}} -{C_{-\theta,x}} \right)} - C_{-\theta,y}\right)^{2},
      $$
      $$
        C_{2}(x_{h}) = 
         \frac{d {x_{h}}- {C_{-\theta,x}}  + \sqrt{3} C_{-\theta,y}  + \sqrt{\Delta(x_{h})}}{2  d}, 
      $$
      and
      $$
        L(x_{h}) = \left\{ 
        \begin{array}{>{\displaystyle}c>{\displaystyle}l}
          \frac{\sin\left(\frac{\pi}{2} - \theta\right)(d x_{h}  - C_{-\theta,x}) + \cos\left(\frac{\pi}{2} - \theta\right)C_{-\theta,y}}{d \sin\left(\frac{5\pi}{6}-\theta\right)},
            & \text{ if } T > \frac{s}{v}, \\
          \frac{\sin\left(\frac{\pi}{2}-\theta\right)  x_{h}}{\sin\left( \frac{5\pi}{6}-\theta\right)},
            & \text{ otherwise.}\\
        \end{array}
        \right.
      $$
    \fi
    \label{lemma:NS}   
  \end{lemma}
  \begin{proof}
    \ifithasappendixforlemmas %
      See Online Appendix.
    \else %
    \begin{figure}[t]
      \centering
      \includegraphics[width=\columnwidth]{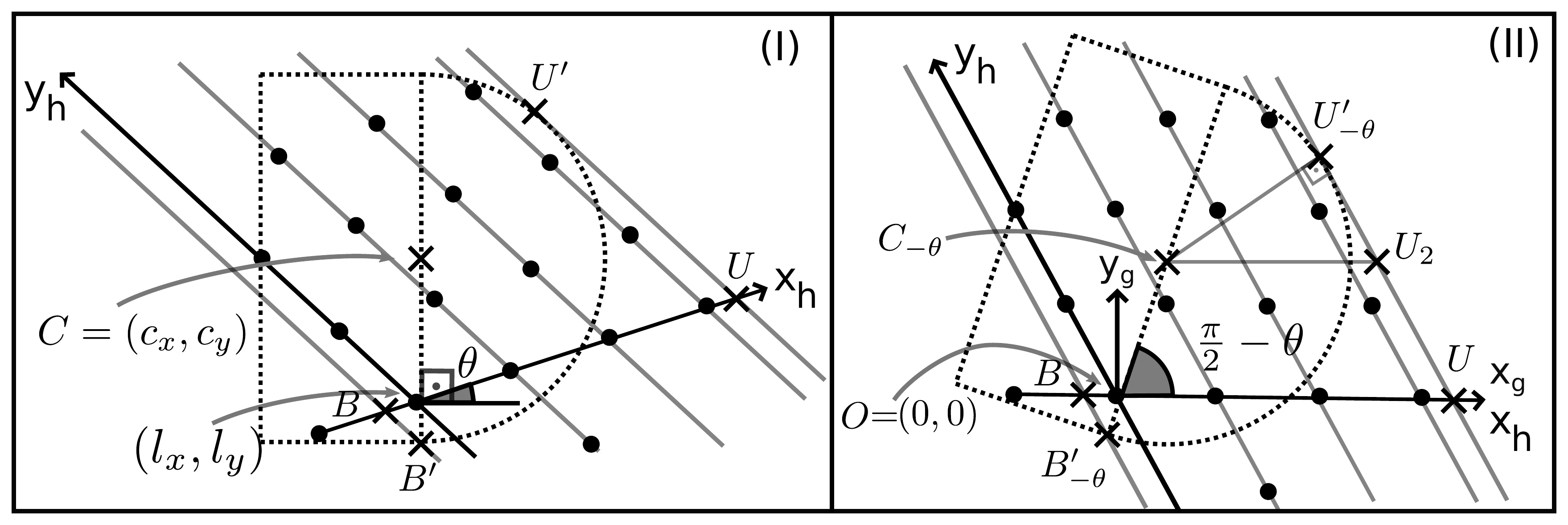}
      \caption{In the space (I) the robots are in the standard coordinate system and the semicircle with centre at $C = (c_{x},c_{y})$ has the lowest point at $B^{'}$. $\protect\overleftrightarrow{CB^{'}}$ has angle $\frac{\pi}{2}$ with the usual $x$-axis, however the $x_{h}$-axis here has angle $\theta$ with it. In (II) we rotate by $-\theta$ with $(l_{x},l_{y})$ as centre of rotation. After this rotation, $B^{'}$, $U^{'}$ and $C$ become $B^{'}_{-\theta}$, $U^{'}_{-\theta}$ and $C_{-\theta}$, respectively, and $\protect\overleftrightarrow{C_{-\theta}B^{'}_{-\theta}}$ has angle $\frac{\pi}{2}-\theta$ in relation to the $x_{g}$-axis and $x_{h}$-axis, which are now coincident lines despite their scale being different. $B$ and $U$ are the minimum and maximum values of the $x_{h}$-axis coordinate for a line parallel to the $y_{h}$-axis on the hexagonal grid coordinate system.}
      \label{fig:hexsemicircle1}
    \end{figure}
    
    Assume $T > \frac{s}{v}$, as shown in Figure \ref{fig:hexnumrobots} (III). The robots are located in the usual Euclidean space. Instead of it, we use in this proof a similar coordinate system transformation for positioning the robots in a hexagonal grid as integer coordinates as we did in the rectangular part. As in the previous lemmas, we call this coordinate system space coordinates $(x_{h},y_{h})$. However, here we are using a $(x_{h},y_{h})$ coordinate system with a different origin and inclination. 
    
    In order to do so, we first redefine a $(x_{g},y_{g})$ coordinate space, that is, we perform rotation by $-\theta$ on the usual Euclidean space about $(l_{x},l_{y})$. The origin of the $(x_{g},y_{g})$ coordinate system is at $(l_{x},l_{y})$. The transformation for $(x_{g},y_{g})$ coordinate system used here is similar to the depicted in the Figure \ref{fig:referencepsi}, but here we are using $-\theta$ and $(l_{x},l_{y})$ instead of $-\psi$ and $(x_{0},y_{0})$, i.e., 
    \if\shortVersion 1
      $
         \left[
        \begin{array}{c}
          x_{g} \\
          y_{g}
        \end{array}
        \right]
        = \left[
        \begin{array}{cc}
          \cos(-\theta) & -\sin(-\theta)\\
          \sin(-\theta) & \cos(-\theta)\\
        \end{array}
        \right]
        \left[
        \begin{array}{c}
          x-l_{x}\\
          y-l_{y}\\
        \end{array}
        \right].
      $
    \else
      $$
         \left[
        \begin{array}{c}
          x_{g} \\
          y_{g}
        \end{array}
        \right]
        = \left[
        \begin{array}{cc}
          \cos(-\theta) & -\sin(-\theta)\\
          \sin(-\theta) & \cos(-\theta)\\
        \end{array}
        \right]
        \left[
        \begin{array}{c}
          x-l_{x}\\
          y-l_{y}\\
        \end{array}
        \right].
      $$
    \fi
    As the coordinate space $(x_{g},y_{g})$ is already translated to the point $(l_{x},l_{y})$, the transformation from the new $(x_{h}, y_{h})$ to the new $(x_{g},y_{g})$ is the same as in (\ref{eq:xhyh2xgyg}). We repeat it below for convenience: 
    \begin{equation}
      \left[
      \begin{array}{c}
        x_{g} \\
        y_{g}
      \end{array}
      \right]
      = \left[
      \begin{array}{cc}
        d & -\frac{d}{2}\\
        0 & \frac{\sqrt{3}d}{2}\\ 
      \end{array}
      \right]
      \left[
      \begin{array}{c}
        x_{h}\\
        y_{h}
      \end{array}
      \right].
      \label{eq:xhyh2xgyglxly}
    \end{equation}
    Despite these differences,  we will keep using the notation $(x_{g},y_{g})$ and $(x_{h},y_{h})$ as we did before for a clean presentation.

    Figure \ref{fig:hexsemicircle1} shows how the semicircle with centre at $C = (c_{x},c_{y}) = (x_{0} + vT - s, y_{0})$ will be after the rotation by $-\theta$ about $(l_{x},l_{y})$, that is,
    \begin{equation}
      C_{-\theta} = \left[
        \begin{array}{cc}
          \cos(-\theta) & -\sin(-\theta)\\
          \sin(-\theta) & \cos(-\theta)\\
        \end{array}
        \right]
        \left[
        \begin{array}{c}
          c_{x} - l_{x}\\
          c_{y} - l_{y}\\
        \end{array}
        \right].
      \label{eq:centertheta}
    \end{equation}
    Hereafter we will use the subscript $-\theta$ on every point presented on the usual Euclidean space to denote the corresponding point on the $(x_{g},y_{g})$ coordinate space.

    We compute first the upper and lower values,  $U$ and $B$, of $x_{h}$ lying on the semicircle. For getting the $U$ value on the $x_{h}$-axis, we draw a line parallel to the $y_{h}$-axis on the rightmost semicircle boundary at the point $U^{'}$ in order to reach the $x_{h}$-axis (Figure \ref{fig:hexsemicircle1} (I)). The corresponding point on the $(x_{g},y_{g})$ space is denoted by $U^{'}_{-\theta}$ (Figure \ref{fig:hexsemicircle1} (II)).
    We compute $U^{'}_{-\theta}$, then we take its $x_{h}$-value on the hexagonal grid coordinate system. $\bigtriangleup U^{'}_{-\theta}C_{-\theta}U_{2}$ in Figure \ref{fig:hexsemicircle1} (II) has $\vert U^{'}_{-\theta}C_{-\theta}\vert =s$ and $\widehat{U^{'}_{-\theta}C_{-\theta}U_{2}} = \pi - \widehat{C_{-\theta}U^{'}_{-\theta}U_{2}} - \widehat{U^{'}_{-\theta}U_{2}C_{-\theta}} = \pi - \pi/2 - \pi/3 = \pi/6.$ Hence, 
    \if\shortVersion 1 
      $
        U^{'}_{-\theta} 
          = C_{-\theta} + s(\cos(\pi/6),\sin(\pi/6))
          =\big(\cos(\theta)(c_{x} - l_{x}) + \sin(\theta)(c_{y} - l_{y}) + \frac{\sqrt{3}s}{2}, 
          \cos(\theta)(c_{y} - l_{y}) -\sin(\theta)(c_{x} - l_{x})   + \frac{s}{2}\big).
      $
    \else
      $$
        \begin{aligned}
          U^{'}_{-\theta} 
            &= C_{-\theta} + s(\cos(\pi/6),\sin(\pi/6))\\
\ifexpandexplanation
            &= C_{-\theta} + \left(\frac{\sqrt{3}s}{2}, \frac{s}{2}\right) \\
            &= \left[
              \begin{array}{cc}
                \cos(-\theta) & -\sin(-\theta)\\
                \sin(-\theta) & \cos(-\theta)\\
              \end{array}
              \right]
              \left[
              \begin{array}{c}
                c_{x} - l_{x}\\
                c_{y} - l_{y}\\
              \end{array}
              \right]
             + \left[
               \begin{array}{c}
                 \frac{\sqrt{3}s}{2} \\ \frac{s}{2}
               \end{array}
               \right]\\
            &=\bigg(\cos(-\theta)(c_{x} - l_{x}) - \sin(-\theta)(c_{y} - l_{y})  + \frac{\sqrt{3}s}{2}, \\
            &\phantom{= \bigg(} \sin(-\theta)(c_{x} - l_{x}) + \cos(-\theta)(c_{y} - l_{y}) + \frac{s}{2}\bigg)\\
\fi
            &=\bigg(\cos(\theta)(c_{x} - l_{x}) + \sin(\theta)(c_{y} - l_{y}) + \frac{\sqrt{3}s}{2}, \\
            &\phantom{= \bigg(}  \cos(\theta)(c_{y} - l_{y}) -\sin(\theta)(c_{x} - l_{x})   + \frac{s}{2}\bigg).
        \end{aligned}
      $$
      
    \fi
    The inverse transformation from (\ref{eq:xhyh2xgyglxly}) is
    \begin{equation}
      \left[
        \begin{array}{c}
          x_{h}\\
          y_{h}\\
        \end{array}
      \right] = 
      \left[
        \begin{array}{cc}
          \frac{1}{d} & \frac{1}{\sqrt{3}d}\\
          0 & \frac{2}{\sqrt{3}d}
        \end{array}
      \right]
      \left[
        \begin{array}{c}
          x_{g}\\
          y_{g}\\
        \end{array}
      \right].
      \label{eq:xgyg2xhyh}
    \end{equation}
    
    Applying the transformation of (\ref{eq:xgyg2xhyh}) to the point $U^{'}_{-\theta}$ we get its $x_{h}$-axis coordinate
\ifexpandexplanation 
    $$
      \begin{aligned}
        U 
          &= \frac{1}{d}\left(\cos(\theta)(c_{x} - l_{x}) + \sin(\theta)(c_{y} - l_{y})  + \frac{\sqrt{3}s}{2} \right) + \\
            &\phantom{=}\ \frac{1}{\sqrt{3}d} \left(\cos(\theta)(c_{y} - l_{y}) -\sin(\theta)(c_{x} - l_{x}) + \frac{s}{2} \right)\\ 
      \end{aligned}
    $$
\else
    \begin{equation}
      \begin{aligned}
        U 
          &= \frac{1}{d}\left(\cos(\theta)(c_{x} - l_{x}) + \sin(\theta)(c_{y} - l_{y})  + \frac{\sqrt{3}s}{2} \right) + \\
            &\phantom{=}\ \frac{1}{\sqrt{3}d} \left(\cos(\theta)(c_{y} - l_{y}) -\sin(\theta)(c_{x} - l_{x}) + \frac{s}{2} \right)\\
\fi
\ifexpandexplanation 
    $$
      \begin{aligned}
          &= \frac{1}{d}\left(\cos(\theta)(c_{x} - l_{x}) + \sin(\theta)(c_{y} - l_{y})  \right) +  \frac{\sqrt{3}s}{2d} +\\
            &\phantom{=}\ \frac{1}{\sqrt{3}d} \left(\cos(\theta)(c_{y} - l_{y}) -\sin(\theta)(c_{x} - l_{x})   \right) + \frac{s}{2\sqrt{3}d}\\
      \end{aligned}
    $$
    $$
      \begin{aligned}
          &= \frac{1}{d}\left(\cos(\theta)(c_{x} - l_{x}) + \sin(\theta)(c_{y} - l_{y})  \right) +  \frac{3s+s}{2\sqrt{3}d}+\\
            &\phantom{=}\ \frac{1}{\sqrt{3}d} \left(\cos(\theta)(c_{y} - l_{y}) -\sin(\theta)(c_{x} - l_{x})   \right)\\
      \end{aligned}
    $$
    $$
      \begin{aligned}
          &= \frac{\cos(\theta)(c_{x} - l_{x}) + \sin(\theta)(c_{y} - l_{y})  }{d}  + \frac{\cos(\theta)(c_{y} - l_{y}) -\sin(\theta)(c_{x} - l_{x}) }{\sqrt{3}d}   +  \frac{2s}{\sqrt{3}d}\\
      \end{aligned}
    $$
    $$
      \begin{aligned}
          &= \frac{1}{\sqrt{3}d}(\sqrt{3}\cos(\theta)(c_{x} - l_{x}) + \sqrt{3}\sin(\theta)(c_{y} - l_{y})  + \cos(\theta)(c_{y} - l_{y}) \\ 
            &\phantom{=}\ -\sin(\theta)(c_{x} - l_{x}) ) +  \frac{2s}{\sqrt{3}d}\\
          &= \frac{2}{\sqrt{3}d}\Bigg(\frac{\sqrt{3}}{2}\cos(\theta)(c_{x} - l_{x}) + \frac{\sqrt{3}}{2}\sin(\theta)(c_{y} - l_{y})    \\
            &\phantom{=}\ + \frac{1}{2}\cos(\theta)(c_{y} - l_{y}) -\frac{1}{2}\sin(\theta)(c_{x} - l_{x}) \Bigg) +  \frac{2s}{\sqrt{3}d}\\
          &= \frac{2}{\sqrt{3}d}\Bigg(\Bigg(\frac{\sqrt{3}}{2}\cos(\theta)-\frac{1}{2}\sin(\theta)\Bigg)(c_{x} - l_{x}) + \Bigg(\frac{\sqrt{3}}{2}\sin(\theta)  \\ 
            &\phantom{=}\  + \frac{1}{2}\cos(\theta)\Bigg)(c_{y} - l_{y}) \Bigg) + \frac{2s}{\sqrt{3}d}\\
      \end{aligned}
    $$
    \begin{equation}
      \begin{aligned}
          &= \frac{2}{\sqrt{3}d}((\sin(\pi/3)\cos(\theta)-\cos(\pi/3)\sin(\theta))(c_{x} - l_{x}) + (\sin(\pi/3)\sin(\theta)  \\ 
            &\phantom{=}\ + \cos(\pi/3)\cos(\theta))(c_{y} - l_{y}) ) +  \frac{2s}{\sqrt{3}d}\\
          &= \frac{2}{\sqrt{3}d}(\sin(\pi/3-\theta)(c_{x} - l_{x}) + \cos(\pi/3-\theta)(c_{y} - l_{y}) ) +  \frac{2s}{\sqrt{3}d}\\
\fi
          &= \frac{2(\sin(\pi/3-\theta)(c_{x} - l_{x}) + \cos(\pi/3-\theta)(c_{y} - l_{y}) + s)}{\sqrt{3}d}\\
      \end{aligned}
      \label{eq:hexsemicircleU1}
    \end{equation}
    As we need the integer coordinate less or equal to this value, we apply the floor function to yield the desired result in (\ref{eq:eq58}).
    
    For getting the $B$ value on the $x_{h}$-axis, we draw a line parallel to the $y_{h}$-axis on the lower semicircle corner at the point $B^{'}$ in order to reach the $x_{h}$-axis (Figure \ref{fig:hexsemicircle1} (I)). We perform a calculation similar to the previous paragraph but using $B^{'}_{-\theta}$ (Figure \ref{fig:hexsemicircle1} (II)). We have $\widehat{C_{-\theta}OU}=\pi/2 - \theta$ (as this is the same angle of $\overleftrightarrow{C B'}$ with $x_{h}$-axis in Figure \ref{fig:hexsemicircle1} (I) which coincides with $x_{g}$-axis in the Figure \ref{fig:hexsemicircle1} (II)). Then, as the vector $C_{-\theta}B^{'}_{-\theta}$ is pointed downwards, it has negative angle with the $x_{g}$-axis, that is, $-\widehat{B_{-\theta}OU} = -(\pi - \widehat{C_{-\theta}OU}) = -(\pi - (\pi/2 - \theta)) = -\pi/2 - \theta$ with $x_{g}$-axis. Also,  $\left\vert \overline{C_{-\theta}B^{'}_{-\theta}}\right\vert = s$ . Consequently, 
    \if\shortVersion 1
      $
        \overrightarrow{C_{-\theta}B^{'}_{-\theta}} = B^{'}_{-\theta} - C_{-\theta} = s(\cos(-\pi/2 - \theta),\sin(-\pi/2 - \theta)) \LR 
      $
      $
          B^{'}_{-\theta} 
            = C_{-\theta} + s(\cos(-\pi/2 - \theta),\sin(-\pi/2 - \theta)) =(\cos(\theta)(c_{x} - l_{x}) + \sin(\theta)(c_{y} - l_{y} -s), \cos(\theta)(c_{y} - l_{y} -s) - \sin(\theta)(c_{x} - l_{x}) ).
      $
    \else
      $$
        \overrightarrow{C_{-\theta}B^{'}_{-\theta}} = B^{'}_{-\theta} - C_{-\theta} = s(\cos(-\pi/2 - \theta),\sin(-\pi/2 - \theta)) \LR \\
      $$
      $$
        \begin{aligned}
          B^{'}_{-\theta} 
            &= C_{-\theta} + s(\cos(-\pi/2 - \theta),\sin(-\pi/2 - \theta)) \\
\ifexpandexplanation
            &= C_{-\theta} + s(\sin(-\theta),-\cos(-\theta)) \\
            &= C_{-\theta} - s(\sin(\theta),\cos(\theta)) \\
            &= \left[
              \begin{array}{cc}
                \cos(-\theta) & -\sin(-\theta)\\
                \sin(-\theta) & \cos(-\theta)\\
              \end{array}
              \right]
              \left[
              \begin{array}{c}
                c_{x} - l_{x}\\
                c_{y} - l_{y}\\
              \end{array}
              \right]
              - \left[
               \begin{array}{c}
                 s\sin(\theta) \\ s\cos(\theta)
               \end{array}
               \right]\\
            &=(\cos(\theta)(c_{x} - l_{x}) + \sin(\theta)(c_{y} - l_{y})  -s\sin(\theta), \\
            &\phantom{= (}  \cos(\theta)(c_{y} - l_{y}) - \sin(\theta)(c_{x} - l_{x})  - s\cos(\theta))\\
\fi
            &=(\cos(\theta)(c_{x} - l_{x}) + \sin(\theta)(c_{y} - l_{y} -s), \\
            &\phantom{= (}  \cos(\theta)(c_{y} - l_{y} -s) - \sin(\theta)(c_{x} - l_{x}) ).\\
        \end{aligned}
      $$ 
    \fi
    Using (\ref{eq:xgyg2xhyh}) on $B^{'}_{-\theta}$ we get,
    \if\shortVersion 1
      $
          B  
            = \frac{1}{d}\left(\cos(\theta)(c_{x} - l_{x}) + \sin(\theta)(c_{y} - l_{y} -s)\right) +  \frac{1}{\sqrt{3}d}\left(\cos(\theta)(c_{y} - l_{y} - s) - \sin(\theta)(c_{x} - l_{x})\right) = \frac{2( \sin(\pi/3-\theta)(c_{x} - l_{x})  + \cos(\pi/3-\theta)(c_{y} - l_{y} -s) )}{\sqrt{3}d}. 
      $
    \else
      $$
        \begin{aligned}
          B  
            &= \frac{1}{d}\left(\cos(\theta)(c_{x} - l_{x}) + \sin(\theta)(c_{y} - l_{y} -s)\right) + \\
            &\phantom{=}  \frac{1}{\sqrt{3}d}\left(\cos(\theta)(c_{y} - l_{y} - s) - \sin(\theta)(c_{x} - l_{x})\right)\\
\ifexpandexplanation 
        \end{aligned}
      $$
      $$
        \begin{aligned}
            &=  \frac{\cos(\theta)(c_{x} - l_{x}) + \sin(\theta)(c_{y} - l_{y} -s)}{d} +   \frac{ \cos(\theta)(c_{y} - l_{y} - s) - \sin(\theta)(c_{x} - l_{x})}{\sqrt{3}d}  \\
            &= \frac{1}{\sqrt{3}d}\Big(\sqrt{3}\cos(\theta)(c_{x} - l_{x}) + \sqrt{3}\sin(\theta)(c_{y} - l_{y} -s) +    \cos(\theta)(c_{y} - l_{y} - s) - \\
              &\phantom{=}\  \sin(\theta)(c_{x} - l_{x})\Big) \\
        \end{aligned}
      $$
      $$
        \begin{aligned}
            &= \frac{2}{\sqrt{3}d}\Bigg(\frac{\sqrt{3}}{2}\cos(\theta)(c_{x} - l_{x}) + \frac{\sqrt{3}}{2}\sin(\theta)(c_{y} - l_{y} -s) +    \frac{1}{2}\cos(\theta)(c_{y} - l_{y} - s) - \\
              &\phantom{=}\ \frac{1}{2}\sin(\theta)(c_{x} - l_{x})\Bigg) \\
            &= \frac{2}{\sqrt{3}d}\Bigg(\frac{\sqrt{3}}{2}\cos(\theta)(c_{x} - l_{x}) - \frac{1}{2}\sin(\theta)(c_{x} - l_{x}) + \frac{\sqrt{3}}{2}\sin(\theta)(c_{y} - l_{y} -s) + \\
              &\phantom{=}\   \frac{1}{2}\cos(\theta)(c_{y} - l_{y} - s) \Bigg) \\
        \end{aligned}
      $$
      $$
        \begin{aligned}
            &= \frac{2}{\sqrt{3}d}( (\sin(\pi/3)\cos(\theta) - \cos(\pi/3)\sin(\theta))(c_{x} - l_{x})  + (\sin(\pi/3)\sin(\theta) + \\
              &\phantom{=}\  \cos(\pi/3)\cos(\theta))(c_{y} - l_{y} -s)  ) \\
\fi
            &= \frac{2( \sin(\pi/3-\theta)(c_{x} - l_{x})  + \cos(\pi/3-\theta)(c_{y} - l_{y} -s) )}{\sqrt{3}d}. \\
        \end{aligned}
      $$
    \fi
    Then, we apply the ceiling function on this value to get an integer coordinate greater or equal to it in order to obtain (\ref{eq:BcasesSC}) for $T > \frac{s}{v}$.
    
    On the hexagonal grid coordinate system, for each $x_{h}$ from $B$ to $U$, we need to find the minimum and maximum $y_{h}$ -- namely $Y_{1}^{S}(x_{h})$ and $Y_{2}^{S}(x_{h})$, respectively -- of a line parallel to $y_{h}$-axis intercepting the $x_{h}$-axis and lying on the semicircle. Depending on the angle of  $\overleftrightarrow{C_{-\theta}B^{'}_{-\theta}}$ with the $x_{h}$-axis, the minimum and maximum $y_{h}$ can be either on the semicircle arc or $\overleftrightarrow{C_{-\theta}B^{'}_{-\theta}}$.  Due to $\theta \in \lbrack 0,\pi/3 \rparen$, the angle of $\overleftrightarrow{C_{-\theta}B^{'}_{-\theta}}$  is in $(\frac{\pi}{6}, \frac{\pi}{2}]$. Thus, the minimum $y_{h}$ value is at the semicircle arc, otherwise the minimum angle of $\overleftrightarrow{C_{-\theta}B^{'}_{-\theta}}$ would be $2\pi/3$, which is the $y_{h}$-axis angle with the $x_{h}$-axis. However, the maximum $y_{h}$ value could be either on $\overleftrightarrow{C_{-\theta}B^{'}_{-\theta}}$ or on the circle, thus we take the lowest, since we want the $y_{h}$ value on the boundary of the semicircle. 
    
    Let $C_{1}(x_{h})$ and $C_{2}(x_{h})$ be functions that respectively return the lowest and the highest $y_{h}$ value at the circle centred at $C_{-\theta}$ and radius $s$ for a $x_{h}$ coordinate value of a parallel-to-$y_{h}$-axis line assuming it intersects the circle. Then, a point $(x_{g},y_{g})$ on the Euclidean space is on that circle if 
    \if\shortVersion 1
      $(x_{g} - C_{-\theta,x})^{2} + (y_{g} - C_{-\theta,y})^{2} = s^{2} \LR 
      \left(d x_{h} - \frac{d y_{h}}{2}  - C_{-\theta,x}\right)^{2} + \left(\frac{\sqrt{3}d y_{h}}{2} - C_{-\theta,y}\right)^{2} = s^{2},$
    \else
      $$(x_{g} - C_{-\theta,x})^{2} + (y_{g} - C_{-\theta,y})^{2} = s^{2} \LR$$ 
      $$\left(d x_{h} - \frac{d y_{h}}{2}  - C_{-\theta,x}\right)^{2} + \left(\frac{\sqrt{3}d y_{h}}{2} - C_{-\theta,y}\right)^{2} = s^{2},$$ 
    \fi
    by (\ref{eq:xhyh2xgyglxly}). 

    Isolating $y_{h}$ and solving the two degree polynomial we get
    \begin{equation}
      y_{h_{1}} = C_{1}(x_{h}) = 
       \frac{d {x_{h}}- {C_{-\theta,x}}  + \sqrt{3} {C_{-\theta,y}}   - \sqrt{\Delta(x_{h})}}{2  d} \text{ and }
       \label{eq:C1hex}
    \end{equation}
    \begin{equation}
      y_{h_{2}} = C_{2}(x_{h}) = 
       \frac{d {x_{h}}- {C_{-\theta,x}} + \sqrt{3} {C_{-\theta,y}}   + \sqrt{\Delta(x_{h})}}{2  d},
      \label{eq:C2hex}
    \end{equation}
    for
    $0 \le \Delta(x_{h}) =   4  s^{2} - \left(\sqrt{3}  {\left(d {x_{h}} -{C_{-\theta,x}} \right)} - C_{-\theta,y}\right)^{2}$. $\Delta(x_{h})$ cannot be negative, otherwise the lines would not intersect this circle, contradicting our assumption.
    
    We denote $L(x_{h})$ a function that returns the $y_{h}$ component of the line $\overleftrightarrow{C_{-\theta}B_{-\theta}}$ for a given $x_{h}$.  The $\overleftrightarrow{C_{-\theta}B_{-\theta}}$ equation for a point in the space $(x_{g},y_{g})$ 
    is 
    \if\shortVersion 1
      $\tan\left(\frac{\pi}{2} - \theta\right) = \frac{y_{g} - C_{-\theta,y}}{x_{g} - C_{-\theta,x}} \Rightarrow L(x_{h}) = y_{h} = \frac{\sin\left(\frac{\pi}{2} - \theta\right)(d x_{h}   - C_{-\theta,x}) + \cos\left(\frac{\pi}{2} - \theta\right)C_{-\theta,y} }{d \sin\left(\frac{5\pi}{6}-\theta\right)}.$
    \else
      $$\tan\left(\frac{\pi}{2} - \theta\right) = \frac{y_{g} - C_{-\theta,y}}{x_{g} - C_{-\theta,x}} \LR 
      \frac{\sin\left(\frac{\pi}{2} - \theta\right)}{\cos\left(\frac{\pi}{2} - \theta\right)} = \frac{y_{g} - C_{-\theta,y}}{x_{g} - C_{-\theta,x}} \LR$$
      $$\sin\left(\frac{\pi}{2} - \theta\right)\left(d x_{h} - \frac{d y_{h}}{2}  - C_{-\theta,x}\right) = \cos\left(\frac{\pi}{2} - \theta\right) \left(\frac{\sqrt{3}d y_{h}}{2} - C_{-\theta,y}\right) \LR$$ 
\ifexpandexplanation
      $$\sin\left(\frac{\pi}{2} - \theta\right)\left(d x_{h}   - C_{-\theta,x}\right) - \sin\left(\frac{\pi}{2} - \theta\right)\frac{d y_{h}}{2} = \cos\left(\frac{\pi}{2} - \theta\right) \frac{\sqrt{3}d y_{h}}{2} - \cos\left(\frac{\pi}{2} - \theta\right) C_{-\theta,y}\LR$$
      $$\frac{\sqrt{3}d y_{h}}{2}\cos\left(\frac{\pi}{2} - \theta\right) + \sin\left(\frac{\pi}{2}  - \theta\right)\frac{d y_{h}}{2} = \sin\left(\frac{\pi}{2} - \theta\right)(d x_{h}   - C_{-\theta,x}) + \cos\left(\frac{\pi}{2} - \theta\right)C_{-\theta,y} \LR$$ 
      $$d \left(\frac{\sqrt{3}}{2}\cos\left(\frac{\pi}{2} - \theta\right) + \sin\left(\frac{\pi}{2}  - \theta\right)\frac{1}{2}\right) y_{h} = \sin\left(\frac{\pi}{2} - \theta\right)(d x_{h}   - C_{-\theta,x}) + \cos\left(\frac{\pi}{2} - \theta\right)C_{-\theta,y} \LR$$ 
\fi
      $$y_{h} = \frac{\sin\left(\frac{\pi}{2} - \theta\right)(d x_{h}   - C_{-\theta,x}) + \cos\left(\frac{\pi}{2} - \theta\right)C_{-\theta,y} }{d \left(\frac{\sqrt{3}}{2}\cos\left(\frac{\pi}{2} - \theta\right) + \sin\left(\frac{\pi}{2}  - \theta\right)\frac{1}{2}\right)}\LR$$ 
\ifexpandexplanation
      $$y_{h} = \frac{\sin\left(\frac{\pi}{2} - \theta\right)(d x_{h}   - C_{-\theta,x}) + \cos\left(\frac{\pi}{2} - \theta\right)C_{-\theta,y} }{d \left(\sin\left(\frac{\pi}{3}\right)\cos\left(\frac{\pi}{2} - \theta\right) + \sin\left(\frac{\pi}{2}  - \theta\right)\cos\left(\frac{\pi}{3}\right)\right)}\LR$$
      $$y_{h} = \frac{\sin\left(\frac{\pi}{2} - \theta\right)(d x_{h}   - C_{-\theta,x}) + \cos\left(\frac{\pi}{2} - \theta\right)C_{-\theta,y} }{d \sin\left(\frac{\pi}{3}+\frac{\pi}{2} - \theta\right)}\LR$$ 
\fi
      $$L(x_{h}) = y_{h} = \frac{\sin\left(\frac{\pi}{2} - \theta\right)(d x_{h}   - C_{-\theta,x}) + \cos\left(\frac{\pi}{2} - \theta\right)C_{-\theta,y} }{d \sin\left(\frac{5\pi}{6}-\theta\right)}.$$

   \fi

    We have that $Y_{1}^{S}(x_{h}) = C_{1}(x_{h})$ and $Y_{2}^{S}(x_{h})$ can be either $\min(L(x_{h}),C_{2}(x_{h}))$ or $\min (L(x_{h}),C_{2}(x_{h})) - 1$. As $T > \frac{s}{v}$, we can have a number of robots inside the rectangle $N_{R}(T,\theta) \ge 1$. If, for some $x_{h}$, $Y'(x_{h}) = \min(L(x_{h}),C_{2}(x_{h})) = \lfloor L(x_{h})\rfloor$, then the robot on $\left(x_{h}, Y'(x_{h})\right)$ is on the line $\overleftrightarrow{C_{-\theta}B_{-\theta}^{'}}$. As this line belongs to the rectangle, the robot was already counted by $N_{R}(T,\theta)$.  Hence, 
    \if\shortVersion 1
      $
        Y_{2}^{S}(x_{h}) = 
          \min(L(x_{h}),C_{2}(x_{h})) - 1,$ 
        if  $\min(L(x_{h}),C_{2}(x_{h})) = \lfloor L(x_{h})\rfloor$ 
        and  $T > \frac{s}{v},$ otherwise,
          $Y_{2}^{S}(x_{h}) =\min(L(x_{h}),C_{2}(x_{h})).$
    \else
      $$
        Y_{2}^{S}(x_{h}) = \left\{
        \begin{array}{>{\displaystyle}c>{\displaystyle}l}
          \min(L(x_{h}),C_{2}(x_{h})) - 1, 
           & \text{ if } \min(L(x_{h}),C_{2}(x_{h})) = \lfloor L(x_{h})\rfloor \\ 
           & \phantom{if} \text{ and } T > \frac{s}{v},\\
          \min(L(x_{h}),C_{2}(x_{h})), 
           & \text{ otherwise. } 
        \end{array}
        \right.
      $$
    \fi
    
    The number of robots inside the semicircle is the number of integer coordinates $(x_{h},y_{h})$ for $x_{h}$ ranging from $B$ to $U$ and $y_{h} \in \left[\left\lceil Y_{1}^{S}(x_{h}) \right\rceil, \left\lfloor Y_{2}^{S}(x_{h}) \right\rfloor\right]$ for each $x_{h}$.  Thus,
    \if\shortVersion 1
      $N_{S}(T,\theta)  = \sum_{x_{h} = B}^{U} \left( \lfloor Y_{2}^{S}(x_{h}) \rfloor - \lceil Y_{1}^{S}(x_{h}) \rceil + 1 \right).$
    \else
      $$N_{S}(T,\theta) = \sum_{x_{h} = B}^{U} \sum_{y_{h} = \lceil Y_{1}^{S}(x_{h}) \rceil}^{\lfloor Y_{2}^{S}(x_{h}) \rfloor} 1 = \sum_{x_{h} = B}^{U} \left( \lfloor Y_{2}^{S}(x_{h}) \rfloor - \lceil Y_{1}^{S}(x_{h}) \rceil + 1 \right).$$
    \fi
    Heed that the last summation can only be used when $\left\lfloor Y_{2}^{S}(x_{h}) \right\rfloor \ge \left \lceil Y_{1}^{S}(x_{h}) \right \rceil $, otherwise a negative number of robots would be summed.

    \begin{figure}[t]
      \centering
      \includegraphics[width=0.84\columnwidth]{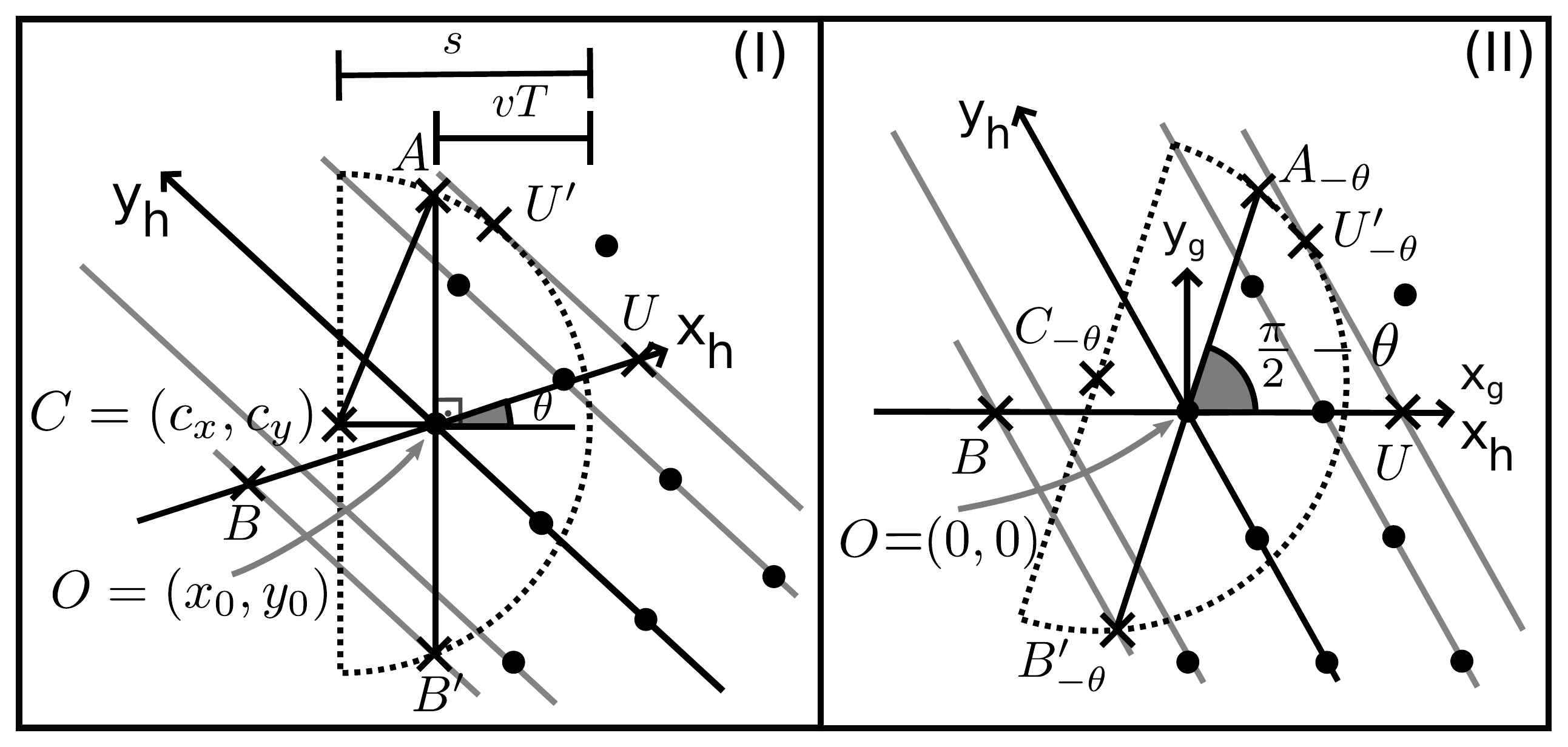} 
      \caption{Similar to the coordinate spaces of Figure \ref{fig:hexsemicircle1}, but for $T \le \frac{s}{v}$. The rotation and hexagonal grid system centres are now $(x_{0},y_{0})$. Notice also in (I) that $\bigtriangleup CAO$ is right with hypotenuse $\overline{CA}$ measuring $s$, and the horizontal cathetus $\overline{CO}$ measures $s - vT$.}
      \label{fig:hexsemicircle2}
    \end{figure}

    \begin{figure}[t]
      \centering
      \includegraphics[width=0.45\columnwidth]{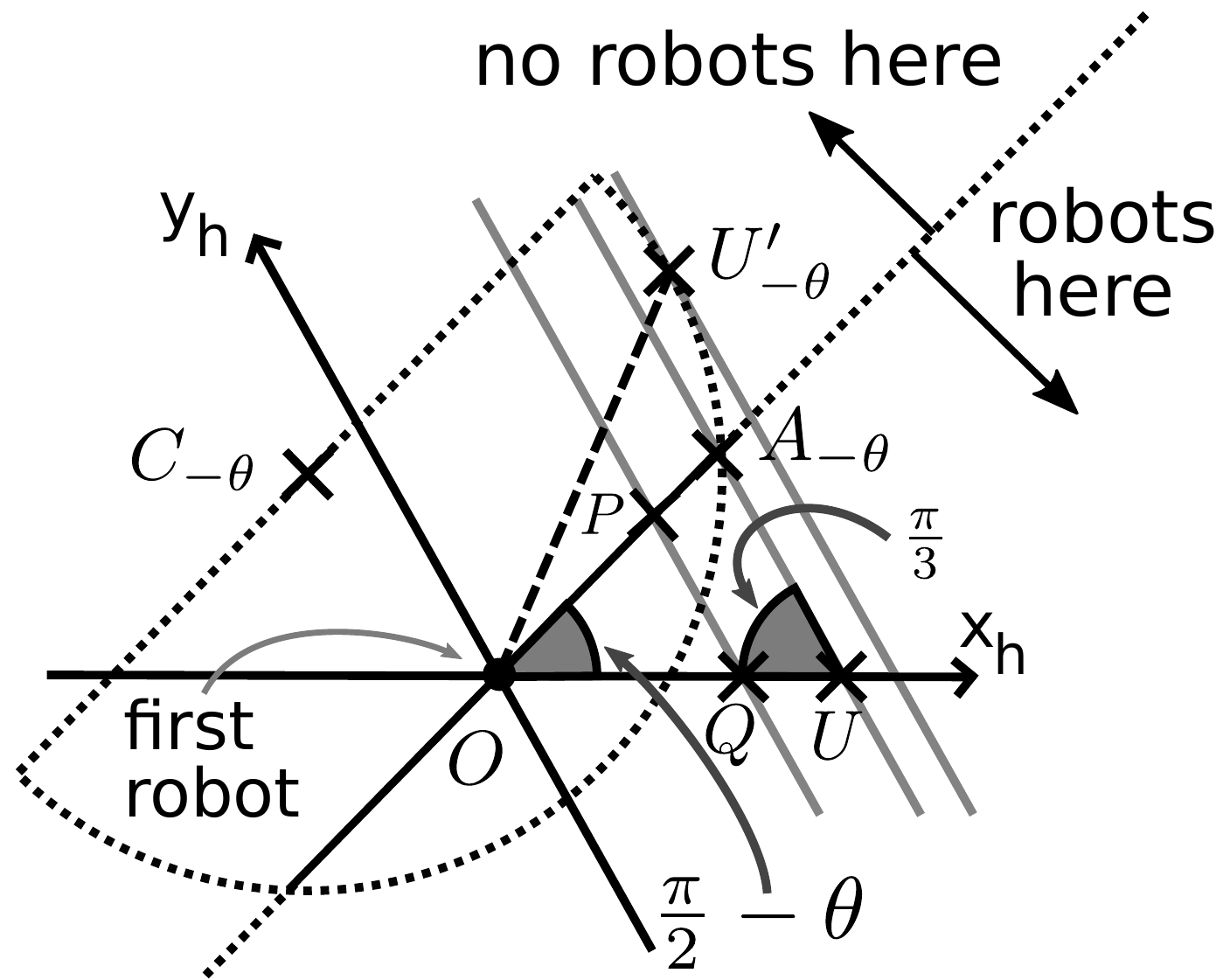}  
      \caption{An example of when the angle $\widehat{U^{'}_{-\theta}OU}$ is greater than $\widehat{A_{-\theta}OU}$. We only consider robots inside the semicircle below the line $\protect \overleftrightarrow{OA_{-\theta}}$, otherwise the robot on $O$ would not be the first robot by assumption. In this case, any line parallel to $y_{h}$-axis crossing the semicircle below $\protect\overline{OA_{-\theta}}$ must have its $x_{h}$-axis coordinate less than or equal to $U$, for example $Q$ projected from $P$.}
      \label{fig:hexsemicircleUA}
    \end{figure}

    Now, assume $T \le \frac{s}{v}$. Then, the semicircle has centre at $C = (c_{x},c_{y}) = (x_{0} - (s - vT) ,y_{0})$ (Figure \ref{fig:hexsemicircle2}). Now, as we do not have the rectangle part, we consider the \emph{last} robot of the rectangular part being the first robot to arrive at the target region, so $(l_{x},l_{y}) = (x_{0},y_{0})$, and, by (\ref{eq:centertheta}),
    \begin{equation}
      C_{-\theta} = \left[
        \begin{array}{cc}
          \cos(-\theta) & -\sin(-\theta)\\
          \sin(-\theta) & \cos(-\theta)\\
        \end{array}
        \right]
        \left[
        \begin{array}{c}
          c_{x} - x_{0}\\
          c_{y} - y_{0}\\
        \end{array}
        \right].
      \label{eq:centertheta2}
    \end{equation}
    
    In the usual Euclidean coordinate space before the rotation about $(x_{0},y_{0})$, we consider the line $\overleftrightarrow{OA}$ perpendicular to the $x$-axis at $O = (x_{0},y_{0})$. This line represents the perpendicular axis such that we wish to count all the robots from it to the arc of the semicircle on its right. From Figure \ref{fig:hexsemicircle2} (I), 
    \if\shortVersion 1
      $r = \vert \overline{AO}\vert = \sqrt{\vert \overline{CA}\vert ^{2} - \vert \overline{CO}\vert ^{2}} = \sqrt{s^{2} - (s - vT)^{2}} = \sqrt{2svT - (vT)^{2}}.$
    \else
      $$r = \vert \overline{AO}\vert = \sqrt{\vert \overline{CA}\vert ^{2} - \vert \overline{CO}\vert ^{2}} = \sqrt{s^{2} - (s - vT)^{2}} = \sqrt{2svT - (vT)^{2}}.$$
    \fi
    
    After the rotation by $-\theta$ about the point $O$, the maximum value for $x_{h}$ is defined by the point $U$. The point $U$ is chosen depending on the angles $\widehat{U^{'}_{-\theta}OU}$  and $\widehat{A_{-\theta}OU}$. When the angle $\widehat{U^{'}_{-\theta}OU}$ is greater than $\widehat{A_{-\theta}OU}$, the value of $U$ is calculated in relation to $A_{-\theta}$, because the line parallel to $y_{h}$-axis intercepting $U^{'}_{-\theta}$ is not inside the semicircle below $\overleftrightarrow{OA_{-\theta}}$ (Figure \ref{fig:hexsemicircleUA}). For comparison, Figure \ref{fig:hexsemicircle2} (II) shows an example where we choose $U$ as the $x_{h}$-axis intersection with the line parallel to $y_{h}$-axis at $U^{'}_{-\theta}$. As we saw before for the case $T > \frac{s}{v}$ (Figure \ref{fig:hexsemicircle1} (II)), the angle of $\overline{C_{-\theta}U^{'}_{\theta}}$ in relation to $x_{h}$-axis is $\pi/6$, consequently, 
    \if\shortVersion 1
      $
           U^{'}_{-\theta} 
            = C_{-\theta} + s\left(\cos\left(\frac{\pi}{6}\right),\sin\left(\frac{\pi}{6}\right)\right)
            = \left( \frac{\sqrt{3}s}{2} + \cos(\theta) (vT - s)  , \frac{s}{2}  - \sin(\theta) (vT - s)   \right),
      $
    \else
      $$ 
        \begin{aligned}
           U^{'}_{-\theta} 
             &= C_{-\theta} + s\left(\cos\left(\frac{\pi}{6}\right),\sin\left(\frac{\pi}{6}\right)\right)\\
\ifexpandexplanation
             &=
                \left[
                \begin{array}{cc}
                  \cos(-\theta) & -\sin(-\theta)\\
                  \sin(-\theta) & \cos(-\theta)\\
                \end{array}
                \right]
                \left[
                \begin{array}{c}
                  c_{x} - x_{0}\\
                  c_{y} - y_{0}\\
                \end{array}
                \right]
                + \left[
                 \begin{array}{c}
                   \frac{\sqrt{3}s}{2} \\ \frac{s}{2}
                 \end{array}
                 \right] \\
             &=
                \left[
                \begin{array}{cc}
                  \cos(\theta) & \sin(\theta)\\
                  -\sin(\theta) & \cos(\theta)\\
                \end{array}
                \right]
                \left[
                \begin{array}{c}
                  vT-s\\
                  0\\
                \end{array}
                \right]
                + \left[
                 \begin{array}{c}
                   \frac{\sqrt{3}s}{2} \\ \frac{s}{2}
                 \end{array}
                 \right]\\ 
\fi
            &= \left( \frac{\sqrt{3}s}{2} + \cos(\theta) (vT - s)  , \frac{s}{2}  - \sin(\theta) (vT - s)   \right),
        \end{aligned}
      $$
    \fi
    from (\ref{eq:centertheta2}), and $\widehat{U^{'}_{-\theta}OU}$ measures
    \if\shortVersion 1
      $
         \arctan\left( \frac{U^{'}_{-\theta,y}}{U^{'}_{-\theta,x}} \right)
            = \arctan\left( \frac{\frac{s}{2} - \sin(\theta) (vT - s)  }{\frac{\sqrt{3}s}{2} + \cos(\theta) (vT - s)}\right).
      $
    \else
      $$
        \begin{aligned}
          \arctan\left( \frac{U^{'}_{-\theta,y}}{U^{'}_{-\theta,x}} \right)
            &= \arctan\left( \frac{\frac{s}{2} - \sin(\theta) (vT - s)  }{\frac{\sqrt{3}s}{2} + \cos(\theta) (vT - s)}\right).
        \end{aligned}
      $$
    \fi
    
    $\widehat{A_{-\theta}OU}$ measures $\frac{\pi}{2} - \theta$, as show in Figure \ref{fig:hexsemicircle2} (II). Thence, 
    \if\shortVersion 1
      $A_{-\theta} = \big(r\cos\big(\frac{\pi}{2} - \theta\big), r\sin\big(\frac{\pi}{2} - \theta\big)\big).$
    \else
      $$A_{-\theta} = \left(r\cos\left(\frac{\pi}{2} - \theta\right), r\sin\left(\frac{\pi}{2} - \theta\right)\right).$$ 
    \fi
    If $\arctan\left( \frac{U^{'}_{-\theta,y}}{U^{'}_{-\theta,x}} \right) \le \widehat{A_{-\theta}OU} = \frac{\pi}{2} - \theta$, 
    we apply (\ref{eq:xgyg2xhyh}) on $U^{'}_{-\theta}$ to get its $x_{h}$-axis coordinate
    \if\shortVersion 1
       $
          U 
          = \frac{1}{d}\big(\frac{\sqrt{3}s}{2} + \cos(\theta) (vT - s)   \big) + \frac{1}{\sqrt{3}d}\big(\frac{s}{2} - \sin(\theta) (vT - s)   \big) = \frac{2\sin(\pi/3 - \theta)(vT - s)}{\sqrt{3}d}   + \frac{2s}{\sqrt{3}d},
      $
    \else
      $$
        \begin{aligned}
          U 
          &= \frac{1}{d}\left(\frac{\sqrt{3}s}{2} + \cos(\theta) (vT - s)   \right) + \frac{1}{\sqrt{3}d}\left(\frac{s}{2} - \sin(\theta) (vT - s)   \right)\\
\ifexpandexplanation
            &= \cos(\theta) \frac{vT - s}{d}  + \frac{\sqrt{3}s}{2d}   - \sin(\theta) \frac{vT - s}{\sqrt{3}d}   + \frac{s}{2\sqrt{3}d}\\
            &= \cos(\theta) \frac{vT - s}{d}     - \sin(\theta) \frac{vT - s}{\sqrt{3}d}  + \frac{\sqrt{3}\sqrt{3}s}{2\sqrt{3}d} + \frac{s}{2\sqrt{3}d}\\
            &= \cos(\theta) \frac{vT - s}{d}     - \sin(\theta) \frac{vT - s}{\sqrt{3}d}  + \frac{3s}{2\sqrt{3}d} + \frac{s}{2\sqrt{3}d}\\
            &= \cos(\theta) \frac{vT - s}{d}     - \sin(\theta) \frac{vT - s}{\sqrt{3}d}  + \frac{4s}{2\sqrt{3}d}\\
\fi
          &= \cos(\theta) \frac{vT - s}{d}     - \sin(\theta) \frac{vT - s}{\sqrt{3}d}   + \frac{2s}{\sqrt{3}d}\\
          &= \sqrt{3}\cos(\theta) \frac{vT - s}{\sqrt{3}d}     - \sin(\theta) \frac{vT - s}{\sqrt{3}d}   + \frac{2s}{\sqrt{3}d}\\
\ifexpandexplanation
            &= 2\left(\frac{\sqrt{3}}{2}\cos(\theta)  - \frac{1}{2}\sin(\theta) \right)\frac{vT - s}{\sqrt{3}d}   + \frac{2s}{\sqrt{3}d}\\
\fi
          &= \frac{2\sin(\pi/3 - \theta)(vT - s)}{\sqrt{3}d}   + \frac{2s}{\sqrt{3}d},\\
        \end{aligned}
      $$
    \fi
    followed by applying floor function to it, as we need the integer coordinate less or equal to this value. This is the same as (\ref{eq:hexsemicircleU1}) by using $(l_{x},l_{y}) = (x_{0},y_{0})$, then we also have (\ref{eq:eq58}) when $\arctan\left( \frac{\frac{s}{2} - \sin(\theta) (vT - s) }{\frac{\sqrt{3}s}{2} + \cos(\theta) (vT - s)} \right) \le  \frac{\pi}{2} - \theta$.
    
    If $\arctan\left( \frac{U^{'}_{-\theta,y}}{U^{'}_{-\theta,x} } \right) > \frac{\pi}{2} - \theta$, then there are no robots to consider on the parallel lines to $y_{h}$-axis between $U^{'}_{-\theta}$ and $A_{-\theta}$, otherwise the robot at $(x_{0},y_{0})$ would not be the first to arrive at the target region. Thus, if $\arctan\left( \frac{U^{'}_{-\theta,y} }{U^{'}_{-\theta,x}} \right) > \frac{\pi}{2} - \theta$, we use the $x_{h}$-coordinate for the point $A_{-\theta}$ on the hexagonal grid space, that is,
    \if\shortVersion 1
      $
          U 
          = \frac{1}{d}\big( r\cos\left(\frac{\pi}{2} - \theta\right)\big) + \frac{1}{\sqrt{3}d}\big(r\sin\left(\frac{\pi}{2} - \theta\right)\big)
          = \frac{2r}{\sqrt{3}d}\cos\left(\theta-\frac{\pi}{3}\right).
      $
    \else
      $$
        \begin{aligned}
          U 
          &= \frac{1}{d}\left( r\cos\left(\frac{\pi}{2} - \theta\right)\right) + \frac{1}{\sqrt{3}d}\left(r\sin\left(\frac{\pi}{2} - \theta\right)\right)\\
          &= \frac{2r}{\sqrt{3}d}\left(\frac{\sqrt{3}}{2} \cos\left(\frac{\pi}{2} - \theta\right) + \frac{1}{2}\sin\left(\frac{\pi}{2} - \theta\right)\right)\\
\ifexpandexplanation
            &= \frac{2r}{\sqrt{3}d}\left(\frac{\sqrt{3}}{2} \sin\left(\theta\right) + \frac{1}{2}\cos\left(\theta\right)\right)\\
\fi
          &= \frac{2r}{\sqrt{3}d}\cos\left(\theta-\frac{\pi}{3}\right).\\
        \end{aligned}
      $$
    \fi
    then we apply the floor function to yield the desired result in (\ref{eq:eq59}).

    Now we will find the minimum value for an integer $x_{h}$ such that a parallel-to-$y_{h}$-axis line is inside the semicircle and starting from the right of $\overleftrightarrow{OA}$ or on it.
    For the calculation of $B$, from Figure \ref{fig:hexsemicircle2} (II), similarly to how we previously did,
    \if\shortVersion 1
      $
          B^{'}_{-\theta} 
            = O + r(\cos(-(\pi/2 + \theta)),\sin(-(\pi/2 + \theta)) 
            = ( - r\sin(\theta),  - r\cos(\theta)),
      $
    \else
      $$
        \begin{aligned}
          B^{'}_{-\theta} 
            &
            = O + r(\cos(-(\pi/2 + \theta)),\sin(-(\pi/2 + \theta)) 
            \\ &= (r\cos(\pi/2+\theta), - r\sin(\pi/2 + \theta)) \\ &
            = ( - r\sin(\theta),  - r\cos(\theta)),
        \end{aligned}
      $$ 
    \fi
    and, by (\ref{eq:xgyg2xhyh}), as $B$ is the $x_{h}$-coordinate of the $B_{-\theta}$,
    \if\shortVersion 1
      $
          B 
            = \frac{1}{d}\left( - r\sin(\theta) \right) + \frac{1}{\sqrt{3}d}\left(-r\cos(\theta)\right)
            = -\frac{2r}{\sqrt{3}d}\sin\left(\theta + \frac{\pi}{6}\right).
      $
    \else
      $$
        \begin{aligned}
          B 
            &= \frac{1}{d}\left( - r\sin(\theta) \right) + \frac{1}{\sqrt{3}d}\left(-r\cos(\theta)\right)
            = -\frac{2r}{\sqrt{3}d}\left(   \frac{1}{2}\cos(\theta) + \frac{\sqrt{3}}{2}\sin(\theta) \right)\\
\ifexpandexplanation
              &= -\frac{2r}{\sqrt{3}d}\left(   \sin(\pi/6)\cos(\theta) + \cos(\pi/6)\sin(\theta) \right)\\
\fi
            &= -\frac{2r}{\sqrt{3}d}\sin\left(\theta + \frac{\pi}{6}\right).\\
        \end{aligned}
      $$
    \fi
    Also, we apply the ceiling function to yield the desired result in (\ref{eq:BcasesSC}).
    
    In this case, $C_{1}(x_{h})$ and $C_{2}(x_{h})$ are equal to (\ref{eq:C1hex}) and (\ref{eq:C2hex}), but $L(x_{h})$ is different from the previous case. The line $\overleftrightarrow{OA_{-\theta}}$ for a point $(x_{g},y_{g})$ in the Euclidean space is
    \if\shortVersion 1
      $ y_{g} = \tan\left(\frac{\pi}{2} - \theta\right)x_{g}  \Rightarrow $ 
      $  L(x_{h}) =  y_{h} = \frac{\sin\left(\frac{\pi}{2}-\theta\right)  x_{h}}{\sin\left( \frac{5\pi}{6}-\theta\right)}.$ 
    \else
      $$ y_{g} = \tan\left(\frac{\pi}{2} - \theta\right)x_{g}  \LR 
          y_{g} = \frac{\sin\left(\frac{\pi}{2} - \theta\right)} {\cos\left(\frac{\pi}{2} - \theta\right)} x_{g}  \LR$$
      $$\frac{\sqrt{3} d y_{h}}{2}  = \frac{\sin\left(\frac{\pi}{2}-\theta\right)}{\cos\left(\frac{\pi}{2}-\theta\right)} \left(d x_{h} - \frac{d y_{h}}{2} \right)  \LR $$
\ifexpandexplanation
        $$\frac{\sqrt{3} d y_{h}}{2}\cos\left(\frac{\pi}{2}-\theta\right)  + \sin\left(\frac{\pi}{2}-\theta\right)\frac{d y_{h}}{2} = \sin\left(\frac{\pi}{2}-\theta\right) d x_{h}   \LR$$
\fi
      $$ y_{h} = \frac{2\sin\left(\frac{\pi}{2}-\theta\right)  x_{h}}{\sqrt{3}\cos\left(\frac{\pi}{2}-\theta\right)  + \sin\left(\frac{\pi}{2}-\theta\right) }   \LR$$
\ifexpandexplanation
        $$y_{h} = \frac{\sin\left(\frac{\pi}{2}-\theta\right)  x_{h}}{\sin(\pi/3)\cos\left(\frac{\pi}{2}-\theta\right)  + \cos(\pi/3)\sin\left(\frac{\pi}{2}-\theta\right) }   \LR$$
        $$y_{h} = \frac{\sin\left(\frac{\pi}{2}-\theta\right)  x_{h}}{\sin(\pi/3 + \pi/2-\theta)}   \LR$$
\fi
      $$  L(x_{h}) =  y_{h} = \frac{\sin\left(\frac{\pi}{2}-\theta\right)  x_{h}}{\sin\left( \frac{5\pi}{6}-\theta\right)}.$$ 
    \fi 
    \fi %
  \end{proof}

  We have $\displaystyle \lim_{T \to \infty} f_{h}(T,\theta) = \lim_{T \to \infty} \frac{N_{R}(T,\theta)}{T} + \lim_{T \to \infty} \frac{N_{S}(T,\theta)-1}{T}$, by Definition \ref{def:throughput2}. As shown below, this limit needs only the rectangle part, because $N_{S}$ is limited by a semicircle with finite radius.
  
  \begin{lemma}
    \if\shortVersion 1
      $\lim_{T\to \infty} \frac{N_{S}(T,\theta)-1}{T} = 0.$
    \else
      $$\lim_{T\to \infty} \frac{N_{S}(T,\theta)-1}{T} = 0.$$
    \fi
    \label{lemma:limitinftyNS}
  \end{lemma}
  \begin{proof} 
    \ifithasappendixforlemmas %
      See Online Appendix.
    \else %
    As $T \to \infty$, we have $T > \frac{s}{v}$. By Lemma \ref{lemma:NS}, $c_{x} = x_{0} + vT - s$, which is the $x$-axis coordinate of the right side of the rectangle. We have that the robots are distant by $d$, so the last robot must be at most distant by $d$ from the point $(c_{x},y_{0})$. Hence, $x_{0} + vT - s - d \le l_{x} \le x_{0} + vT - s$, and $y_{0} - d \le l_{y} \le y_{0} + d$, so $0 = c_{x} - (x_{0} + vT - s) \le c_{x} - l_{x} \le c_{x} - (x_{0} + vT - s - d) = d$ and $-d \le y_{0} - l_{y} \le  d$. Then, $-d \le C_{-\theta,x}, C_{-\theta,y} \le d$. 
\ifexpandexplanation
    Also, we have that $\theta \in \lbrack 0,\pi/3 \rparen$, so $-1/2 \le \cos(\pi/3 - \theta) \le 1$.
\fi
    Thus,
    \if\shortVersion 1
      $
          B 
          =   \big\lceil\frac{2( \sin(\pi/3-\theta)(c_{x} - l_{x})  + \cos(\pi/3-\theta)(y_{0} - l_{y} -s) )}{\sqrt{3}d}\big\rceil
          \ge \big\lceil\frac{2(\cos(\pi/3-\theta)(-d -s) )}{\sqrt{3}d}\big\rceil 
          \ge  \big\lceil\frac{-2(1 +\frac{s}{d}) }{\sqrt{3}}\big\rceil 
          =  \big\lceil-\frac{2}{\sqrt{3}} -\frac{s}{\sqrt{3}d}\big\rceil 
          \ge  -\frac{2}{\sqrt{3}} -\frac{s}{\sqrt{3}d}, 
      $ and
      $
          U 
          = \big \lfloor \frac{2(\sin(\pi/3-\theta)(c_{x} - l_{x}) + \cos(\pi/3-\theta)(y_{0} - l_{y}) + s)}{\sqrt{3}d} \big \rfloor
          \le \big \lfloor \frac{2(\sin(\pi/3-\theta)d + \cos(\pi/3-\theta)d + s)}{\sqrt{3}d} \big \rfloor
          \le \big \lfloor \frac{2(2d + s)}{\sqrt{3}d} \big \rfloor
          \le  \frac{4}{\sqrt{3}} + \frac{2s}{\sqrt{3}d}, 
      $
   \else
      $$
        \begin{aligned}
          B 
          &=   \left\lceil\frac{2( \sin(\pi/3-\theta)(c_{x} - l_{x})  + \cos(\pi/3-\theta)(y_{0} - l_{y} -s) )}{\sqrt{3}d}\right\rceil\\
          &\ge \left\lceil\frac{2(\cos(\pi/3-\theta)(-d -s) )}{\sqrt{3}d}\right\rceil 
          = \left\lceil\frac{-2\cos(\pi/3-\theta)(1 +\frac{s}{d}) }{\sqrt{3}}\right\rceil 
          \ge  \left\lceil\frac{-2(1 +\frac{s}{d}) }{\sqrt{3}}\right\rceil \\
          &=  \left\lceil-\frac{2}{\sqrt{3}} -\frac{s}{\sqrt{3}d}\right\rceil 
          \ge  -\frac{2}{\sqrt{3}} -\frac{s}{\sqrt{3}d}, 
        \end{aligned}
      $$
\ifexpandexplanation 
      $$
        \begin{aligned}
          U 
          &= \left \lfloor \frac{2(\sin(\pi/3-\theta)(c_{x} - l_{x}) + \cos(\pi/3-\theta)(y_{0} - l_{y}) + s)}{\sqrt{3}d} \right \rfloor\\
          &\le \left \lfloor \frac{2(\sin(\pi/3-\theta)d + \cos(\pi/3-\theta)d + s)}{\sqrt{3}d} \right \rfloor
          \le \left \lfloor \frac{2(2d + s)}{\sqrt{3}d} \right \rfloor
            = \left \lfloor \frac{4}{\sqrt{3}} + \frac{2s}{\sqrt{3}d} \right \rfloor\\ & 
          \le  \frac{4}{\sqrt{3}} + \frac{2s}{\sqrt{3}d}, 
        \end{aligned}
      $$
\else
      $$
        \begin{aligned}
          U 
          &= \left \lfloor \frac{2(\sin(\pi/3-\theta)(c_{x} - l_{x}) + \cos(\pi/3-\theta)(y_{0} - l_{y}) + s)}{\sqrt{3}d} \right \rfloor\\
          &\le \left \lfloor \frac{2(\sin(\pi/3-\theta)d + \cos(\pi/3-\theta)d + s)}{\sqrt{3}d} \right \rfloor
          \le \left \lfloor \frac{2(2d + s)}{\sqrt{3}d} \right \rfloor
          \le  \frac{4}{\sqrt{3}} + \frac{2s}{\sqrt{3}d}, 
        \end{aligned}
      $$
\fi
    \fi
    and for any integer $x_{h} \in [B,U]$, as $\Delta(x_{h})$ cannot be negative,
    \if\shortVersion 1
      $
          0 \le \Delta(x_{h}) 
            =   4  s^{2} - \big(\sqrt{3}  {\big(d {x_{h}} -{C_{-\theta,x}} \big)} - C_{-\theta,y}\big)^{2} 
            \le   4  s^{2},
      $
      $
          \lceil Y_{1}^{S}(x_{h}) \rceil 
            \ge Y_{1}^{S}(x_{h})   
            = \frac{d {x_{h}}- {C_{-\theta,x}} + \sqrt{3} C_{-\theta,y}   - \sqrt{\Delta(x_{h})}}{2  d}  
            \ge \frac{d {x_{h}}- d - \sqrt{3} d   - 2s}{2  d}  
            = \frac{x_{h} - 1 - \sqrt{3}}{2} - \frac{s}{d} ,
      $
      and
      $
          \lfloor Y_{2}^{S}(x_{h}) \rfloor
             \le  Y_{2}^{S}(x_{h}) 
             \le \min(L(x_{h}), C_{2}(x_{h}))  \le C_{2}(x_{h})  
             = \frac{d {x_{h}}- {C_{-\theta,x}}  + \sqrt{3} C_{-\theta,y}  + \sqrt{\Delta(x_{h})}}{2  d}
             \le \frac{d {x_{h}} + d + \sqrt{3} d  + 2s}{2  d} 
               = \frac{x_{h}+1+\sqrt{3}}{2} + \frac{s}{d}.
      $
    \else
      $$
        \begin{aligned}
          0 \le \Delta(x_{h}) 
            =   4  s^{2} - \left(\sqrt{3}  {\left(d {x_{h}} -{C_{-\theta,x}} \right)} - C_{-\theta,y}\right)^{2} 
            \le   4  s^{2},
        \end{aligned}
      $$
      $$
        \begin{aligned}
          \lceil Y_{1}^{S}(x_{h}) \rceil 
            &\ge Y_{1}^{S}(x_{h})   
            = \frac{d {x_{h}}- {C_{-\theta,x}} + \sqrt{3} C_{-\theta,y}   - \sqrt{\Delta(x_{h})}}{2  d}  \\
            &\ge \frac{d {x_{h}}- d - \sqrt{3} d   - 2s}{2  d}  
            = \frac{x_{h} - 1 - \sqrt{3}}{2} - \frac{s}{d} ,
        \end{aligned}
      $$
      and
      $$
        \begin{aligned}
          \lfloor Y_{2}^{S}(x_{h}) \rfloor
             &\le  Y_{2}^{S}(x_{h}) 
             \le \min(L(x_{h}), C_{2}(x_{h}))  \le C_{2}(x_{h})  
             \\&= \frac{d {x_{h}}- {C_{-\theta,x}}  + \sqrt{3} C_{-\theta,y}  + \sqrt{\Delta(x_{h})}}{2  d}
             \le \frac{d {x_{h}} + d + \sqrt{3} d  + 2s}{2  d} 
             \\ & = \frac{x_{h}+1+\sqrt{3}}{2} + \frac{s}{d}.
        \end{aligned}
      $$
    \fi
    Thus,
    \if\shortVersion 1
      $
          0 \le N_{S}(T,\theta) = 
          \sum_{x_{h}=B}^{U} \big(\lfloor Y_{2}^{S}(x_{h}) \rfloor - \lceil Y_{1}^{S}(x_{h}) \rceil + 1 \big)
          \le \sum_{x_{h}=B}^{U}\big( \frac{x_{h}+1+\sqrt{3}}{2} + \frac{s}{d} - \big(\frac{x_{h} - 1 - \sqrt{3}}{2} - \frac{s}{d}\big) + 1 \big)  
          = \sum_{x_{h}=B}^{U}  \big(\frac{2s}{d} + \sqrt{3} + 2\big) 
          \le 
          \big( \frac{4}{\sqrt{3}} + \frac{2s}{\sqrt{3}d} -\big( -\frac{2}{\sqrt{3}} -\frac{s}{\sqrt{3}d}  \big) + 1\big)  \big(\frac{2s}{d} + \sqrt{3} + 2\big) 
          = \big( 2\sqrt{3} + \frac{\sqrt{3}s}{d}  + 1\big)  \big(\frac{2s}{d} + \sqrt{3} + 2\big) 
      $
      $
        \Rightarrow 0 = \lim_{T\to\infty} \frac{-1}{T} 
          \le \lim_{T\to \infty} \frac{N_{S}(T,\theta)-1}{T} 
          \le \lim_{T \to \infty} \frac{1}{T}\big(\big( 2\sqrt{3} + \frac{\sqrt{3}s}{d}  + 1\big)  \big(\frac{2s}{d} + \sqrt{3} + 2\big) -1\big) = 0.
      $
  \else
\ifexpandexplanation 
      $$
        \begin{aligned}
          \phantom{\Rightarrow } 0 &\le N_{S}(T,\theta) = 
          \sum_{x_{h}=B}^{U} \left(\lfloor Y_{2}^{S}(x_{h}) \rfloor - \lceil Y_{1}^{S}(x_{h}) \rceil + 1 \right)
          \\&\le \sum_{x_{h}=B}^{U}\left( \frac{x_{h}+1+\sqrt{3}}{2} + \frac{s}{d} - \left(\frac{x_{h} - 1 - \sqrt{3}}{2} - \frac{s}{d}\right) + 1 \right)  
            \\&= \sum_{x_{h}=B}^{U} \left(\frac{x_{h}+1+\sqrt{3}}{2} + \frac{s}{d} - \frac{x_{h} - 1 - \sqrt{3}}{2} + \frac{s}{d} + 1  \right)
            \\&= \sum_{x_{h}=B}^{U} \left(\frac{\sqrt{3}}{2} + \frac{1}{2} + \frac{2s}{d} + \frac{1}{2}  + \frac{\sqrt{3}}{2} + 1 \right) 
            \\&= \sum_{x_{h}=B}^{U}  \left(\frac{2s}{d} + \frac{\sqrt{3} + 1 + 1 + \sqrt{3}}{2} + 1 \right) \\
            &= \sum_{x_{h}=B}^{U}  \left(\frac{2s}{d} + \frac{2\sqrt{3} + 4}{2} \right) 
          \\
        \end{aligned}
      $$
      $$
        \begin{aligned}  
          &
          = \sum_{x_{h}=B}^{U}  \left(\frac{2s}{d} + \sqrt{3} + 2\right) 
            = (U - B + 1) \left(\frac{2s}{d} + \sqrt{3} + 2\right)
          \\&\le 
          \left( \frac{4}{\sqrt{3}} + \frac{2s}{\sqrt{3}d} -\left( -\frac{2}{\sqrt{3}} -\frac{s}{\sqrt{3}d}  \right) + 1\right)  \left(\frac{2s}{d} + \sqrt{3} + 2\right)\\ 
            &= \left( \frac{4}{\sqrt{3}} + \frac{2s}{\sqrt{3}d}  +\frac{2}{\sqrt{3}} +\frac{s}{\sqrt{3}d}   + 1\right)  \left(\frac{2s}{d} + \sqrt{3} + 2\right)\\
          &= \left( 2\sqrt{3} + \frac{\sqrt{3}s}{d}  + 1\right)  \left(\frac{2s}{d} + \sqrt{3} + 2\right) 
        \end{aligned}
      $$
\else 
      $$
        \begin{aligned}
          \phantom{\Rightarrow } 0 &\le N_{S}(T,\theta) = 
          \sum_{x_{h}=B}^{U} \left(\lfloor Y_{2}^{S}(x_{h}) \rfloor - \lceil Y_{1}^{S}(x_{h}) \rceil + 1 \right)
          \\&\le \sum_{x_{h}=B}^{U}\left( \frac{x_{h}+1+\sqrt{3}}{2} + \frac{s}{d} - \left(\frac{x_{h} - 1 - \sqrt{3}}{2} - \frac{s}{d}\right) + 1 \right)  
          \\&
          = \sum_{x_{h}=B}^{U}  \left(\frac{2s}{d} + \sqrt{3} + 2\right) 
          \\&\le 
          \left( \frac{4}{\sqrt{3}} + \frac{2s}{\sqrt{3}d} -\left( -\frac{2}{\sqrt{3}} -\frac{s}{\sqrt{3}d}  \right) + 1\right)  \left(\frac{2s}{d} + \sqrt{3} + 2\right)\\ 
          &= \left( 2\sqrt{3} + \frac{\sqrt{3}s}{d}  + 1\right)  \left(\frac{2s}{d} + \sqrt{3} + 2\right) 
        \end{aligned}
      $$
\fi
      $$
      \begin{aligned}
        \Rightarrow 0 = \lim_{T\to\infty} \frac{-1}{T} 
          &\le \lim_{T\to \infty} \frac{N_{S}(T,\theta)-1}{T} 
        \\
          &\le \lim_{T \to \infty} \frac{1}{T}\left(\left( 2\sqrt{3} + \frac{\sqrt{3}s}{d}  + 1\right)  \left(\frac{2s}{d} + \sqrt{3} + 2\right) -1\right) = 0.
      \end{aligned}
      $$
    \fi
    Hence, the result follows from the sandwich theorem.
    \fi %
  \end{proof}
  
  As we obtained that $\displaystyle\lim_{T\to \infty} \frac{N_{S}(T,\theta)-1}{T} = 0$, hereafter we only calculate the limit for the number of robots inside the rectangle. By Lemmas \ref{lemma:NR} to \ref{lemma:MiddleInterval}, if $n_{l}^{+}-1 < K'$  we have
  \if\shortVersion 1
    $
        \lim_{T\to \infty} f_{h}(T,\psi) 
        = \lim_{T\to \infty} \frac{1}{T} \sum_{x_{h}=-n_{l}^{-}}^{n_{l}^{-}} \big(\lfloor Y_{2}^{R}(x_{h}) \rfloor - \lceil Y_{1}^{R}(x_{h}) \rceil + 1 \big)
          + \lim_{T\to \infty} \frac{1}{T} \sum_{x_{h}=n_{l}^{-} + 1}^{n_{l}^{+}-1} \big( \lfloor Y_{2}^{R}(x_{h}) \rfloor - \lceil Y_{1}^{R}(x_{h}) \rceil + 1 \big),
    $
  \else
    $$
      \begin{aligned}
        \lim_{T\to \infty} f_{h}(T,\psi) 
        &= \lim_{T\to \infty} \frac{1}{T} \sum_{x_{h}=-n_{l}^{-}}^{n_{l}^{-}} \left(\lfloor Y_{2}^{R}(x_{h}) \rfloor - \lceil Y_{1}^{R}(x_{h}) \rceil + 1 \right)\\
          &+ \lim_{T\to \infty} \frac{1}{T} \sum_{x_{h}=n_{l}^{-} + 1}^{n_{l}^{+}-1} \left( \lfloor Y_{2}^{R}(x_{h}) \rfloor - \lceil Y_{1}^{R}(x_{h}) \rceil + 1 \right),
      \end{aligned}
    $$
  \fi
  otherwise,
  \if\shortVersion 1
    $
        \lim_{T\to \infty} f_{h}(T,\psi) 
        = \lim_{T\to \infty} \frac{1}{T} \sum_{x_{h}=-n_{l}^{-}}^{n_{l}^{-}} \big(\lfloor Y_{2}^{R}(x_{h}) \rfloor - \lceil Y_{1}^{R}(x_{h}) \rceil + 1 \big)
          + \lim_{T\to \infty} \frac{1}{T} \sum_{x_{h}=n_{l}^{-} + 1}^{K' -1} \big( \lfloor Y_{2}^{R}(x_{h}) \rfloor - \lceil Y_{1}^{R}(x_{h}) \rceil + 1 \big)
          + \lim_{T\to \infty} \frac{1}{T} \sum_{x_{h}=K'}^{n_{l}^{+}-1} \big( \lfloor Y_{2}^{R}(x_{h}) \rfloor - \lceil Y_{1}^{R}(x_{h}) \rceil + 1 \big).
    $
  \else
    $$
      \begin{aligned}
        \lim_{T\to \infty} f_{h}(T,\psi) 
        &= \lim_{T\to \infty} \frac{1}{T} \sum_{x_{h}=-n_{l}^{-}}^{n_{l}^{-}} \left(\lfloor Y_{2}^{R}(x_{h}) \rfloor - \lceil Y_{1}^{R}(x_{h}) \rceil + 1 \right)\\
          &+ \lim_{T\to \infty} \frac{1}{T} \sum_{x_{h}=n_{l}^{-} + 1}^{K' -1} \left( \lfloor Y_{2}^{R}(x_{h}) \rfloor - \lceil Y_{1}^{R}(x_{h}) \rceil + 1 \right)\\
          &+ \lim_{T\to \infty} \frac{1}{T} \sum_{x_{h}=K'}^{n_{l}^{+}-1} \left( \lfloor Y_{2}^{R}(x_{h}) \rfloor - \lceil Y_{1}^{R}(x_{h}) \rceil + 1 \right).\\
      \end{aligned}
    $$ 
  \fi
  To clarify, the third summation is zero in the case of $n_{l}^{+}-1 < K'$, while the second summation goes until $\min(n_{l}^{+}-1,K'-1)$ in both cases. Each one will be individually solved assuming $\psi \neq \pi/6$. Later, we will see that the final result holds for $\psi = \pi/6$ as well. The following lemmas will be useful soon.
  
  \begin{lemma} 
    Assume $\psi \neq \pi/6$.
    \if\shortVersion 1
      $
      \lim_{T\to \infty} \frac{1}{T} \sum_{x_{h}=-n_{l}^{-}}^{n_{l}^{-}} \left( \lfloor Y_{2}^{R}(x_{h}) \rfloor - \lceil Y_{1}^{R}(x_{h}) \rceil + 1 \right) = 0.
      $
    \else
      $$
      \lim_{T\to \infty} \frac{1}{T} \sum_{x_{h}=-n_{l}^{-}}^{n_{l}^{-}} \left( \lfloor Y_{2}^{R}(x_{h}) \rfloor - \lceil Y_{1}^{R}(x_{h}) \rceil + 1 \right) = 0.
      $$ 
    \fi
    \label{lemma:lim1st}
  \end{lemma} 
  \begin{proof} 
    \ifithasappendixforlemmas %
      See Online Appendix.
    \else %
    As for any $x$, $x - 1 < \lfloor x \rfloor \le x \le \lceil x \rceil < x + 1$,
    \if\shortVersion 1
      $
           \lim_{T\to \infty} \frac{1}{T}\sum_{x_{h}=-n_{l}^{-}}^{n_{l}^{-}} \big( Y_{2}^{R}(x_{h})   -  Y_{1}^{R}(x_{h}) -1  \big) 
           < \lim_{T\to \infty} \frac{1}{T} \sum_{x_{h}=-n_{l}^{-}}^{n_{l}^{-}} \big(\lfloor Y_{2}^{R}(x_{h}) \rfloor - \lceil Y_{1}^{R}(x_{h}) \rceil + 1 \big) 
           \le \lim_{T\to \infty} \frac{1}{T} \sum_{x_{h}=-n_{l}^{-}}^{n_{l}^{-}} \big( Y_{2}^{R}(x_{h})   -  Y_{1}^{R}(x_{h}) +1 \big).
      $
    \else
      $$
        \begin{aligned}
           &\lim_{T\to \infty} \frac{1}{T}\sum_{x_{h}=-n_{l}^{-}}^{n_{l}^{-}} \left( Y_{2}^{R}(x_{h})   -  Y_{1}^{R}(x_{h}) -1  \right) \\
           &< \lim_{T\to \infty} \frac{1}{T} \sum_{x_{h}=-n_{l}^{-}}^{n_{l}^{-}} \left(\lfloor Y_{2}^{R}(x_{h}) \rfloor - \lceil Y_{1}^{R}(x_{h}) \rceil + 1 \right) \\
           &\le \lim_{T\to \infty} \frac{1}{T} \sum_{x_{h}=-n_{l}^{-}}^{n_{l}^{-}} \left( Y_{2}^{R}(x_{h})   -  Y_{1}^{R}(x_{h}) +1 \right).
        \end{aligned}
      $$
    \fi
    By Lemma \ref{lemma:Interval1}, the first and last summations do not depend on $T$, so both sides have limit equal to 0. By the sandwich theorem, we have the result. 
    \fi %
  \end{proof}
  
  \begin{lemma}
    Assume $\psi \neq \pi/6$. For $K' = \left\lceil\frac{2 (vT-s) \cos(\psi - \pi/6) - 2s\sin(\vert \psi - \pi/6\vert ) }{\sqrt{3}d}\right\rceil$,
    \if\shortVersion 1
      $
          \lim_{T\to \infty} \frac{1}{T} \sum_{x_{h}=K'}^{n_{l}^{+}-1} \left( \lfloor Y_{2}^{R}(x_{h}) \rfloor - \lceil Y_{1}^{R}(x_{h}) \rceil + 1 \right)= 0.
      $
    \else
      $$
          \lim_{T\to \infty} \frac{1}{T} \sum_{x_{h}=K'}^{n_{l}^{+}-1} \left( \lfloor Y_{2}^{R}(x_{h}) \rfloor - \lceil Y_{1}^{R}(x_{h}) \rceil + 1 \right)= 0.
      $$
    \fi
    \label{lemma:lim2nd}
  \end{lemma} 
  \begin{proof}
    \ifithasappendixforlemmas %
      See Online Appendix.
    \else %
    If $K' > n_{l}^{+}-1$, this limit is already zero, so we focus this proof on the other case. We have, analogously to the previous lemma, 
    \begin{equation}
      \begin{aligned}
         &\lim_{T\to \infty} \frac{1}{T}\sum_{x_{h}=K'}^{n_{l}^{+}-1} \left( Y_{2}^{R}(x_{h})   -  Y_{1}^{R}(x_{h}) -1 \right)   \\
         &<\lim_{T\to \infty} \frac{1}{T} \sum_{x_{h}=K'}^{n_{l}^{+}-1} \left( \lfloor Y_{2}^{R}(x_{h}) \rfloor - \lceil Y_{1}^{R}(x_{h}) \rceil + 1 \right)  \\
         &\le\lim_{T\to \infty} \frac{1}{T} \sum_{x_{h}=K'}^{n_{l}^{+}-1}\left( Y_{2}^{R}(x_{h})   -  Y_{1}^{R}(x_{h}) +1 \right).
      \end{aligned}
      \label{eq:limit2ineq}
    \end{equation}
    
    For any constant $c$, we have 
    \begin{equation}
      \lim_{T \to \infty}\frac{1}{T}\sum_{x_{h} = K'}^{n_{l}^{+}-1}c = 0,
      \label{eq:limzerolemma101}
    \end{equation}
    because the number of $x_{h}$ indexes  in the summation is limited by a finite number of integer outcomes that depends on $T$. In other words, the number of indexes in the above summation is $n_{l}^{+} - K'$ such that $\frac{4s\sin(\vert \psi - \pi/6\vert ) }{\sqrt{3}d} - 1 < n_{l}^{+} - K' \le \frac{4s\sin(\vert \psi - \pi/6\vert ) }{\sqrt{3}d} + 1$. The last inequality is obtained by counting how many $x_{h}$ are used in the summation and knowing that $2y -1 <\lfloor x + y \rfloor - \lceil x - y \rceil + 1 \le 2y + 1$ for any $x,y \in \Real$. Thus, for any $T$, $n_{l}^{+} - K'$  can only range from $\left\lceil\frac{4s\sin(\vert \psi - \pi/6\vert ) }{\sqrt{3}d}\right\rceil -1$ to $\left\lfloor\frac{4s\sin(\vert \psi - \pi/6\vert ) }{\sqrt{3}d}\right\rfloor + 1$. This yields to three possible integer numbers, if $\frac{4s\sin(\vert \psi - \pi/6\vert ) }{\sqrt{3}d} \in \Zeta$, or four, otherwise. Thus, a finite range of outcomes, none of them having $T$. Hence, for all outcomes, the limit on the left side of (\ref{eq:limzerolemma101}) is zero.

    Assume $\psi > \pi/6$ (for $\psi < \pi/6$ the result is the same).  From Lemma \ref{lemma:endcase}, 
    \begin{equation}
      \begin{aligned}
         &Y_{2}^{R}(x_{h}) - Y_{1}^{R}(x_{h}) 
         = 
          \frac{\frac{2 (v T-s)}{d \cos(\psi)}-2 x_{h}}{\sqrt{3} \tan(\psi) - 1}-\frac{\frac{2y_1}{d} + 2\tan(\psi) x_h}{\sqrt{3} + \tan(\psi)} \\
         &= \frac{\frac{2 (v T-s)}{d \cos(\psi)}}{\sqrt{3} \tan(\psi) - 1}-\frac{\frac{2y_1}{d}}{\sqrt{3} + \tan(\psi)} - \Bigg(\frac{2}{\sqrt{3} \tan(\psi) - 1} \\
         &+ \frac{2\tan(\psi) }{\sqrt{3} + \tan(\psi)}\Bigg) x_{h}.
      \end{aligned}
      \label{eq:t1y2y1xh}
    \end{equation}
    
    For the second term above, by (\ref{eq:limzerolemma101}), 
    $ \displaystyle
      \lim_{T\to\infty}\sum_{x_{h}=K'}^{n_{l}^{+}-1} \frac{\frac{2y_1}{d}}{\sqrt{3} + \tan(\psi)} = 0.
    $
    
    For the first term,
    \begin{equation}
      \begin{aligned}
        &\sum_{x_{h}=K'}^{n_{l}^{+}-1}  \frac{1}{T}\frac{\frac{2 (v T-s)}{d \cos(\psi)}}{\sqrt{3} \tan(\psi) - 1} 
        = \sum_{x_{h}=K'}^{n_{l}^{+}-1}   \frac{2 \left(v -\frac{s}{T}\right)}{d \cos(\psi)(\sqrt{3} \tan(\psi) - 1)}  \\
        &= \sum_{x_{h}=K'}^{n_{l}^{+}-1}    \frac{2 \left(v -\frac{s}{T}\right)}{d (\sqrt{3} \sin(\psi) - \cos(\psi))} 
        =\sum_{x_{h}=K'}^{n_{l}^{+}-1}   \frac{v -\frac{s}{T}}{d\sin(\psi-\pi/6)}\\
        &=\sum_{x_{h}=K'}^{n_{l}^{+}-1}   \frac{v}{d\sin(\psi-\pi/6)}  -\frac{1}{T}\sum_{x_{h}=K'}^{n_{l}^{+}-1} \frac{s}{d\sin(\psi-\pi/6)},
      \end{aligned}
      \label{eq:t102l7}
    \end{equation}
    due to $\frac{\sqrt{3}}{2} \sin(\psi) - \frac{1}{2}\cos(\psi) = \sin(\psi - \pi/6)$. Let $L$ be the number of terms on the summation of (\ref{eq:t102l7}). As discussed above, $L$ is an integer in $\Big\{\left\lceil\frac{4s\sin(\vert \psi - \pi/6\vert ) }{\sqrt{3}d}\right\rceil -1, \dots, \left\lfloor\frac{4s\sin(\vert \psi - \pi/6\vert ) }{\sqrt{3}d}\right\rfloor + 1 \Big\}$, so
    \begin{equation}
      \begin{aligned}
        &\sum_{x_{h}=K'}^{n_{l}^{+}-1}  \frac{1}{T}\frac{\frac{2 (v T-s)}{d \cos(\psi)}}{\sqrt{3} \tan(\psi) - 1} 
        =\frac{Lv}{d\sin(\psi-\pi/6)} - \frac{1}{T}\sum_{x_{h}=K'}^{n_{l}^{+}-1} \frac{s}{d\sin(\psi-\pi/6)}.
      \end{aligned}
      \label{eq:t102l72}
    \end{equation}

    Also, we have
    \if\shortVersion 1 
      $
           \frac{2}{\sqrt{3} \tan(\psi) - 1} + \frac{2\tan(\psi) }{\sqrt{3} + \tan(\psi)} 
            = \frac{\sqrt{3}}{2\sin(\psi - \pi/6) \cos(\psi - \pi/6)} 
      $
    \else
\ifexpandexplanation
      $$
        \begin{aligned}
           &\frac{2}{\sqrt{3} \tan(\psi) - 1} + \frac{2\tan(\psi) }{\sqrt{3} + \tan(\psi)} 
             = \frac{2(\sqrt{3} + \tan(\psi)) + 2\tan(\psi)(\sqrt{3} \tan(\psi) - 1) }{(\sqrt{3} \tan(\psi) - 1)(\sqrt{3} + \tan(\psi))} \\
               &= \frac{2\sqrt{3} + 2\tan(\psi) + 2\sqrt{3}\tan^{2}(\psi)  - 2 \tan(\psi) }{(\sqrt{3} \tan(\psi) - 1)(\sqrt{3} + \tan(\psi))} \\
               &= \frac{2\sqrt{3} + 2\sqrt{3}\tan^{2}(\psi)   }{(\sqrt{3} \tan(\psi) - 1)(\sqrt{3} + \tan(\psi))} \\
             &= \frac{2\sqrt{3}( 1 + \tan^{2}(\psi))   }{(\sqrt{3} \tan(\psi) - 1)(\sqrt{3} + \tan(\psi))} 
               \\&= \frac{2\sqrt{3}\sec^{2}(\psi)   }{(\sqrt{3} \tan(\psi) - 1)(\sqrt{3} + \tan(\psi))} \\&
             = \frac{2\sqrt{3}   }{(\sqrt{3} \sin(\psi) - \cos(\psi))(\sqrt{3}\cos(\psi) + \sin(\psi))} \\
             &= \frac{\sqrt{3}}{2\sin(\psi - \pi/6) \cos(\psi - \pi/6)} \\
        \end{aligned}
      $$ 
\else 
      $$
        \begin{aligned}
           &\frac{2}{\sqrt{3} \tan(\psi) - 1} + \frac{2\tan(\psi) }{\sqrt{3} + \tan(\psi)} 
             = \frac{2(\sqrt{3} + \tan(\psi)) + 2\tan(\psi)(\sqrt{3} \tan(\psi) - 1) }{(\sqrt{3} \tan(\psi) - 1)(\sqrt{3} + \tan(\psi))} \\
             &= \frac{2\sqrt{3}( 1 + \tan^{2}(\psi))   }{(\sqrt{3} \tan(\psi) - 1)(\sqrt{3} + \tan(\psi))} 
             = \frac{2\sqrt{3}   }{(\sqrt{3} \sin(\psi) - \cos(\psi))(\sqrt{3}\cos(\psi) + \sin(\psi))} \\
             &= \frac{\sqrt{3}}{2\sin(\psi - \pi/6) \cos(\psi - \pi/6)} \\
        \end{aligned}
      $$ 
\fi
    \fi
    as $1 + \tan^{2}(\psi) = \sec^{2}(\psi)$  and $\frac{\sqrt{3}}{2}\cos(\psi) + \frac{1}{2}\sin(\psi) = \cos(\psi - \pi/6)$. Hence, for the last term in (\ref{eq:t1y2y1xh}),
    \begin{equation}
    \begin{aligned}
      &\frac{1}{T}\sum_{x_{h}=K'}^{n_{l}^{+}-1} \frac{\sqrt{3}}{2\sin(\psi - \pi/6) \cos(\psi - \pi/6)}x_{h} 
        \\
      &=\frac{1}{T}\frac{\sqrt{3}}{2\sin(\psi - \pi/6) \cos(\psi - \pi/6)} \frac{(n_{l}^{+}-1 + K') (n_{l}^{+}-K')}{2}       \\
      &=\frac{\sqrt{3}LG}{4T\sin(\psi - \pi/6) \cos(\psi - \pi/6)},       \\
      \end{aligned} 
    \label{eq:t101l7}
    \end{equation}
    for an integer $G = n_{l}^{+} - 1 + K'$.  As $2x - 1 <\lfloor x + y \rfloor + \lceil x - y \rceil < 2x+1$ for any $x,y \in \Real$,  $G \in \left(\frac{4 (vT-s) \cos(\psi - \pi/6)}{\sqrt{3}d}-1,\frac{4 (vT-s) \cos(\psi - \pi/6)}{\sqrt{3}d}+1\right)$.

    For the lowest bound on G, using (\ref{eq:t102l72}) and (\ref{eq:t101l7})
    \if\shortVersion 1
      $
          \lim_{T\to \infty} \frac{1}{T} \sum_{x_{h}=K'}^{n_{l}^{+}-1} \big(  Y_{2}^{R}(x_{h})  - Y_{1}^{R}(x_{h}) \big)
          =\lim_{T\to \infty} \frac{1}{T} \sum_{x_{h}=K'}^{n_{l}^{+}-1} 
           \bigg( \frac{\frac{2 (v T-s)}{d \cos(\psi)}}{\sqrt{3} \tan(\psi) - 1} - \big(\frac{2}{\sqrt{3} \tan(\psi) - 1} + \frac{2\tan(\psi) }{\sqrt{3} + \tan(\psi)}\big) x_{h} \bigg)
          =\lim_{T\to \infty} \bigg(\frac{Lv}{d\sin(\psi-\pi/6)}  - \frac{1}{T}\sum_{x_{h}=K'}^{n_{l}^{+}-1} \frac{s}{d\sin(\psi-\pi/6)}  \big. 
          \phantom{=} \big. -\frac{\sqrt{3}L\big(\frac{4 (vT-s) \cos(\psi - \pi/6)}{\sqrt{3}d}-1\big)}{4T\sin(\psi - \pi/6) \cos(\psi - \pi/6)} \bigg) 
          =\lim_{T\to \infty} \bigg( \frac{\sqrt{3}L}{4T\sin(\psi - \pi/6) \cos(\psi - \pi/6)} - \frac{1}{T}\sum_{x_{h}=K'}^{n_{l}^{+}-1} \frac{s}{d\sin(\psi-\pi/6)} \bigg)
          = 0,
      $
    \else
      $$
        \begin{aligned}
          &\lim_{T\to \infty} \frac{1}{T} \sum_{x_{h}=K'}^{n_{l}^{+}-1} \left(  Y_{2}^{R}(x_{h})  - Y_{1}^{R}(x_{h}) \right)\\
          &=\lim_{T\to \infty} \frac{1}{T} \sum_{x_{h}=K'}^{n_{l}^{+}-1} 
           \left( \frac{\frac{2 (v T-s)}{d \cos(\psi)}}{\sqrt{3} \tan(\psi) - 1} - \left(\frac{2}{\sqrt{3} \tan(\psi) - 1} + \frac{2\tan(\psi) }{\sqrt{3} + \tan(\psi)}\right) x_{h} \right)\\
          &=\lim_{T\to \infty} \left(\frac{Lv}{d\sin(\psi-\pi/6)}  - \frac{1}{T}\sum_{x_{h}=K'}^{n_{l}^{+}-1} \frac{s}{d\sin(\psi-\pi/6)}  \right. \\
          &\phantom{=} \left. -\frac{\sqrt{3}L\left(\frac{4 (vT-s) \cos(\psi - \pi/6)}{\sqrt{3}d}-1\right)}{4T\sin(\psi - \pi/6) \cos(\psi - \pi/6)} \right) \\
        \end{aligned}
      $$
\ifexpandexplanation
      $$
        \begin{aligned}
            &=\lim_{T\to \infty} \left( \frac{Lv}{d\sin(\psi-\pi/6)} -\frac{L\frac{4 (vT-s) \cos(\psi - \pi/6)}{d}-\sqrt{3}L}{4T\sin(\psi - \pi/6) \cos(\psi - \pi/6)} - \frac{1}{T}\sum_{x_{h}=K'}^{n_{l}^{+}-1} \frac{s}{d\sin(\psi-\pi/6)} \right)\\ 
        \end{aligned}
      $$
      $$
        \begin{aligned}
            &=\lim_{T\to \infty} \left( \frac{Lv}{d\sin(\psi-\pi/6)} -\frac{Lv }{d\sin(\psi - \pi/6)} + \frac{\sqrt{3}L}{4T\sin(\psi - \pi/6) \cos(\psi - \pi/6)} - \right.\\
              &\phantom{=} \left. \frac{1}{T}\sum_{x_{h}=K'}^{n_{l}^{+}-1} \frac{s}{d\sin(\psi-\pi/6)} \right)\\
        \end{aligned}
      $$
\fi
      $$
        \begin{aligned}
          &=\lim_{T\to \infty} \left( \frac{\sqrt{3}L}{4T\sin(\psi - \pi/6) \cos(\psi - \pi/6)} - \frac{1}{T}\sum_{x_{h}=K'}^{n_{l}^{+}-1} \frac{s}{d\sin(\psi-\pi/6)} \right)
          = 0,\\
        \end{aligned}
      $$
    \fi
    due to (\ref{eq:limzerolemma101}) on the second term and, as 
    \if\shortVersion 1 
     $L\in \Big\{\Big\lceil\frac{4s\sin(\vert \psi - \pi/6\vert ) }{\sqrt{3}d}\Big\rceil -1, \dots, \Big\lfloor\frac{4s\sin(\vert \psi - \pi/6\vert ) }{\sqrt{3}d}\Big\rfloor + 1 \Big\},$
    \else
      $$L\in \Big\{\Big\lceil\frac{4s\sin(\vert \psi - \pi/6\vert ) }{\sqrt{3}d}\Big\rceil -1, \dots, \Big\lfloor\frac{4s\sin(\vert \psi - \pi/6\vert ) }{\sqrt{3}d}\Big\rfloor + 1 \Big\},$$ 
    \fi
    no element in this finite set has the term $T$.
    
    For the highest bound on G, we have the same limit. Hence, by the sandwich theorem applied on the results for both bounds of G, we get 
    \begin{equation}
      \begin{aligned}
        &\lim_{T\to \infty} \frac{1}{T} \sum_{x_{h}=K'}^{n_{l}^{+}-1} \left(  Y_{2}^{R}(x_{h})  - Y_{1}^{R}(x_{h}) \right)
       =0.
      \end{aligned}
      \label{eq:limity2y1zero1}
    \end{equation}
    Using (\ref{eq:limity2y1zero1}) and (\ref{eq:limzerolemma101}) on the bounds of (\ref{eq:limit2ineq}) and the sandwich theorem again concludes with the desired value.
    \fi %
  \end{proof}

  \begin{lemma} 
    Assume $\psi \neq \pi/6$.
    \if\shortVersion 1
      $
        \lim_{T\to \infty} \frac{1}{T} \sum_{x_{h}=n_{l}^{-} + 1}^{\min(n_{l}^{+}-1,K'-1)} \big(\lfloor Y_{2}^{R}(x_{h}) \rfloor - \lceil Y_{1}^{R}(x_{h}) \rceil + 1\big)
      $
    \else
      $$
        \lim_{T\to \infty} \frac{1}{T} \sum_{x_{h}=n_{l}^{-} + 1}^{\min(n_{l}^{+}-1,K'-1)} \left(\lfloor Y_{2}^{R}(x_{h}) \rfloor - \lceil Y_{1}^{R}(x_{h}) \rceil + 1\right)
      $$
    \fi
    exists and is bounded by
    \if\shortVersion 1
      $
          \big(\frac{4vs}{\sqrt{3}d^{2}} - \frac{2 v  \cos(\psi - \pi/6)}{\sqrt{3}d}, \frac{4vs}{\sqrt{3}d^{2}} +  \frac{2 v  \cos(\psi - \pi/6)}{\sqrt{3}d}\big].
      $
    \else
      $$
        \begin{aligned} 
          \left(\frac{4vs}{\sqrt{3}d^{2}} - \frac{2 v  \cos(\psi - \pi/6)}{\sqrt{3}d}, \frac{4vs}{\sqrt{3}d^{2}} +  \frac{2 v  \cos(\psi - \pi/6)}{\sqrt{3}d}\right].
        \end{aligned}
      $$ 
    \fi
    \label{lemma:lim3rd}
  \end{lemma}
  \begin{proof}
  \ifithasappendixforlemmas %
      See Online Appendix.
  \else %
  The next lemmas will be useful for proving this lemma.

\begin{lemma}
  For any $a,b >0, a\lfloor x \rfloor - b\lfloor y \rfloor < ax - by + a + b$.
  \label{lemma:flooraxby}
\end{lemma}
\begin{proof}  
As mentioned before, by the definition of floor function $\lfloor x \rfloor = x - frac(x)$,  where $frac$ is the function that returns the fractional part of the number $x$, such that $0 \le frac(x) < 1$ \citep{graham1994concrete},
\if\shortVersion 1
  $
    a\lfloor x \rfloor - b\lfloor y \rfloor 
      = a x - a frac(x) - by + b frac(y) 
      < ax - by + b -a frac(x) $  
    $
      < ax - by + b + a  $ 
    because  $frac(y)<1$ and  $-a frac(x) \le 0 < a$.
\else
  $$
  \begin{aligned}
    a\lfloor x \rfloor - b\lfloor y \rfloor 
      &= a x - a frac(x) - by + b frac(y) 
        &
    \\
      &< ax - by + b -a frac(x) 
        &[\text{because } frac(y)<1]
    \\
      & < ax - by + b + a & [\text{as } -a frac(x) \le 0 < a].
  \end{aligned}
  $$  
\fi
\end{proof}

\begin{lemma}
  Let $c,d,A_{1},B_{1},A_{2},B_{2} \in \Real$, $c > 0$ and $I_{1}\in \Zeta$. Then, 
  \if\shortVersion 1
    $
      \lim_{n\to\infty}{\sum_{i=I_{1}+1}^{\lfloor cn + d \rfloor} \frac{frac(-(A_{1}i+B_{1})) + frac(A_{2}i+B_{2})}{n}}.
    $ exists
  \else
    the limit below exists:
    $$
      \lim_{n\to\infty}{\sum_{i=I_{1}+1}^{\lfloor cn + d \rfloor} \frac{frac(-(A_{1}i+B_{1})) + frac(A_{2}i+B_{2})}{n}}.
    $$
  \fi
  \label{lemma:limsum1Ri}
\end{lemma}
\begin{proof}
  For convergence, we show that for $R(i) = frac(-(A_{1}i+B_{1})) + frac(A_{2}i+B_{2})$, $(a_{n})_{n \in \N^{*}} = \left(
  \sum_{i=I_{1}+1}^{\lfloor cn + d \rfloor } \frac{R(i)}{n}
  \right)_{n\in \N^{*}}$ is a Cauchy sequence. Take $\epsilon > 0$ and choose
  $N >  \frac{4\vert I_{1}-d+1\vert }{\epsilon}$.
  Let $n,m \in \N^{*}$ and $n > m > N.$ 
  We have
  \if\shortVersion 1
    $
        \vert a_{n} - a_{m}\vert 
        = \big\vert \sum_{i=I_{1}+1}^{\lfloor cn + d \rfloor} \frac{R(i)}{n} - \sum_{i=I_{1}+1}^{\lfloor cm + d \rfloor} \frac{R(i)}{m}\big\vert 
        = \big\vert \frac{1}{nm} \big( m\sum_{i=I_{1}+1}^{\lfloor cn + d \rfloor} R(i) -  n\sum_{i=I_{1}+1}^{\lfloor cm + d \rfloor} R(i) \big)\big\vert 
        = \big\vert \frac{1}{nm} \big( m\sum_{i=\lfloor cm + d \rfloor  + 1}^{\lfloor cn + d \rfloor} R(i) + (m-  n)\sum_{i=I_{1}+1}^{\lfloor cm + d \rfloor } R(i)\big) \big\vert 
        < \frac{2}{\vert nm\vert }\vert m(\lfloor cn + d \rfloor  - (\lfloor cm + d \rfloor + 1) + 1)  + (m-  n)(\lfloor cm + d \rfloor  -(I_{1}+1)+1 ) \vert 
        = \frac{2}{\vert nm\vert }\vert m\lfloor cn + d \rfloor - n\lfloor cm + d \rfloor  - (m-n) I_{1} \vert 
        < \frac{2}{\vert nm\vert }\vert m( cn + d ) - n( cm + d )  + m + n - (m-n) I_{1}  \vert   
        = \frac{2}{\vert nm\vert }\vert (n - m) (I_{1}-d)   + m + n  \vert 
        < 2\big\vert \frac{ (n + m) (I_{1}-d)   + m + n  }{nm} \big\vert 
        = 2\big\vert \frac{ (m + n) (I_{1}-d + 1) }{nm} \big\vert 
        = 2\vert I_{1}-d + 1\vert \frac{ m + n  }{nm} 
        = 2\vert I_{1}-d + 1\vert \big(\frac{1}{n}+ \frac{1}{m}\big)
        < 2\vert I_{1}-d + 1\vert \frac{2}{N}
        = \frac{4\vert I_{1}-d+1\vert }{N}
        < \epsilon.
    $  
    by knowing that $\left\lfloor cn + d \right \rfloor > \left \lfloor cm + d \right \rfloor$, $R(i) < 2$  for any  $i$ and Lemma \ref{lemma:flooraxby}.
  \else
    $$
      \begin{aligned}
          &\vert a_{n} - a_{m}\vert 
          = \left\vert \sum_{i=I_{1}+1}^{\lfloor cn + d \rfloor} \frac{R(i)}{n} - \sum_{i=I_{1}+1}^{\lfloor cm + d \rfloor} \frac{R(i)}{m}\right\vert {}
          = \left\vert \frac{1}{nm}\left( m\sum_{i=I_{1}+1}^{\lfloor cn + d \rfloor} R(i) -  n\sum_{i=I_{1}+1}^{\lfloor cm + d \rfloor} R(i)\right) \right\vert 
          \\
          &= \left\vert \frac{1}{nm}\left( m\sum_{i=\lfloor cm + d \rfloor  + 1}^{\lfloor cn + d \rfloor} R(i) + (m-  n)\sum_{i=I_{1}+1}^{\lfloor cm + d \rfloor } R(i) \right) \right\vert 
            \hspace*{.6cm} [\text{as } \left\lfloor cn + d \right \rfloor > \left \lfloor cm + d \right \rfloor]
      \end{aligned}
    $$
    $$
      \begin{aligned}
          &< 2\left\vert \frac{m(\lfloor cn + d \rfloor  - (\lfloor cm + d \rfloor + 1) + 1)  + (m-  n)(\lfloor cm + d \rfloor  -(I_{1}+1)+1 )}{nm} \right\vert 
              \hspace*{0.5cm} [\text{as } R(i) < 2 \text{ for any } i]
          \\
\ifexpandexplanation
          &= 2\left\vert \frac{m\lfloor cn + d \rfloor  - m(\lfloor cm + d \rfloor + 1) + m  + m \lfloor cm + d \rfloor  - m(I_{1}+1)+ m  -  n\lfloor cm + d \rfloor  +n(I_{1}+1) -n }{nm} \right\vert 
          \\
          &= 2\left\vert \frac{m\lfloor cn + d \rfloor  - m\lfloor cm + d \rfloor -m + m  + m \lfloor cm + d \rfloor  - mK -m + m  -  n\lfloor cm + d \rfloor  +nK +n -n }{nm} \right\vert 
          \\
          &= 2\left\vert \frac{m\lfloor cn + d \rfloor  - mK   -  n\lfloor cm + d \rfloor  +nK }{nm} \right\vert 
          = 2\left\vert \frac{m\lfloor cn + d \rfloor - n\lfloor cm + d \rfloor  - (m-n) I_{1} }{nm} \right\vert 
          \\
\else
          &= 2\left\vert \frac{m\lfloor cn + d \rfloor - n\lfloor cm + d \rfloor  - (m-n) I_{1} }{nm} \right\vert 
          \\
\fi
      \end{aligned}
    $$
    $$
      \begin{aligned}
          &< 2\left\vert \frac{m( cn + d ) - n( cm + d )  + m + n - (m-n) I_{1} }{nm} \right\vert    \hspace*{1.8cm} [\text{Lemma }\ref{lemma:flooraxby}]\\
\ifexpandexplanation
              &= 2\left\vert \frac{(m -n) (d-I_{1})   + m + n  }{nm} \right\vert \\
\fi
          &= 2\left\vert \frac{ (n - m) (I_{1}-d)   + m + n  }{nm} \right\vert 
          < 2\left\vert \frac{ (n + m) (I_{1}-d)   + m + n  }{nm} \right\vert 
          \\
          &= 2\left\vert \frac{ (m + n) (I_{1}-d + 1) }{nm} \right\vert 
          = 2\vert I_{1}-d + 1\vert \frac{ m + n  }{nm} 
          = 2\vert I_{1}-d + 1\vert \left(\frac{1}{n}+ \frac{1}{m}\right)
          \\
          &< 2\vert I_{1}-d + 1\vert \frac{2}{N}
          = \frac{4\vert I_{1}-d+1\vert }{N}
          < \epsilon.
      \end{aligned}
    $$  
  \fi
\end{proof}

  To prove the existence, we have $\lceil x \rceil = x +  frac(-x)$, for any real number $x$\footnote{Heed that using this definition of $frac$, $frac(1.7) = 0.7$ and $frac(-1.7) = 0.3$.}, because 
  \if\shortVersion 1
    $
        frac(-x)  
          = -x - \lfloor -x \rfloor  $ 
          $= -x - (-\lceil x \rceil)  $  $ 
        = -x + \lceil x \rceil 
        \LR  \lceil x \rceil = x +  frac(-x) $
    by the definition of $\lfloor -x \rfloor$ and $\lfloor -x \rfloor = -\lceil x \rceil$.
  \else
    $$
      \begin{aligned}
        frac(-x)  
          &= -x - \lfloor -x \rfloor & [\text{def. of }\lfloor -x \rfloor ] \\
          &= -x - (-\lceil x \rceil) & [\lfloor -x \rfloor = -\lceil x \rceil] \\
        &= -x + \lceil x \rceil &\\
        \LR  \lceil x \rceil &= x +  frac(-x).&\\
      \end{aligned}
    $$
  \fi
  Thus,
    \if\shortVersion 1
      $
          \lim_{T\to \infty} \frac{1}{T} \sum_{x_{h}=n_{l}^{-} + 1}^{\min(n_{l}^{+}-1,K'-1)} \big(\lfloor Y_{2}^{R}(x_{h}) \rfloor - \lceil Y_{1}^{R}(x_{h}) \rceil + 1 \big) 
          = \lim_{T\to \infty} \frac{1}{T} \sum_{x_{h}=n_{l}^{-} + 1}^{\min(n_{l}^{+}-1,K'-1)}  \big( Y_{2}^{R}(x_{h})  -  Y_{1}^{R}(x_{h}) + 1 \big) -
           \lim_{T\to \infty} \frac{1}{T} \sum_{x_{h}=n_{l}^{-} + 1}^{\min(n_{l}^{+}-1,K'-1)} \big( frac\big(-Y_{1}^{R}(x_{h})\big) +frac\big(Y_{2}^{R}(x_{h})\big)  \big).
       $
   \else
      $$
        \begin{aligned}
          &\lim_{T\to \infty} \frac{1}{T} \sum_{x_{h}=n_{l}^{-} + 1}^{\min(n_{l}^{+}-1,K'-1)} \left(\lfloor Y_{2}^{R}(x_{h}) \rfloor - \lceil Y_{1}^{R}(x_{h}) \rceil + 1 \right) 
          \\
          =& \lim_{T\to \infty} \frac{1}{T} \sum_{x_{h}=n_{l}^{-} + 1}^{\min(n_{l}^{+}-1,K'-1)}  \left( Y_{2}^{R}(x_{h})  -  Y_{1}^{R}(x_{h})  - frac\left(-Y_{1}^{R}(x_{h})\right) \right.\\
          &
            \left. -frac\left(Y_{2}^{R}(x_{h})\right) + 1 \right)
          \\
        \end{aligned}
      $$
      $$
        \begin{aligned}
          =& \lim_{T\to \infty} \frac{1}{T} \sum_{x_{h}=n_{l}^{-} + 1}^{\min(n_{l}^{+}-1,K'-1)}  \left( Y_{2}^{R}(x_{h})  -  Y_{1}^{R}(x_{h}) + 1 \right) -
          \\
          & \lim_{T\to \infty} \frac{1}{T} \sum_{x_{h}=n_{l}^{-} + 1}^{\min(n_{l}^{+}-1,K'-1)} \left( frac\left(-Y_{1}^{R}(x_{h})\right) +frac\left(Y_{2}^{R}(x_{h})\right)  \right).
          \\
        \end{aligned}
      $$
    \fi
    The limit of the first term above exists and its value is presented below on (\ref{eq:limplus1above}). 
    The existence of the limit for the second term was shown by Lemma \ref{lemma:limsum1Ri} for any outcome of $\min(n_{l}^{+}-1,K'-1)$, because, if $ \left\lfloor\frac{2 (vT-s) \cos(\pi/6 - \theta)  + 2s\sin(\vert \pi/6 - \theta\vert ) }{\sqrt{3}d}\right\rfloor = n_{l}^{+}-1 \le K'-1$, $c = \frac{2 v \cos(\psi - \pi/6)}{\sqrt{3}d}$ and $d =  \frac{2s(\sin(\vert \pi/6 - \theta\vert )-\cos(\pi/6 - \theta)) }{\sqrt{3}d}$ on the Lemma \ref{lemma:limsum1Ri}. If $n_{l}^{+}-1 > K'-1 = \left\lceil\frac{2 (vT-s) \cos(\psi - \pi/6) - 2s\sin(\vert \psi - \pi/6\vert ) }{\sqrt{3}d} - 1\right\rceil$, as for any $x$, $\lceil x \rceil = \lfloor x \rfloor$ or $\lceil x \rceil = \lfloor x \rfloor + 1$ depending on whether $x$ is an integer or not, then $K' - 1 = \Big\lfloor \frac{2 (vT-s) \cos(\psi - \pi/6)}{\sqrt{3}d} - \frac{2s\sin(\vert \psi - \pi/6\vert )}{\sqrt{3}d} -1\Big\rfloor$ or $K' - 1 = \Big\lfloor \frac{2 (vT-s) \cos(\psi - \pi/6)}{\sqrt{3}d} - \frac{2s\sin(\vert \psi - \pi/6\vert )}{\sqrt{3}d}  \Big\rfloor$. For both cases, on the Lemma \ref{lemma:limsum1Ri} $c = \frac{2 v \cos(\psi - \pi/6)}{\sqrt{3}d}$ as well, but for the former case, $d = -\frac{2s(\sin(\vert \psi - \pi/6\vert )+\cos(\pi/6 - \theta))}{\sqrt{3}d} - 1$, and for the latter, $d = - \frac{2s(\sin(\vert \psi - \pi/6\vert )+\cos(\pi/6 - \theta))}{\sqrt{3}d}$.

    To get the bounds, we have
    \begin{equation}
      \begin{aligned}
         &\lim_{T\to \infty} \frac{1}{T}\sum_{x_{h}=n_{l}^{-} + 1}^{\min(n_{l}^{+}-1,K'-1)} \left( Y_{2}^{R}(x_{h})   -  Y_{1}^{R}(x_{h}) -1 \right)   \\
         &< \lim_{T\to \infty} \frac{1}{T} \sum_{x_{h}=n_{l}^{-} + 1}^{\min(n_{l}^{+}-1,K'-1)}\left( \lfloor Y_{2}^{R}(x_{h}) \rfloor - \lceil Y_{1}^{R}(x_{h}) \rceil + 1 \right)  \\
         &\le \lim_{T\to \infty} \frac{1}{T} \sum_{x_{h}=n_{l}^{-} + 1}^{\min(n_{l}^{+}-1,K'-1)} \left( Y_{2}^{R}(x_{h})   -  Y_{1}^{R}(x_{h}) +1 \right), 
      \end{aligned}
      \label{eq:limit3ineq}
    \end{equation}
    and by Lemma \ref{lemma:MiddleInterval}, as $T \to \infty$,
    \if\shortVersion 1
      $
        Y_{2}^{R}(x_{h}) - Y_{1}^{R}(x_{h}) 
        = 
           \frac{\frac{y_2}{d} + \tan(\psi) x_h}{{\frac{\sqrt{3} + \tan(\psi)}{2}}} - 
            \frac{\frac{y_1}{d} + \tan(\psi) x_h}{\frac{\sqrt{3} + \tan(\psi)}{2}} 
        = \frac{2s}{d\cos(\psi - \pi/6)},
      $
    \else
\ifexpandexplanation 
      $$
      \begin{aligned}
        Y_{2}^{R}(x_{h}) - Y_{1}^{R}(x_{h}) 
        &= 
           \frac{\frac{y_2}{d} + \tan(\psi) x_h}{{\frac{\sqrt{3} + \tan(\psi)}{2}}} - 
            \frac{\frac{y_1}{d} + \tan(\psi) x_h}{\frac{\sqrt{3} + \tan(\psi)}{2}} 
                \\ &
                = \frac{\frac{y_{2}-y_{1}}{d}}{{\frac{\sqrt{3} + \tan(\psi)}{2}}}
                = \frac{\frac{2s}{d\cos(\psi)}}{{\frac{\sqrt{3} + \tan(\psi)}{2}}}
                = \frac{4s}{d(\sqrt{3}\cos(\psi) + \sin(\psi))}
                \\ &
        = \frac{2s}{d\cos(\psi - \pi/6)},
      \end{aligned}
      $$ 
\else
      $$
      \begin{aligned}
        Y_{2}^{R}(x_{h}) - Y_{1}^{R}(x_{h}) 
        &= 
           \frac{\frac{y_2}{d} + \tan(\psi) x_h}{{\frac{\sqrt{3} + \tan(\psi)}{2}}} - 
            \frac{\frac{y_1}{d} + \tan(\psi) x_h}{\frac{\sqrt{3} + \tan(\psi)}{2}} 
        = \frac{2s}{d\cos(\psi - \pi/6)},
      \end{aligned}
      $$ 
\fi
    \fi
    by (\ref{eq:2sy2y1}).
    
    For the first limit at (\ref{eq:limit3ineq}) in the case of $\min(n_{l}^{+}-1,K'-1) = n_{l}^{+}-1$,
    \if\shortVersion 1
      \begin{equation}
           \lim_{T\to \infty} \frac{1}{T}\sum_{x_{h}=n_{l}^{-} + 1}^{n_{l}^{+} - 1 } \left( Y_{2}^{R}(x_{h})   -  Y_{1}^{R}(x_{h}) -1 \right) 
           =  \frac{4vs}{\sqrt{3}d^{2}} -  \frac{2 v  \cos(\psi - \pi/6)}{\sqrt{3}d}.    
        \label{eq:limitbelowmid}
      \end{equation}
    \else
        $$
           \lim_{T\to \infty} \frac{1}{T}\sum_{x_{h}=n_{l}^{-} + 1}^{n_{l}^{+} - 1 } \left( Y_{2}^{R}(x_{h})   -  Y_{1}^{R}(x_{h}) -1 \right) 
        $$
        $$
           =  \lim_{T\to \infty} \frac{1}{T}\sum_{x_{h}=n_{l}^{-} + 1}^{n_{l}^{+} - 1 } \left( \frac{2s}{d\cos(\psi - \pi/6)} - 1 \right)
        $$
        $$
           =  \lim_{T\to \infty} \frac{1}{T} \left(n_{l}^{+} -  n_{l}^{-} - 1\right)  \left(\frac{2s}{d\cos(\psi - \pi/6)} - 1\right)
        $$
        \begin{align}
          &= \left(\frac{2s}{d\cos(\psi - \pi/6)} - 1\right) \left(  \lim_{T\to \infty} \frac{1}{T} n_{l}^{+} - \lim_{T\to \infty} \frac{1}{T} (n_{l}^{-} - 1 )\right) \nonumber \\
          &= \left(\frac{2s}{d\cos(\psi - \pi/6)} - 1\right)   \lim_{T\to \infty} \frac{1}{T} n_{l}^{+} \nonumber \\
          &=  \left(\frac{2s}{d\cos(\psi - \pi/6)} - 1\right)  \frac{2 v \cos(\psi - \pi/6)}{\sqrt{3}d}. \nonumber \\
          &=  \frac{4vs}{\sqrt{3}d^{2}} -  \frac{2 v  \cos(\psi - \pi/6)}{\sqrt{3}d}.
        \label{eq:limitbelowmid}
        \end{align}
    \fi
    Above we get $\displaystyle \lim_{T \to\infty}\frac{1}{T} n_{l}^{+} = \frac{2 v \cos(\psi - \pi/6)}{\sqrt{3}d}$ by using the sandwich theorem and the inequality $x - 1 < \lfloor x \rfloor \le x$ to get the bounds on $n_{l}^{+}$.

    Similarly, for the last limit at (\ref{eq:limit3ineq}) in the case of $\min(n_{l}^{+}-1,K'-1) = n_{l}^{+}-1$, 
    \if\shortVersion 1
      $
           \lim_{T\to \infty} \frac{1}{T}\sum_{x_{h}=n_{l}^{-} + 1}^{n_{l}^{+}-1} \left(Y_{2}^{R}(x_{h})   -  Y_{1}^{R}(x_{h}) +1 \right)
           =\frac{4vs}{\sqrt{3}d^{2}} +  \frac{2 v  \cos(\psi - \pi/6)}{\sqrt{3}d}.
      $
    \else
      $$
      \begin{aligned}
           &\lim_{T\to \infty} \frac{1}{T}\sum_{x_{h}=n_{l}^{-} + 1}^{n_{l}^{+}-1} \left(Y_{2}^{R}(x_{h})   -  Y_{1}^{R}(x_{h}) +1 \right)
           =\frac{4vs}{\sqrt{3}d^{2}} +  \frac{2 v  \cos(\psi - \pi/6)}{\sqrt{3}d}.
      \end{aligned}
      $$
    \fi
    The limits above in the case of $\min(n_{l}^{+}-1,K'-1) = K'-1$ yields the same result because of the sandwich theorem, the inequality $x  \le \lceil x \rceil < x + 1$, and
    \if\shortVersion 1
      $
          \frac{2 v \cos(\psi - \pi/6)}{\sqrt{3}d}
          = \lim_{T \to \infty} \frac{1}{T} \big( \frac{2 (vT-s) \cos(\psi - \pi/6) - 2s\sin(\vert \psi - \pi/6\vert ) }{\sqrt{3}d} \big)
           \le \lim_{T \to \infty} \frac{1}{T}(K' -1) = \lim_{T \to \infty} \frac{1}{T}K' 
          = \lim_{T \to \infty} \frac{1}{T} \big\lceil \frac{2 (vT-s) \cos(\psi - \pi/6) - 2s\sin(\vert \psi - \pi/6\vert ) }{\sqrt{3}d} \big\rceil \\
          < \lim_{T \to \infty} \frac{1}{T} \big( \frac{2 (vT-s) \cos(\psi - \pi/6) - 2s\sin(\vert \psi - \pi/6\vert ) }{\sqrt{3}d} + 1 \big) 
          = \frac{2 v \cos(\psi - \pi/6)}{\sqrt{3}d},
      $
    \else
      $$
        \begin{aligned} 
            &\frac{2 v \cos(\psi - \pi/6)}{\sqrt{3}d}
            = \lim_{T \to \infty} \frac{1}{T} \left( \frac{2 (vT-s) \cos(\psi - \pi/6) - 2s\sin(\vert \psi - \pi/6\vert ) }{\sqrt{3}d} \right)
            \\ 
            & \le \lim_{T \to \infty} \frac{1}{T}(K' -1) = \lim_{T \to \infty} \frac{1}{T}K' \\
            &= \lim_{T \to \infty} \frac{1}{T} \left\lceil \frac{2 (vT-s) \cos(\psi - \pi/6) - 2s\sin(\vert \psi - \pi/6\vert ) }{\sqrt{3}d} \right\rceil 
            \\ 
            &<\lim_{T \to \infty} \frac{1}{T} \left( \frac{2 (vT-s) \cos(\psi - \pi/6) - 2s\sin(\vert \psi - \pi/6\vert ) }{\sqrt{3}d} + 1 \right) \\
            &= \frac{2 v \cos(\psi - \pi/6)}{\sqrt{3}d},
        \end{aligned}
      $$
    \fi
    so,  $\displaystyle \lim_{T \to \infty} \frac{1}{T}K' = \lim_{T \to \infty} \frac{1}{T}n_{l}^{+}.$ Consequently, the limit below exists and 
    \begin{equation}
      \lim_{T\to \infty} \frac{1}{T}\sum_{x_{h}=n_{l}^{-} + 1}^{\min(n_{l}^{+}-1,K'-1) } \left(Y_{2}^{R}(x_{h})   -  Y_{1}^{R}(x_{h}) +1 \right) =\frac{4vs}{\sqrt{3}d^{2}} +  \frac{2 v  \cos(\psi - \frac{\pi}{6})}{\sqrt{3}d}.
    \label{eq:limplus1above}
    \end{equation}
    
    Finally, using the bounds provided by (\ref{eq:limitbelowmid}) and (\ref{eq:limplus1above}) we have the expected result.
    \fi %
  \end{proof}
  
  By Lemmas \ref{lemma:lim1st}, \ref{lemma:lim2nd} and \ref{lemma:lim3rd} we have for $\psi \neq \pi/6$
  \begin{equation} 
    \begin{aligned}
      \lim_{T \to \infty}  f_{h}(T,\psi) \in&
         \left(\frac{4vs}{\sqrt{3}d^{2}} - \frac{2 v  \cos(\psi - \pi/6)}{\sqrt{3}d}, \frac{4vs}{\sqrt{3}d^{2}} +  \frac{2 v  \cos(\psi - \pi/6)}{\sqrt{3}d}\right].
    \end{aligned}
    \label{eq:whollynotpi6}
  \end{equation}
  For $\psi = \pi/6$, by (\ref{eq:30degreescasepsi}),
  \if\shortVersion 1
    $
      \lim_{T \to \infty} \frac{1}{T}\sum_{x_{h}=0}^{
      \big\lfloor\frac{2  (vT-s) }{\sqrt{3}d} \big\rfloor
      } \Big(\frac{\sqrt{3}y_{2} + d x_{h}}{2d}  -  \frac{\sqrt{3}y_{1} + d x_{h}}{2d}  - 1 \Big)
      <
      \lim_{T \to \infty} f_{h}(T,\pi/6) 
      \le \lim_{T \to \infty} \frac{1}{T} \sum_{x_{h}=0}^{
      \big\lfloor\frac{2  (vT-s) }{\sqrt{3}d} \big\rfloor
      } \Big(\frac{\sqrt{3}y_{2} + d x_{h}}{2d}  - \frac{\sqrt{3}y_{1} + d x_{h}}{2d}  + 1 \Big),
    $
  \else
    $$
      \begin{array}{>{\displaystyle}c}
      \lim_{T \to \infty} \frac{1}{T}\sum_{x_{h}=0}^{
      \left\lfloor\frac{2  (vT-s) }{\sqrt{3}d} \right\rfloor
      } \left(\frac{\sqrt{3}y_{2} + d x_{h}}{2d}  -  \frac{\sqrt{3}y_{1} + d x_{h}}{2d}  - 1 \right)
      <
      \lim_{T \to \infty} f_{h}(T,\pi/6) 
      \\
      \le \lim_{T \to \infty} \frac{1}{T} \sum_{x_{h}=0}^{
      \left\lfloor\frac{2  (vT-s) }{\sqrt{3}d} \right\rfloor
      } \left(\frac{\sqrt{3}y_{2} + d x_{h}}{2d}  - \frac{\sqrt{3}y_{1} + d x_{h}}{2d}  + 1 \right),
      \end{array}
    $$
  \fi
  with
  \if\shortVersion 1
    $
        \lim_{T \to \infty} \frac{1}{T}\sum_{x_{h}=0}^{ \big\lfloor\frac{2  (vT-s) }{\sqrt{3}d} \big\rfloor} \big( \frac{\sqrt{3}y_{2} + d x_{h}}{2d}  -  \frac{\sqrt{3}y_{1} + d x_{h}}{2d}  + 1 \big) 
        = \lim_{T \to \infty} \frac{1}{T}\sum_{x_{h}=0}^{ \big\lfloor\frac{2  (vT -s)}{\sqrt{3}d} \big\rfloor} \big( \frac{\sqrt{3}s}{d\cos(\pi/6)}  + 1 \big)
        = \lim_{T \to \infty} \frac{1}{T} \big\lfloor \frac{2  (vT-s) }{\sqrt{3}d} + 1\big\rfloor \big(\frac{2s}{d}  + 1\big)
        = \frac{2  v }{\sqrt{3}d}  \big(\frac{2s}{d}  + 1\big),
    $
  \else
\ifexpandexplanation 
   $$
      \begin{aligned}
        &\lim_{T \to \infty} \frac{1}{T}\sum_{x_{h}=0}^{ \left\lfloor\frac{2  (vT-s) }{\sqrt{3}d} \right\rfloor} \left( \frac{\sqrt{3}y_{2} + d x_{h}}{2d}  -  \frac{\sqrt{3}y_{1} + d x_{h}}{2d}  + 1 \right) \\
        &= \lim_{T \to \infty} \frac{1}{T}\sum_{x_{h}=0}^{ \left\lfloor\frac{2  (vT -s)}{\sqrt{3}d} \right\rfloor} \left( \frac{\sqrt{3}s}{d\cos(\pi/6)}  + 1 \right)
          \\
          &= \lim_{T \to \infty} \frac{1}{T}\sum_{x_{h}=0}^{ \left\lfloor\frac{2  (vT-s) }{\sqrt{3}d} \right\rfloor} \left(\frac{2s}{d}  + 1 \right)\\
          & 
        = \lim_{T \to \infty} \frac{1}{T} \left\lfloor \frac{2  (vT-s) }{\sqrt{3}d} + 1\right\rfloor \left(\frac{2s}{d}  + 1\right)\\&
        = \frac{2  v }{\sqrt{3}d}  \left(\frac{2s}{d}  + 1\right),\\
      \end{aligned}
  $$ 
\else 
   $$
      \begin{aligned}
        &\lim_{T \to \infty} \frac{1}{T}\sum_{x_{h}=0}^{ \left\lfloor\frac{2  (vT-s) }{\sqrt{3}d} \right\rfloor} \left( \frac{\sqrt{3}y_{2} + d x_{h}}{2d}  -  \frac{\sqrt{3}y_{1} + d x_{h}}{2d}  + 1 \right) \\
        &= \lim_{T \to \infty} \frac{1}{T}\sum_{x_{h}=0}^{ \left\lfloor\frac{2  (vT -s)}{\sqrt{3}d} \right\rfloor} \left( \frac{\sqrt{3}s}{d\cos(\pi/6)}  + 1 \right)
        = \lim_{T \to \infty} \frac{1}{T} \left\lfloor \frac{2  (vT-s) }{\sqrt{3}d} + 1\right\rfloor \left(\frac{2s}{d}  + 1\right)\\&
        = \frac{2  v }{\sqrt{3}d}  \left(\frac{2s}{d}  + 1\right),\\
      \end{aligned}
  $$ 
\fi
  \fi
  from (\ref{eq:2sy2y1}) and, as similarly done before, $\displaystyle  \lim_{T \to \infty} \frac{1}{T} \left\lfloor \frac{2  (vT-s) }{\sqrt{3}d} + 1\right\rfloor =  \frac{2  v }{\sqrt{3}d}$ by using the sandwich theorem and the inequality $x - 1 < \lfloor x \rfloor \le x$ to get the bounds on the floor function; and
  \if\shortVersion 1
    $
      \lim_{T \to \infty} \frac{1}{T}\sum_{x_{h}=0}^{ \big\lfloor\frac{2  (vT-s) }{\sqrt{3}d} \big\rfloor} \big( \frac{\sqrt{3}y_{2} + d x_{h}}{2d}  -  \frac{\sqrt{3}y_{1} + d x_{h}}{2d}  - 1 \big)
      = \frac{2  v }{\sqrt{3}d}  \big(\frac{2s}{d}  - 1\big).
    $
  \else
    $$
    \begin{aligned}
      &\lim_{T \to \infty} \frac{1}{T}\sum_{x_{h}=0}^{ \left\lfloor\frac{2  (vT-s) }{\sqrt{3}d} \right\rfloor} \left( \frac{\sqrt{3}y_{2} + d x_{h}}{2d}  -  \frac{\sqrt{3}y_{1} + d x_{h}}{2d}  - 1 \right)
      = \frac{2  v }{\sqrt{3}d}  \left(\frac{2s}{d}  - 1\right).\\
    \end{aligned}
    $$
  \fi
  Accordingly, 
  \if\shortVersion 1
    $\lim_{T \to \infty} f_{h}(T,\pi/6) \in \big(\frac{2  v }{\sqrt{3}d}  \big(\frac{2s}{d}  - 1\big),\frac{2  v }{\sqrt{3}d}  \big(\frac{2s}{d}  + 1\big)\big],$
  \else
    $$\lim_{T \to \infty} f_{h}(T,\pi/6) \in \left(\frac{2  v }{\sqrt{3}d}  \left(\frac{2s}{d}  - 1\right),\frac{2  v }{\sqrt{3}d}  \left(\frac{2s}{d}  + 1\right)\right],$$
  \fi
  which are the same values in (\ref{eq:whollynotpi6}) if used $\psi = \pi/6$. 
  Lemmas \ref{lemma:nlmnlp}--\ref{lemma:MiddleInterval}, \ref{lemma:lim1st}, \ref{lemma:lim2nd} and \ref{lemma:lim3rd} used $\psi$, so, after replacing $\psi$ by $\pi/3 - \theta$, we conclude the Proposition \ref{prop:hexthroughputbounds}.  
\end{proof}

\begin{figure}
  \centering
  \includegraphics[width=0.6\columnwidth]{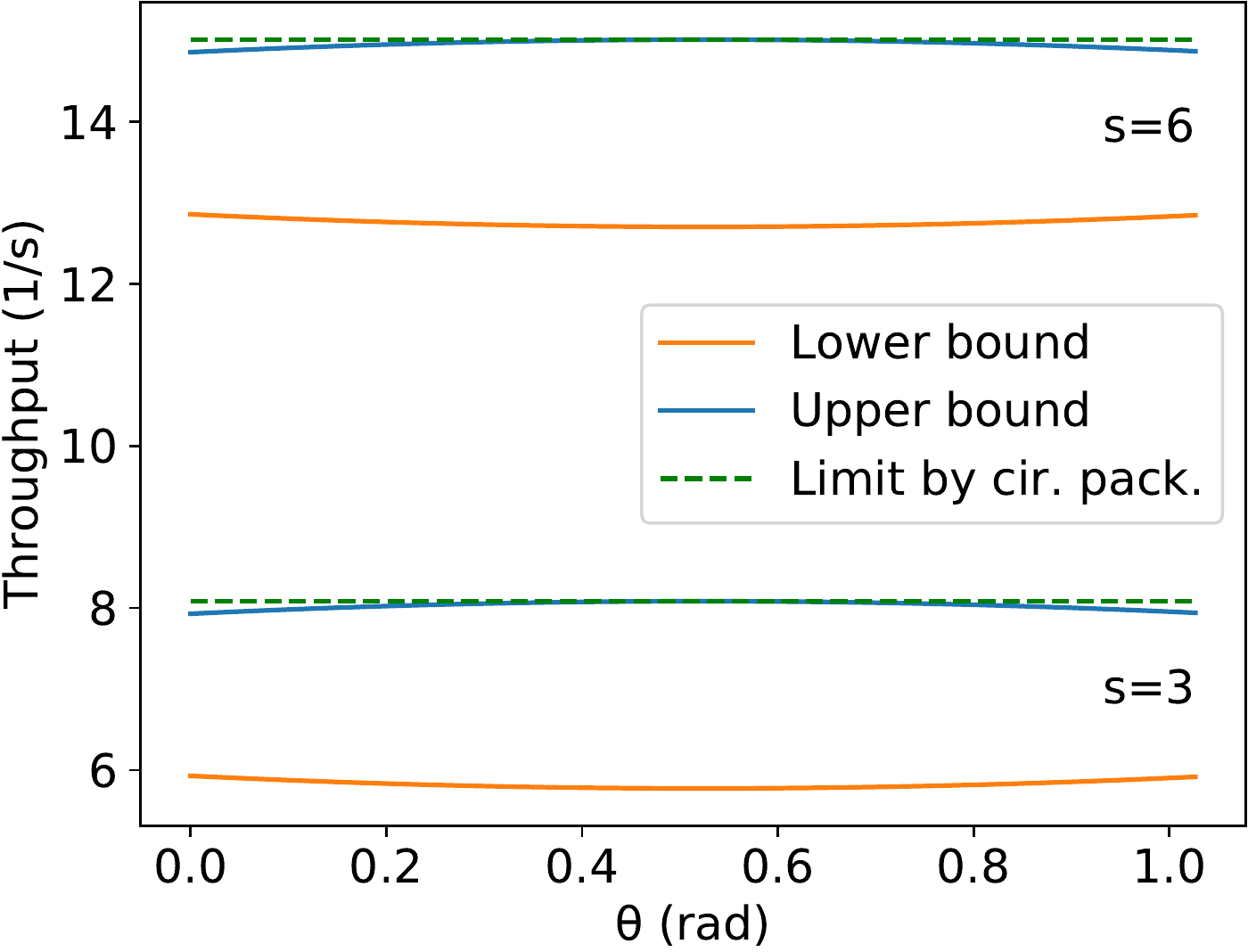} 
  \caption{Limit given by (\ref{eq:limitnoangle}) using the circle packing results and the lower and upper bounds of the hexagonal packing limit by (\ref{eq:hexthroughputbounds}) for $\theta \in \lbrack 0,\pi/3 \rparen$, $d = 1$ m, $v = 1$ m/s and $s \in \{3,6\}$ m.}
  \label{fig:limitshexpack}
\end{figure}

The upper and lower bounds presented on (\ref{eq:hexthroughputbounds}) are below or equal the maximum asymptotic throughput presented by the Proposition \ref{prop:triangularthroughput}, equation (\ref{eq:limitnoangle}). The result of the Proposition \ref{prop:triangularthroughput} only concerns the maximum asymptotic throughput and do not consider the hexagonal packing angle $\theta$, while Proposition \ref{prop:hexthroughputbounds} gives a lower bound and tightens the bounds for a given $\theta$. Figure \ref{fig:limitshexpack} presents an example comparison of these equations for two different values of $s$. As expected, the maximum asymptotic throughput under the optimal density assumption (in (\ref{eq:limitnoangle})) is a possible value of the throughput using hexagonal packing and is above or equal the interval in (\ref{eq:hexthroughputbounds}) for any given $\theta$. However, for practical robotic swarms applications, a certain hexagonal packing angle must be fixed depending on the expected height of the corridor,  target size and the minimum distance between the robots, resulting in a throughput below or equal to the upper value presented on Proposition \ref{prop:triangularthroughput}. 

\begin{figure}[t!]
  \centering 
  \subfloat[For 99 samples,  $T = 43$ s, $s=3$ m, $d = 1$ m.]{\includegraphics[width=0.483\columnwidth]{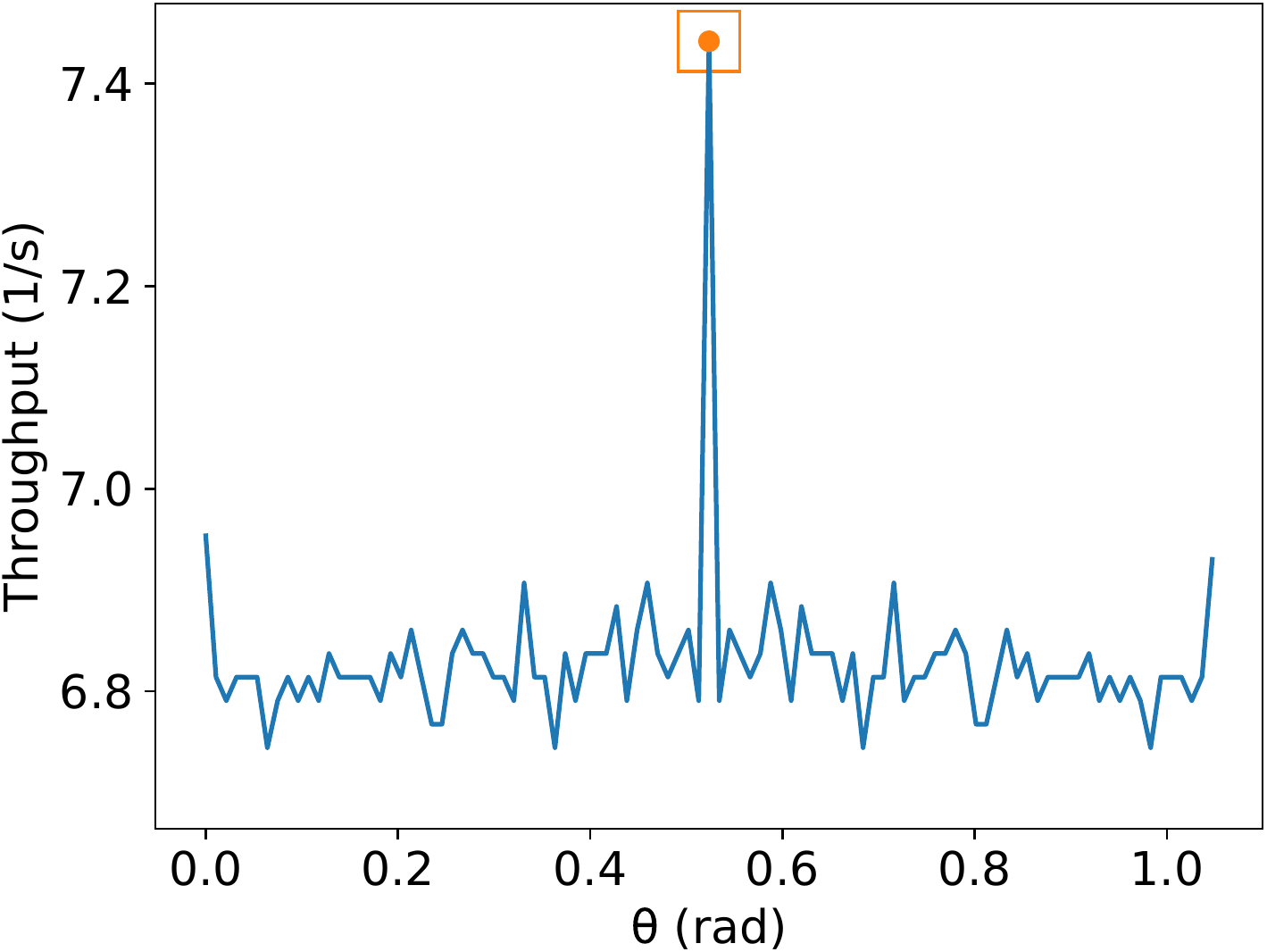}}
  \,
  \subfloat[For 100 samples, $T = 43$ s, $s=3$ m, $d = 1$ m.]{\includegraphics[width=0.483\columnwidth]{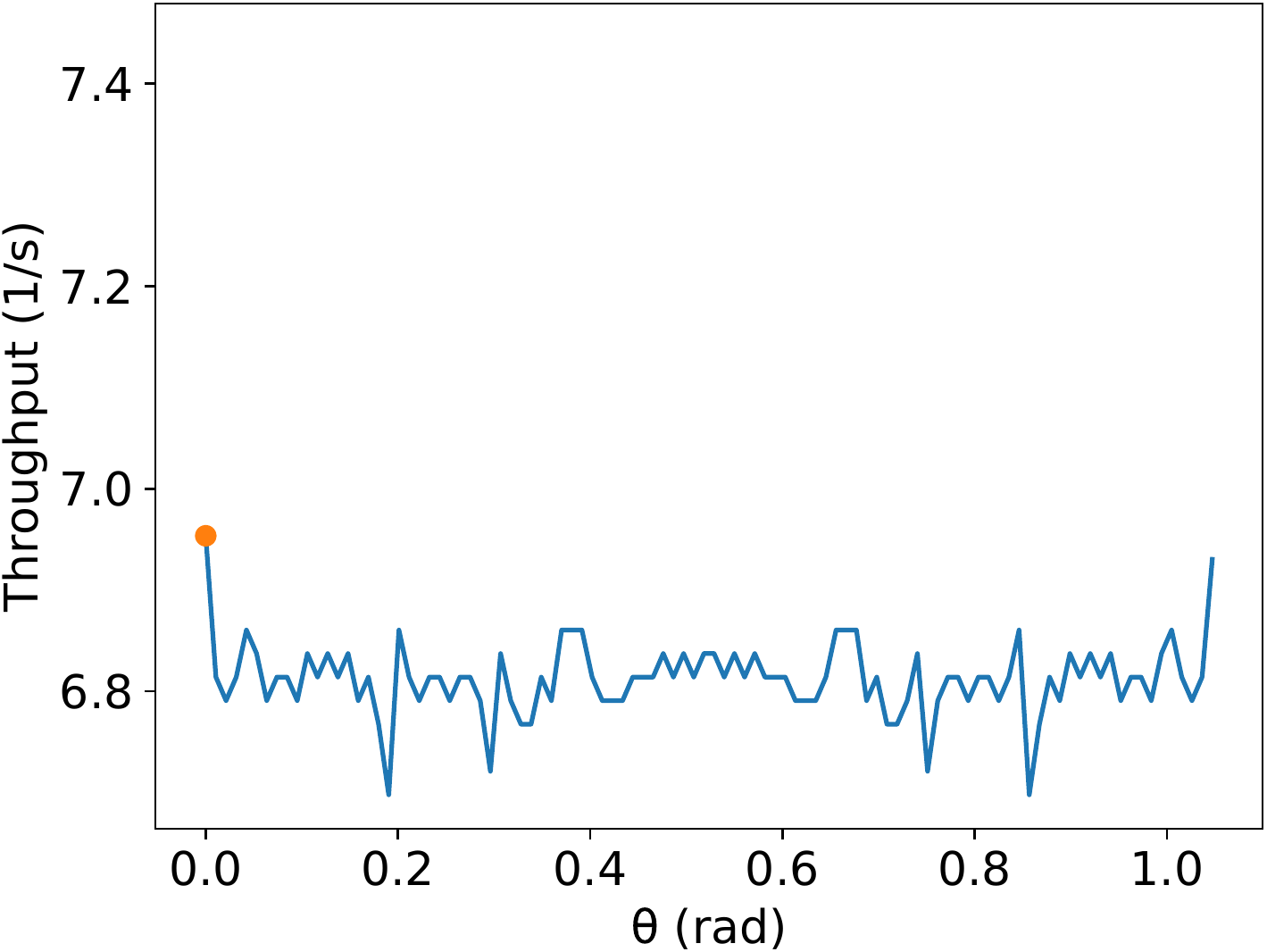}}
  \\
  \subfloat[For 99 samples, $T=30$ s, $s=2.5$ m and $d=0.66$ m.]{\includegraphics[width=0.483\columnwidth]{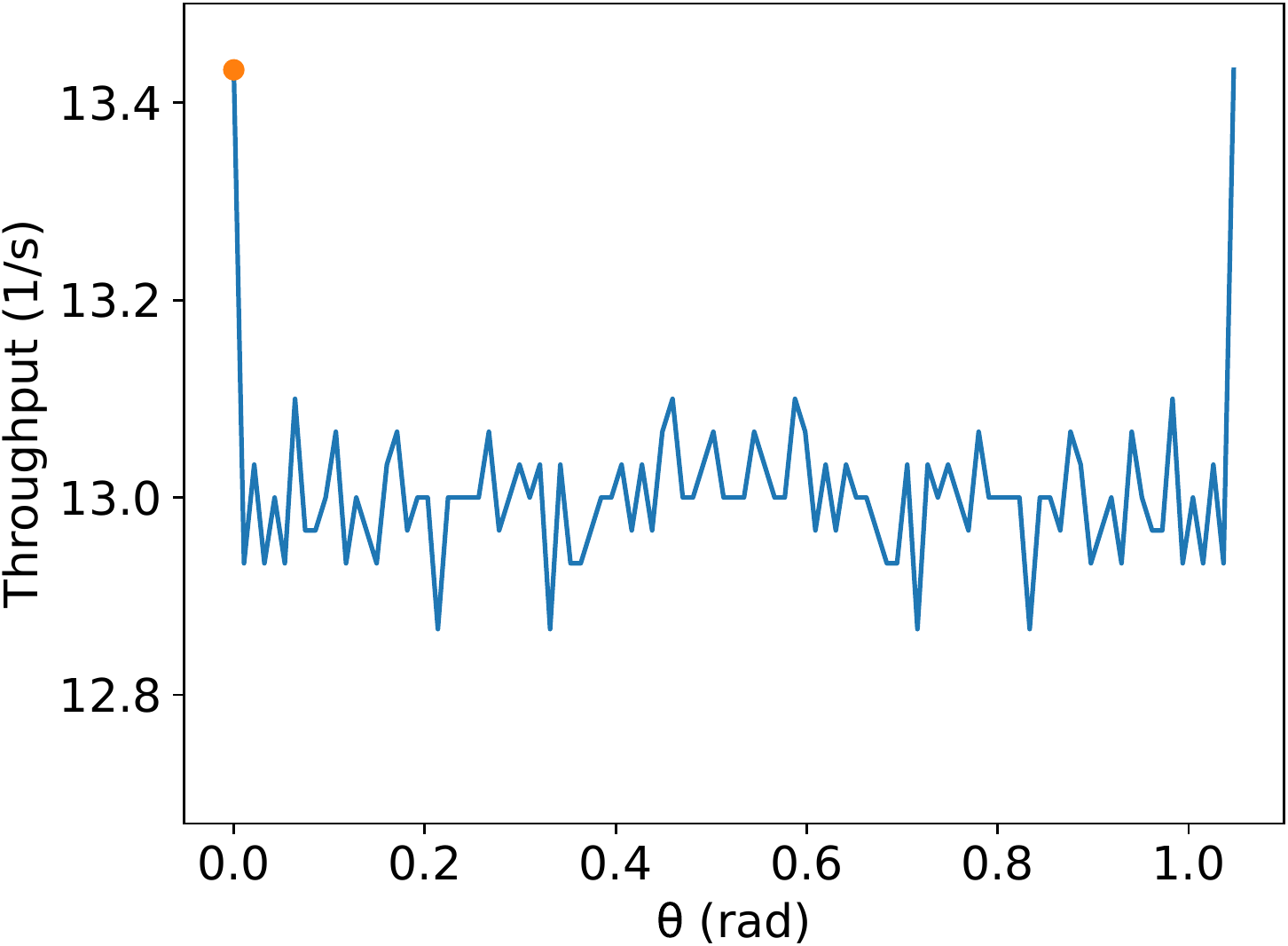}} 
  \,
  \subfloat[For 100 samples, $T=30$ s, $s=2.5$ m and $d=0.66$ m.]{\includegraphics[width=0.483\columnwidth]{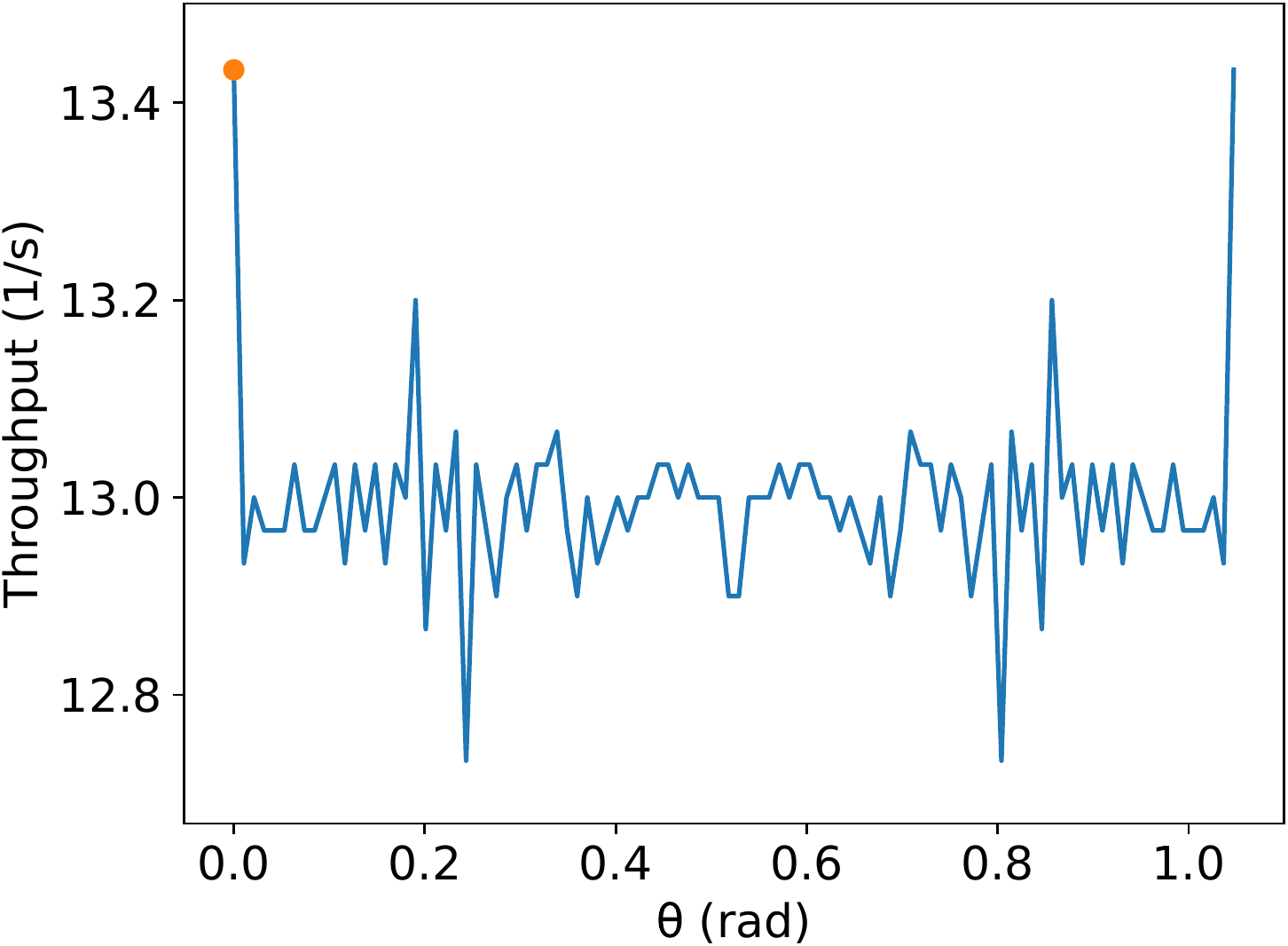}}
  \caption{Examples of (\ref{eq:hexthroughput}) varying $\theta$ from 0 to $\frac{\pi}{3}$ for different and randomly generated values of $T$, $s$, and $d$. In the graphs, $\theta$ is over the $x$-axis and the number of robots inside the given rectangle is over the $y$-axis. We used 99 samples on the images on the left-hand side and 100 samples on the right-hand side for each plot, and $v = 1$ m/s. The maximum value in each image is represented by an orange circle and a rectangle represents the maximum between the left and the right image. No square means the maximum values in both sides are equal. It continues in the Figure \ref{fig:plottheta2}.}
  \label{fig:plottheta1}
\end{figure}

\begin{figure}[t!]
  \centering 
  \subfloat[For 99 samples,  $T = 4$ s, $s = 2$ m and $d = 0.13$ m.]{\includegraphics[width=0.483\columnwidth]{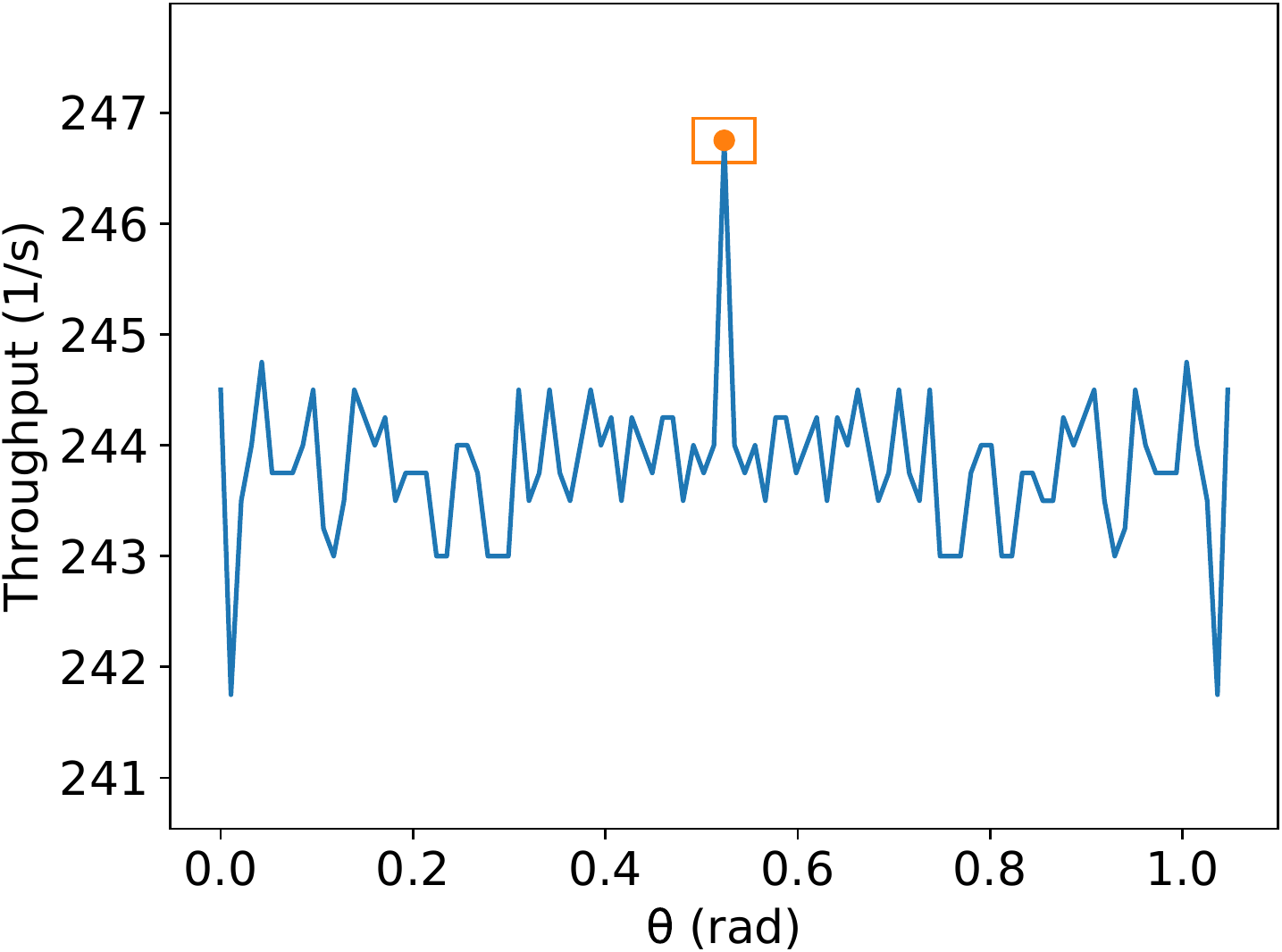}}
  \,
  \subfloat[For 100 samples, $T = 4$ s, $s = 2$ m and $d = 0.13$ m.]{\includegraphics[width=0.483\columnwidth]{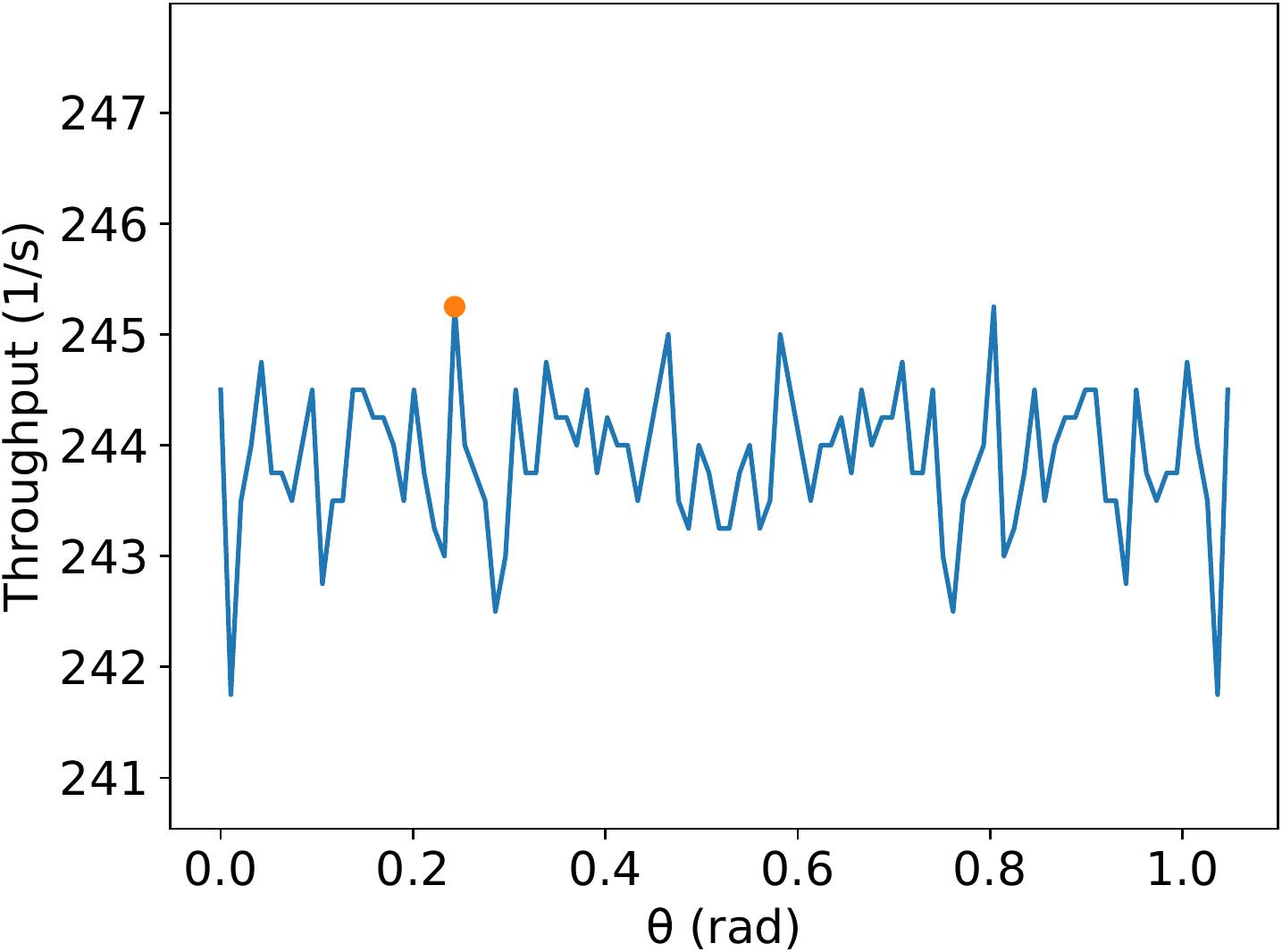}}
  \\
  \subfloat[For 99 samples, $T = 100$ s, $s=2.40513$ m and $d=1$ m.]{\includegraphics[width=0.483\columnwidth]{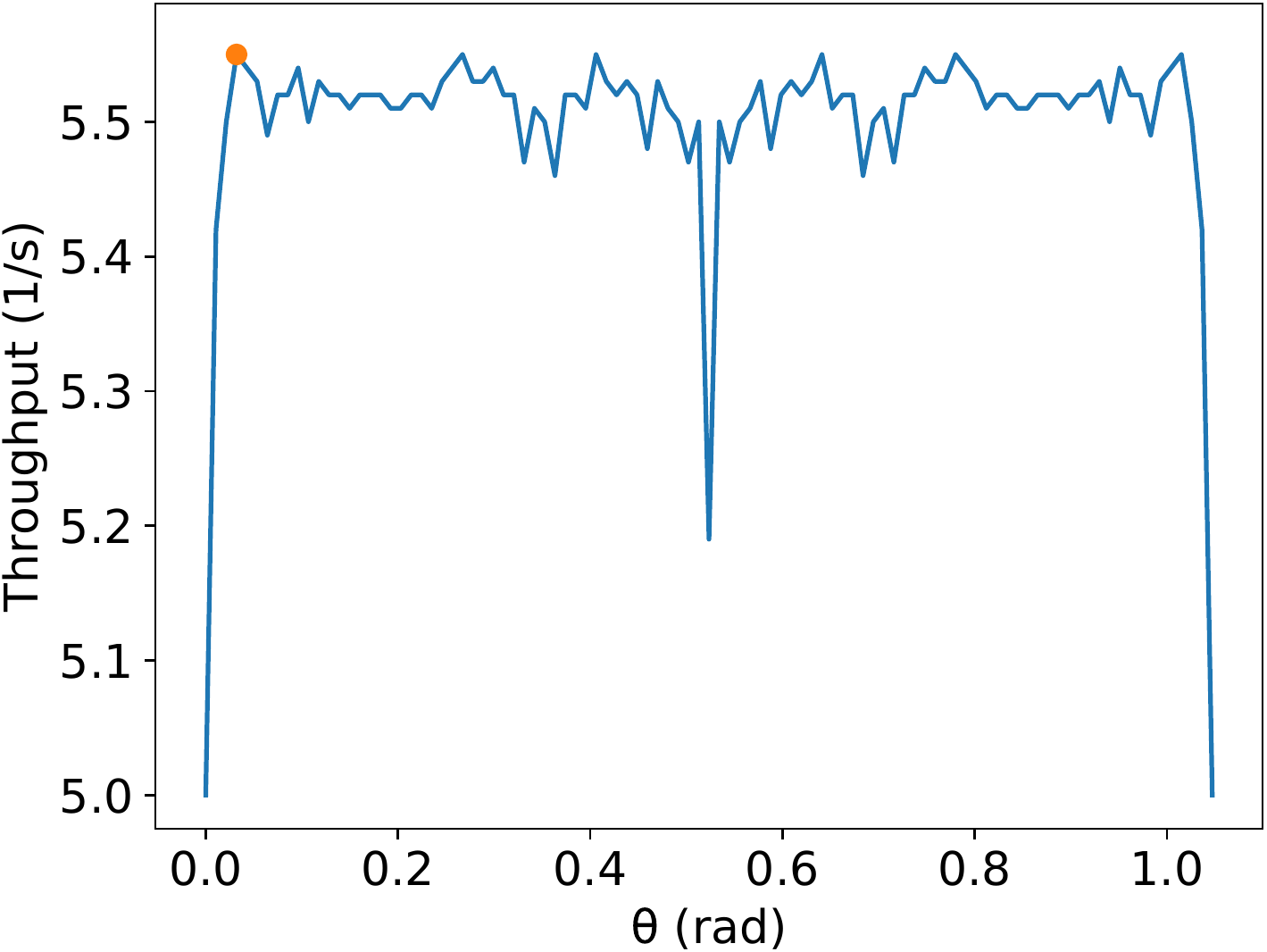}} 
  \,
  \subfloat[For 100 samples, $T = 100$ s, $s=2.40513$ m and $d=1$ m.]{\includegraphics[width=0.483\columnwidth]{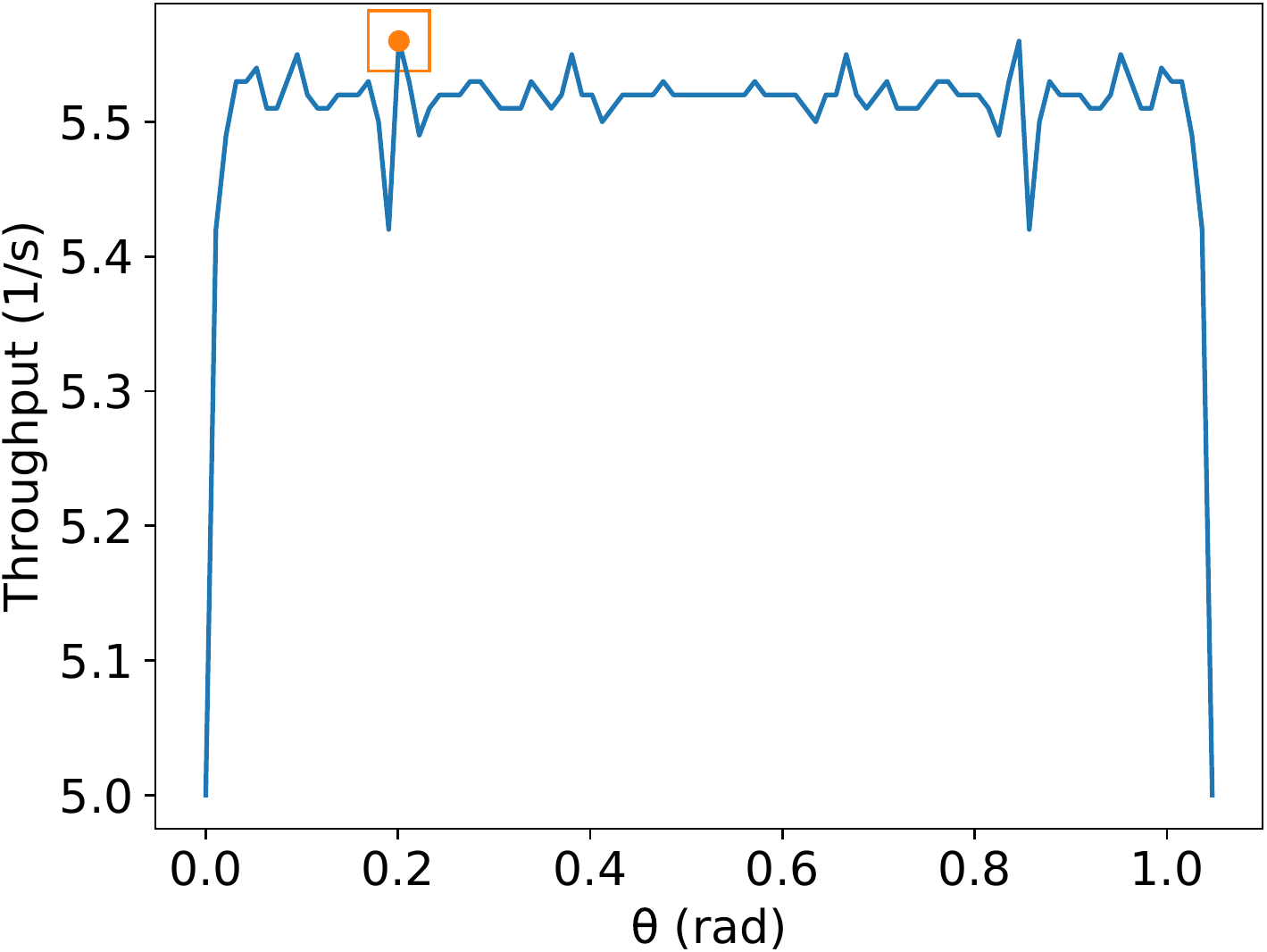}} 
  \caption{Continuation of Figure \ref{fig:plottheta1}.}
  \label{fig:plottheta2}
\end{figure}

\begin{figure}[t!]
  \centering 
  \subfloat[For $10^{7}$ samples,  $T = 43$ s, $s=3$ m, $d = 1$ m.]{\includegraphics[width=0.483\columnwidth]{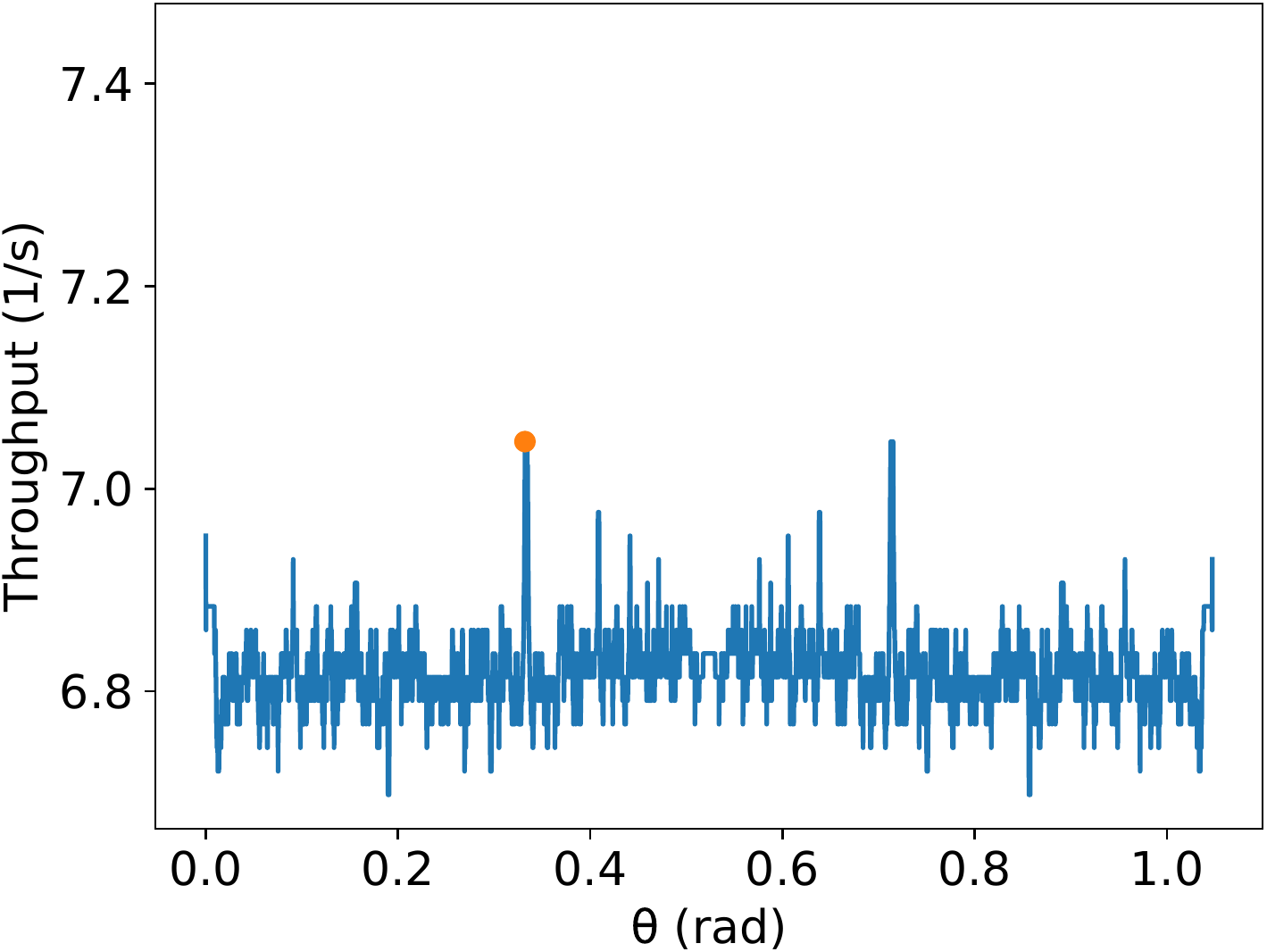}}
  \,
  \subfloat[For $10^{7}+1$ samples, $T = 43$ s, $s=3$ m, $d = 1$ m.]{\includegraphics[width=0.483\columnwidth]{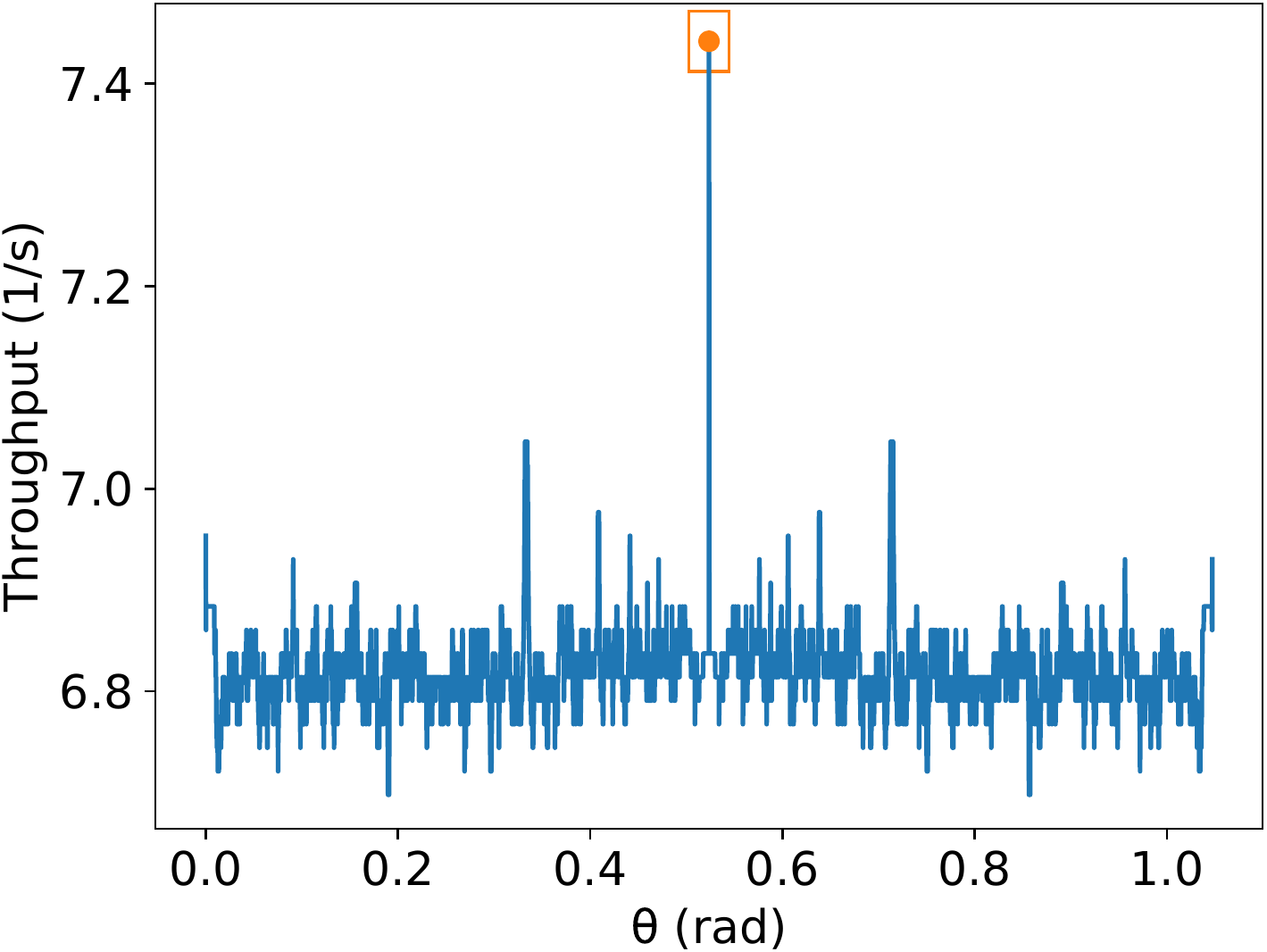}}
  \\
  \subfloat[For $10^{7}$ samples, $T=30$ s, $s=2.5$ m and $d=0.66$ m.]{\includegraphics[width=0.483\columnwidth]{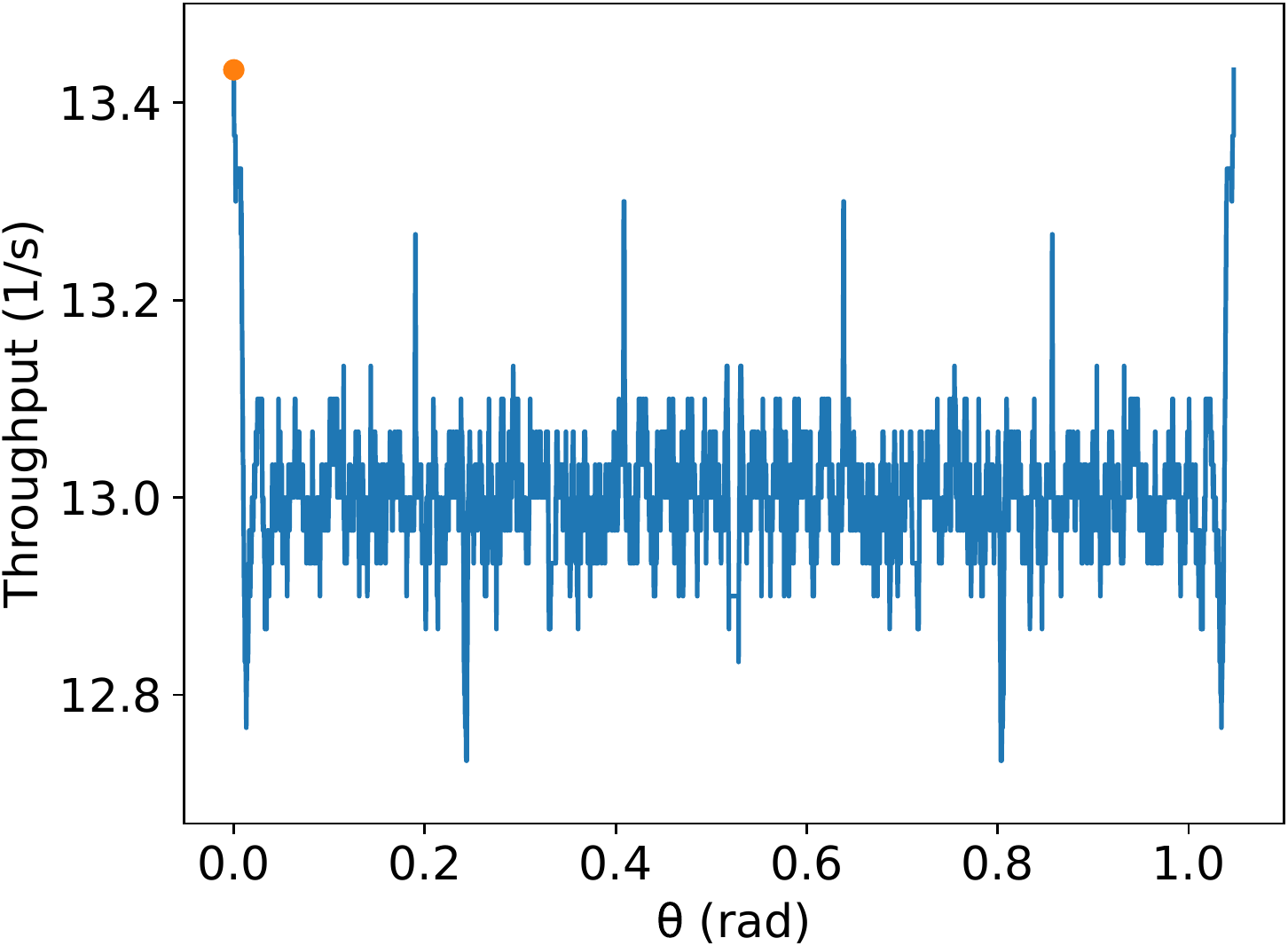}} 
  \,
  \subfloat[For $10^{7}+1$ samples, $T=30$ s, $s=2.5$ m and $d=0.66$ m.]{\includegraphics[width=0.483\columnwidth]{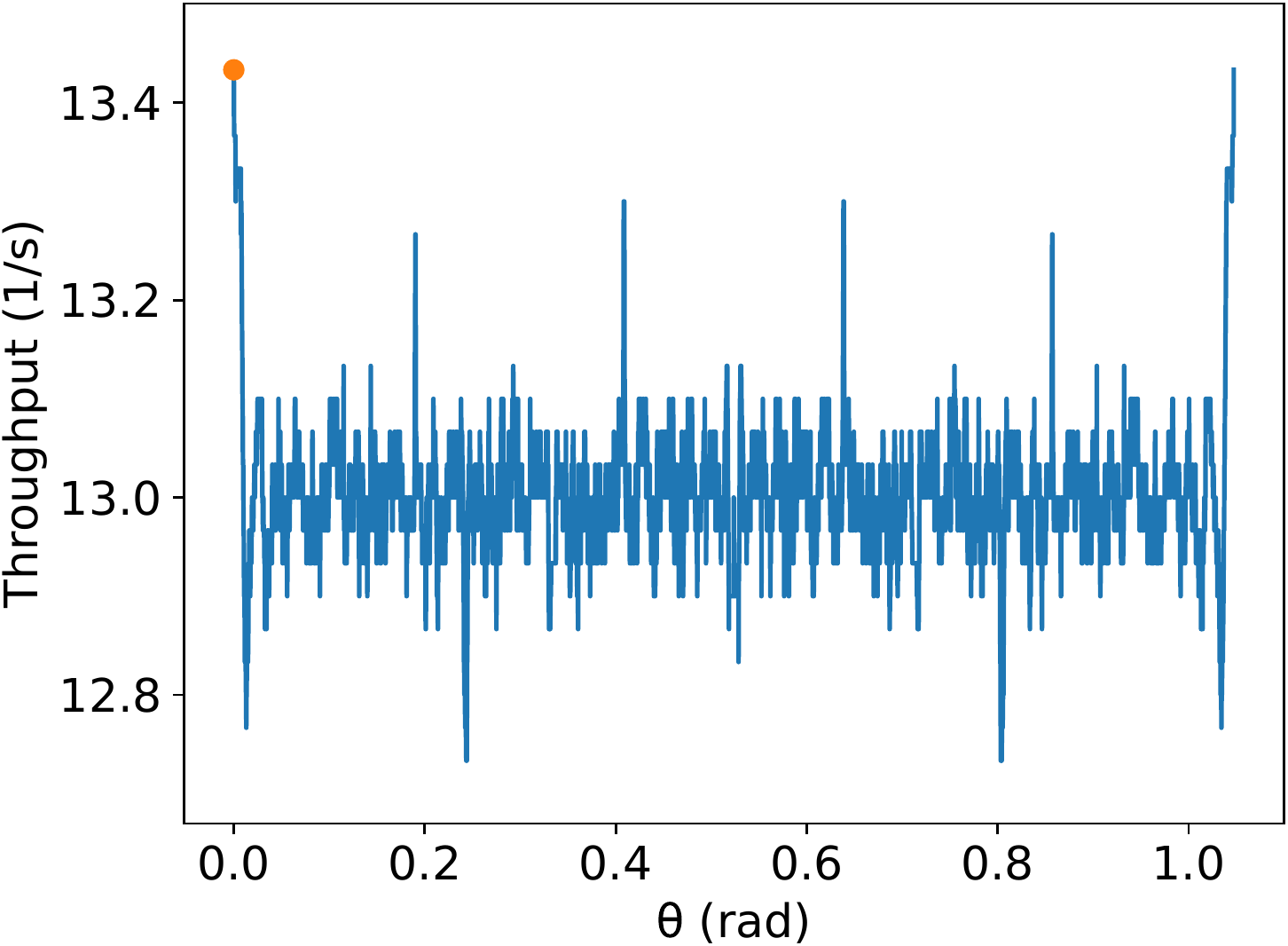}}
  \caption{Similar to Figures \ref{fig:plottheta1} and \ref{fig:plottheta2} but using 10000000 and 10000001 equally spaced points for $\theta \in \lbrack 0,\pi/3 \rparen$. It continues in the Figure \ref{fig:plottheta4}.} 
  \label{fig:plottheta3}
\end{figure}

\begin{figure}[t!]
  \centering 
  \subfloat[For $10^{7}$ samples,  $T = 4$ s, $s = 2$ m and $d = 0.13$ m.]{\includegraphics[width=0.483\columnwidth]{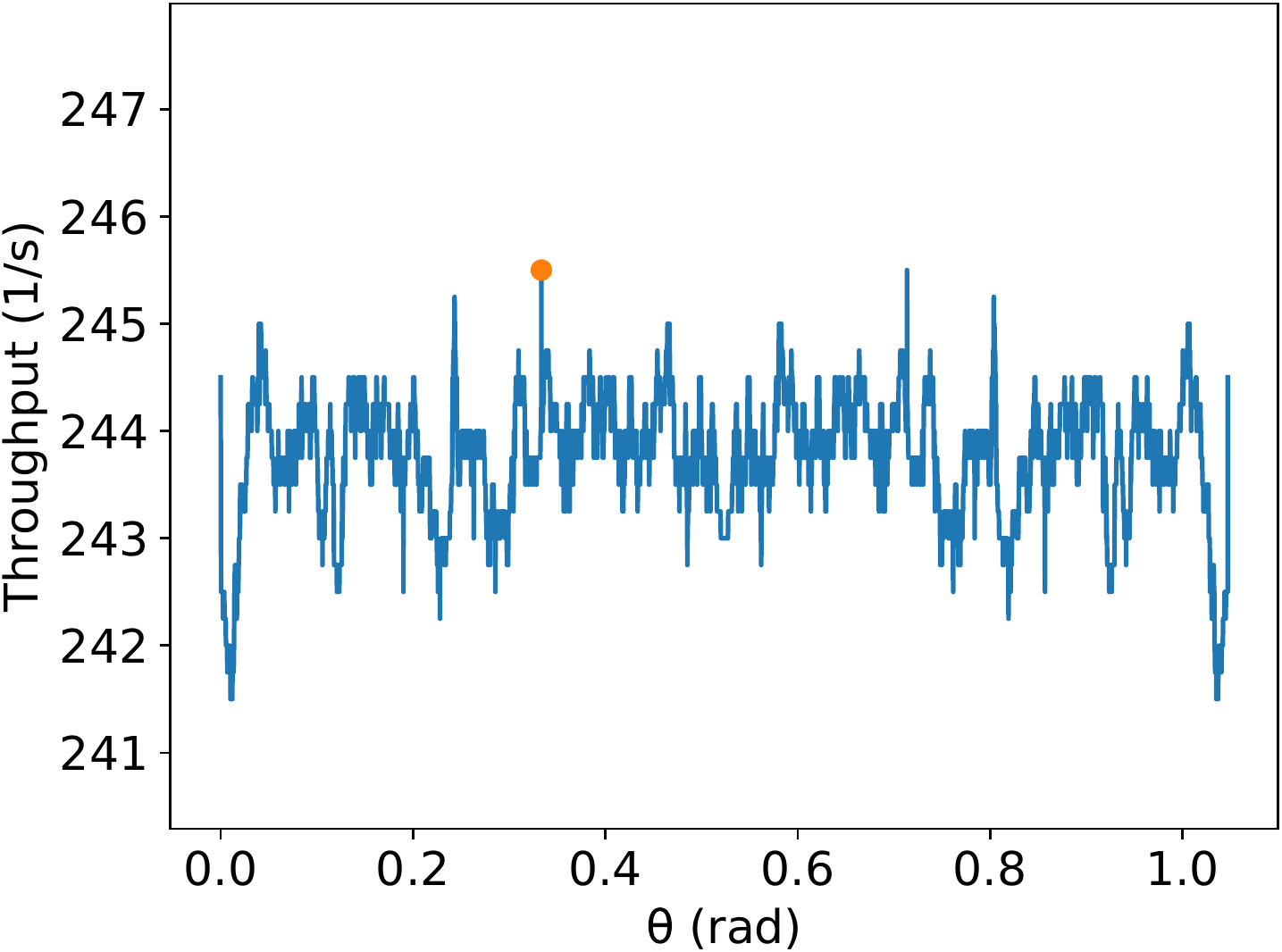}}
  \,
  \subfloat[For $10^{7}+1$ samples, $T = 4$ s, $s = 2$ m and $d = 0.13$ m.]{\includegraphics[width=0.483\columnwidth]{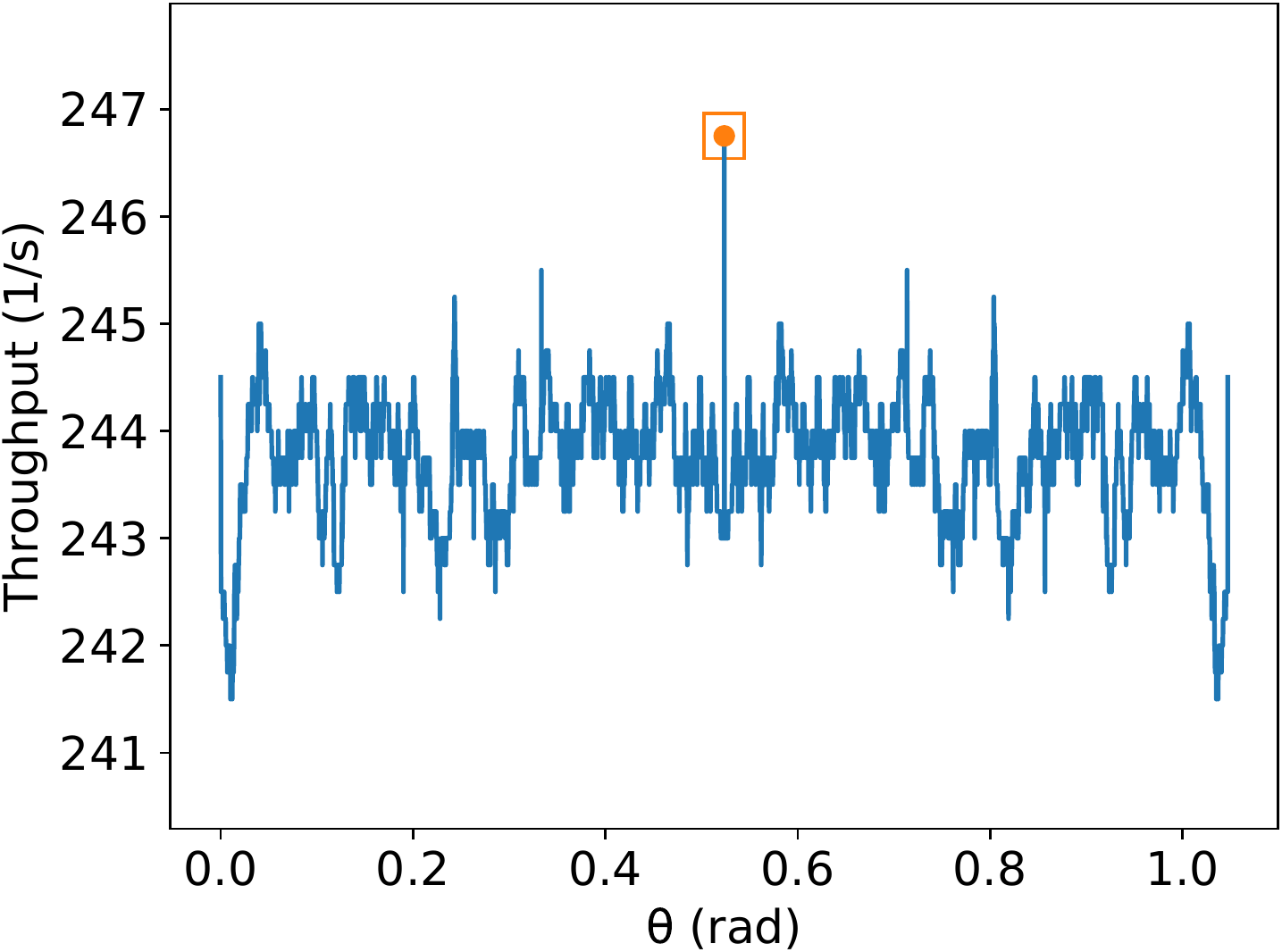}}
  \\
  \subfloat[For $10^{7}$ samples, $T = 100$ s, $s=2.40513$ m and $d=1$ m.]{\includegraphics[width=0.483\columnwidth]{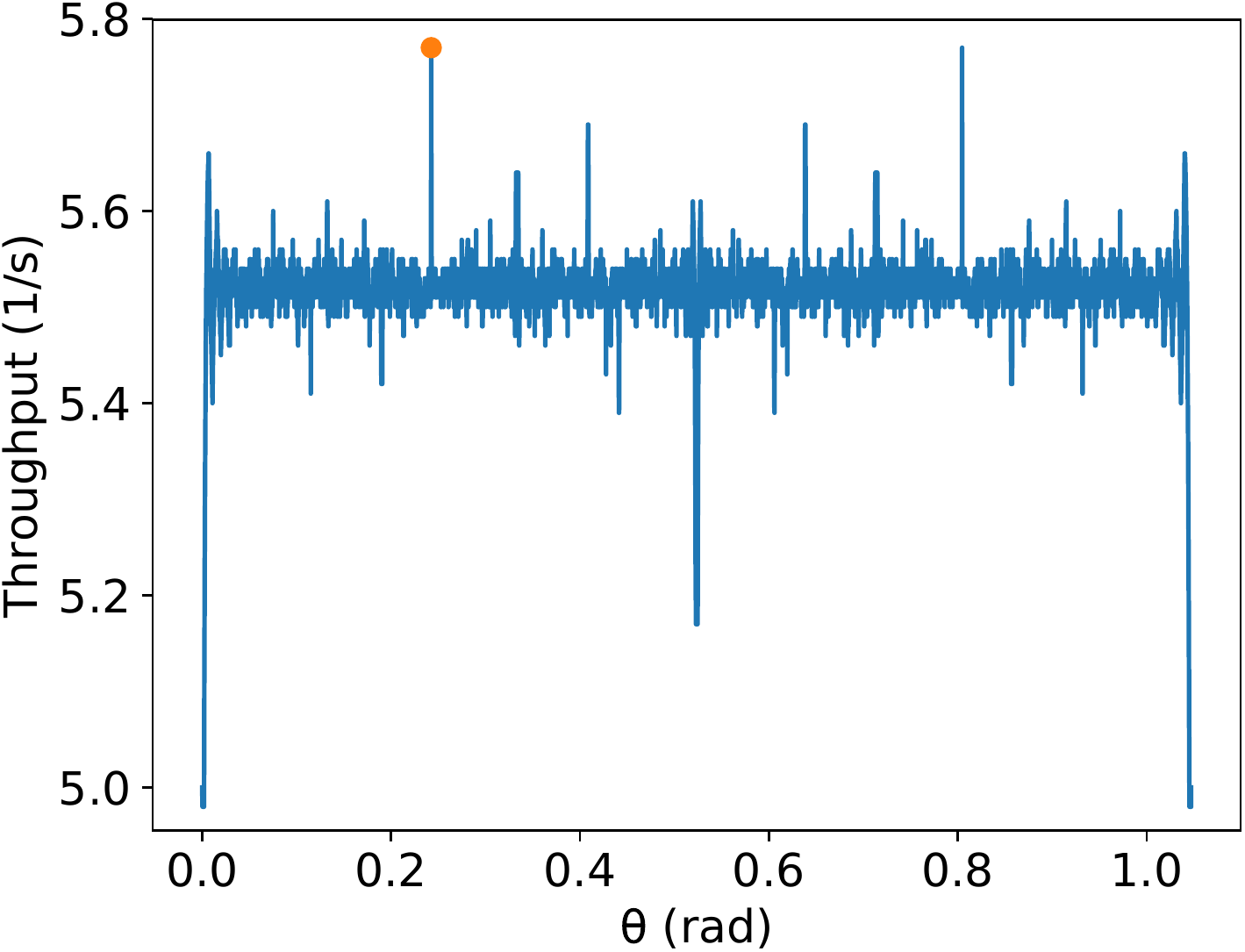}} 
  \,
  \subfloat[For $10^{7}+1$ samples, $T = 100$ s, $s=2.40513$ m and $d=1$ m.]{\includegraphics[width=0.483\columnwidth]{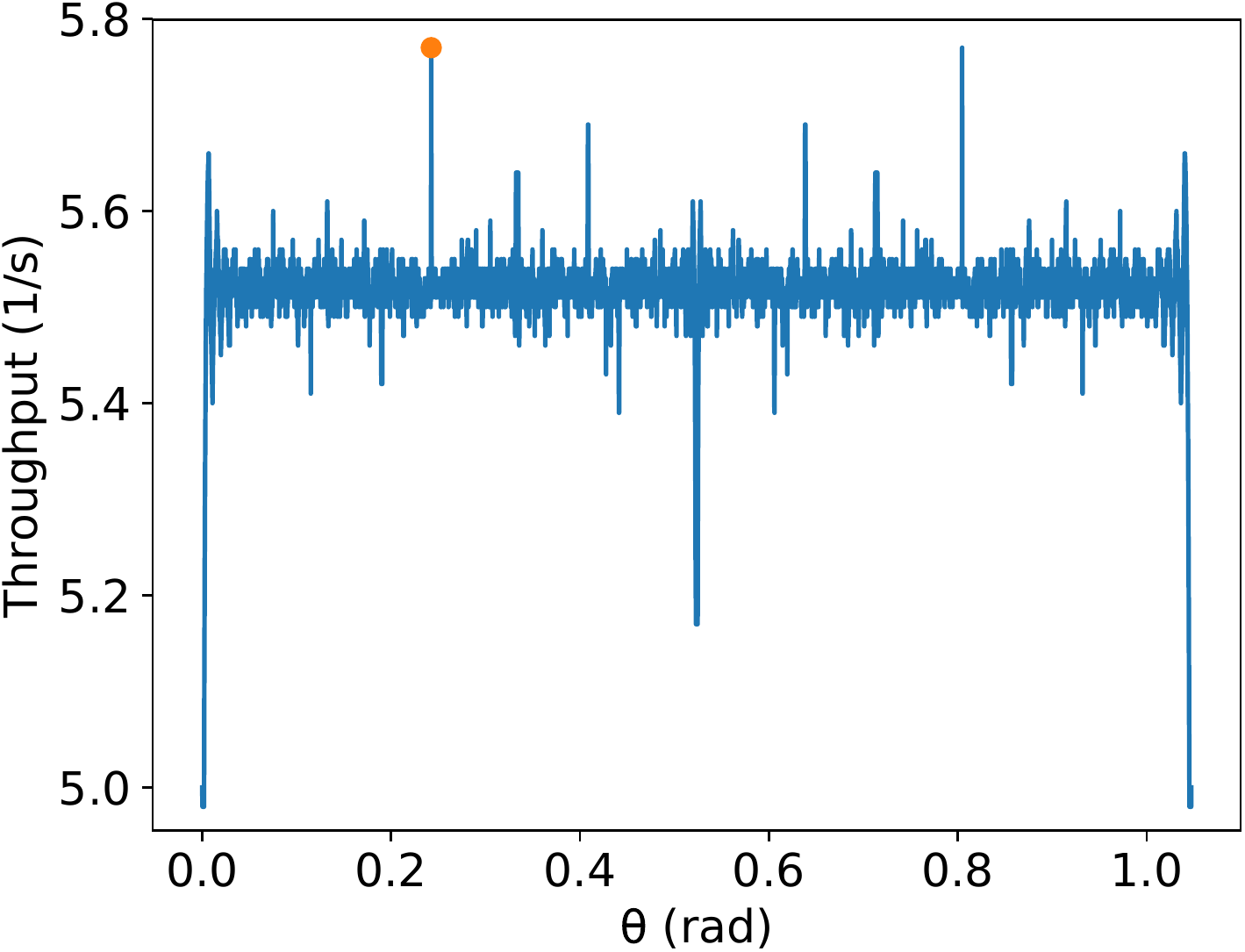}} 
  \caption{Continuation of Figure \ref{fig:plottheta3}.}
  \label{fig:plottheta4}
\end{figure}

On the other hand, due to the discontinuities of (\ref{eq:hexthroughput}), it is difficult to get an exact value of $\theta$ which maximises the throughput given the other parameters. Also, there is no specific value of $\theta$ which achieves the maximum throughput for all possible values of the other parameters. For instance, Figures \ref{fig:plottheta1}-\ref{fig:plottheta4} present the result of this equation for some randomly generated parameters and a different number of samples of $\theta$ equally spaced and taken from the domain interval, that is, from $0$ to $\pi/3$, including these values.

Each one of the Figures \ref{fig:plottheta1}-\ref{fig:plottheta4} presents two different sets of parameters.
In Figures \ref{fig:plottheta1} and \ref{fig:plottheta2}, we use 99 equally spaced values for $\theta \in \lbrack 0,\pi/3 \rparen$ on the left-hand side images and 100 on the right-hand side, then we compare the maximum on each side and choose the best one. We do the same in Figures \ref{fig:plottheta3} and \ref{fig:plottheta4}, but using $10^{7}$ and $10^{7}+1$. 
Figures  
\ref{fig:plottheta1} (a),  
\ref{fig:plottheta2} (a),
\ref{fig:plottheta3} (b)
and \ref{fig:plottheta4} (b)
show an example that $\theta \approx \pi/6$ reaches the maximum throughput, and in Figures 
\ref{fig:plottheta1} (c) and (d), and \ref{fig:plottheta3} (c) and (d)  the maximum is at $\theta = 0$.
Moreover, Figures 
\ref{fig:plottheta2} (c) and (d) has its maximum for $\theta$ different from the other examples. Figures
\ref{fig:plottheta1} (c) and \ref{fig:plottheta1} (d)
have the same maximum, despite the plotting being different. This also occurs in Figures
\ref{fig:plottheta3} (c) and (d), and Figures
\ref{fig:plottheta4} (c) and \ref{fig:plottheta4} (d).
If we know the parameters, we can find an approximate best candidate for $\theta$  by searching several values, as we presented. However, as far as we know, getting the true value which maximises that equation by a closed form is an open problem.

Additionally, notice that whenever the number of samples is odd, it is sampled the value $\theta= \pi/6$. We observe in these figures that when the maximum is at $\theta = \pi/6$, it tends to be higher than the maximum found without considering it. For instance, compare the maximum found on the pairs (a) and (b) in Figures \ref{fig:plottheta1}-\ref{fig:plottheta4}. On the other hand, $\theta = \pi/6$ is not always the optimal value. 
Thus, we suggest to compute first the value for $\theta=\pi/6$ and compare it with the result for a search for the maximum for any chosen number of samples in the interval from $\theta \in \lbrack 0,\pi/3 \rparen$.

\subsubsection{Touch and run strategy}
\label{sec:touchandrun}

\begin{figure}[t]
  \centering
  \subfloat[]{
    \includegraphics[width=0.47\columnwidth]{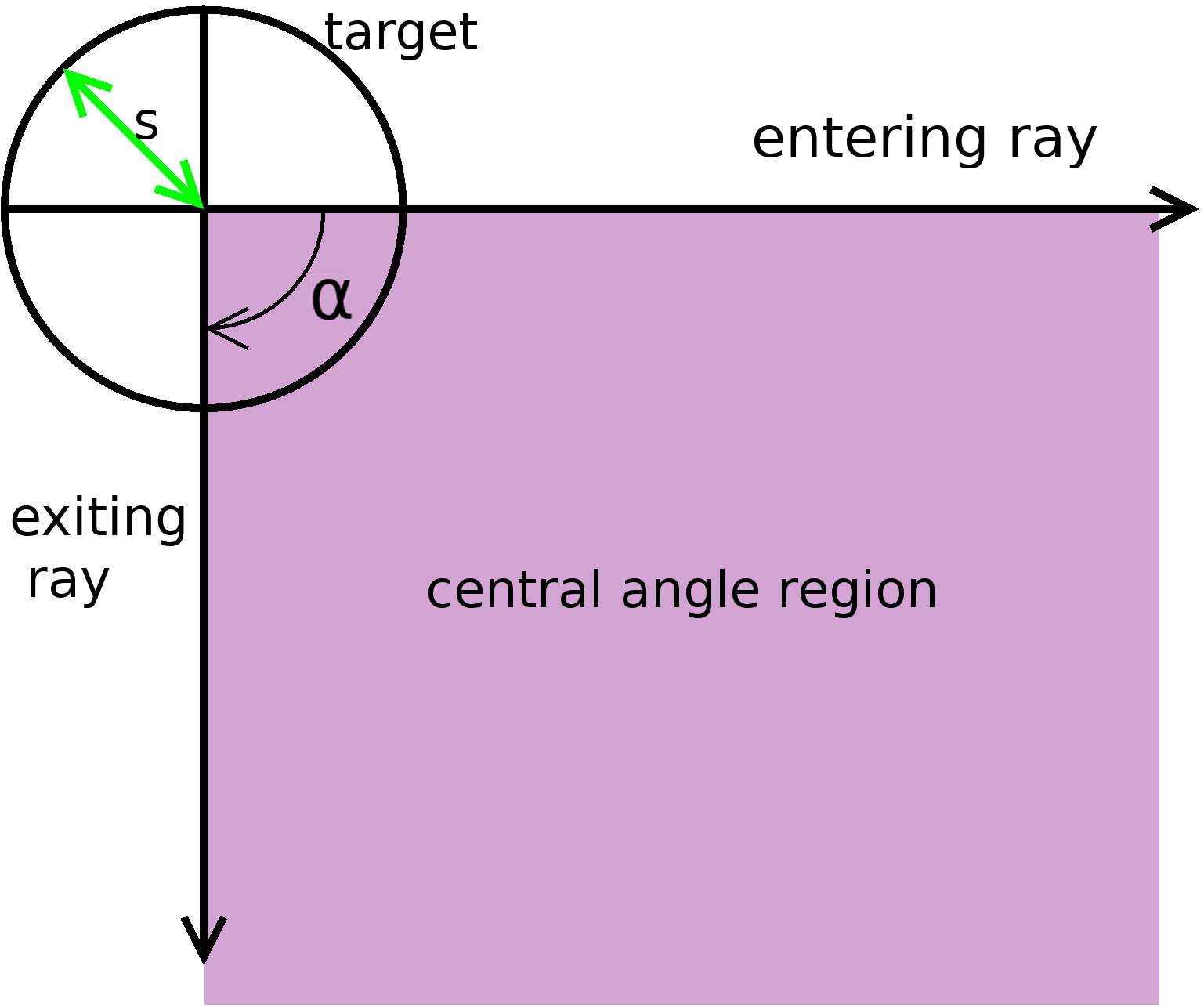}
  }
  \subfloat[]{
    \includegraphics[width=0.47\columnwidth]{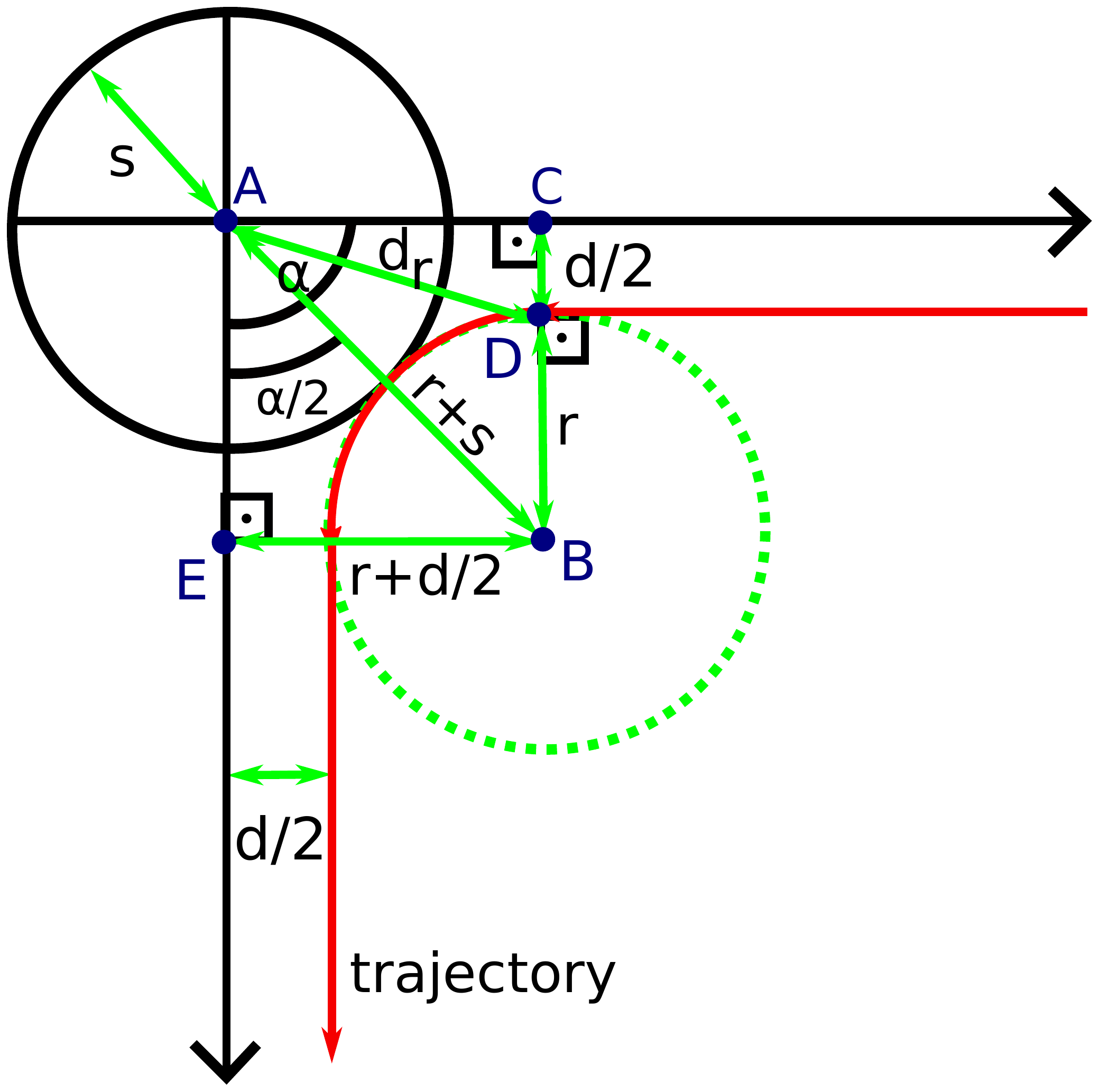}
  }
  \caption{Illustration of the touch and run strategy. (a) Central angle region and its exiting and entering rays defined by the angle $\alpha$. (b) Trajectory of a robot next to the target in red. Here we have the relationship between the target area radius ($s$), the minimum safety distance between the robots ($d$), the turning radius ($r$), the central region angle ($\alpha$) and the distance from the target centre for a robot to begin turning ($d_{r}$) -- used as justification for (\ref{eq:distTargetToTurn}) and (\ref{eq:relationangles}). The green dashed circle represents the whole turning circle.}
  \label{fig:theoretical:central_angle_linka}
\end{figure}

We now discuss the \emph{touch and run} strategy. Since a robot should spend as little time as possible near the target, we imagined a simple scenario where robots travel in predefined curved lanes and tangent to the target area where they spend minimum time on the target. 
To avoid collisions with other robots, the trajectory of a robot nearby the target is circular,
and the distance between each robot must be at least $d$ at any part of the trajectory. 
Hence, no lane crosses another, and each lane occupies a region defined by an angle in the target area that we denote by $\alpha$, shown in Figure \ref{fig:theoretical:central_angle_linka} (a).

Figure \ref{fig:theoretical:central_angle_linka} (b) shows the trajectory of a robot towards the target region following that strategy. The robot first follows the boundary of the central angle region -- that is, the entering ray -- at a distance $d/2$.
Then, it arrives at a distance of $s$ of the target centre using a circular trajectory with a turning radius $r$. 
As the trajectory is tangent to the target shape, it is close enough to consider the robot reached the target region.
Finally, the robot leaves the target by following the second boundary of the central angle region -- that is, the exiting ray -- at a distance $d/2$. Depending on the value of $\alpha$, it is possible to fit several of these lanes around the target.
For example, when $\alpha = \pi / 2$, it is possible to fit 4 lanes (Figure \ref{fig:theoretical:trajectory}). The robots in each lane must maintain a distance $d_{o}$ between each other -- which is calculated depending on the values of $d$, $s$, $r$ and the number of lanes $K$ as shown later.

\begin{figure}[t!]
  \centering
  \includegraphics[width=0.9\columnwidth]{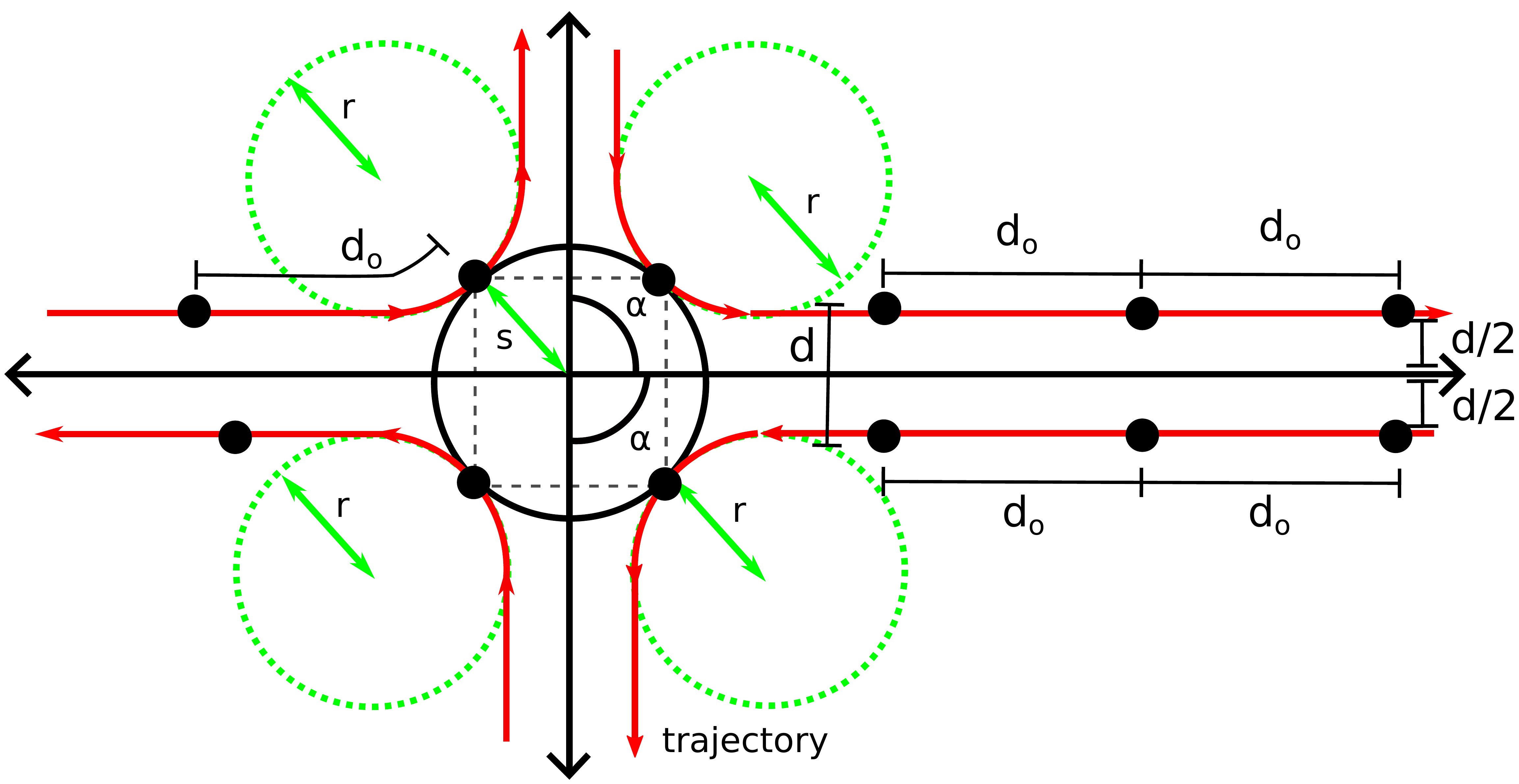} 
  \caption{Theoretical trajectory in red, for $\alpha = \pi/2$ and $K=4$. Robots are black dots and $d_{o}$ is the desired distance between the robots in the same lane. When  robots of all lanes simultaneously occupy the target region, their positions are the vertices of a regular polygon -- here, it is represented by a grey square inside the target region. }
  \label{fig:theoretical:trajectory}
\end{figure}

The lemma below concerns the distance to the target centre where the robots will start turning on the curved path. It will also be useful in the discussion about experiments using this strategy on Section \ref{sec:hitandrunexperiments}. 

\begin{lemma}
  The distance $d_{r}$ to the target centre for the robot to start turning is 
  \begin{equation}
     d_{r} = \sqrt{s(2r+s)-r d}. 
    \label{eq:distTargetToTurn}
  \end{equation}
  \label{lemma:drturn}
\end{lemma}
\begin{proof}
  \ifithasappendixforlemmas %
      See Online Appendix.
  \else %
  In Figure \ref{fig:theoretical:central_angle_linka} (b) we show the distance $d_{r}$ from the target centre where the robots begin turning.  By symmetry, this is the same distance from the target centre where the robots stop turning. From the right triangle $ABC$ on that figure, we have $\vert \overline{AC}\vert = \sqrt{(r+s)^{2} - (r+d/2)^{2}}$ and from $\bigtriangleup ACD$, $d_{r} = \sqrt{(d/2)^{2} + \vert \overline{AC}\vert ^{2}}$. Thus,
  \if\shortVersion 1
    $
     d_{r} = \sqrt{(d/2)^{2} + (r+s)^{2} - (r+d/2)^{2}}  = \sqrt{s(2r+s)-r d}. 
    $
  \else
    $$
     d_{r} = \sqrt{(d/2)^{2} + (r+s)^{2} - (r+d/2)^{2}} = \sqrt{s(2r+s)-r d}. 
    $$
  \fi
  \fi %
\end{proof}

We now present a lemma about the turning radius, then we define the domain of $K$ and $\alpha$, in order to calculate the throughput for the touch and run strategy.

\begin{lemma}
  The central region angle $\alpha$, the minimum distance between the robots $d$ and the turning radius $r$ are related by
  \begin{equation}
  r = \frac{s   \sin(\alpha / 2) - d/2}{1 - \sin(\alpha / 2)}.
  \label{eq:relationangles}
  \end{equation}
  \label{prop:relationangles}
\end{lemma}
\begin{proof} 
  \ifithasappendixforlemmas %
      See Online Appendix.
  \else %
  From Figure \ref{fig:theoretical:central_angle_linka} (b), we can see that the right triangle ABE has angle $\widehat{EAB} = \alpha/2$,  hypotenuse $r + s$ and cathetus $r + d/2$. Hence, it directly follows that
  \if\shortVersion 1
    $
    sin(\alpha / 2) = \frac{r + d/2}{r + s}
    \LR
    r = \frac{s   \sin(\alpha / 2) - d/2}{1 - \sin(\alpha / 2)}.
    $
  \else
    $$
    sin(\alpha / 2) = \frac{r + d/2}{r + s}
    \LR
    r = \frac{s   \sin(\alpha / 2) - d/2}{1 - \sin(\alpha / 2)}.
    $$
  \fi
  \fi %
\end{proof}

\begin{proposition}
  Let $K$ be the number of curved trajectories around the target area, $\alpha$ be the angle of each central area region, and $r$ the turning radius of the robot for the curved trajectory of this central area region. For a given $d > 0$ and $s \ge d/2$, the domain of $K$ is
  \begin{equation} 
    3 \le K  \le \frac{\pi}{\arcsin \left( \frac{d}{2s} \right)}, \text{ and }
    \label{eq:Kbounds}
  \end{equation}
  \begin{equation}
    \alpha = \frac{2 \pi}{K}.
    \label{eq:ak}
  \end{equation}
  \label{prop:Kboundsrk}
\end{proposition}
\begin{proof}
The number of trajectories $K$ must be greater or equal to 3.
	The reason is that for the minimum possible value for $s$, $s = d/2$, $K = 2$ is enough to have parallel lanes. However, starting with $K = 3$, 
	curved trajectories are needed to guarantee that robots of one lane do not interfere with robots 
	from another lane.

Also, we have $K$ identical trajectories around the target, each taking a central angle of $\alpha$.
As a result, the value of $\alpha$ given $K$ is
$
\alpha = \frac{2 \pi}{K},
$
implying that
$
0 < \alpha \le \frac{2 \pi}{3}. 
$

Additionally, in the worst case, one robot of each lane arrives in the target region at the same time. When robots of all lanes simultaneously occupy the target region, their positions can be seen as the vertices of a regular polygon which must be inscribed in the circular target region of radius $s$ (e.g., in Figure \ref{fig:theoretical:trajectory} we have a square whose sides are greater than $d$). The number of robots on the target region at the same time must be limited by the maximum number of sides of an inscribed regular polygon of a side with minimum side greater or equal to $d$. The side of a K regular polygon inscribed in a circle of radius $s$ measures $2 s \sin\left(\frac{\pi}{K}\right)$. Thence, $2 s \sin\left(\frac{\pi}{K}\right) \ge d \Rightarrow  \frac{\pi}{\arcsin\left(\frac{d}{2 s} \right)} \ge K.$  
\end{proof}

Now that we have determined the correct parametrisation for the touch and run strategy, we determine its throughput in the next proposition.

\begin{proposition}
  Assuming the touch and run strategy and that the first robot of every lane begins at the same distance from the target, given a target radius $s$,
  the constant linear robot speed $v$,
  a minimum distance  between robots $d$
  and the number of lanes $K$, the throughput for a given instant $T$ is given by
  
  \begin{equation}
  f_{t}(K,T) = 
    \frac{1}{T}\left(K\left\lfloor \frac{vT}{d_{o}} +1 \right\rfloor - 1\right), \text{ for }
    \label{eq:throughputhitandruntime}
  \end{equation}
  \begin{equation}
    d_{o} = \max(d,d'), \text{ and }
    \label{eq:do}
  \end{equation}
  \begin{equation}
    d' = 
    \begin{cases}
       r  (\pi - \alpha)  + \frac{d - 2  r  \cos(\alpha / 2)}  {\sin(\alpha / 2)}, & \text{ if } 2  r  \cos(\alpha / 2) < d,\\
      2  r  \arcsin\left( \frac{d}{2 r}\right), & \text{ otherwise, }
    \end{cases}
    \label{eq:dprime}
  \end{equation}
  with $r$ obtained from (\ref{eq:relationangles}). Also,
  \begin{equation}
  f_{t}(K) = \lim_{T\to \infty}{f_{t}(K,T)}= \frac{Kv }{d_{o}}. 
  \label{eq:throughputhitandrunlimit}
  \end{equation}
  \label{prop:throughputsdk}
\end{proposition} 
\begin{proof}
  \begin{figure}[t!]
    \centering
    \includegraphics[width=0.55\columnwidth]{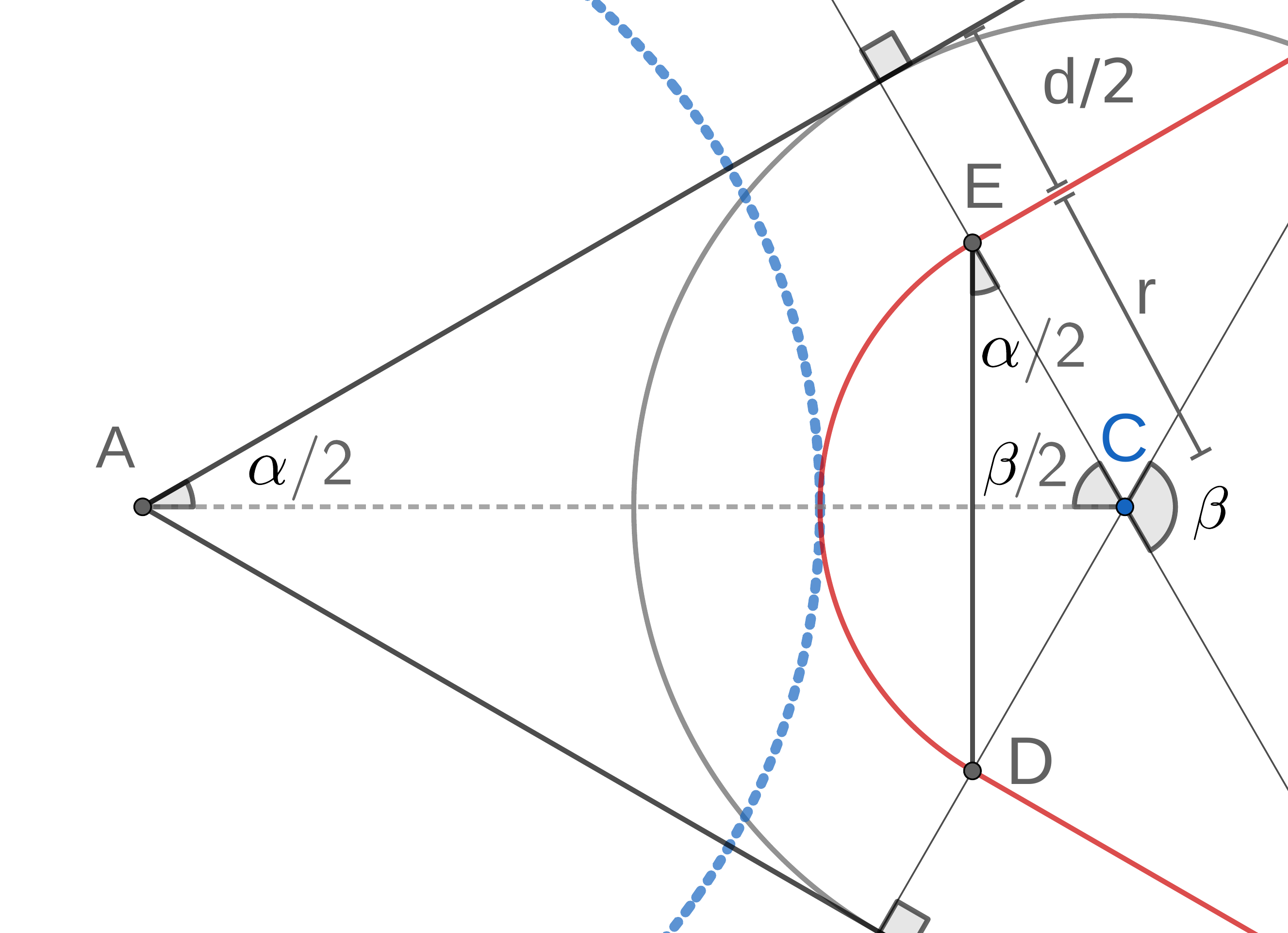}
    \caption{ The red line represents the trajectory of robots in one lane. $\alpha$ is the central angle for the lane. The dashed blue circle of centre A is the target. C is the centre of the circle of radius $r$ from the circular trajectory. The grey circle of centre C has a radius of $r + d/2$. Points D and E represent the connection between the curved path and the straight path. We have $\beta = \pi - \alpha$ due to the symmetry and the fact that the sum of the angles of $\bigtriangleup \mathit{ECD}$ is equal to $\pi$.}
    \label{fig:coolpaths:1}
  \end{figure}
  
  Using the touch and run strategy, each lane is distant by at least $d$ from each other. However, the minimum distance between robots on the same lane $d_{o}$ must be checked at the beginning of the curved path, as their distance decreases if assuming linear constant velocity. We distinguish two cases based on Figure \ref{fig:coolpaths:1}:
  \begin{enumerate}
  \item  $\vert \overline{ED}\vert < d$:
  Two robots cannot be on the lane curved path; 
  \item $\vert \overline{ED}\vert \ge d$:
  More than one robot can occupy the lane curved path. 
  \end{enumerate}
  
  These cases affect the minimum distance between robots $d_{o}$ such that they can follow the trajectory without decreasing their linear speed. In both cases, they need to satisfy the minimum distance $d$ if they are turning on the curved path.  From Figure \ref{fig:coolpaths:1}, 
  \begin{equation}
  \vert \overline{ED}\vert =
  2r \sin \left( \frac{\beta}{2} \right) = 
  2r \sin \left( \frac{\pi}{2} - \frac{\alpha}{2} \right)
  = 2r \cos
  \left(
  \frac{\alpha}{2}
  \right).
  \label{eq:EDv}
  \end{equation}

  \begin{figure}
    \centering 
    \subfloat[$\vert \overline{ED}\vert < d$]{
      \includegraphics[width=0.45\linewidth]{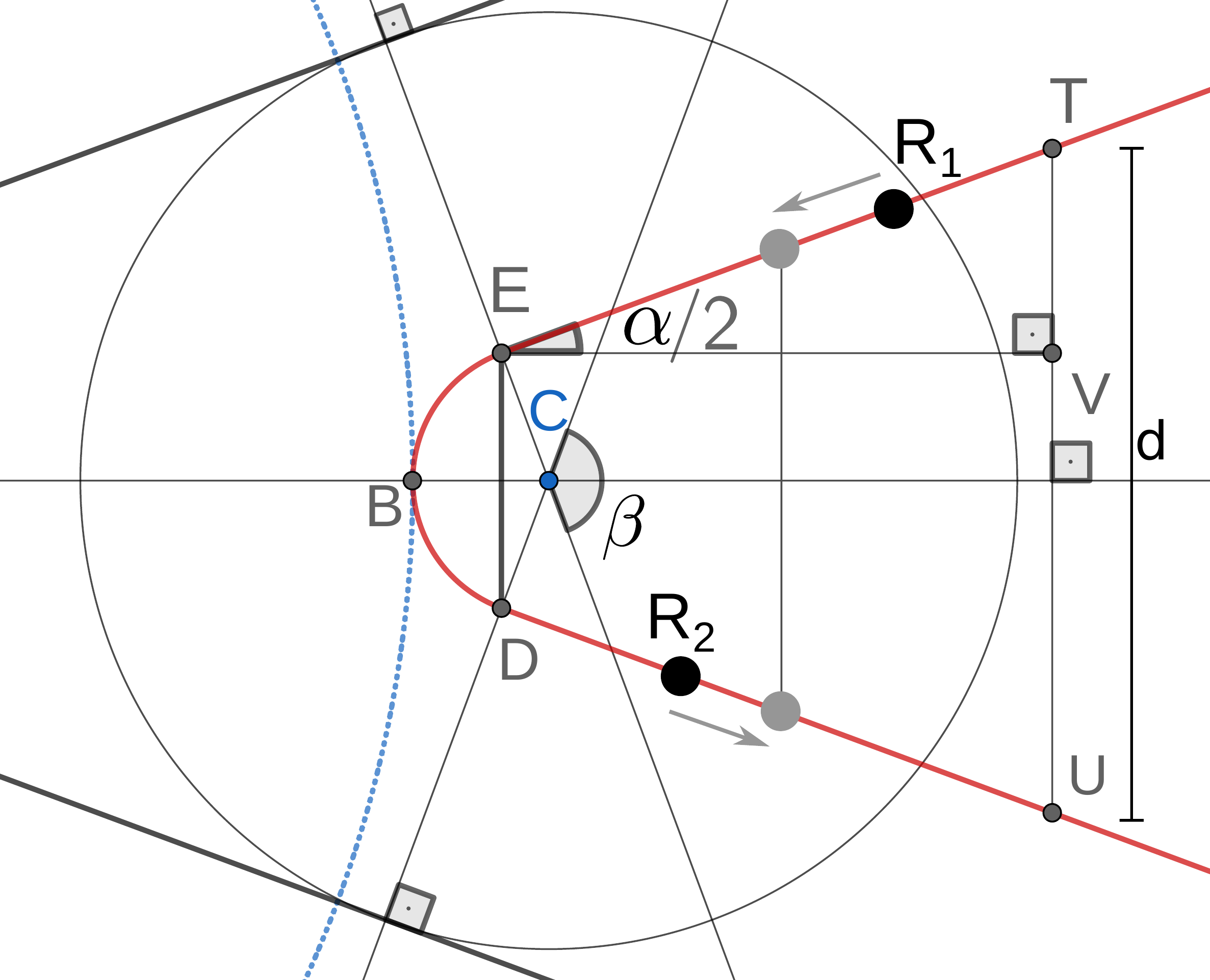}
    }\qquad
    \subfloat[$\vert \overline{ED}\vert \ge d$]{  
      \includegraphics[width=0.25\linewidth]{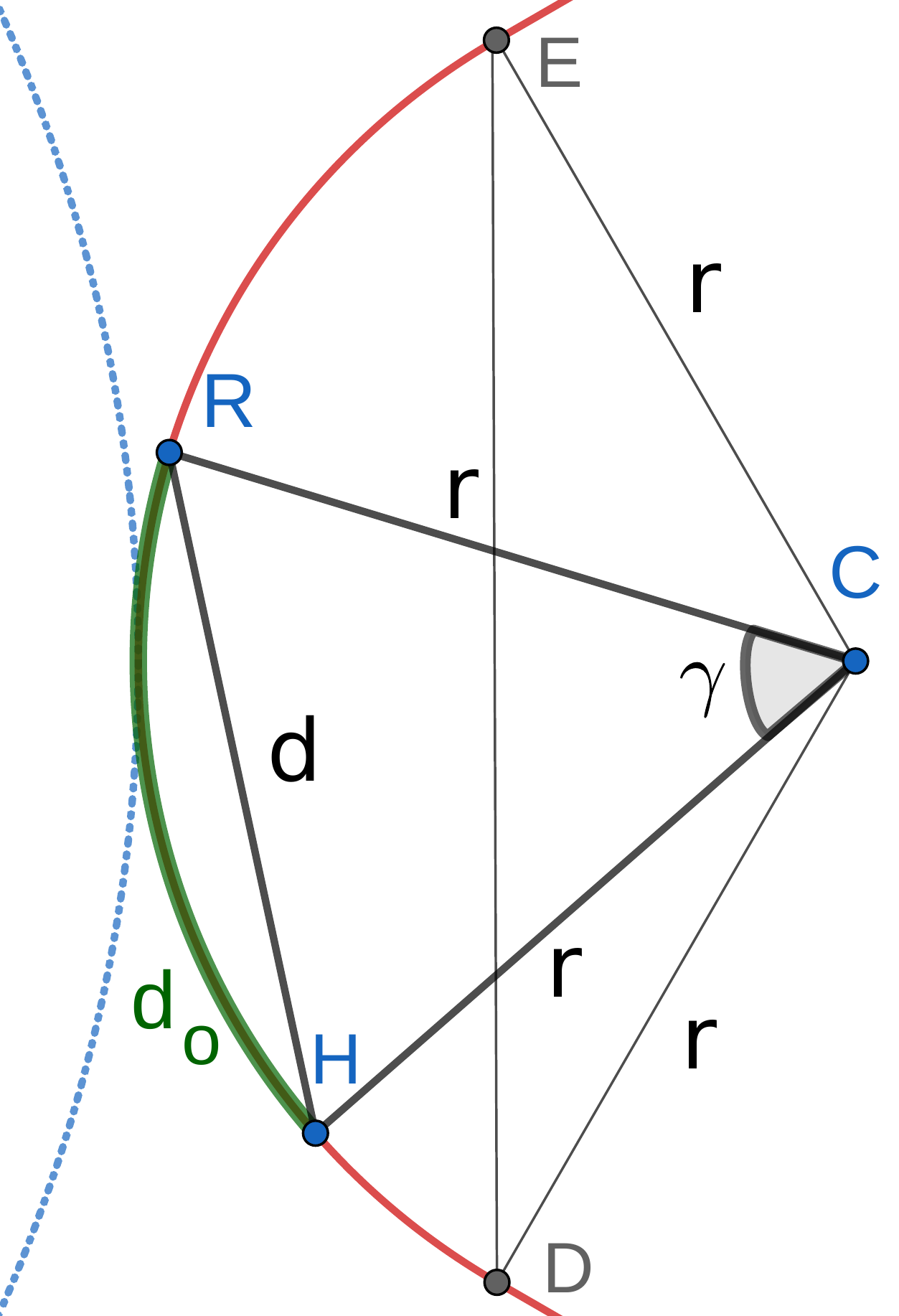}
    }
    \caption{Enlargements of Figure \ref{fig:coolpaths:1}. (a) The robots $R_{1}$ and $R_{2}$ are the black dots on the red line representing the trajectory. If the delay between $R_{1}$ and $R_{2}$ is less than the time for a robot to run from $T$ to $U$ following the red trajectory, there will be some instant which $R_{1}$ and $R_{2}$ will be vertically aligned. Their positions at that instant are represented by grey dots in front of them. Hence, their distance would be less than $d$. The right triangle $TVE$ has side $\overline{TV}$, which can be measured using $\overline{ED}$. (b) $d_{o}$ denotes the minimum arc length for two robots located at any two points $R$ and $H$ on $\wideparen{ED}$ such that they are distant by at least $d$. $\gamma$ is the angle defining the arc $d_{o}$ for the circle of centre $C$.}
    \label{fig:coolpaths23}
  \end{figure}

  In the case 1, in Figure \ref{fig:coolpaths23} (a), we define two points T and U on the lane such that the distance between them is $\vert \overline{TU}\vert = d$ and their distances to the target are equal.
  The delay between one robot at T and another at U is equal to 
  \if\shortVersion 1
    $\Delta t_1 = \frac{\vert \overline{TE}\vert + \vert \wideparen{ED}\vert + \vert \overline{DU}\vert }{v},$
  \else
    $$\Delta t_1 = \frac{\vert \overline{TE}\vert + \vert \wideparen{ED}\vert + \vert \overline{DU}\vert }{v},$$
  \fi
  that is, the time for running through the straight line TE, the curved path ED and the straight line DU.
  
  For any delay less than  $\Delta t_1 $ between two robots, say $R_{1}$ and $R_{2}$, there is an instant of time when $R_{1}$ is on the path between B and T and $R_{2}$ is on the path between B and U and they are vertically aligned (Figure \ref{fig:coolpaths23} (a)). In this case, the distance between $R_{1}$ and $R_{2}$ is below $\vert \overline{TU}\vert $, so they do not respect the minimum distance $d$ between them.
  Hence, the minimum delay between two robots in case 1 is $\Delta t_1$.
  
  From Figure \ref{fig:coolpaths:1}, we have
  $\vert \wideparen{ED}\vert =  r \beta =  r (\pi - \alpha)$.
  For calculating the value of $\vert \overline{TE}\vert $ and $\vert \overline{DU}\vert $ from Figure \ref{fig:coolpaths23} (a), we observe that $\vert \overline{TE}\vert = \vert \overline{DU}\vert $ by symmetry. Thus,
  \if\shortVersion 1
    $
    \vert \overline{VT}\vert 
       = \frac{d}{2} - \frac{\vert \overline{ED}\vert }{2} 
       = \frac{d}{2} - r \cos \left( \frac{\alpha}{2} \right) 
    $
    From Figure \ref{fig:coolpaths23} (a) and (\ref{eq:EDv}).
  \else
    $$
    \begin{aligned}
    \vert \overline{VT}\vert 
      & = \frac{d}{2} - \frac{\vert \overline{ED}\vert }{2} 
        & [\text{From Figure \ref{fig:coolpaths23} (a)}]\\
      & = \frac{d}{2} - r \cos \left( \frac{\alpha}{2} \right) 
      & [\text{From (\ref{eq:EDv})}].\\
    \end{aligned}
    $$
  \fi
  As $\bigtriangleup \mathit{TVE}$ is right, $\vert \overline{TE}\vert = \frac{\vert \overline{VT}\vert }{\sin(\alpha / 2)}$. Thence,
  \if\shortVersion 1
    $
    \Delta t_1 = 
    \frac{r (\pi - \alpha) + 
    2 
    \frac{d/2 - r \cos(\alpha / 2)}{\sin(\alpha / 2)}
    }{v} = 
    \frac{r (\pi - \alpha)}{v}
    +
    \frac{d - 2 r \cos(\alpha / 2)}{v \sin(\alpha / 2)}
    $
  \else
    $$
    \Delta t_1 = 
    \frac{r (\pi - \alpha) + 
    2 
    \frac{d/2 - r \cos(\alpha / 2)}{\sin(\alpha / 2)}
    }{v} = 
    \frac{r (\pi - \alpha)}{v}
    +
    \frac{d - 2 r \cos(\alpha / 2)}{v \sin(\alpha / 2)}
    $$
  \fi
  and 
  \if\shortVersion 1
    $d_{o} = \max\left(d, v \Delta t_1\right)  
     = \max\left(d,r  (\pi - \alpha)  + \frac{d - 2  r  \cos\left(\alpha/2\right)}{\sin(\alpha/2)}\right).$
    Here
  \else
    $$d_{o} = \max\left(d, v \Delta t_1\right)  
     = \max\left(d,r  (\pi - \alpha)  + \frac{d - 2  r  \cos\left(\alpha/2\right)}{\sin(\alpha/2)}\right).$$
    Above
  \fi
  we used $\max$ function because the result of $v \Delta t_{1}$ can still be less than $d$, depending on $\alpha$, $r$ and $d$.
  
  In the case 2,  we need to check the minimum distance $d$ when two robots are on the circular part $\wideparen{ED}$ in Figure \ref{fig:coolpaths23} (b). From this figure, $\bigtriangleup CRH$ is isosceles, so $\gamma = 2 \arcsin\left( \frac{d}{2r}\right)$. Thus, to keep constant velocity, the delay between two robots in this case is
  \if\shortVersion 1 
    $
    \Delta t_2 = \frac{d_{o}}{v} = \frac{r  \gamma}{v} = \frac{2r}{v} \arcsin \left( \frac{d}{2r} \right).
    $
  \else
    $$
    \Delta t_2 = \frac{d_{o}}{v} = \frac{r  \gamma}{v} = \frac{2r}{v} \arcsin \left( \frac{d}{2r} \right).
    $$
  \fi
  Then, 
  \if\shortVersion 1
    $d_{o} = \max\left(d, v \Delta t_2\right) = \max\left(d, 2r\arcsin\left(\frac{d}{2r}\right)\right).$ 
  \else
    $$d_{o} = \max\left(d, v \Delta t_2\right) = \max\left(d, 2r\arcsin\left(\frac{d}{2r}\right)\right).$$ 
  \fi
  We used $\max$ function for a similar reason as exposed before. After rearranging, we have (\ref{eq:do}) and (\ref{eq:dprime}).

  For calculating the throughput $f_{t}(K,T)$ for $K$ lanes and a given time $T$ after the arrival of the first robot, we get the number of robots reaching the target region by the time $T$, then we use the Definition \ref{def:throughput2}. As we assume that the first robot of every lane begins at the same distance from the target, at time $T=0$ we have $K$ robots simultaneously arriving. Then, after $d_{o}/v$ units of time, we have $K$ more robots arriving and this keeps happening every $d_{o}/v$ units of time. Denote $N(K,T)$ the total number of robots that have arrived at the target region from $K$ lanes by time $T$. Thus, we have:
  \if\shortVersion 1
    $N(K,T) =K\left\lfloor \frac{T}{\frac{d_{0}}{v}} + 1\right\rfloor  = K\left\lfloor \frac{vT}{d_{o}} + 1\right\rfloor,$
  \else
    $$N(K,T) =K\left\lfloor \frac{T}{\frac{d_{0}}{v}} + 1\right\rfloor  = K\left\lfloor \frac{vT}{d_{o}} + 1\right\rfloor,$$
  \fi
  so, by Definition \ref{def:throughput2},
  \if\shortVersion 1
    $f_{t}(K,T) = 
      \frac{1}{T}\left(K\left\lfloor \frac{vT}{d_{o}} +1 \right\rfloor - 1\right).
    $
  \else
    $$f_{t}(K,T) = 
      \frac{1}{T}\left(K\left\lfloor \frac{vT}{d_{o}} +1 \right\rfloor - 1\right).
    $$
  \fi
  
  As for every number $x$, $\lfloor x \rfloor = x - frac(x)$ and $0 \le frac(x) < 1$, then distributing $\frac{1}{T}$ for each term, we get
  \if\shortVersion 1
    $
        f_{t}(K) 
          = \lim_{T \to \infty} f_{t}(K,T)
          = \frac{Kv }{d_{o}}.
    $  
  \else
    $$
      \begin{aligned}
        f_{t}(K) &= \lim_{T \to \infty} f_{t}(K,T)
          = \lim_{T \to \infty} \frac{1}{T}\left(K\left( \frac{vT}{d_{o}} + 1 \right) - frac\left( \frac{vT}{d_{o}} + 1 \right) - 1\right)\\
          &= \lim_{T \to \infty} \frac{K}{T}\left( \frac{vT}{d_{o}} + 1 \right)
          = \frac{Kv }{d_{o}}.
      \end{aligned}
    $$ 
  \fi
\end{proof}

\begin{figure}
  \centering
  \includegraphics[width=0.5\linewidth]{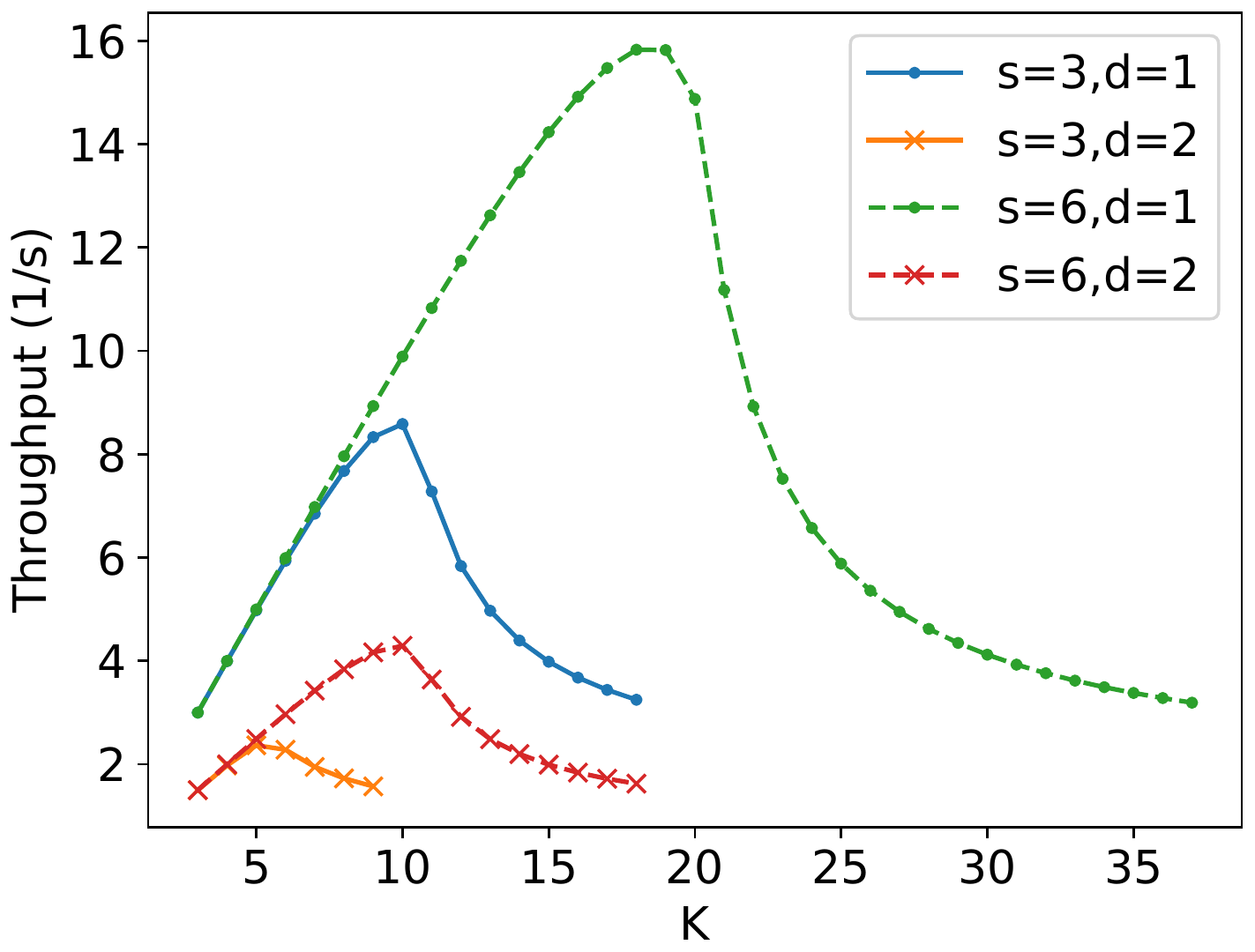}
  \caption{Plotting of the asymptotic throughput of the touch and run strategy (given by (\ref{eq:throughputhitandrunlimit})) for some values of $s$ and $d$, in metres, and $v = 1$ m/s, for the interval of values for $K$ obtained by (\ref{eq:Kbounds}). }
  \label{fig:grafKthroughput}
\end{figure}

Figure \ref{fig:grafKthroughput} presents examples of (\ref{eq:throughputhitandrunlimit}) for some parameters. Observe that the maximum throughput for different values of $s, d$ and $v$ can be found by linear search in the interval obtained by (\ref{eq:Kbounds}).

\subsubsection{Comparison of the strategies}
\label{sec:subseccomparison}

The parallel lanes strategy has the lowest of the limits in relation to $u = \frac{s}{d}$, the ratio between the radius of the target region and the minimum distance between the robots. However, its asymptotic value is still higher than the minimum possible asymptotic throughput for hexagonal packing just for some values of $u$. In this section, we will make explicit the dependence on the argument $u$ in every throughput function we defined previously to compare them with respect to this ratio. Let $f_{p}(u) = \displaystyle \lim_{T\to \infty} f_{p}(T,u)$ and $f_{h}^{min}(u)$ be the asymptotic throughput for the parallel lanes strategy and the lower asymptotic throughput for the hexagonal packing strategy for a ratio $u$, respectively. Hence, by Proposition \ref{methodology:bigtarget:independent_straight:proof},
\if\shortVersion 1
  $ f_{p}(u) = \lfloor 2u + 1 \rfloor \frac{v}{d}, $
\else
  $$ f_{p}(u) = \lfloor 2u + 1 \rfloor \frac{v}{d}, $$
\fi
and, by (\ref{eq:hexthroughputbounds}) using $\theta = \pi/6$ as it minimises the lower bound of $\displaystyle \lim_{T\to \infty} f(T,\theta)$ in Proposition \ref{prop:hexthroughputbounds},
\if\shortVersion 1
  $f_{h}^{min}(u)=\frac{2 }{\sqrt{3}}  \left(2u  - 1\right)\frac{v}{d}.$
\else
  $$f_{h}^{min}(u)=\frac{2 }{\sqrt{3}}  \left(2u  - 1\right)\frac{v}{d}.$$
\fi

\begin{proposition}
   There are some $u < \frac{\sqrt{3} + 2 }{4-2 \sqrt{3} }$ such that $f_{p}(u) > f_{h}^{min}(u)$, and for every $u \ge \frac{\sqrt{3} + 2 }{4-2 \sqrt{3} }, f_{p}(u) \le f_{h}^{min}(u)$.
\end{proposition}
\begin{proof}
  For any $u <   \frac{\sqrt{3} + 2 }{4-2 \sqrt{3} }$, $(2u + 1)\frac{v}{d} > f_{h}^{min}(u)$, due to
\if\shortVersion 1
  \begin{equation}
      (2 u + 1) \frac{v}{d} > \frac{2 }{\sqrt{3}}  \left(2u  - 1\right)\frac{v}{d} \LR 
      u < \frac{-1 - \frac{2 }{\sqrt{3}}}{2   - \frac{4 }{\sqrt{3}}} 
      = \frac{\sqrt{3} + 2 }{4-2 \sqrt{3} }. 
    \label{eq:equivuhex1}
  \end{equation}
\else
  \begin{equation}
    \begin{aligned}
      &\phantom{\LR\ } (2 u + 1) \frac{v}{d} > \frac{2 }{\sqrt{3}}  \left(2u  - 1\right)\frac{v}{d} \LR 
      2 u + 1  > \frac{2 }{\sqrt{3}}  \left(2u  - 1\right) \\ 
      &\LR 2 u  - \frac{4 }{\sqrt{3}} u > -1 - \frac{2 }{\sqrt{3}}   \LR 
      u < \frac{-1 - \frac{2 }{\sqrt{3}}}{2   - \frac{4 }{\sqrt{3}}} 
\ifexpandexplanation
        = \frac{-\sqrt{3} - 2 }{2\sqrt{3}   - 4 } 
\fi
      = \frac{\sqrt{3} + 2 }{4-2 \sqrt{3} }. 
    \end{aligned}
    \label{eq:equivuhex1}
  \end{equation}
\fi
We have that $f_{p}(u) = (2u + 1)\frac{v}{d}$ when $2u + 1 \in \Zeta$. Also, as $u < \frac{\sqrt{3} + 2 }{4-2 \sqrt{3} } < 7$, $u$ can be a number satisfying $(2u + 1) = \lfloor2u + 1\rfloor$. Thus, there are some values of $u$ such that $f_{p}(u) = \lfloor2u + 1\rfloor\frac{v}{d} > f_{h}^{min}(u)$. 

From the equivalence in (\ref{eq:equivuhex1}) and because for any $x$, $\lfloor x \rfloor \le x$, it follows that for any $u \ge \frac{\sqrt{3} + 2 }{4-2 \sqrt{3} },  f_{p}(u) \le (2 u + 1) \frac{v}{d}  \le f_{h}^{min}(u).$   
\end{proof}

Figure \ref{fig:fpfhminfhmax} shows an example of $f_{h}^{min}(u)$, $f_{p}(u)$ and the maximum possible asymptotic throughput of the hexagonal packing $f_{h}^{max}(u) = \frac{2 }{\sqrt{3}}  \left(2u  + 1\right)\frac{v}{d}$ for $u \in [0,10]$. Observe that, from the left side of $u=7$, $f_{p}(u)$ has some values above $f_{h}^{min}(u)$ even though they are below $f_{h}^{max}(u)$ for every $u$.

\begin{figure}
  \centering  
  \begin{minipage}[t]{0.49\linewidth}
    \includegraphics[width=\linewidth]{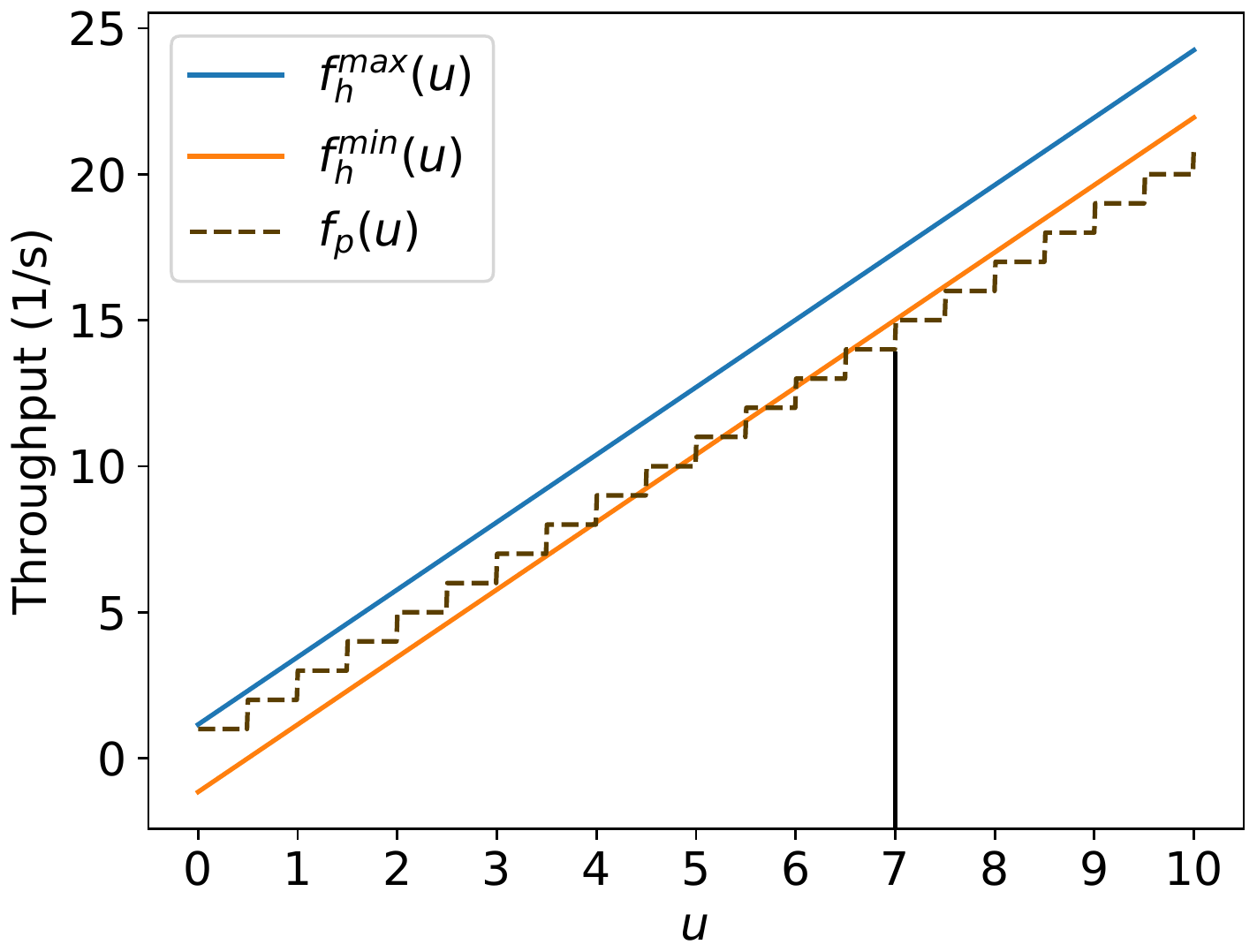}
    \caption{Example of $u$ values such that $f_{h}^{max}(u) > f_{p}(u)$ for $v = 1$ m/s and $d = 1$ m.}
    \label{fig:fpfhminfhmax}
  \end{minipage}\,\,
  \begin{minipage}[t]{0.49\linewidth}  
    \centering 
    \includegraphics[width=\linewidth]{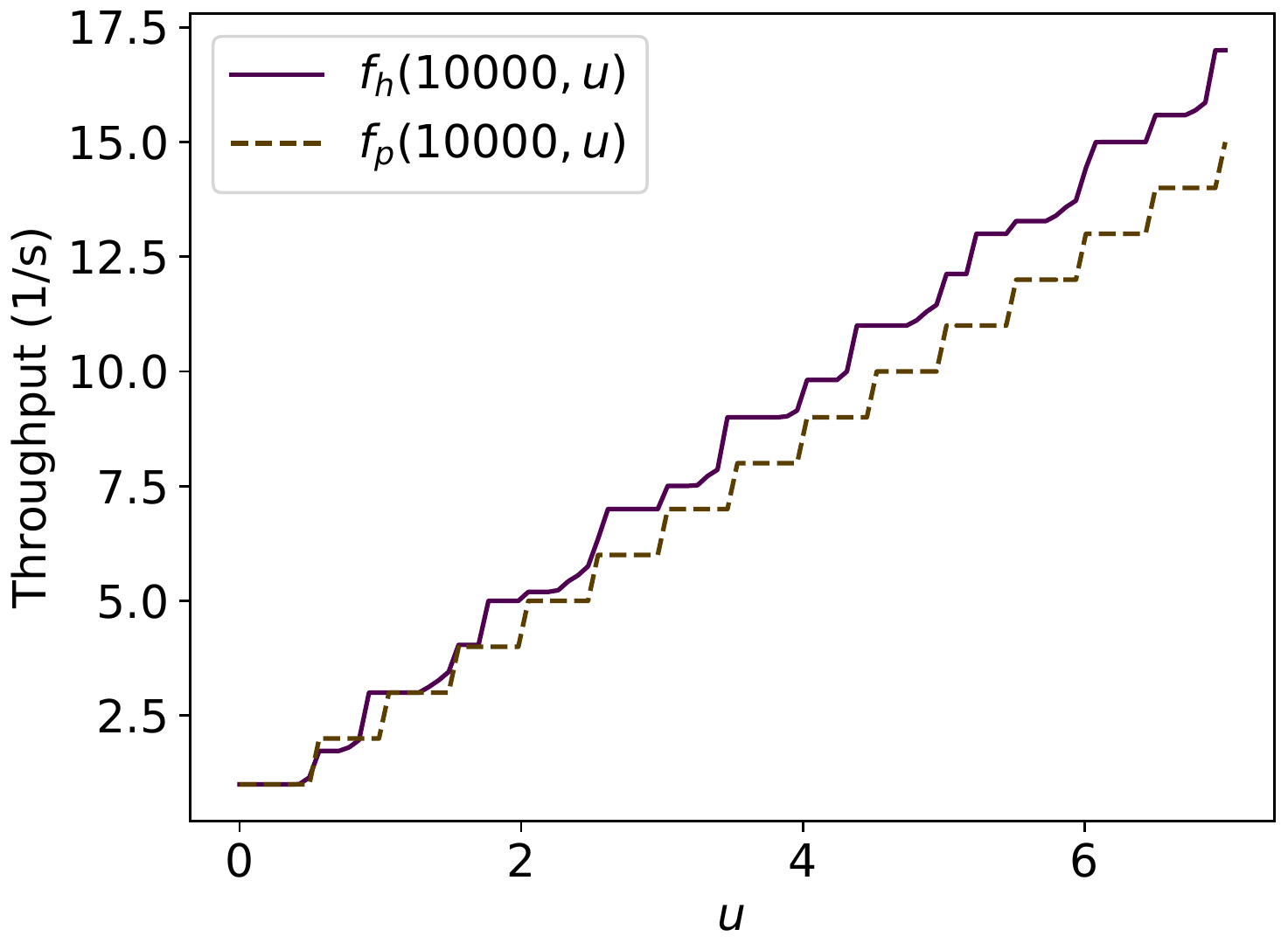}
  \caption{Comparison of $f_{p}(T,u)$ and $f_{h}(T,u)$ for $u \in [0,7]$, $T=10000$ s, $v = 1$ m/s and $d = 1$ m.}
  \label{fig:fhfpLarge}
  \end{minipage}
\end{figure}

Because of this proposition, we are certain that for values of $u \ge \frac{\sqrt{3} + 2 }{4-2 \sqrt{3} } \approx 7$ the hexagonal packing strategy at the limit will have higher throughput than parallel lanes. However, for values $u < \frac{\sqrt{3} + 2 }{4-2 \sqrt{3}}$, there is the possibility of the parallel lanes strategy being better than hexagonal packing. As we do not have an exact asymptotic throughput for the hexagonal packing strategy for a given angle $\theta$, we can numerically find the best $\theta$ using large values of $T$ on (\ref{eq:hexthroughput}); then, after choosing $\theta$, we calculate the numerical approximation of the asymptotic throughput using this fixed $\theta$ and those $T$ values. This result can be compared with the throughput for the same large values of $T$ for the parallel lanes strategy using (\ref{eq:parallelT}). Furthermore, in a scenario with the target region only being accessed by a corridor with a finite height, the maximum time $T$ can be inferred by its size, then the exact throughput for this specific value can be calculated by (\ref{eq:hexthroughput}) and (\ref{eq:parallelT}) as stated before, but using only this specific value $T$, instead of a set of large values, to decide which strategy is more suitable.

Let $f_{h}(T,\theta,u)$ and $f_{p}(T,u)$ be (\ref{eq:hexthroughput}) and (\ref{eq:parallelT}) making explicit the parameter $u$. Let $\theta^{*}$ be the outcome from the search of the  $\theta$ which  maximises $f_{h}(T,\theta,u)$ by numeric approximation. Thus, we define $f_{h}(T,u) = f_{h}(T,\theta^{*},u)$.  Figure \ref{fig:fhfpLarge} illustrates the result of the procedure mentioned above for $T = 10000$ for 100 equally spaced values of $u\in[0,7]$ and seeking the maximum throughput using 1000 evenly spaced points between $\lbrack 0,\pi/3 \rparen$ to find the best $\theta$ for the hexagonal packing strategy. Then, we compare it with the result for $\theta=\pi/6$ as we explained previously when we discussed Figures \ref{fig:plottheta1}-\ref{fig:plottheta4}. Observe that for $u \in [0.5,0.9]$ there is some values for which $f_{h}(10000,u) < f_{p}(10000,u)$. Figure \ref{fig:fhfpZoom} shows this by 100 equally spaced values of $u \in [0.4,1]$ for different values of $v$.  This occurs because, for such values of $u$, using square packing fits more robots inside the circle over the time than hexagonal packing, as we will show in Section \ref{sec:hexparexperiments}.

\begin{figure}[t!]
  \centering
  \subfloat[$v=0.1$ m/s]{\includegraphics[width=0.49\linewidth]{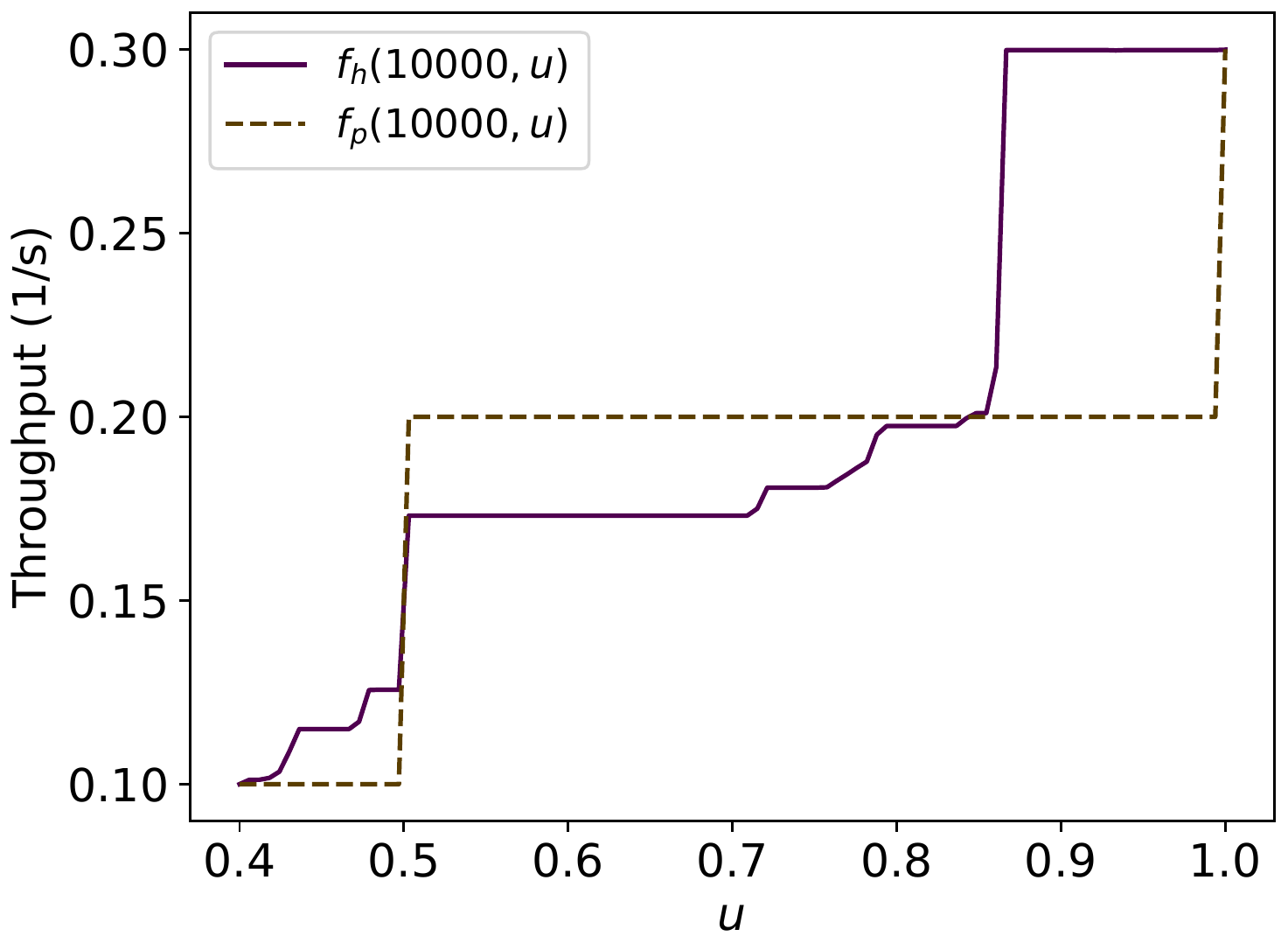}}
  \subfloat[$v=1$ m/s]{\includegraphics[width=0.49\linewidth]{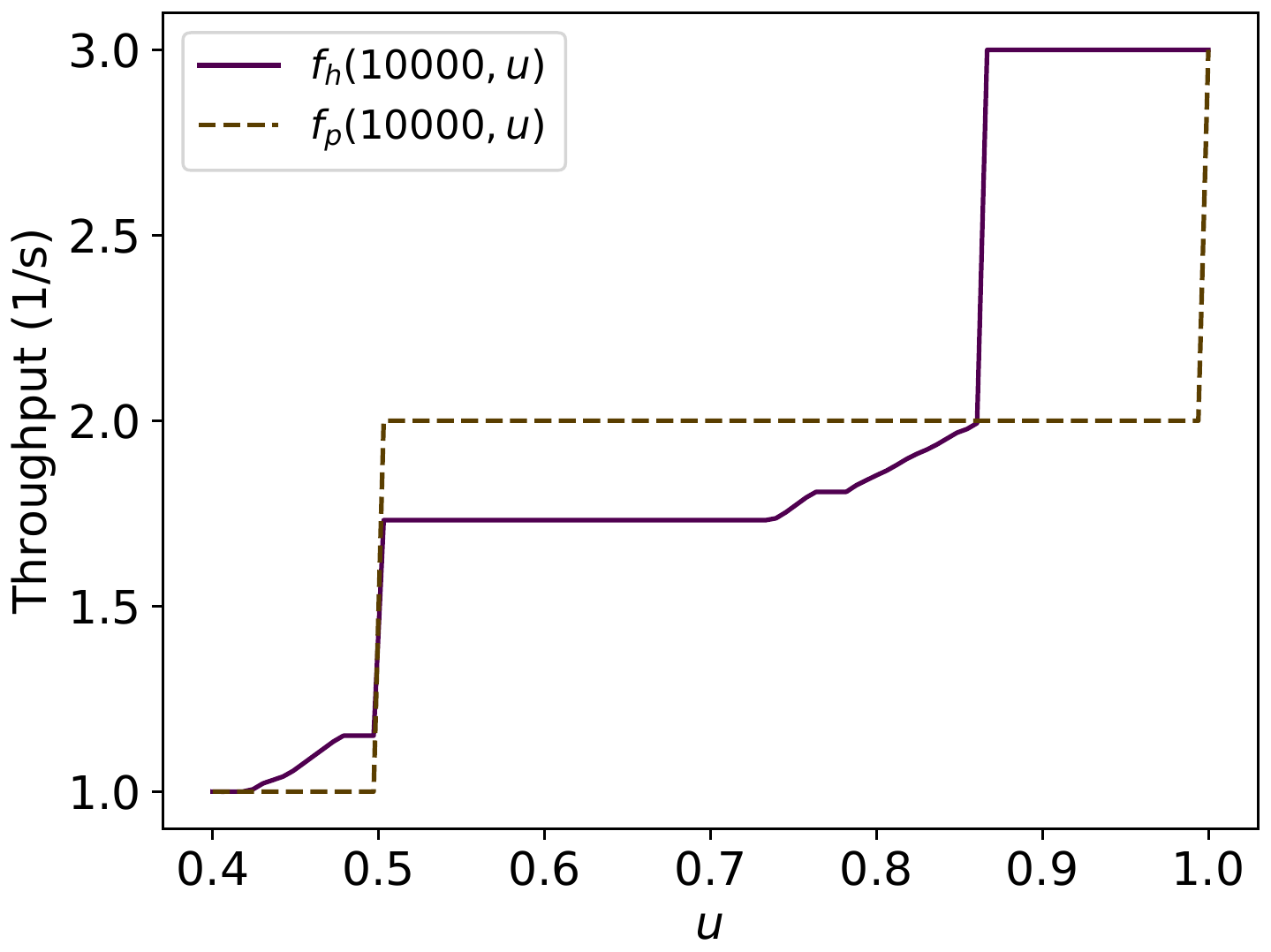}}
  \caption{Comparison of $f_{p}$ and $f_{h}$ for $u \in [0.4,1]$, $T=10000$ s, $v \in \{0.1,1\}$ m/s and $d = 1$ m. The difference in the lines of $f_{h}$ is due to $\theta^{*}$ being different for each value of $v$.}
  \label{fig:fhfpZoom}
\end{figure}

Additionally, the asymptotic throughput of the touch and run strategy, $f_{t}(u) = \displaystyle \lim_{T \to \infty} f_{t}(T,u)$, for  higher values of $u$ is greater than the maximum possible asymptotic value of the hexagonal packing $f_{h}^{max}(u) =\frac{2 }{\sqrt{3}}  \left(2u  + 1\right)\frac{v}{d}$, as shown later by numeric experimentation. Before presenting this result, we need to verify which values of $u$ are allowed by $f_{t}(u)$  and to express the asymptotic throughput of the touch and run strategy from Proposition \ref{prop:throughputsdk} in terms of the ratio $u$.

From Proposition \ref{prop:Kboundsrk} we have that the possible number of lanes $K \in \{3, \dots, K(u)\}$ with 
$
    K (u) = \Bigl\lfloor \frac{\pi}{\arcsin \left( \frac{1}{2u} \right)}\Bigr\rfloor 
$. Consequently,
  $f_{t}(u)$ is only allowed for any $u \ge \frac{1}{\sqrt{3}}.$ In fact,
  by Proposition \ref{prop:Kboundsrk}, $K \ge 3$, then $\frac{\pi}{\arcsin\left(\frac{1}{2u}\right)} \ge \Bigl\lfloor\frac{\pi}{\arcsin\left(\frac{1}{2u}\right)}\Bigr\rfloor \ge 3 \Rightarrow 
      \frac{\pi}{3} \ge \arcsin\left(\frac{1}{2u}\right) \Leftrightarrow
      \sin\left(\frac{\pi}{3}\right) \ge \frac{1}{2u} \Leftrightarrow
      \frac{\sqrt{3}}{2} \ge \frac{1}{2u} \Leftrightarrow
    u \ge \frac{1}{\sqrt{3}}.$

We show below the algebraic manipulations for  expressing the asymptotic throughput of the touch and run strategy from Proposition \ref{prop:throughputsdk} in terms of the ratio $u$. The asymptotic throughput expressed in (\ref{eq:throughputhitandrunlimit}) is
\ifexpandexplanation
  \begin{equation}
    \begin{aligned}
        \frac{Kv }{d_{o}} 
           & 
      = \frac{K}{\frac{d_{o}}{d}}\frac{v}{d} 
           & 
      = \frac{K}{\frac{\max(d,d')}{d}}\frac{v}{d} 
          & [\text{(\ref{eq:do})}] \\
      = \frac{K}{\max(1,\frac{d'}{d})}\frac{v}{d},
     \end{aligned}
     \label{eq:tempthrcurve}
  \end{equation}
\else
  \begin{equation}
        \frac{Kv }{d_{o}} 
      = \frac{K}{\frac{d_{o}}{d}}\frac{v}{d} 
      = \frac{K}{\frac{\max(d,d')}{d}}\frac{v}{d} 
      = \frac{K}{\max(1,\frac{d'}{d})}\frac{v}{d},
     \label{eq:tempthrcurve}
  \end{equation}
\fi
for an integer $K \in \{3, \dots, K(u)\}$. From (\ref{eq:ak}), $\alpha = \frac{2 \pi}{K}$, and, from (\ref{eq:relationangles}),
\if\shortVersion 1
  $
      \frac{r}{d} =  \frac{\frac{s}{d}   \sin(\alpha / 2) - \frac{d}{2d}}{1 - \sin(\alpha / 2)} = \frac{u   \sin(\frac{\pi}{K}) - \frac{1}{2}}{1 - \sin(\frac{\pi}{K})} \defeq r(u,K),
  $
\else
  $$
    \begin{aligned}
      \frac{r}{d} =  \frac{\frac{s}{d}   \sin(\alpha / 2) - \frac{d}{2d}}{1 - \sin(\alpha / 2)} = \frac{u   \sin(\frac{\pi}{K}) - \frac{1}{2}}{1 - \sin(\frac{\pi}{K})} \defeq r(u,K),
    \end{aligned}
  $$
\fi
resulting
\begin{equation}
  \begin{aligned}
    \frac{d'}{d} 
      &= 
      \begin{cases}
         \frac{r}{d}  (\pi - \alpha)  + \frac{d - 2  r  \cos(\alpha / 2)}  {d\sin(\alpha / 2)}, & \text{ if } 2  r  \cos(\alpha / 2) < d,\\
        2  \frac{r}{d}  \arcsin\left( \frac{d}{2 r}\right), & \text{ otherwise, }
      \end{cases}
       & [\text{by (\ref{eq:dprime})}]
      \\
      &=  
      \begin{cases}
        \frac{r}{d}  \left(\pi - \frac{2 \pi}{K}\right)  + \frac{1 - 2  \frac{r}{d}  \cos(\frac{\pi}{K})}  {\sin(\frac{\pi}{K})}, & \text{ if } 2  \frac{r}{d}  \cos(\frac{ \pi}{K} ) < 1,\\
        2  \frac{r}{d}  \arcsin\left( \left(2 \frac{r}{d}\right)^{-1}\right), & \text{ otherwise, }
      \end{cases} 
      \\ &=  
      \begin{cases}
        r(u,K)  \left(\pi - \frac{2 \pi}{K}\right)  + \frac{1 - 2  r(u,K)  \cos(\frac{\pi}{K})}  {\sin(\frac{\pi}{K})}, \\ 
           \hspace{3.3cm} \text{ if } 2  r(u,K)  \cos(\frac{ \pi}{K} ) < 1,\\
        2  r(u,K)  \arcsin\left( \frac{1}{2 r(u,K)}\right),  \text{ otherwise, }
      \end{cases}
      \\ &\defeq d'(u,K).  
  \end{aligned}
  \label{eq:dprimed}
\end{equation}
Thus, from (\ref{eq:tempthrcurve}) and (\ref{eq:dprimed}), 
$  f_{t}(u,K) = \frac{K}{\max(1,d'(u,K))}\frac{v}{d}  $ and the upper throughput for the touch and run strategy in terms of $u$ is given by 
\if\shortVersion 1
  $
      f_{t}(u) =  \max_{K \in \{3,\dots,K(u)\}} f_{t}(u,K) = \max_{K \in \{3,\dots,K(u)\}} \frac{K}{\max(1,d'(u,K))}\frac{v}{d} 
      = \frac{K^{*}(u)}{\max(1,d'(u,K^{*}(u)))}\frac{v}{d},
  $ 
\else
  $$ 
    \begin{aligned} 
      f_{t}(u) &=  \max_{K \in \{3,\dots,K(u)\}} f_{t}(u,K) = \max_{K \in \{3,\dots,K(u)\}} \frac{K}{\max(1,d'(u,K))}\frac{v}{d} \\
      &= \frac{K^{*}(u)}{\max(1,d'(u,K^{*}(u)))}\frac{v}{d},
    \end{aligned}
  $$ 
\fi
for some function $K^{*}(u)$ that finds this maximum in $\{3,\dots,K(u)\}$. Similarly, for a fixed maximum time $T$, we have by (\ref{eq:throughputhitandruntime}) $f_{t}(T,u) = \displaystyle \max_{K \in \{3,\dots,K(u)\}} f_{t}(K,T,u)$. 

Figure \ref{fig:numhigherftfh} presents a comparison of the asymptotic throughput $f_{t}(u)$  and the lower and upper values of the asymptotic throughput of the hexagonal packing $f_{h}^{min}(u)$ and $f_{h}^{max}(u)$   for values of $u$ ranging from $1/\sqrt{3}$ to 1000. Observe that the asymptotic throughput of the touch and run strategy is greater than the  maximum possible asymptotic throughput of the hexagonal packing strategy for almost all values of $u$, except for some in $(1.12,1.25)$ (Figure \ref{fig:numhigherftfh} (b)).

\begin{figure}[t!]
  \centering 
  \subfloat[\text{$u \in [1/\sqrt{3}, 1000]$}]{\includegraphics[width=0.49\linewidth]{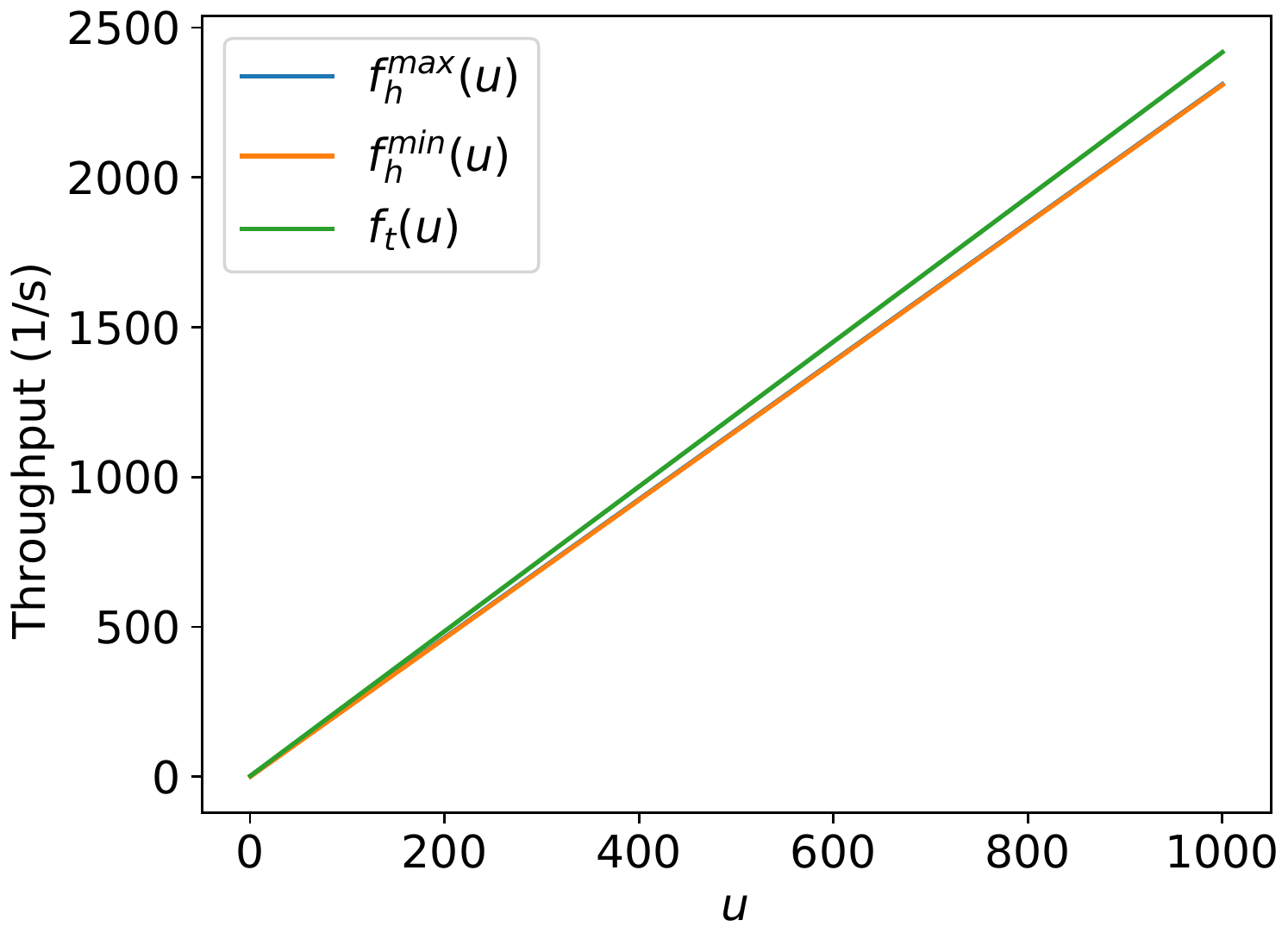}}
  \subfloat[\text{$u \in [1/\sqrt{3}, 1.25]$}]{\includegraphics[width=0.455\linewidth]{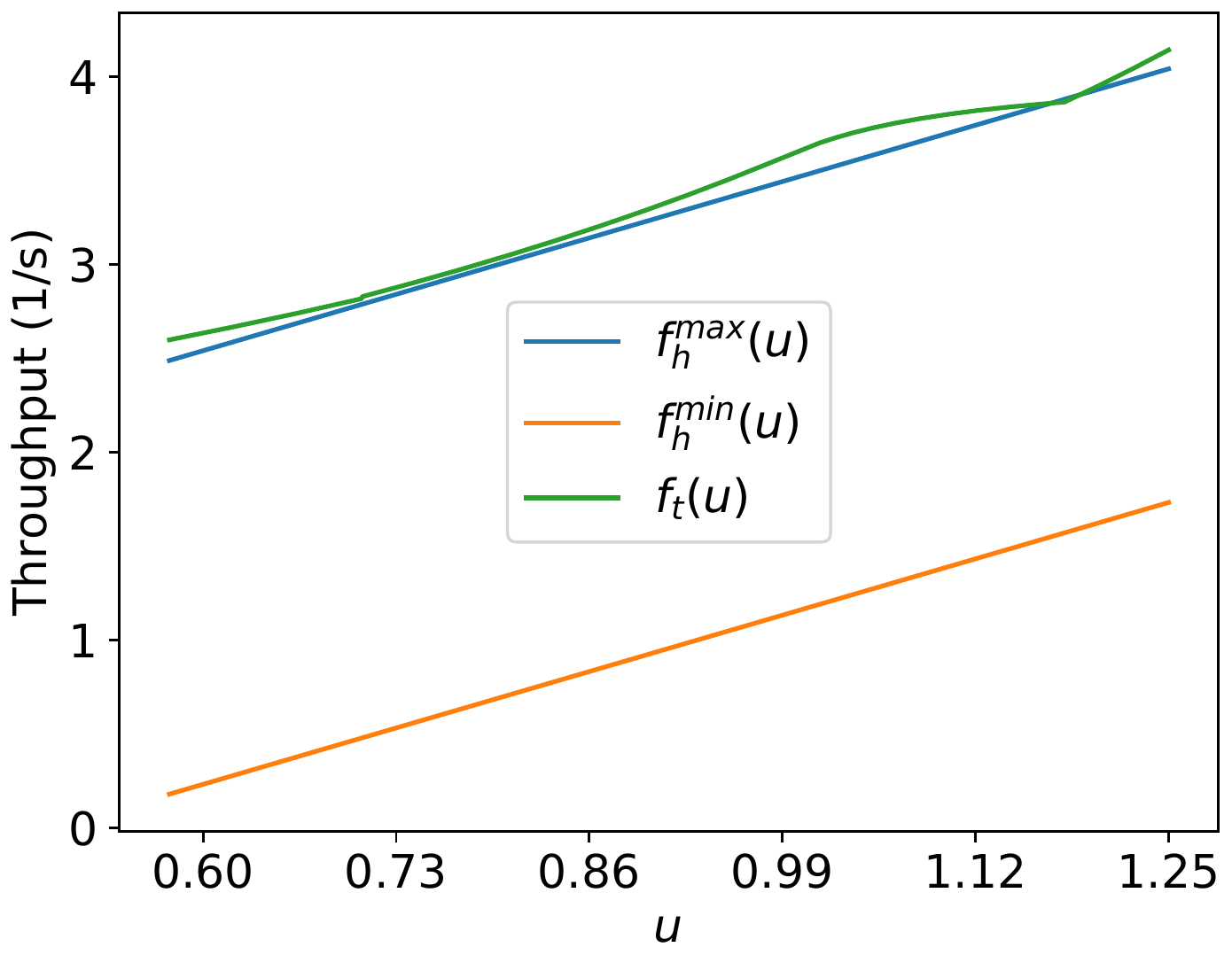}}
  \caption{Graph varying $u$ for $f_{h}^{min}(u)$, $f_{h}^{max}(u)$ and $f_{t}(u)$ with $v = 1$ m/s and $d = 1$ m for different intervals of $u$. In (a), $f_{h}^{min}(u)$ and $f_{h}^{max}(u)$ are almost overlapped. In (b), $f_{t}(u) > f_{h}^{max}(u)$ for all $u$, except in an interval within (1.12,1.25).} 
  \label{fig:numhigherftfh}
\end{figure}

\begin{figure}[t!]
  \centering
  \includegraphics[width=0.49\linewidth]{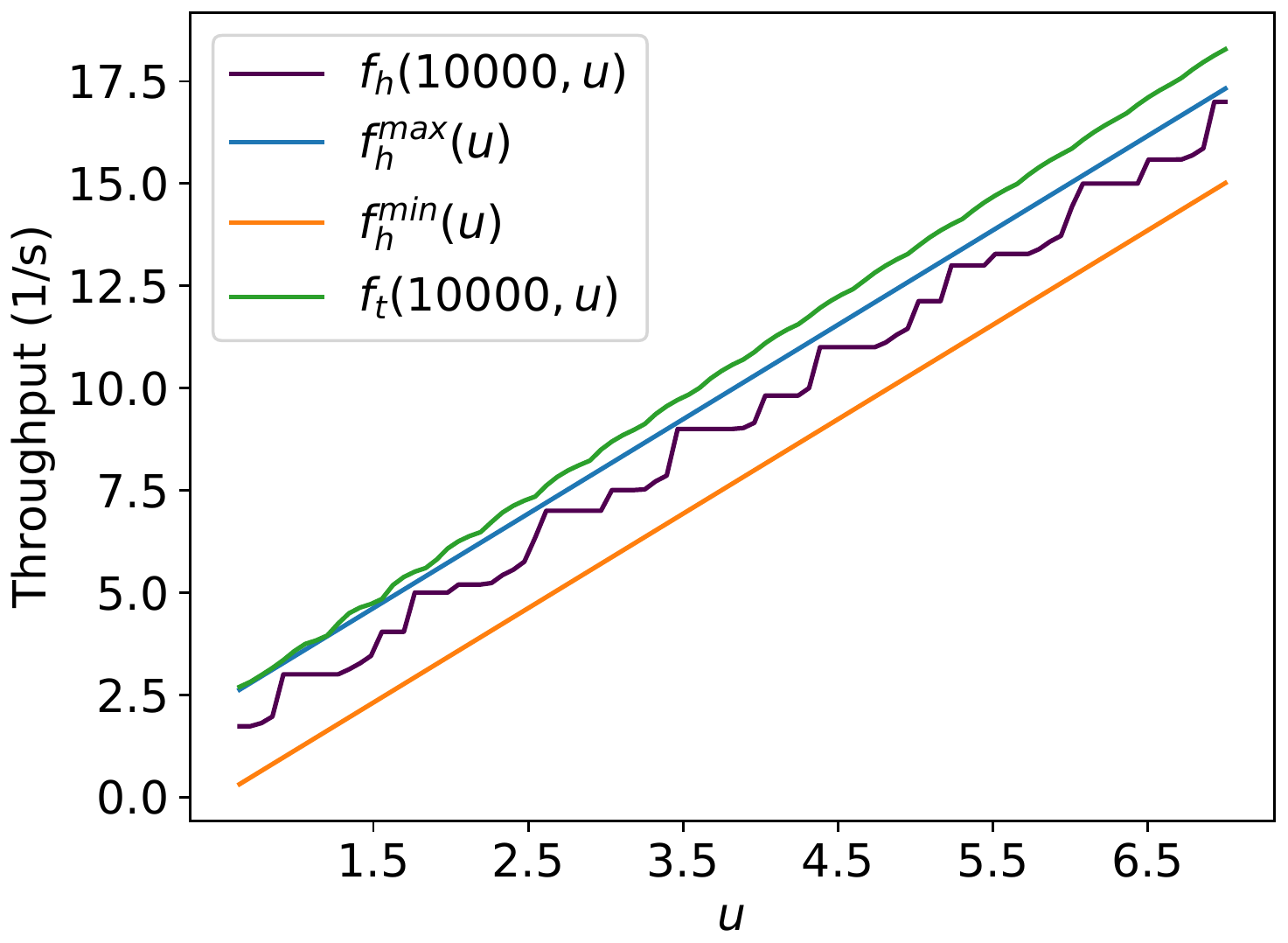}
  \caption{Example for $T=10000$ s, $v = 1$ m/s, $d = 1$ m and 100 equally spaced points of $u \in [1/\sqrt{3},7]$. We have $f_{h}(T,u) < f_{t}(T,u)$, albeit $f_{h}^{max}(u) \ge f_{t}(T,u)$ for a few values of $u < 1.5$.}
  \label{fig:ftbelowfh}
\end{figure}

Additionally, we performed numerical experiments for $f_{t}(T,u)$ and $f_{h}(T,u)$ using fixed time $T = 10000$ in (\ref{eq:throughputhitandruntime}), (\ref{eq:hexthroughput}) and $u \in [1/\sqrt{3}, 7]$. For finding $\theta^{*}$, we use the same procedure described before to compare $f_{h}(T,u)$ and $f_{p}(T,u)$. 
Figure \ref{fig:ftbelowfh} shows the result. It suggests the touch and run strategy has higher throughput than hexagonal packing for large values of $T$. Although hexagonal packing has lower asymptotic throughput than the touch and run for almost all $u$ values,  it is suitable for $u < \frac{1}{\sqrt{3}}$ whenever it surpasses the parallel lanes strategy.

For real-world applications and assuming the robots are in constant velocity and distance between other robots, the hexagonal packing strategy is adequate for a situation where the target is placed in a constrained region, for example, walls in north and south positions. In this example, the number of lanes used in the touch and run strategy would be reduced because of the surrounding walls. In an unconstrained scenario, if the ratio $u$ and the maximum time $T$ are known, the throughput value of the hexagonal packing strategy from (\ref{eq:hexthroughput}) (for the $\theta$ which maximises it) can be compared with the throughput of the touch and run strategy from (\ref{eq:throughputhitandruntime}) (for $K^{*}(u)$) to choose which strategy should be applied. However, assuming constant velocity and distance between robots in a swarm is not practical, because other robots influence the movement in the environment. Hence, we used these strategies as inspiration to propose novel algorithms based on potential fields for robotic swarms in \citep{arxivAlgorithms}.

\section{Experiments and Results}
\label{sec:experimentresults}

In order to evaluate our approach, we executed several simulations using the Stage robot simulator \citep{PlayerStage}  for testing the equations presented in the theoretical section (Section \ref{sec:theoreticalresults}). 
Hyperlinks to the video of executions are available in the captions of each corresponding figure. They are in real-time so that the reader can compare the time and screenshots presented in the figures in this section with those in the supplied videos.\footnote{The source codes of each experimented strategy are in \url{https://github.com/yuri-tavares/swarm-strategies}.}

We ran experiments for all strategies considering $s > 0$. We could not make experiments for point-like targets, because a point with a fixed value is nearly impossible to be reached by a moving robot in Stage computer simulations due to the necessity of exact synchronization of the sampling frequency of positions made by the simulator and the robot velocity. Hence, we must use a circular area with a radius $s>0$ around the target to identify that a robot reached it. After presenting the experiments and results for all strategies for circular target region with radius $s > 0$, we compare them experimentally considering the analysis previously discussed in Section \ref{sec:subseccomparison}.

We saved for each robot its arrival time in milliseconds since the start of the experiment.  We subtracted the arrival time of every robot by the arrival time of the first robot. By doing so, the experiment is assumed to begin in time $T=0$ without worrying about the initial inertia. After this, we registered the number of robots ($N$) for each time value ($T$).

To alleviate some of the numerical errors caused by the floating-point representation, we used rounding on the 13th decimal place before using floor and ceiling functions on the equations presented. For example, in nowadays computers, by using double variables in C or float in Python, if you divide 9.6 by 1.6 the result is  5.999999999999999 for 15 decimal places formatting, but it should be 6. If we applied the floor function on the computer result it would give us 5 instead of the expected 6.

For all experiments in this section, the robots are distant from each other by $d = 1$ m. In the figures of this section, black robots indicate they reached the target, and red did not.  Also, we did not repeat the experiments for the points on the graph plotting of this section because the velocity and initial positions are constant, so there is no random aspect, and we obtain the same results for different runs for that particular point.

\subsection{Compact lanes}
For compact lanes simulations, we used $v=1$ m/s, and the first robot to reach the target is at the bottom lane and starts at the target. For a target area radius $s$, such that $0 < s < \sqrt{3}d/4$, we used $s = 0.3$ m, and for $\sqrt{3}d/4 \le s < d/2$, we choose $s = 0.45$ m. Figure \ref{fig:exptriangulartiling2} shows screenshots of the simulation using  $s = 0.3$ m during $T = 7.1$ s and Figure \ref{fig:exptriangulartiling1} for $s = 0.45$ m and $T = 10.1$ s.

\begin{figure}[t!]
  \centering
  \begin{minipage}[t]{0.48\linewidth}  
    \centering
    \subfloat[0 s: beginning of the simulation.]{\includegraphics[width=\linewidth]{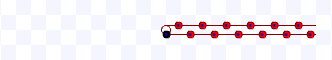}}\\
    \subfloat[After 2.7 s.]{\includegraphics[width=\linewidth]{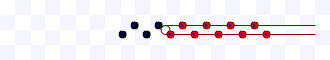}}\\
    \subfloat[After 6.7 s]{\includegraphics[width=\linewidth]{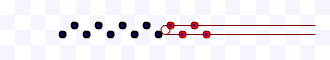}}\\
    \subfloat[5 s: ending of the simulation]{\includegraphics[width=\linewidth]{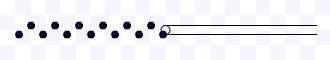}}
    \caption{Simulation on Stage for compact lanes strategy using $s = 0.3$ m, $d = 1$ m during $T = 7.1$ s. Available on \url{https://youtu.be/e1cWJzWhQmQ}.}
    \label{fig:exptriangulartiling2}
  \end{minipage}\quad
  \begin{minipage}[t]{0.48\linewidth}
    \centering
    \subfloat[0 s: beginning of the simulation.]{\includegraphics[width=\linewidth]{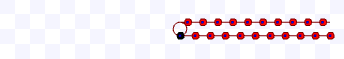}}\\
    \subfloat[After 3.5 s.]{\includegraphics[width=\linewidth]{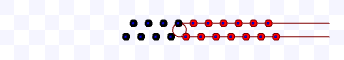}}\\
    \subfloat[After 7 s]{\includegraphics[width=\linewidth]{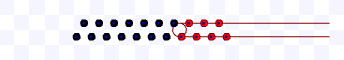}}\\
    \subfloat[10.1 s: ending of the simulation]{\includegraphics[width=\linewidth]{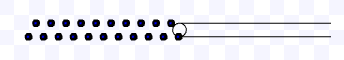}}
    \caption{Simulation on Stage for compact lanes strategy using $s = 0.45$ m, $d = 1$ m during $T = 10.1$ s. Available on \url{https://youtu.be/9OXGC1w83j0}.}
    \label{fig:exptriangulartiling1}
  \end{minipage}
\end{figure}

We ran experiments in order to verify the throughput for a given time and the asymptotic throughput calculated by (\ref{eq:giventime1}) to (\ref{eq:limit2}). Figure \ref{fig:results1qw} shows the throughput for different values of time obtained by the experiments in Stage, i.e. $(N-1)/T$, in comparison with the calculated value by (\ref{eq:giventime1}) and (\ref{eq:limit1}) for $s=0.3$ m and by (\ref{eq:giventime2}) and (\ref{eq:limit2}) for $s=0.45$ m. These figures confirm  that the equations presented at the theoretical section agree with the throughput obtained by simulations. 

\begin{figure}[t!]
  \centering
  \includegraphics[width=0.5\linewidth]{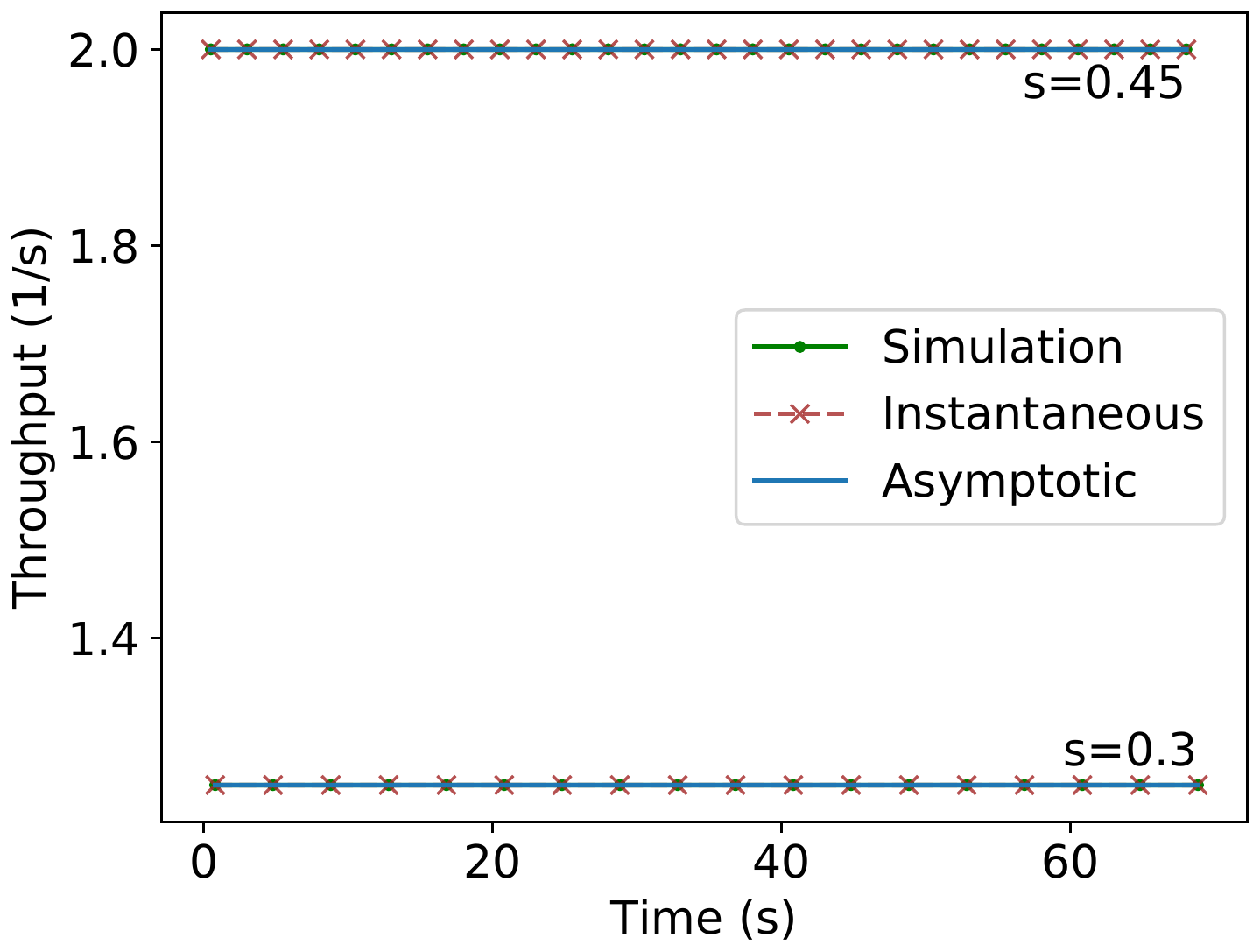}
  \caption{Throughput versus time plotting for compact lanes strategy for different values of $s$. ``Simulation'' stands for the data obtained from Stage, ``Instantaneous'' for the equations of the throughtput for a given time calculated in (\ref{eq:giventime1}) and (\ref{eq:giventime2}), and ``Asymptotic'' for the asymptotic throughtput obtained from (\ref{eq:limit1}) and (\ref{eq:limit2}). The mentioned equations results match with the data obtained from simulations.} \label{fig:results1qw} 
\end{figure}

\subsection{Parallel lanes}

We experimented with the parallel lanes strategy for $v=1$ m/s and $s \in \{3,6\}$ m. Figures \ref{fig:exppar1} and \ref{fig:exppar2} present screenshots from executions using these parameters.

\begin{figure}[t!]
  \centering
  \subfloat[0 s: beginning of the simulation.]{\includegraphics[width=0.49\linewidth]{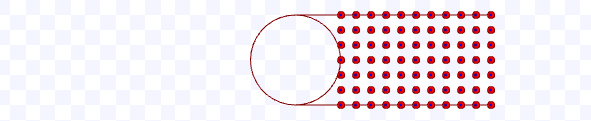}}\,
  \subfloat[After 6.5 s.]{\includegraphics[width=0.49\linewidth]{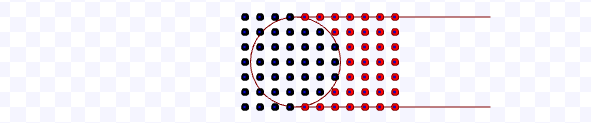}}\\
  \subfloat[13 s: ending of the simulation]{\includegraphics[width=0.49\linewidth]{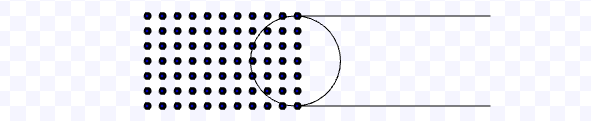}}
  \caption{Simulation on Stage for parallel lanes strategy using $s = 3$ m, $d = 1$ m during $T = 13$ s. Available on \url{https://youtu.be/2Y1RHc9YVaw}.}
  \label{fig:exppar1}
\end{figure}

\begin{figure}[t!]
  \centering
  \subfloat[0 s: beginning of the simulation.]{\includegraphics[width=0.49\linewidth]{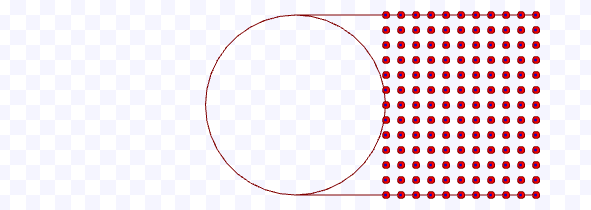}}\,
  \subfloat[After 8 s.]{\includegraphics[width=0.49\linewidth]{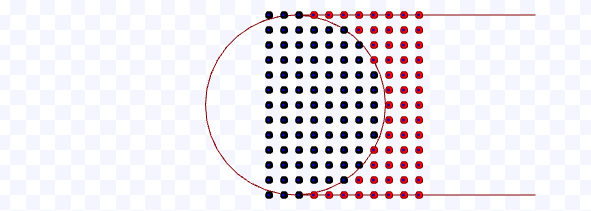}}\\
  \subfloat[16 s: ending of the simulation]{\includegraphics[width=0.49\linewidth]{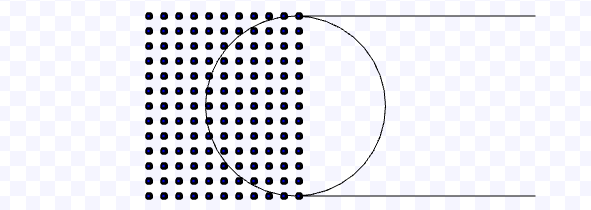}}
  \caption{Simulation on Stage for parallel lanes strategy using $s = 6$ m, $d = 1$ m during $T = 16$ s. Available on \url{https://youtu.be/TVdka65fi1g}.}
  \label{fig:exppar2}
\end{figure}

In order to verify the throughput for a given time calculated by (\ref{eq:parallelT}) and its asymptotic value as in (\ref{eq:parallelLimit}), we compare them with the throughput obtained from Stage simulations. Figure \ref{fig:tppar} (a) presents these comparisons. Observe that the values from (\ref{eq:parallelT}) are almost aligned with the values from simulation, except for some points. The difference in those points is due to the floating-point error discussed in the beginning of Section \ref{sec:experimentresults} that happens in the division before the use of floor or ceiling functions used on (\ref{eq:parallelT}). As expected, the values of (\ref{eq:parallelT}) approximates to (\ref{eq:parallelLimit}) as time passes. As the running time is proportional to the number of robots in our experiments, observe that higher throughput per time is reflected as a lower arrival time of the last robot per number of robots (Figure \ref{fig:tppar} (b)). Also, note that in the lower arrival time per number of robots graph, those values tends to infinite as the horizontal axis values grows.

\begin{figure}[t!]
  \centering
  \subfloat[]{\includegraphics[width=0.49\linewidth]{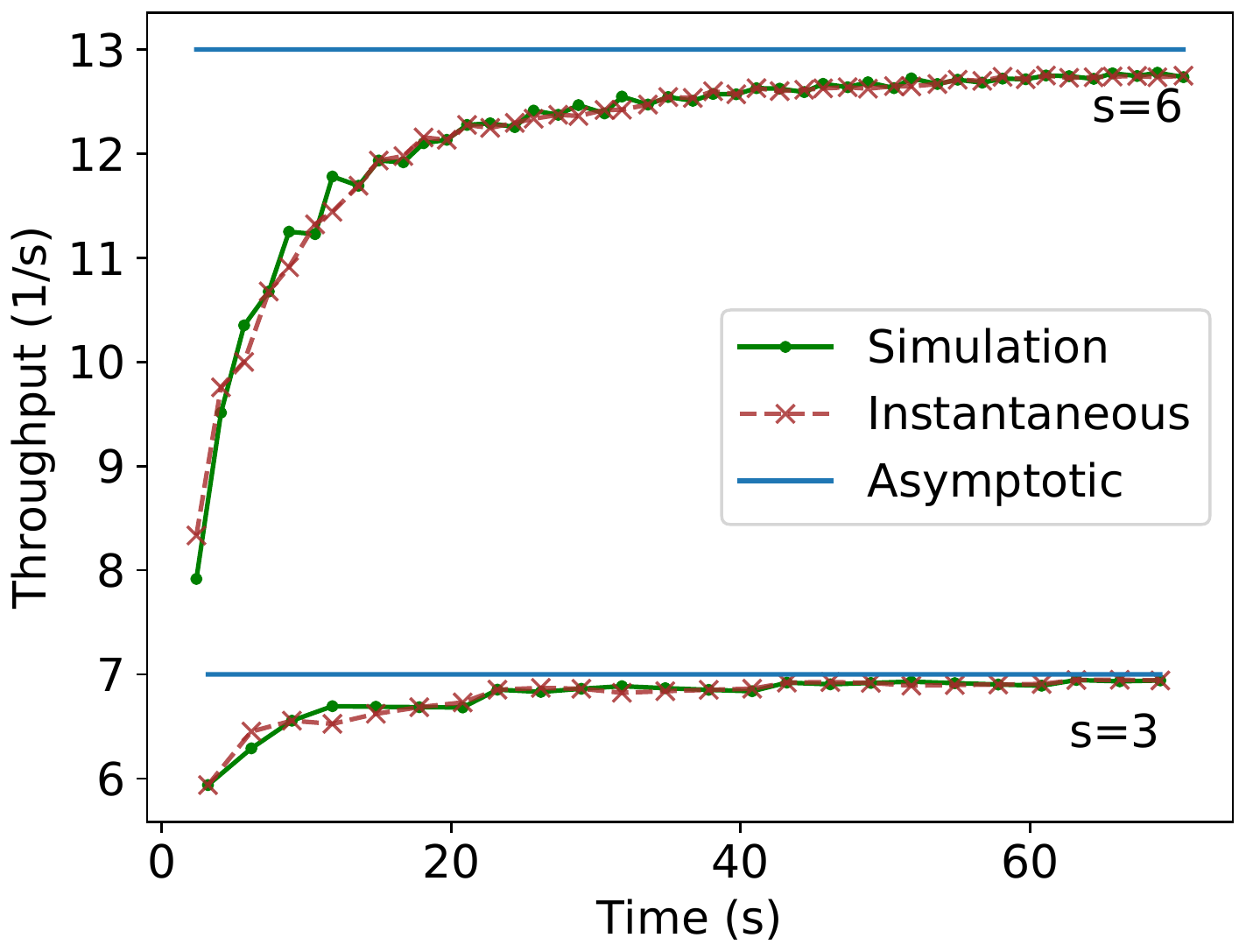}}
  \subfloat[]{\includegraphics[width=0.49\linewidth]{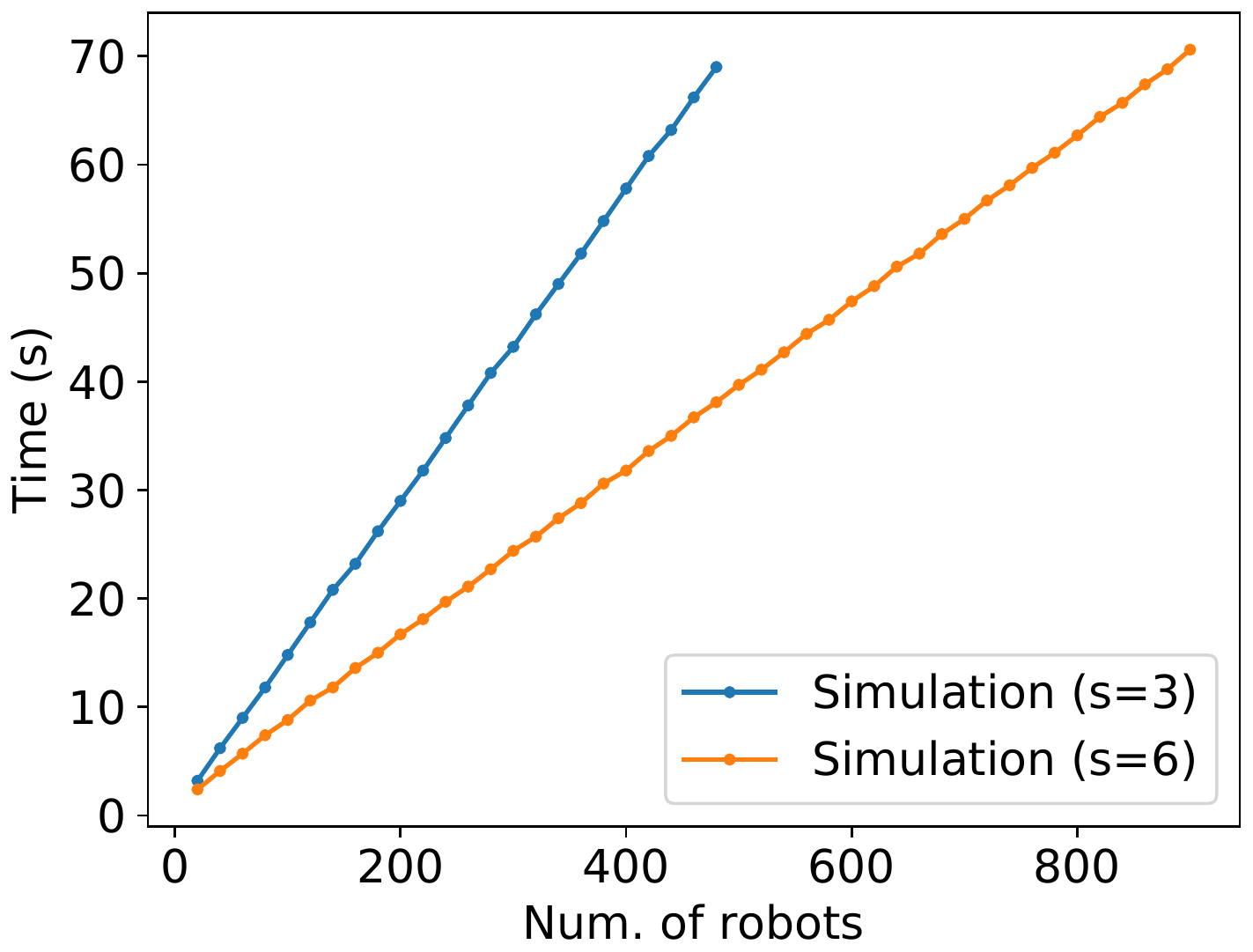}}
  \caption{(a) Throughput versus time plotting for parallel lanes strategy for $s \in \{3,6\}$ m. ``Simulation'' stands for the data obtained from Stage, ``Instantaneous'' for the equations of the throughtput for a given time calculated in (\ref{eq:parallelT}), and ``Asymptotic'' for the asymptotic throughtput obtained from (\ref{eq:parallelLimit}). (b) Number of robots versus time of arrival of the last robot for the same data.}
  \label{fig:tppar} 
\end{figure}

\subsection{Hexagonal packing}

\begin{figure}[t!]
  \centering
  \begin{minipage}[b]{0.48\linewidth}
    \centering
    \subfloat[0 s: beginning of the simulation.]{\includegraphics[width=0.7\linewidth]{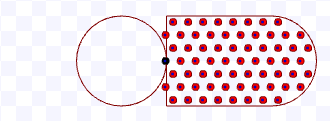}}\\
    \subfloat[After 4.9 s.]{\includegraphics[width=0.7\linewidth]{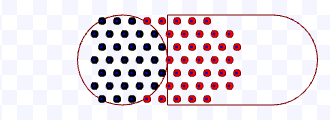}}\\
    \subfloat[9.8 s: ending of the simulation.]{\includegraphics[width=0.7\linewidth]{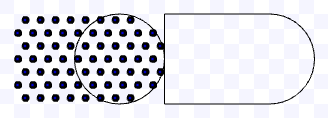}}
    \caption{Simulation on Stage for hexagonal packing strategy using $s = 3$ m, $\theta = 0$  during $T = 9.8$ s. Available on \url{https://youtu.be/6_LgZWFOWd0}.}
    \label{fig:exphex1}
  \end{minipage}\quad
  \begin{minipage}[b]{0.48\linewidth}
    \centering
    \subfloat[0 s: beginning of the simulation.]{\includegraphics[width=0.7\linewidth]{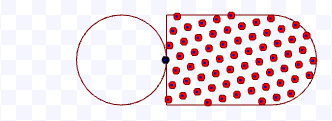}}\\
    \subfloat[After 5 s.]{\includegraphics[width=0.7\linewidth]{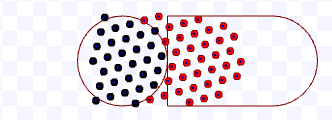}}\\
    \subfloat[10 s: ending of the simulation.]{\includegraphics[width=0.7\linewidth]{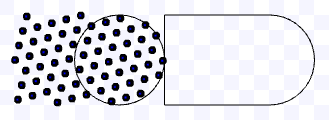}}
    \caption{Simulation on Stage for hexagonal packing strategy using $s = 3$ m, $\theta = \pi/12$ during $T = 10$ s. Available on \url{https://youtu.be/Wji8XlSQJBQ}.}
  \end{minipage}
\end{figure}

\begin{figure}[t!]
  \centering
  \begin{minipage}[b]{0.48\linewidth}
    \centering
    \subfloat[0 s: beginning of the simulation.]{\includegraphics[width=0.7\linewidth]{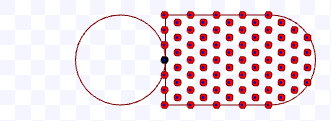}}\\
    \subfloat[After 4.9 s.]{\includegraphics[width=0.7\linewidth]{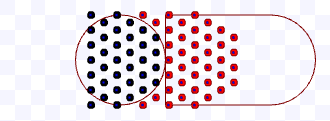}}\\
    \subfloat[10 s: ending of the simulation.]{\includegraphics[width=0.7\linewidth]{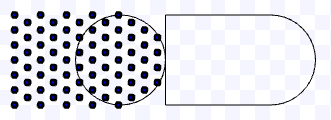}}
    \caption{Simulation on Stage for hexagonal packing strategy using $s = 3$ m, $\theta = \pi/6$  during $T = 10$ s. Available on \url{https://youtu.be/szOBU8no_sU}.}
  \end{minipage}\quad
  \begin{minipage}[b]{0.48\linewidth}
    \centering
    \subfloat[0 s: beginning of the simulation.]{\includegraphics[width=0.7\linewidth]{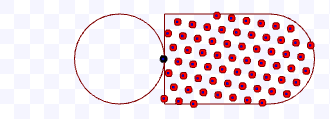}}\\
    \subfloat[After 4.9 s.]{\includegraphics[width=0.7\linewidth]{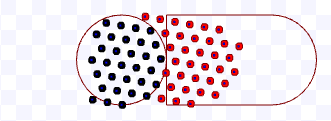}}\\
    \subfloat[10 s: ending of the simulation.]{\includegraphics[width=0.7\linewidth]{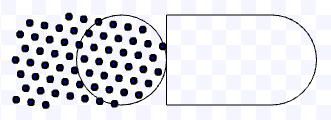}}
    \caption{Simulation on Stage for hexagonal packing strategy using $s = 3$ m, $\theta = 5\pi/18$  during $T = 10$ s. Available on \url{https://youtu.be/jRLgaF7Te1Q}.}
  \end{minipage}
\end{figure}

\begin{figure}[t!]
  \centering
  \begin{minipage}[b]{0.48\linewidth}
    \centering
    \subfloat[0 s: beginning of the simulation.]{\includegraphics[width=0.7\linewidth]{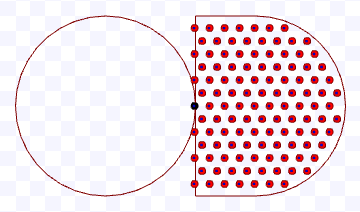}}\\
    \subfloat[After 4.9 s.]{\includegraphics[width=0.7\linewidth]{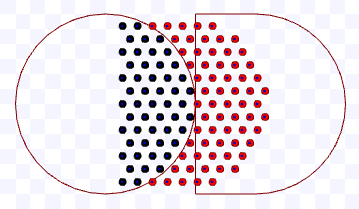}}\\
    \subfloat[9.8 s: ending of the simulation.]{\includegraphics[width=0.7\linewidth]{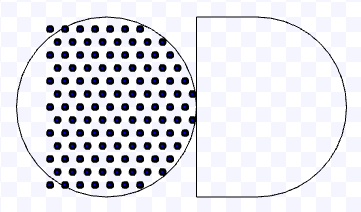}}
    \caption{Simulation on Stage for hexagonal packing strategy using $s = 6$ m, $\theta = 0$ during $T = 9.8$ s. Available on \url{https://youtu.be/v0FK8YpGrL8}.}
  \end{minipage}\quad
  \begin{minipage}[b]{0.48\linewidth}
    \centering
    \subfloat[0 s: beginning of the simulation.]{\includegraphics[width=0.7\linewidth]{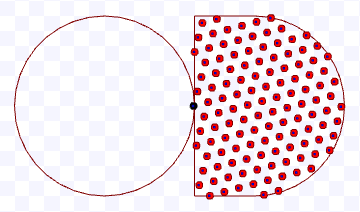}}\\
    \subfloat[After 5 s.]{\includegraphics[width=0.7\linewidth]{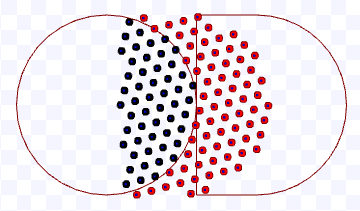}}\\
    \subfloat[10.1 s: ending of the simulation.]{\includegraphics[width=0.7\linewidth]{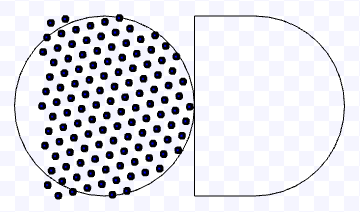}}
    \caption{Simulation on Stage for hexagonal packing strategy using $s = 6$ m, $\theta = \pi/12$ during $T = 10.1$ s. Available on \url{https://youtu.be/OBS_HADH5OE}.}
  \end{minipage}
\end{figure}

\begin{figure}[t!]
  \centering
  \begin{minipage}[b]{0.48\linewidth}
    \centering
    \subfloat[0 s: beginning of the simulation.]{\includegraphics[width=0.7\linewidth]{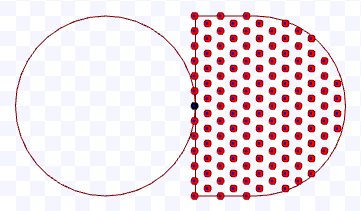}}\\
    \subfloat[After 4.9 s.]{\includegraphics[width=0.7\linewidth]{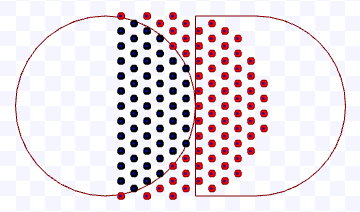}}\\
    \subfloat[10 s: ending of the simulation.]{\includegraphics[width=0.7\linewidth]{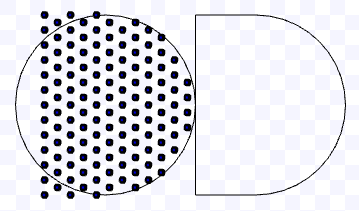}}
    \caption{Simulation on Stage for hexagonal packing strategy using $s = 6$ m, $\theta = \pi/6$ during $T = 10$ s. Available on \url{https://youtu.be/-KX7ziOp8b0}.}
  \end{minipage}\quad
  \begin{minipage}[b]{0.48\linewidth}
    \centering
    \subfloat[0 s: beginning of the simulation.]{\includegraphics[width=0.7\linewidth]{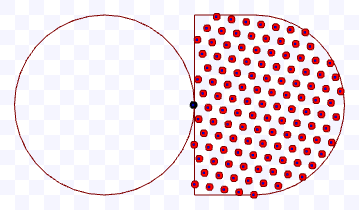}}\\
    \subfloat[After 4.9 s.]{\includegraphics[width=0.7\linewidth]{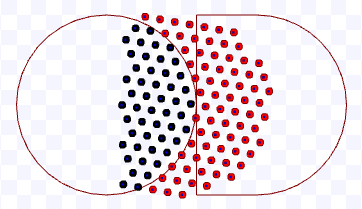}}\\
    \subfloat[10 s: ending of the simulation.]{\includegraphics[width=0.7\linewidth]{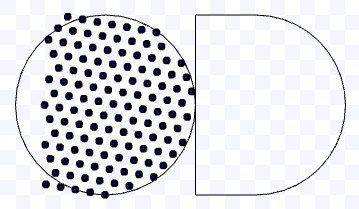}}
    \caption{Simulation on Stage for hexagonal packing strategy using $s = 6$ m, $\theta = 5\pi/18$ during $T = 10$ s. Available on \url{https://youtu.be/GRYRnH5CrhU}.}
    \label{fig:exphex8}
  \end{minipage}
\end{figure}

We experimented the hexagonal packing for $v=1$ m/s, and the combination of the following  variables and values: $s \in \{3,6\}$ m and $\theta \in \{0, \pi/12, \pi/6, 5\pi/18\}$. Figures \ref{fig:exphex1}-\ref{fig:exphex8} present screenshots from executions using these parameters.

\begin{figure}[t!]
  \subfloat[$\theta = 0$]{\includegraphics[width=0.495\linewidth]{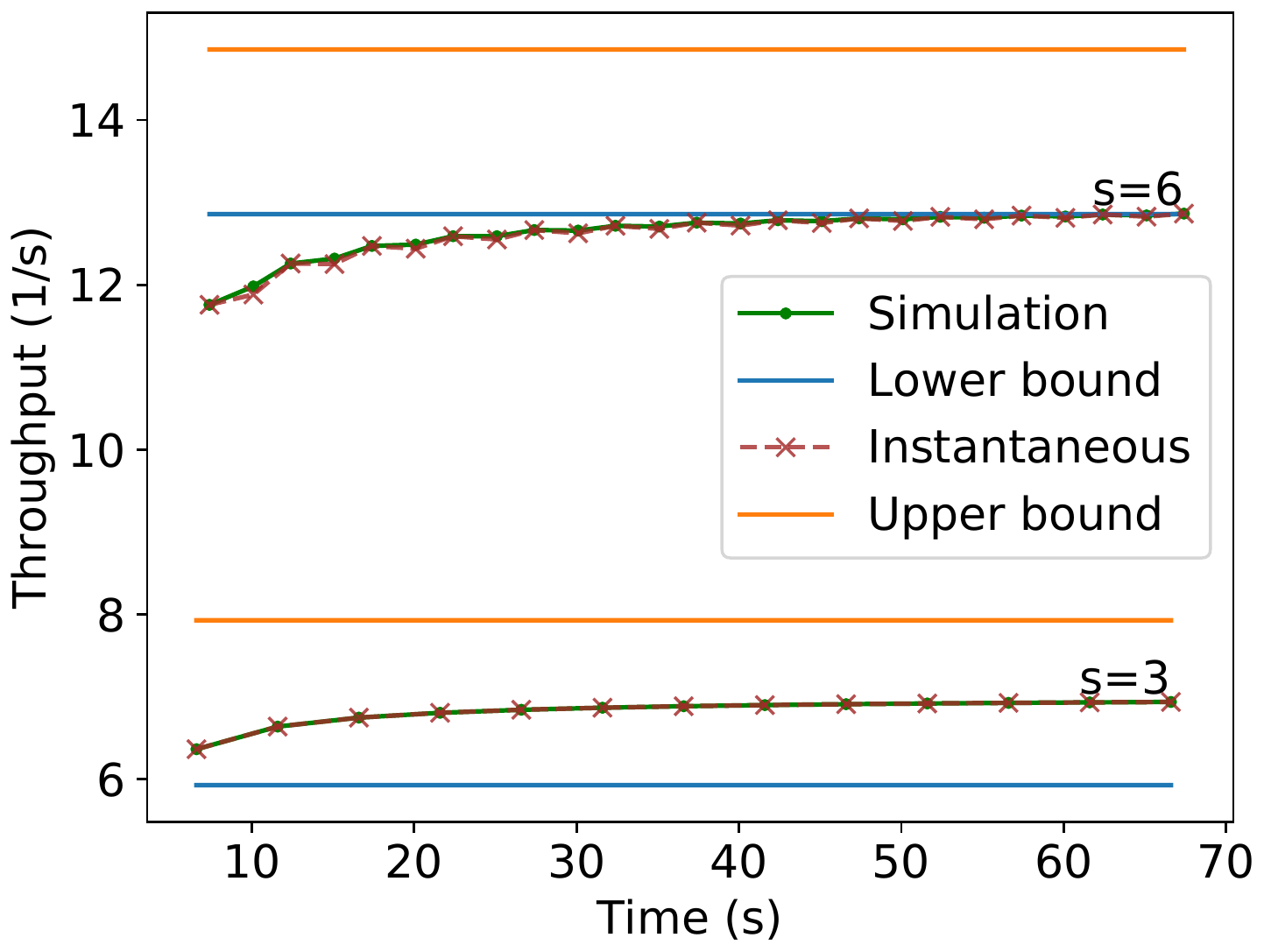}}
  \subfloat[$\theta = \pi/12$]{\includegraphics[width=0.495\linewidth]{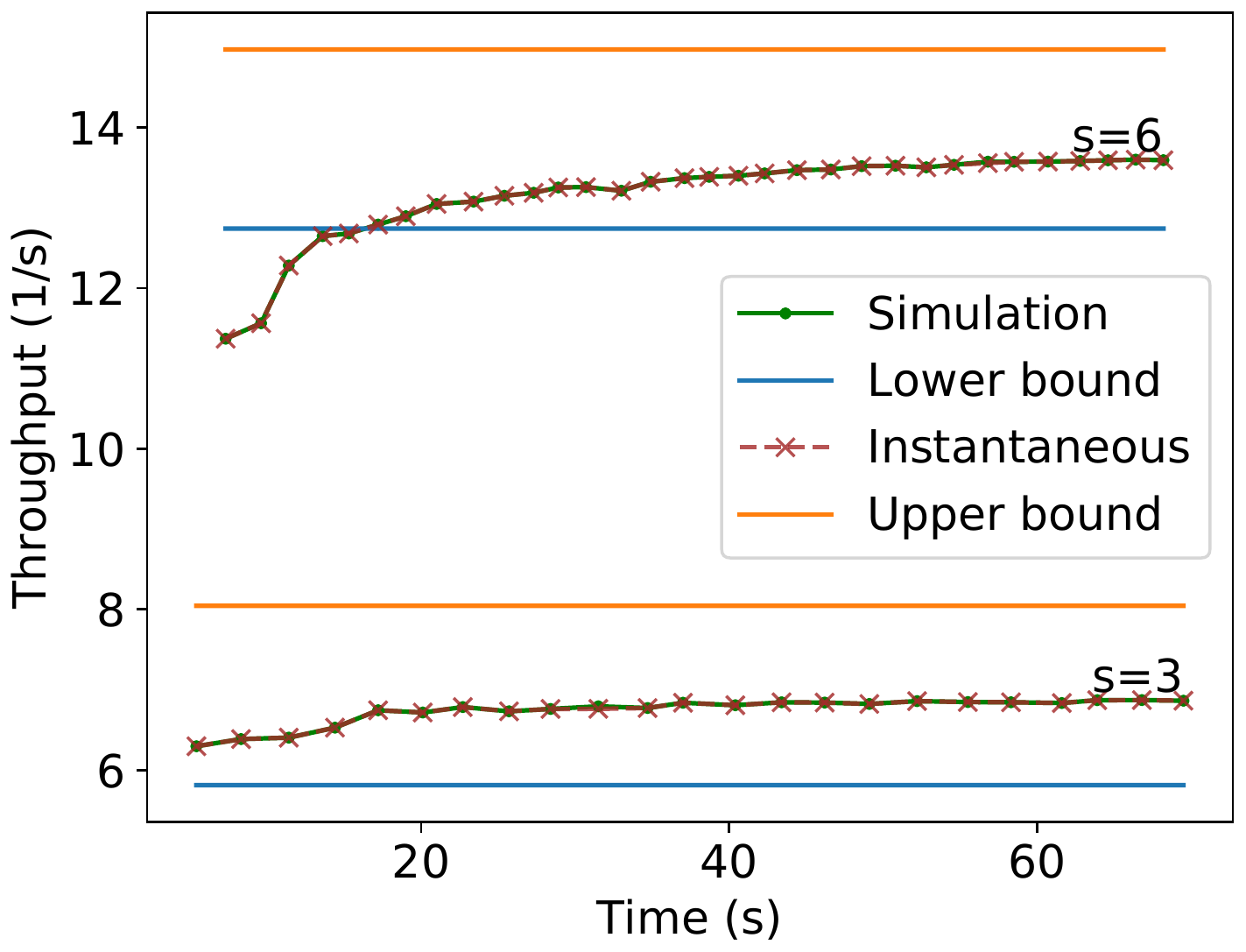}}\\
  \subfloat[$\theta = \pi/6$]{\includegraphics[width=0.495\linewidth]{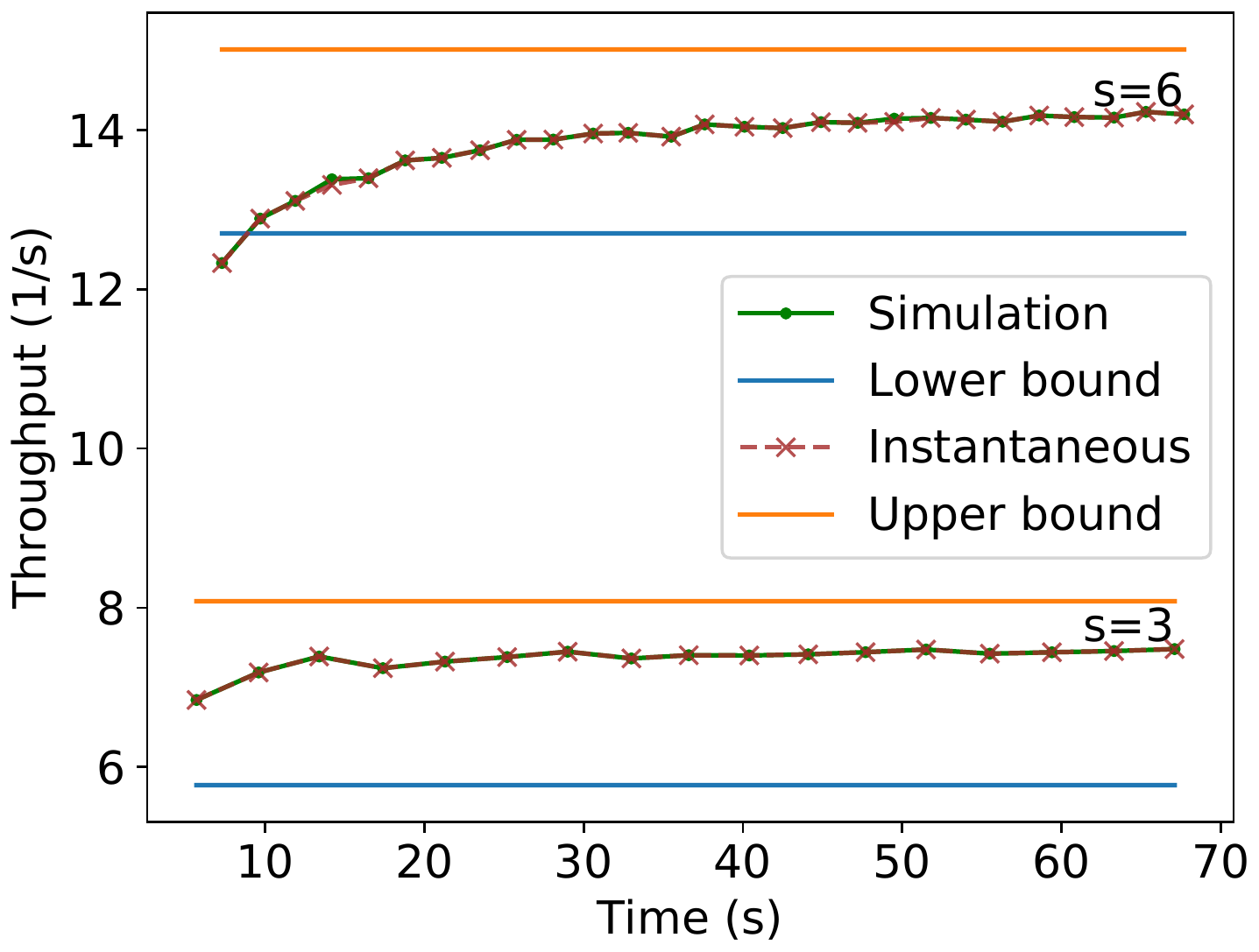}}
  \subfloat[$\theta = 5\pi/18$]{\includegraphics[width=0.495\linewidth]{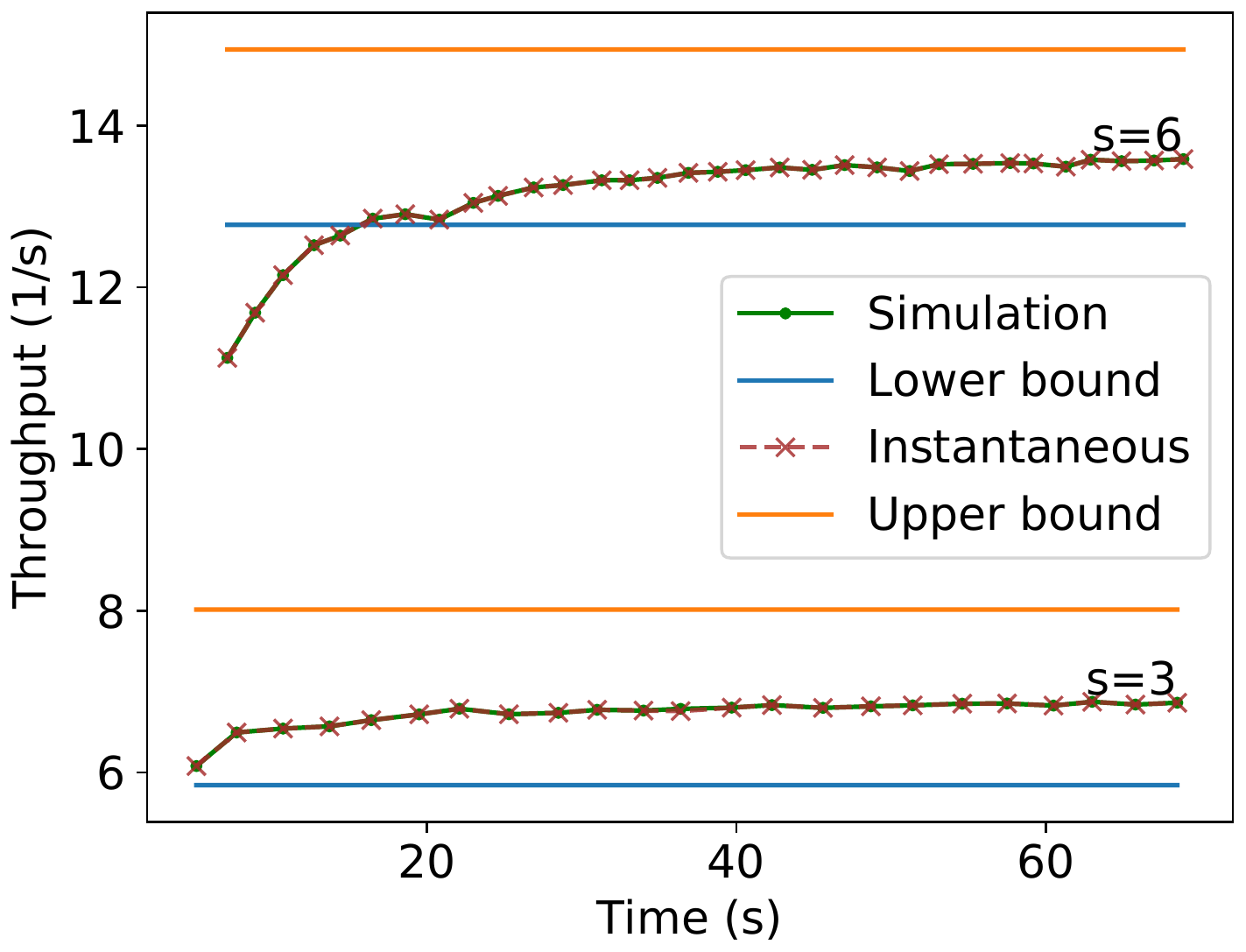}}
  \caption{Throughput versus time comparison of the simulation on Stage, upper and lower bounds on the asymptotic value and the theoretical instantaneous equation for the throughput for a given hexagonal packing angle for different values of $s$ and $\theta$.}
  \label{fig:tphex}
\end{figure}

\begin{figure}[t!]
  \subfloat[$\theta = 0$]{\includegraphics[width=0.495\linewidth]{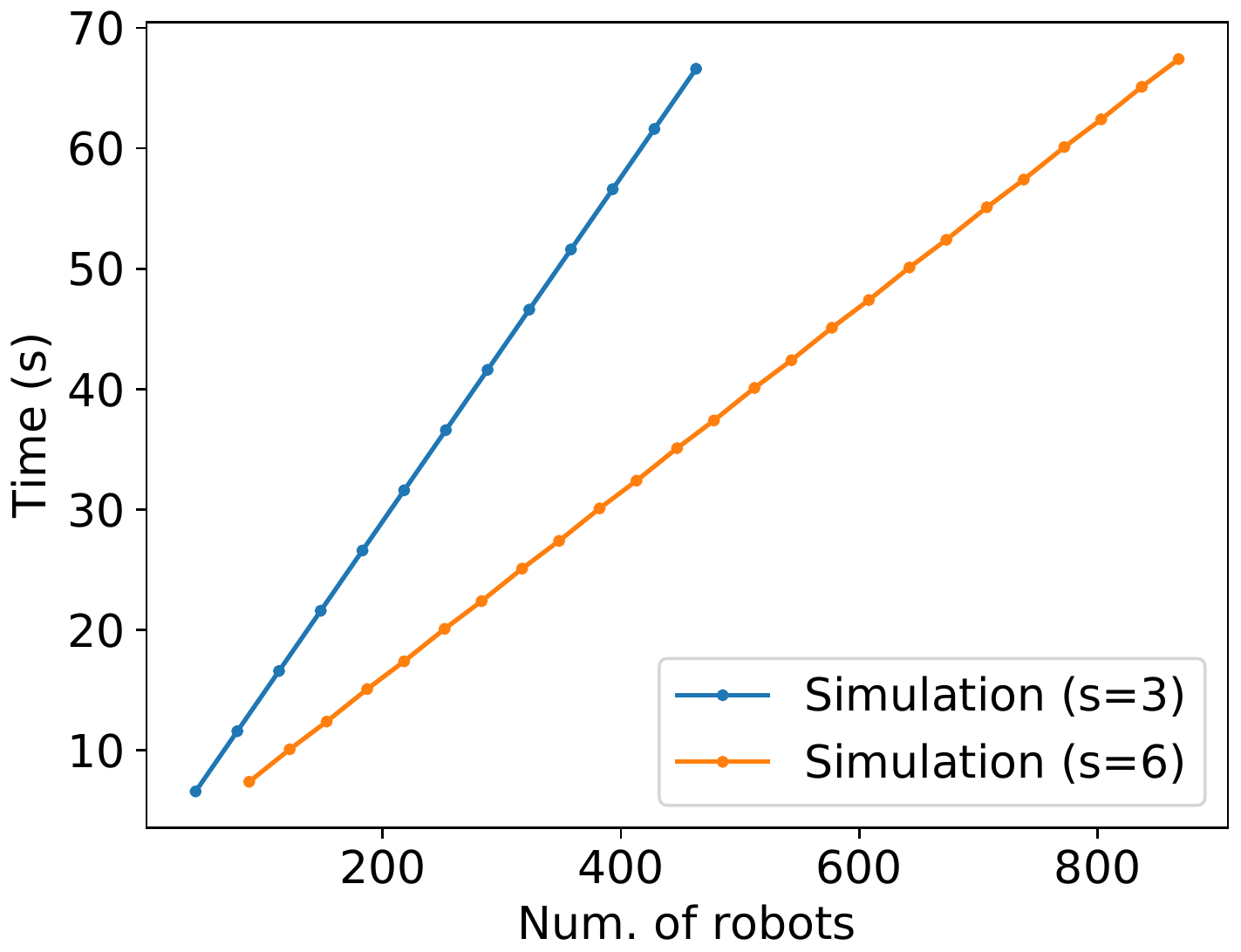}}
  \subfloat[$\theta = \pi/12$]{\includegraphics[width=0.495\linewidth]{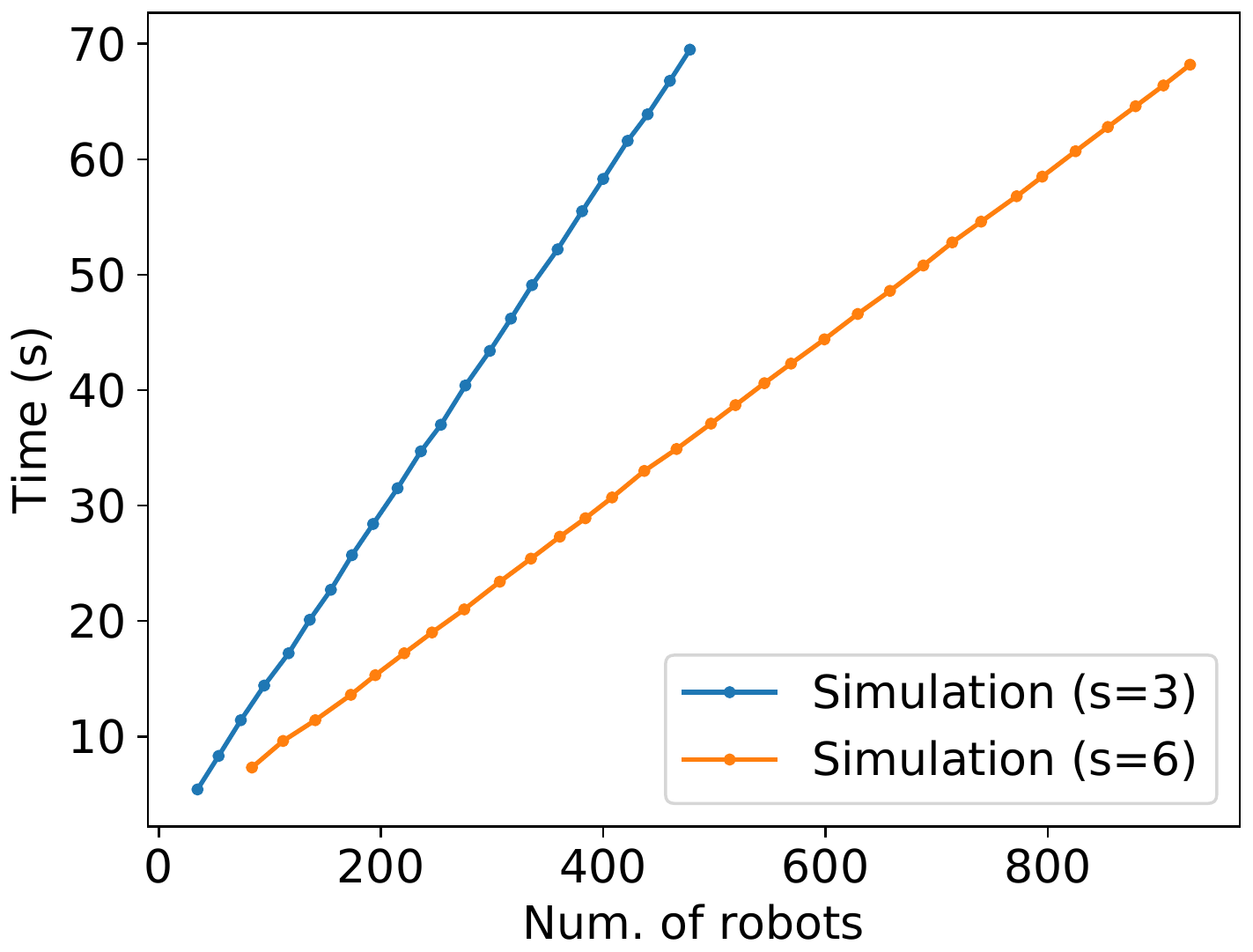}}\\
  \subfloat[$\theta = \pi/6$]{\includegraphics[width=0.495\linewidth]{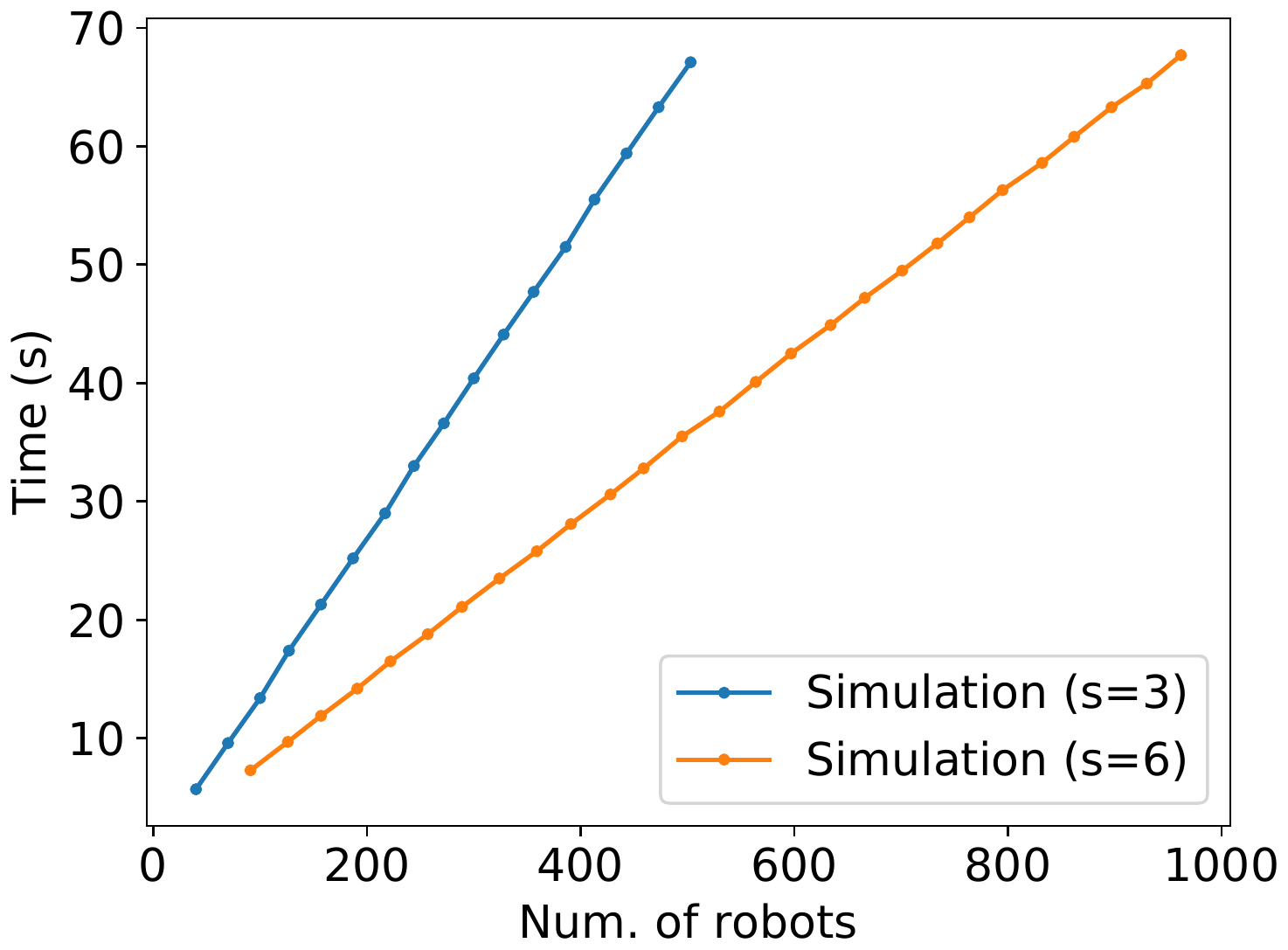}}
  \subfloat[$\theta = 5\pi/18$]{\includegraphics[width=0.495\linewidth]{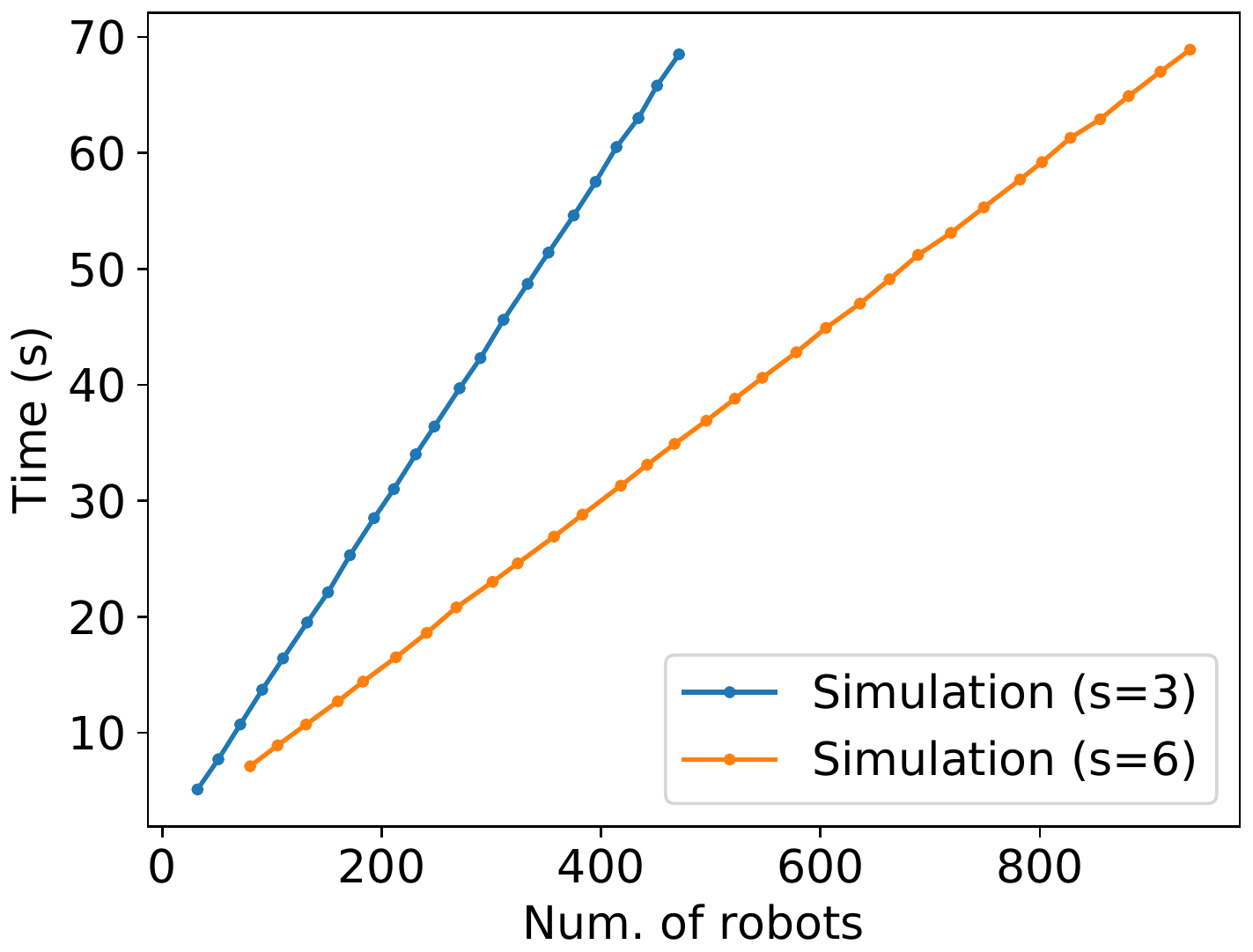}}
  \caption{Time of arrival at the target of the last robot versus number of robots for the same simulations in Figure \ref{fig:tphex}.}
  \label{fig:timehex}
\end{figure}

In order to evaluate the throughput for a given time and angle calculated in (\ref{eq:hexthroughput}) and the bounds on the asymptotic throughput as in (\ref{eq:hexthroughputbounds}), we compare them with the throughput obtained from Stage simulations. Figure \ref{fig:tphex} presents these comparisons. Observe that the values from (\ref{eq:hexthroughput}) are almost aligned with the values from simulation, except for some points. The difference in those points is also due to the floating point error -- discussed in the introduction of Section \ref{sec:experimentresults} -- over the divisions and trigonometric functions done before the use of floor or ceiling functions used on (\ref{eq:hexthroughput}). Also, due to floating point error, in our computation of (\ref{eq:whereIusedepsilon}), instead of using $\min(L(x_{h}),C_{2}(x_{h})) = \lfloor L(x_{h})\rfloor$, we check $\vert \min(L(x_{h}),C_{2}(x_{h})) - \lfloor L(x_{h})\rfloor\vert < 0.001$.

Additionally, note in Figure \ref{fig:tphex} that for any value of $s$ or $\theta$, as the time passes, the values of (\ref{eq:hexthroughput}) asymptotically approaches some value inside the bounds given by (\ref{eq:hexthroughputbounds}). Although the exact asymptotic value could not be given for the presented parameters, the experiments show that the bounds are correct. In the same manner as occured for parallel lanes, higher throughput per time is reflected as a lower arrival time of the last robot per number of robots, and it tends to infinite as the number of robots grows (Figure \ref{fig:timehex}). 

\subsection{Touch and run}
\label{sec:hitandrunexperiments}

For the touch and run strategy, the robots maintain the linear velocity over all the experiment, then turn at fixed constant rotational speed $\omega = v/r$, for $r$ obtained from (\ref{eq:relationangles}), when they are next to the target centre by the distance $d_{r}$ obtained from (\ref{eq:distTargetToTurn}). After they get on the target region,  when they are distant from the target centre by $d_{r}$, they leave the curved path, stop turning and follow the linear exiting lane. On that lane, to stabilise their path following, the robots follow the queue using turning speed equals to $\gamma - \beta$, such that $\beta$ is the angle of the exiting lane and $\gamma$ is the robot orientation angle, both in relation to the $x$-axis.

We used $v=0.1$ m/s because the robots we used on Stage have maximum turning speed of $\pi/2$ rad/s. Choosing a low velocity implies in more number of lanes $K$, as the turning speed $\omega = v/r$, and $r$ varies over $K$ and $s$. In addition, low linear speed diminishes time measurement error, since the positions of the robots are sampled at every $0.1$ s by the Stage simulator. Their positions are not guaranteed to be obtained on the exact moment they are far from the target centre by $d_{r}$, thus this also yields error in time measurement for their arrival on the target area.

\begin{figure}[t!]
  \centering
  \begin{minipage}[b]{0.48\linewidth}
    \centering
    \subfloat[0 s: beginning of the simulation.]{\includegraphics[width=0.7\linewidth]{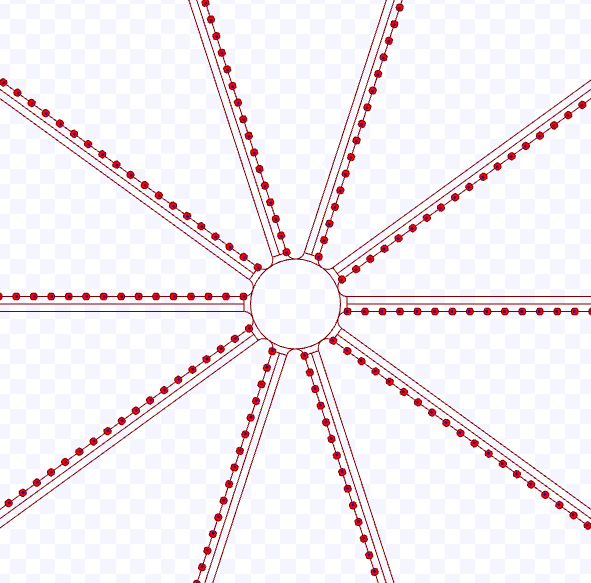}}\\
    \subfloat[After 114 s.]{\includegraphics[width=0.7\linewidth]{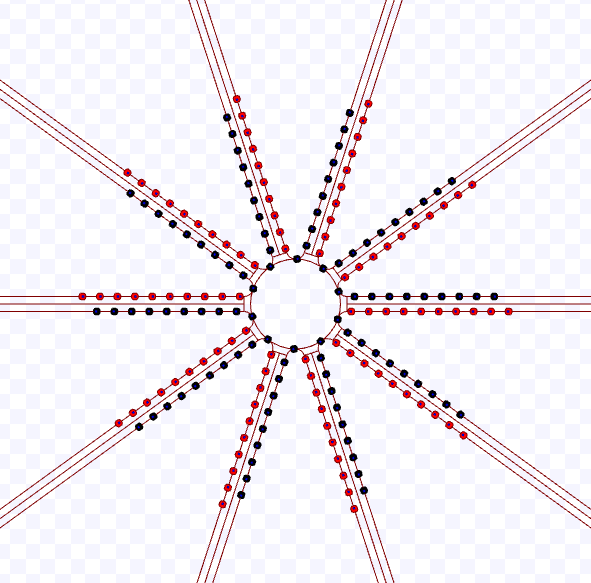}}\\
    \subfloat[228 s: ending of the simulation.]{\includegraphics[width=0.7\linewidth]{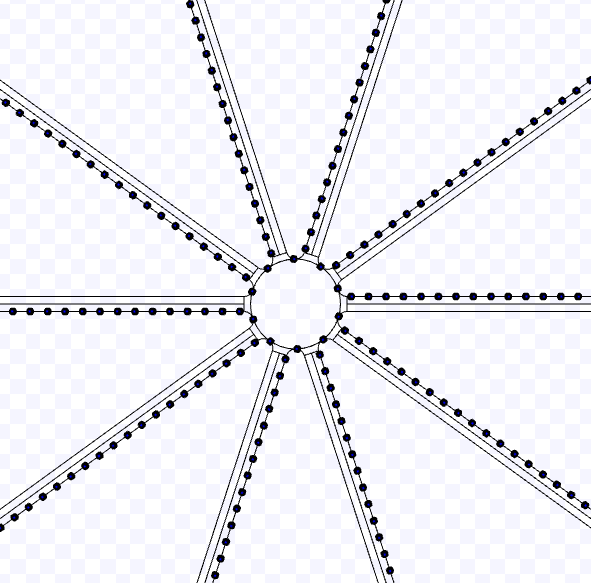}}
    \caption{Simulation on Stage for the touch and run strategy using $s = 3$ m, $K=10$ during $T = 228$ s at $v = 0.1$ m/s. Available on \url{https://youtu.be/Z-ruOMYFyBU}.}
    \label{fig:expit1}
  \end{minipage}\quad
  \begin{minipage}[b]{0.48\linewidth}
    \centering
    \subfloat[0 s: beginning of the simulation.]{\includegraphics[width=0.7\linewidth]{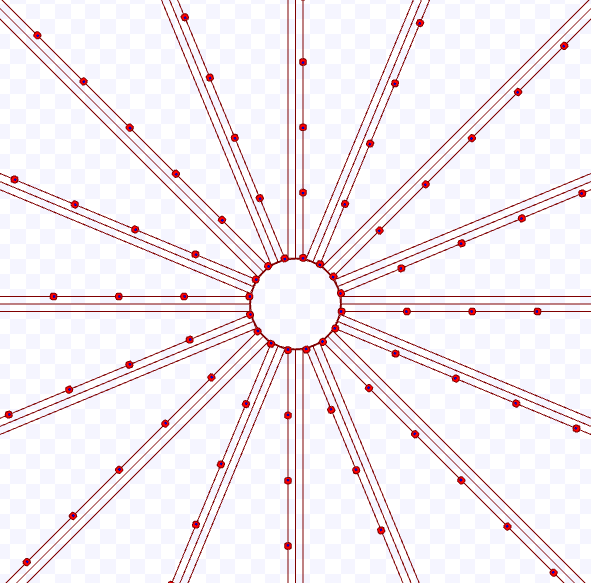}}\\
    \subfloat[After 261.6 s.]{\includegraphics[width=0.7\linewidth]{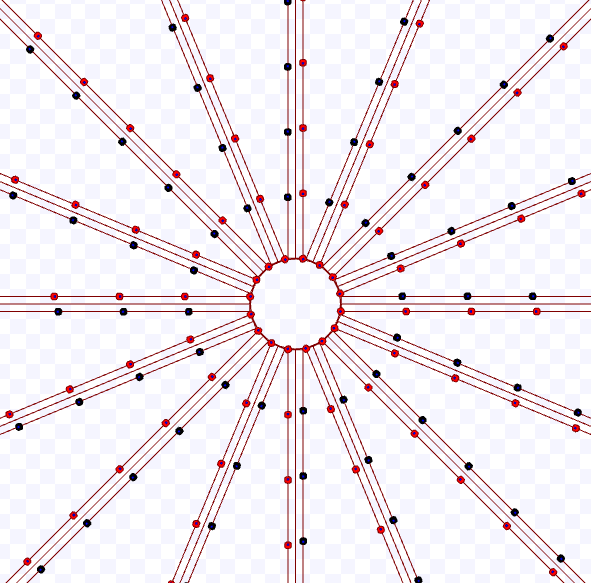}}\\
    \subfloat[523.1 s: ending of the simulation.]{\includegraphics[width=0.7\linewidth]{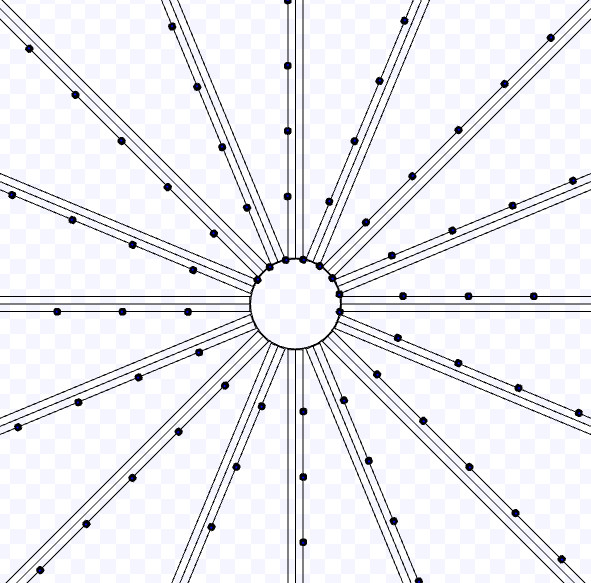}}
    \caption{Simulation on Stage for the touch and run strategy using $s = 3$ m, $K = 16$ during $T = 523.1$ s at $v = 0.1$ m/s. Available on \url{https://youtu.be/FvAqv0zD4_Y}.}
  \end{minipage}
\end{figure}

\begin{figure}[t!]
  \centering
  \begin{minipage}[b]{0.48\linewidth}
    \centering
    \subfloat[0 s: beginning of the simulation.]{\includegraphics[width=0.7\linewidth]{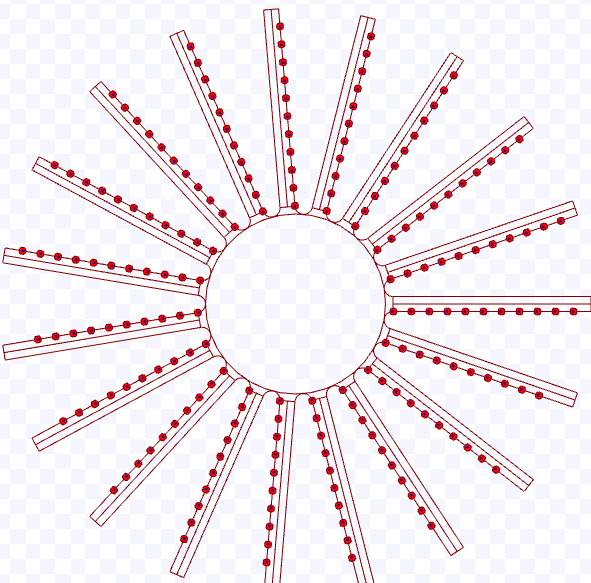}}\\
    \subfloat[After 63.6 s.]{\includegraphics[width=0.7\linewidth]{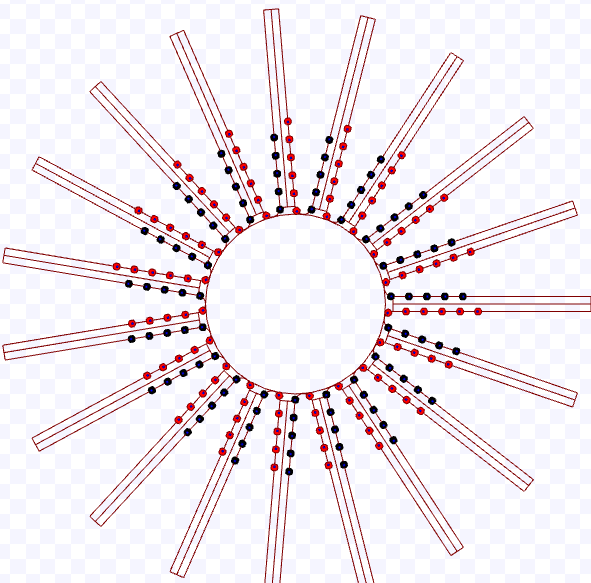}}\\
    \subfloat[127.4 s: ending of the simulation.]{\includegraphics[width=0.7\linewidth]{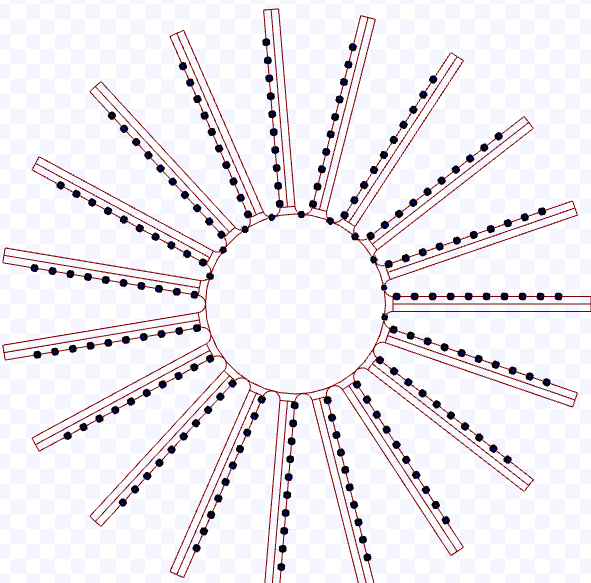}}
    \caption{Simulation on Stage for the touch and run strategy using $s = 6$ m, $K=19$ during $T = 127.4$ s at $v = 0.1$ m/s. Available on \url{https://youtu.be/xJVoVCIjX5k}.}
  \end{minipage}\quad
  \begin{minipage}[b]{0.48\linewidth}
    \centering
    \subfloat[0 s: beginning of the simulation.]{\includegraphics[width=0.7\linewidth]{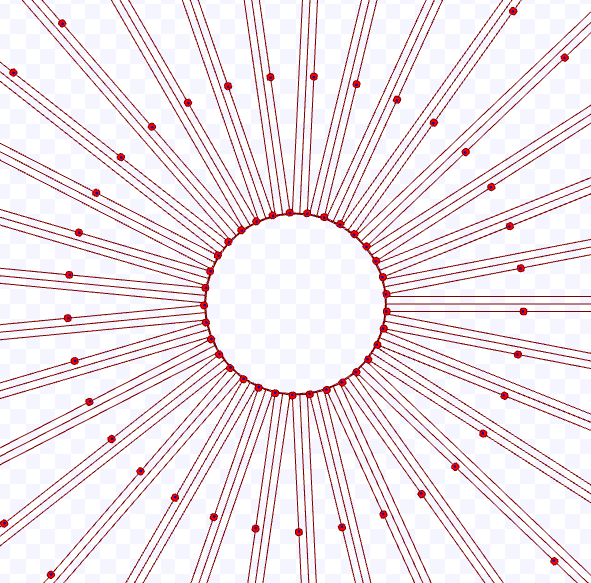}}\\
    \subfloat[After 274 s.]{\includegraphics[width=0.7\linewidth]{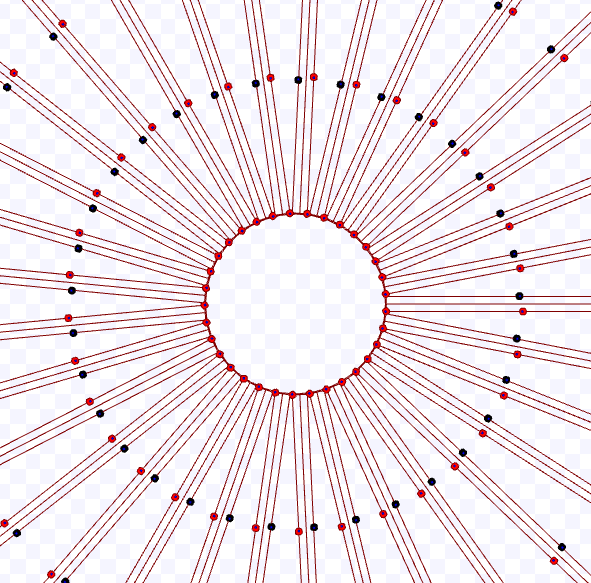}}\\
    \subfloat[548 s: ending of the simulation.]{\includegraphics[width=0.7\linewidth]{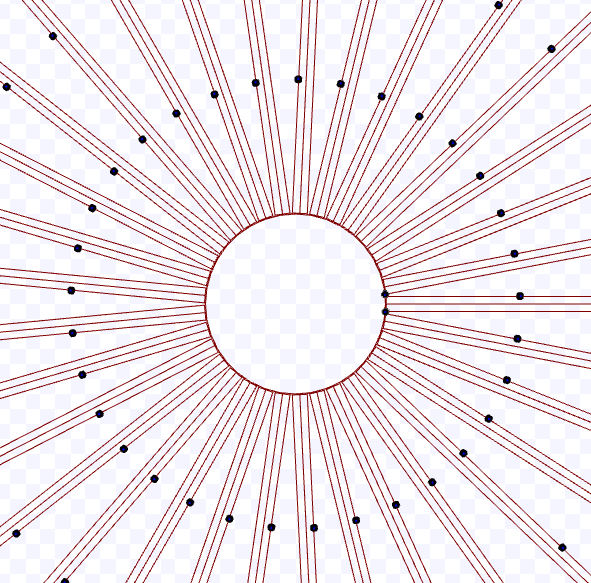}}
    \caption{Simulation on Stage for the touch and run strategy using $s = 6$ m, $K=33$ during $T = 548$ s at $v = 0.1$ m/s. Available on \url{https://youtu.be/-xZz84npKV4}.}
    \label{fig:expit4}
  \end{minipage}
\end{figure}

We used $s \in \{3, 6\}$ m  and all allowed $K$ values for experimenting the touch and run strategy with 200 robots. By (\ref{eq:Kbounds}), for the former $s$ value we have a maximum $K = 18$ and for the later, $K = 37$. However, as the maximum angular velocity is limited, the allowed $K$ values range   for $s=3$ m is reduced to $\{3 ,\dots, 16\}$ and for $s = 6$ m, $\{3,\dots,33\}$. Figures \ref{fig:expit1}-\ref{fig:expit4} presents screenshots from executions using some of these parameters. The circle in the middle of these figures is the target region and the lines where the robots are above represent the curved trajectory which they follow by the touch and run strategy. 

\begin{figure}[t!]
  \subfloat[$s = 3$ m and $K \in \{10,16\}$.]{\includegraphics[width=0.5\linewidth]{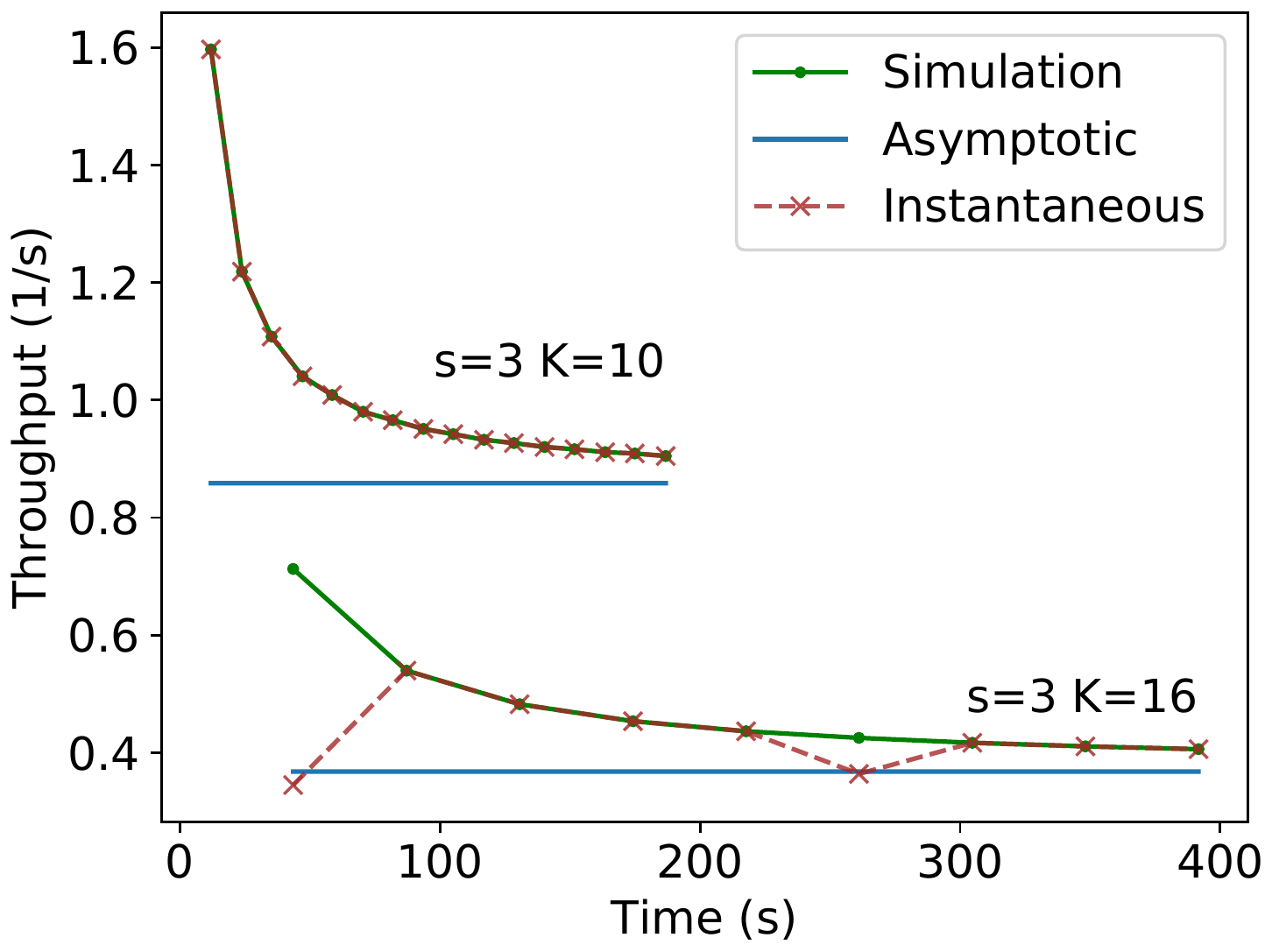}}
  \subfloat[$s = 6$ m and $K \in \{19,33\}$.]{\includegraphics[width=0.5\linewidth]{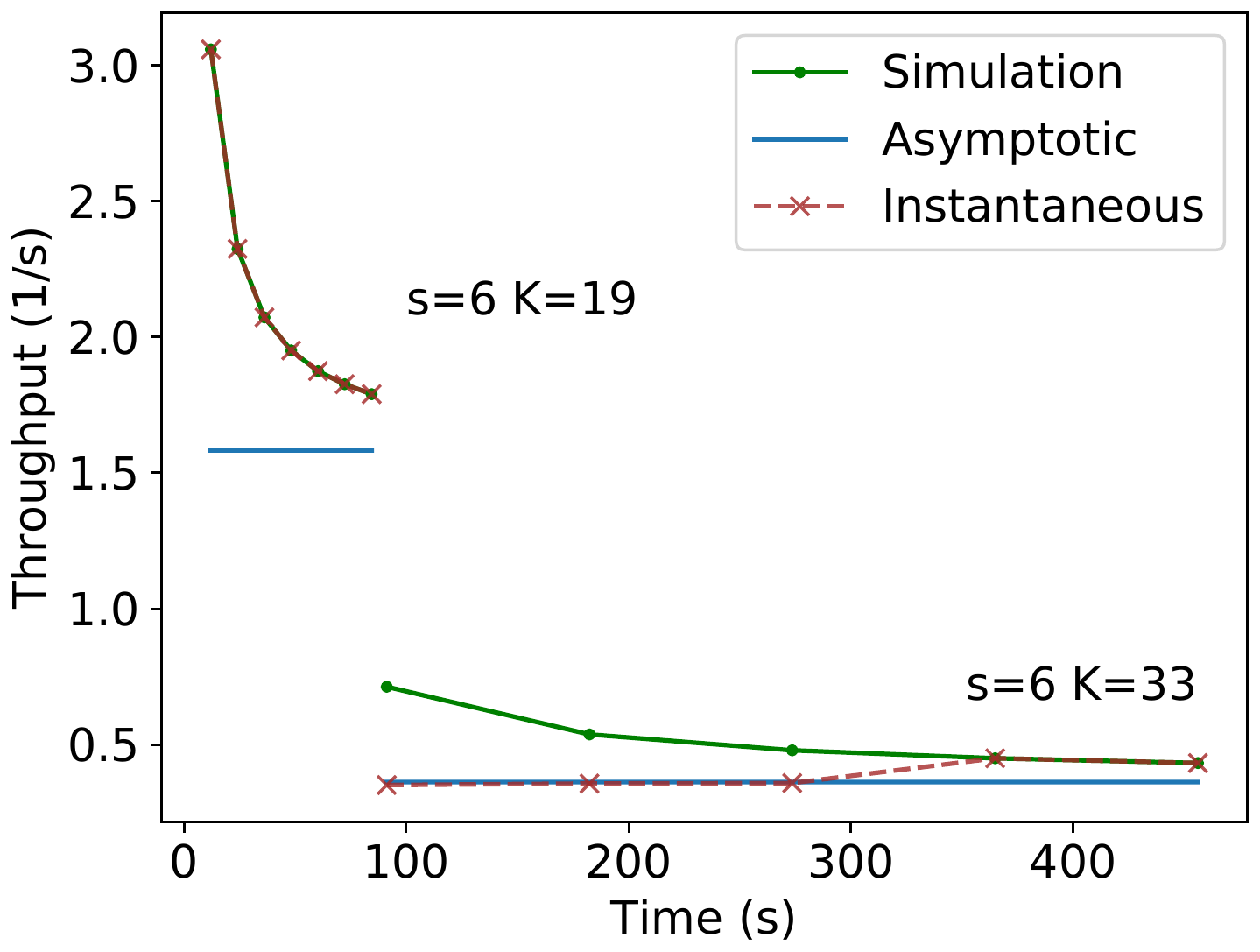}}\\
  \caption{Throughput versus time comparison of the touch and run simulation on Stage with asymptotic values and the theoretical instantaneous equation for the throughput for different values of $s$ and $K$.}
  \label{fig:tpit}
\end{figure}

\begin{figure}[t!]
  \subfloat[$s = 3$ m and $K \in \{10,16\}$.]{\includegraphics[width=0.5\linewidth]{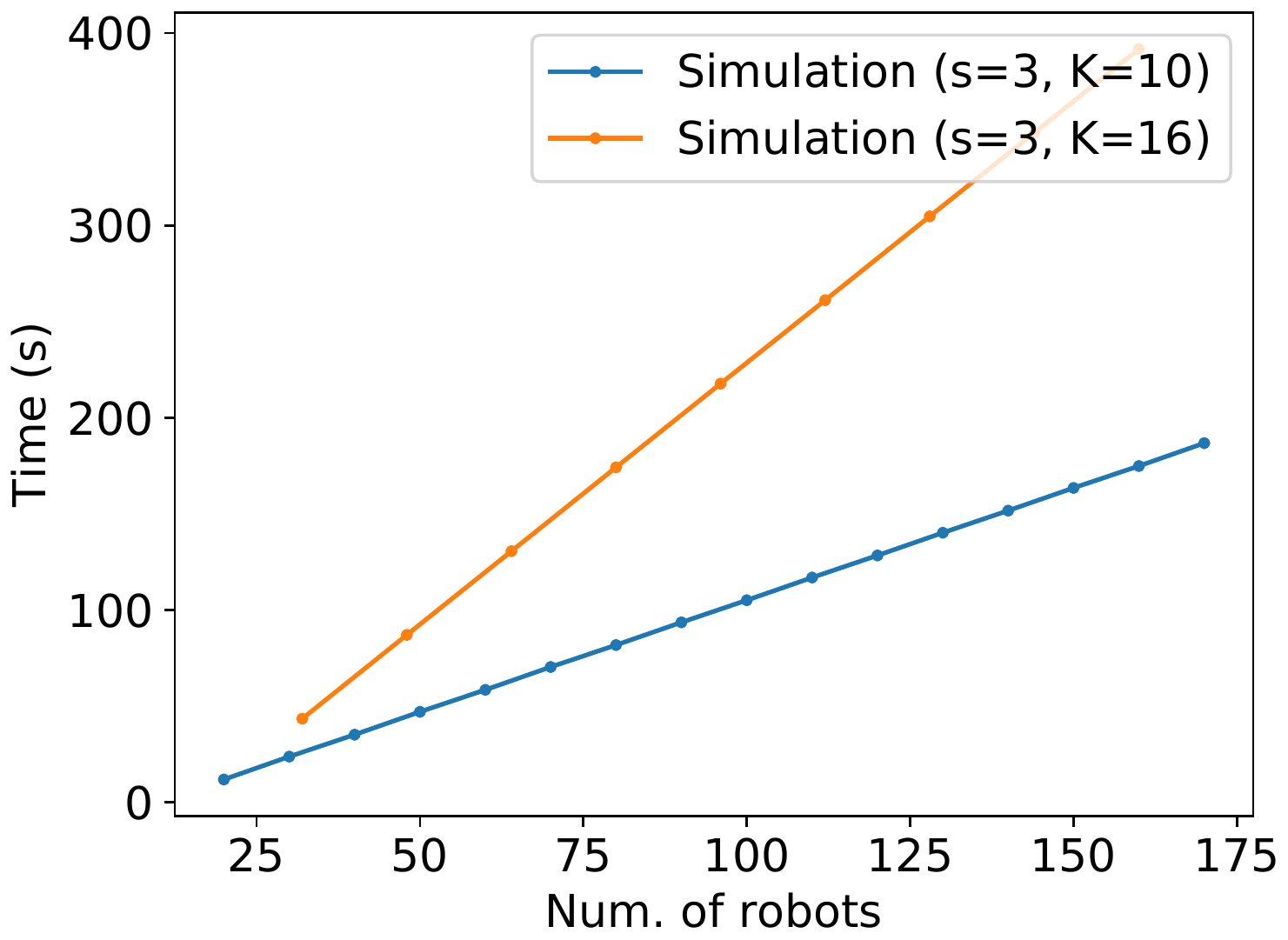}}
  \subfloat[$s = 6$ m and $K \in \{19,33\}$.]{\includegraphics[width=0.5\linewidth]{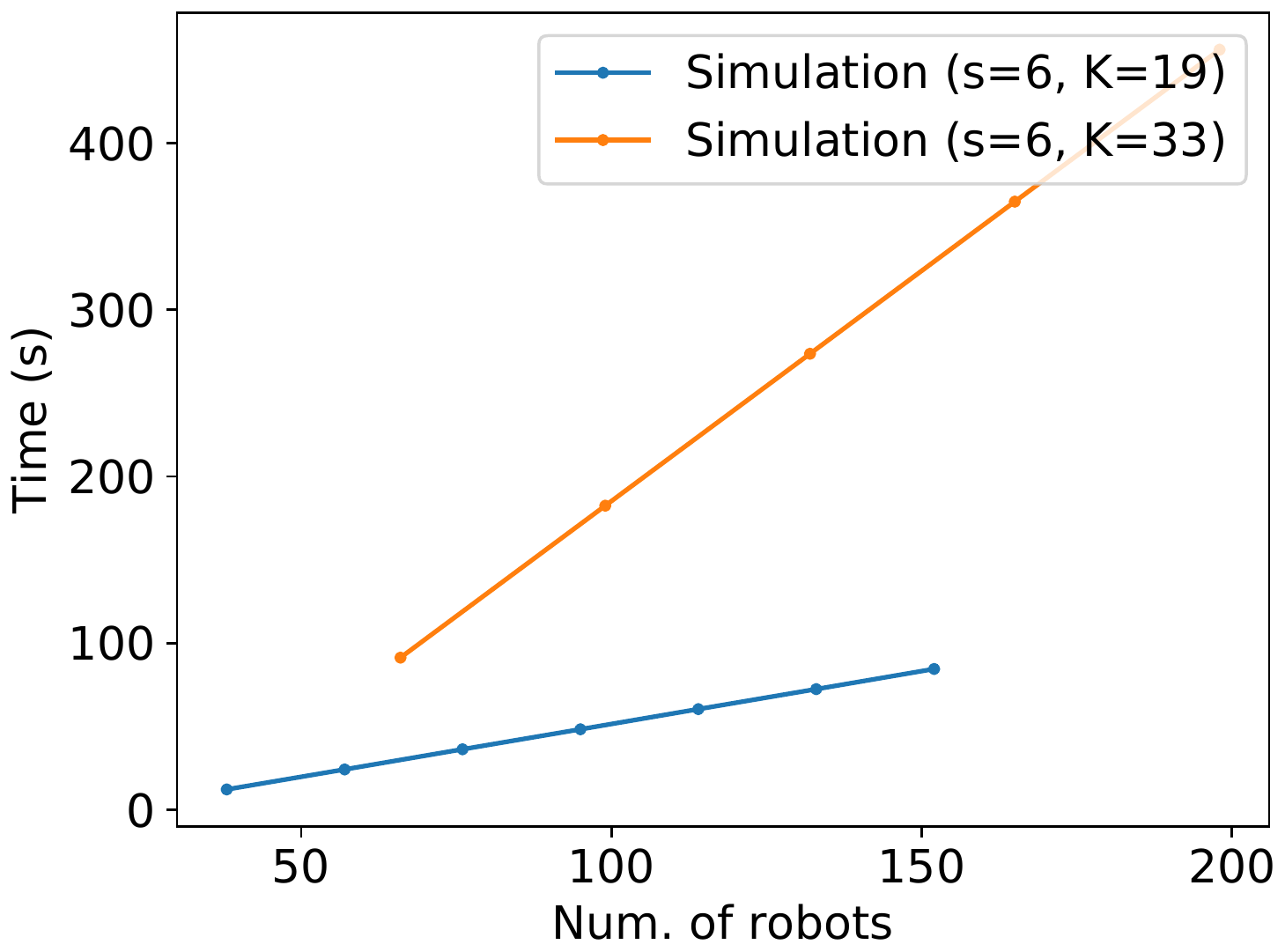}}\\
  \caption{Time of arrival at the target of the last robot versus number of robots for the same simulations in Figure \ref{fig:tpit}.}
  \label{fig:tsit21}
\end{figure}

Figure \ref{fig:tpit} presents the comparison of (\ref{eq:throughputhitandruntime}) and (\ref{eq:throughputhitandrunlimit}) for the throughput for a given time, the bound on its asymptotic value and the one obtained from Stage simulations. Although we fixed the total number of robots and the linear velocity, the arrival times and the number of robots to reach the target change for each parameter used in this figure since the distance between the robots per lane varies and the number of robots simultaneously arriving is, in most cases, the number of lanes. Also, we did not plot the first two arrival times because the first one is zero, yielding an indetermination by the throughput definition, and the second one is still too small in relation to the others, making the resultant throughput too high compared with the rest, producing, thus, an incomprehensible graph.

Observe that the values from (\ref{eq:throughputhitandruntime}) are almost equal to the values from simulation, except for some points. The difference in those points is due to the floating-point error in the divisions and trigonometric functions before the use of floor function used on (\ref{eq:throughputhitandruntime}) -- already mentioned in the introduction of Section \ref{sec:experimentresults} -- as well as the time measurement errors for the arrival of the robots on the target area as explained at the beginning of this section. As expected, the values of (\ref{eq:throughputhitandruntime}) tend to get nearer to the asymptotic value given by (\ref{eq:throughputhitandrunlimit}). Differently from the previous strategies, notice that, for small values of $T$, the throughput is higher than for larger ones because, for a fixed $K$, (\ref{eq:throughputhitandruntime}) is decreasing for $T$.  As occurred for the previous strategies, higher throughput per time is reflected as a lower arrival time of the last robot per number of robots, which tends to infinite as the number of robots grows (Figure \ref{fig:tsit21}).

\begin{figure}[t!]
  \centering 
  \includegraphics[width=0.55\linewidth]{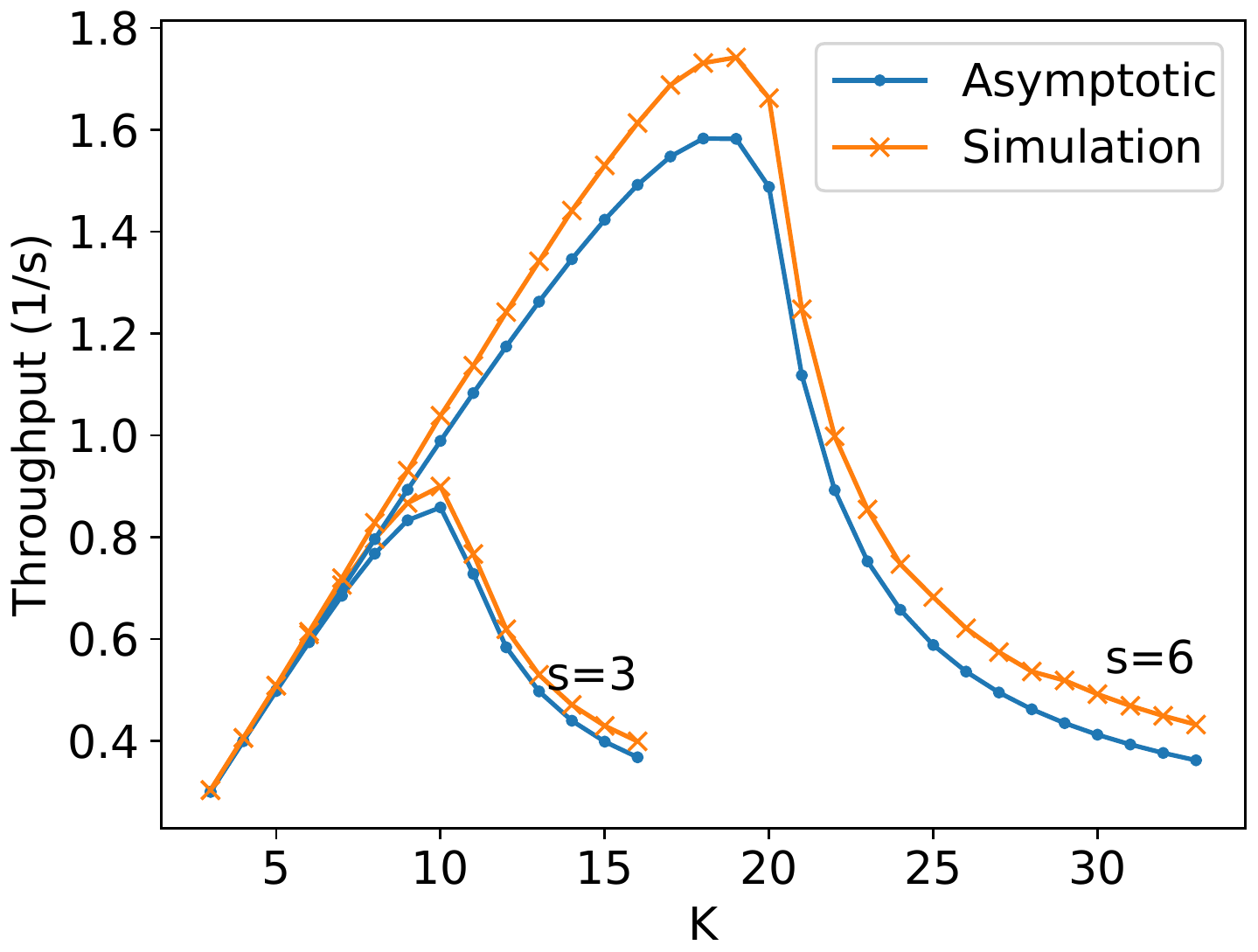} 
  \caption{Throughput versus the number of lanes comparison of the simulation on Stage and asymptotic throughput for $s \in \{3,6\}$ m.}
  \label{fig:ksit}
\end{figure}

Figure \ref{fig:ksit} shows a comparison of the throughput at the end of the experiment -- that is, for 200 robots and considering the difference between the time to reach the target region spent by the last robot and the first -- and the asymptotic throughput obtained by (\ref{eq:throughputhitandrunlimit}) for all the possible number of lanes ($K$) for the used parameters and $s \in \{3,6\}$ m. Confirming our results, the simulation values tend to get near to the asymptotic value.

\subsection{Comparison between hexagonal packing and parallel lanes}

\label{sec:hexparexperiments}

As discussed in Section \ref{sec:subseccomparison}, we observed that the parallel lane strategy has a higher throughput than hexagonal packing for values of $u = s/d$ from 0.5 to a value about 0.85 and for high values of $T$, despite the parallel lanes having lower asymptotic throughput for other values of $u$.  In order to validate this observation, we performed experiments on Stage for these strategies using $T=10000$ s, $v=0.1$ m/s, $d=1$ m and $s$ ranging from 0.4 to 0.95 m in increments of 0.05 m. For hexagonal packing, we computed the best hexagonal packing angle $\theta$ using the same method mentioned at the end of the theoretical section, that is, we searched the maximum throughput using 1000 evenly spaced points between $\lbrack 0,\pi/3 \rparen$ to find the best $\theta$, then we compared with the result for $\pi/6$.

\begin{figure}[t!]
  \centering
  \subfloat[$v = 0.1$ m/s.]{\includegraphics[width=0.49\linewidth]{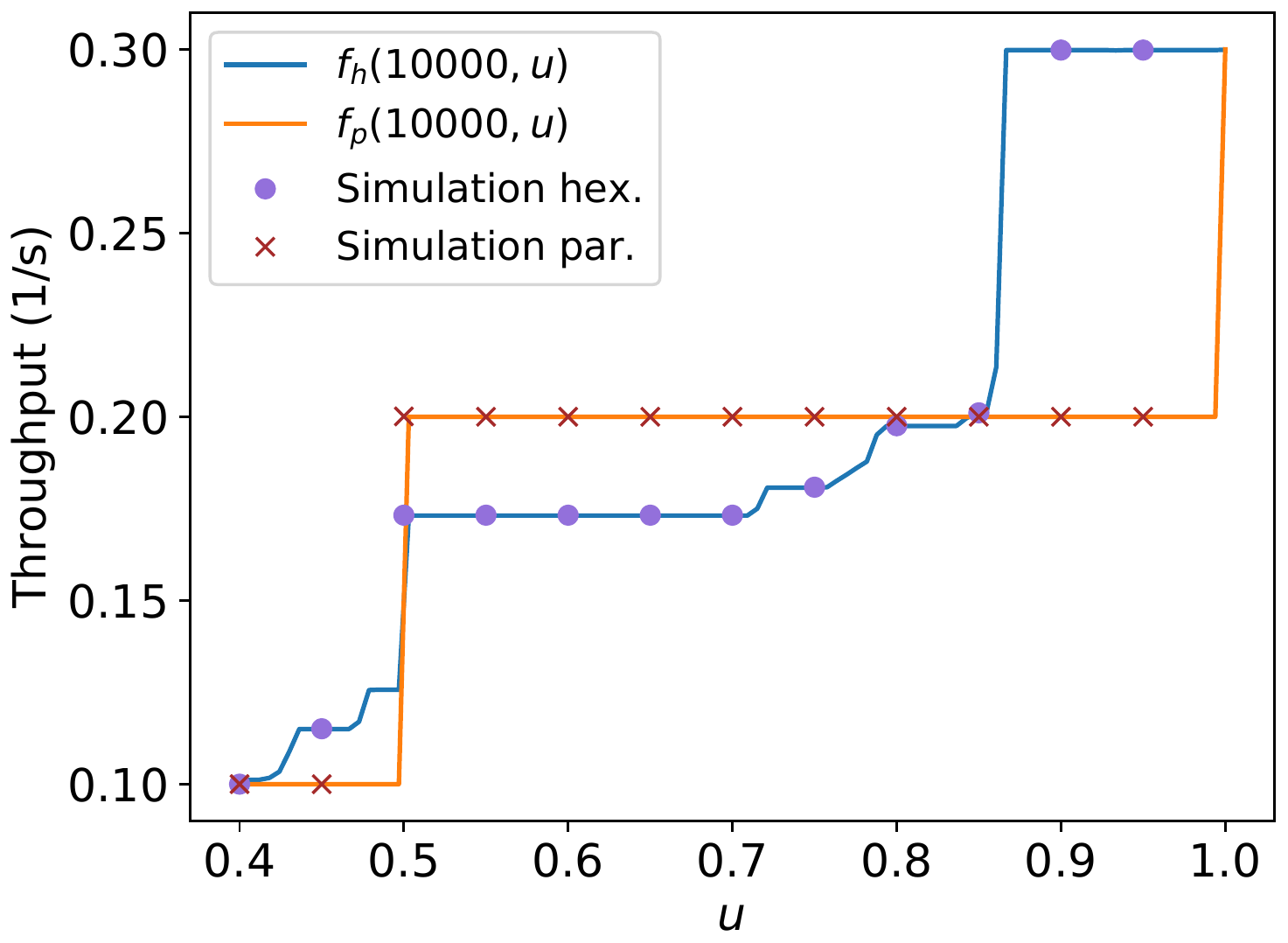} }
  \subfloat[$v = 1$ m/s.]{\includegraphics[width=0.49\linewidth]{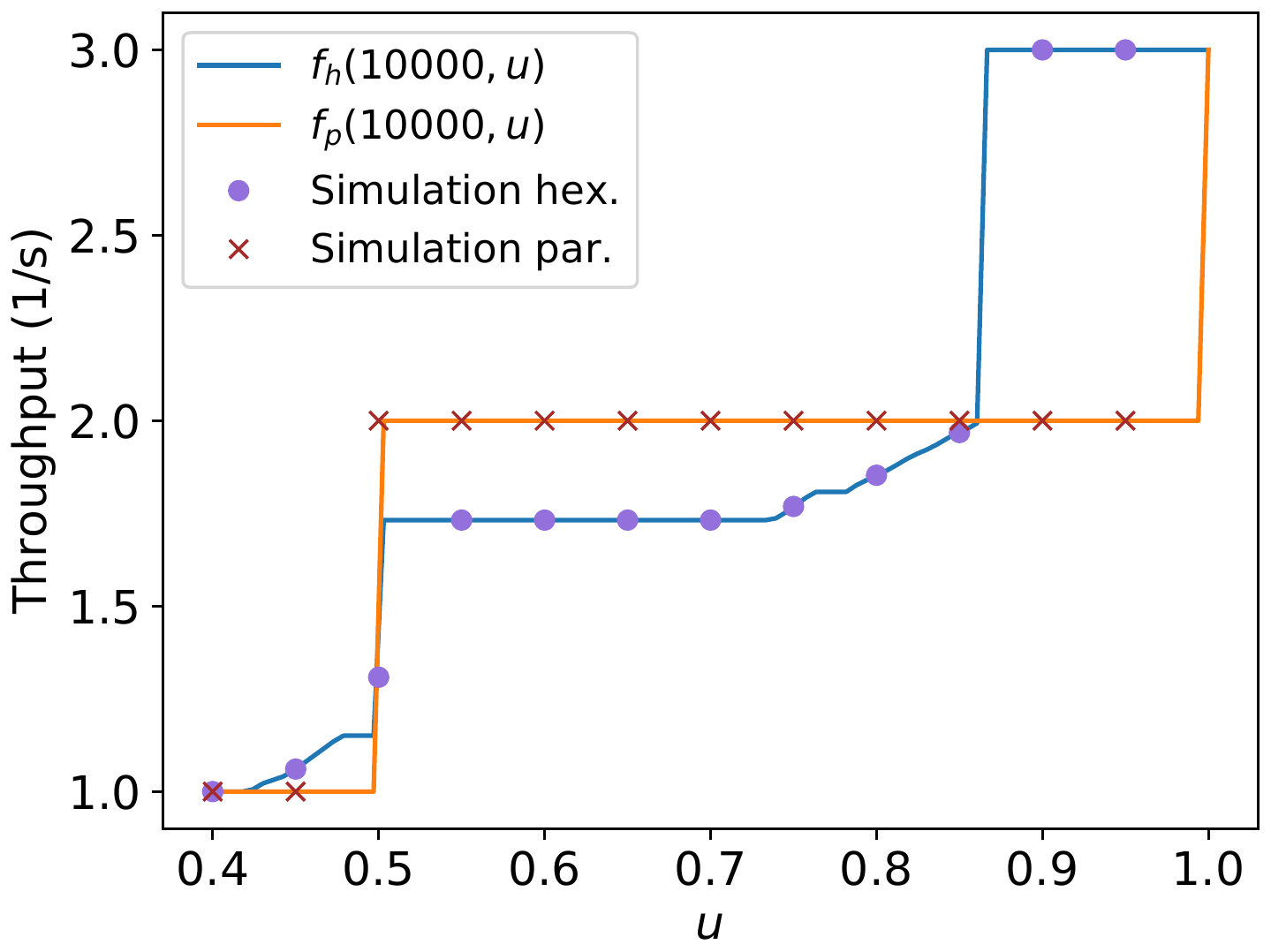} }
  \caption{Throughput versus ratio $u=s/d$ comparing hexagonal packing and parallel lanes strategies for $v\in \{0.1,1\}$ m/s, including results from Stage simulations. The functions $f_{h}$ and $f_{p}$ are the same presented in Figure \ref{fig:fhfpZoom}. The labels ``Simulation hex.'' and ``Simulation par.''  stand for the throughput resultant from the experiments with hexagonal packing and parallel lanes strategies, respectively.}
  \label{fig:compHexParExperiments}
\end{figure}

\begin{figure}[t!]
  \subfloat[Hexagonal packing with best $\theta$ for $s = 0.5$ m. Available on \url{https://youtu.be/IZBnFHLKXUA}.]{\includegraphics[width=0.495\linewidth]{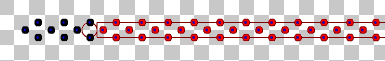}}\ 
  \subfloat[Parallel lanes for $s = 0.5$ m. Available on \url{https://youtu.be/YYv1dJFkdPA}.]{\includegraphics[width=0.495\linewidth]{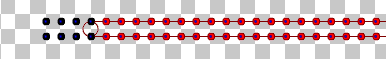} }\\
  \subfloat[Hexagonal packing with best $\theta$ for $s = 0.85$ m. Available on \url{https://youtu.be/r9X0fsnngm0}.]{\includegraphics[width=0.495\linewidth]{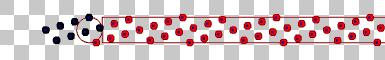}}\ 
  \subfloat[Parallel lanes for $s = 0.85$ m. Available on \url{https://youtu.be/0cx-bHPIong}.]{\includegraphics[width=0.495\linewidth]{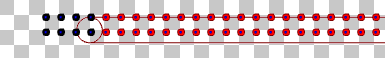} }\\
  \caption{Screenshots of the Stage simulation for hexagonal packing and parallel lanes strategy for $d=1$ m and $s\in \{0.5,0.85\}$ m showing that the square packing fits more robots than hexagonal packing over the time in these cases. The robots run from right to left at constant linear speed $v=0.1$ m/s. We highlighted the grey squares -- which measure $1\times1$ m$^{2}$ -- to help estimate the time needed to  about eight robots arrive on the target region. }
  \label{fig:stageCompHP}
\end{figure}

Figure \ref{fig:compHexParExperiments} presents the results from the experiments with Stage and the theoretical results showed earlier. 
The throughput improvement for the values of $u=s/d$ where the parallel lanes strategy overcomes the hexagonal packing is mainly caused by the square packing being more effective than hexagonal packing for fitting the robots inside the circle over the time for those values. To illustrate this, Figure \ref{fig:stageCompHP} shows the execution for $v=0.1$ m/s, $d=1$ m and $s\in\{0.5,0.85\}$ m. Observe in those figures that when the robots are arranged in squares, more robots arrive per unit of time than using hexagonal packing.
To help visualise this, heed that in Figure \ref{fig:stageCompHP} (a) there are $N=9$ robots in black, occupying a rectangle including the circular target area with a width of approximately $W\approx4.5$ m (this distance can be roughly measured by the grey squares, counting from the two last black robots on the right side to the first one in the left side). As we assumed $v=0.1$ m/s, the throughput in this case is approximately $\frac{N-1}{\frac{W}{v}} \approx \frac{(9-1)0.1}{4.5} \approx 0.178$ s$^{-1}$. Making similar calculations, we have for Figures \ref{fig:stageCompHP} (b), \ref{fig:stageCompHP} (c) and \ref{fig:stageCompHP} (d) the approximate throughputs $\frac{(8-1)0.1}{3} \approx 0.233$ s$^{-1}$, $\frac{(8-1)0.1}{4}=0.175$ s$^{-1}$ and $\frac{(8-1)0.1}{3} \approx 0.233$ s$^{-1}$, respectively. The results from the parallel lanes in this illustration -- about 0.233 for both values of $s$ -- surpass the values for the hexagonal packing.

\section{Conclusion}
\label{sec:conclusion}

A novel metric was proposed for measuring the effectiveness of algorithms to minimise congestions in a swarm of robots trying to reach the same goal: the common target area throughput. Also, we defined the asymptotic throughput for the common target area as the throughput as the time tends to infinity. Assuming the robots have constant velocity and distance between each other, we showed how to calculate the asymptotic throughput for different theoretical strategies to arrive in the common circular target region: (i) making parallel queues to get in the target region, (ii) using a corridor with robots in hexagonal packing to enter in the region and (iii) following curved trajectories to touch the region boundary. These strategies where the inspiration of new algorithms using artificial potential fields in \citep{arxivAlgorithms}. As the growing number of robots needs to be considered in robotic swarms, the asymptotic throughput abstracts the measure of how many robots can reach the target as time passes and, consequently, as the number of robots increases. Should a closed form for the asymptotic throughput be given, we can use it to compare algorithms as it will be always finite. 

When we assumed constant velocity and distance between robots, we were able to provide theoretical calculations of the throughput for a given time and the asymptotic throughput for the different theoretical strategies. Based solely on these calculations, we could compare which strategy is better. If we can theoretically calculate the asymptotic throughput of any algorithm, we may use it to compare them to decide which one is better. When a robotic swarm has all robots going to the same target, the function relating the number of robots and the time of arrival on the target region tends to infinite as the number of robots grows, while the function relating the number of robots and throughput tends to a finite number, because the asymptotic throughput is finite. When an algorithm exhibits a lower target region arrival time for a number of robots, the throughput is higher, so the comparison by this latter metric can replace the former as it reflects the same ordering but reverse. 

Although we do not have closed asymptotic equations tackling dynamic speed and inter-robot distance, we believe that, if we had them, we could compare only by the analytical asymptotic throughput, as we accomplished when the linear velocity and distance between the robots are constant. Thus, for common target area congestion in robotic swarms, the throughput is well suited for comparing algorithms due to its abstraction of the rate of the target area access as the number of robots grows, whether we have or not the closed throughput equation.

\section*{Acknowledgements}

Yuri Tavares dos Passos gratefully acknowledges the Faculty of Science and Technology at Lancaster University for the scholarship and the Universidade Federal do Recôncavo da Bahia for granting the leave of absence for finishing his PhD.

\bibliographystyle{elsarticle-num-names}
\bibliography{swarmIntelligence}   

\end{document}